\documentclass{article}
\usepackage{graphicx} 
\usepackage[left=1in,top=1in,right=1in,nohead,bottom=1in]{geometry}
\usepackage{graphicx}
\usepackage{amsmath, amssymb, enumerate}
\usepackage{natbib} 
\usepackage{booktabs}
\usepackage{algorithm}
\usepackage{algpseudocode}
\usepackage{enumitem}
\usepackage{hyperref}
\hypersetup{colorlinks,linkcolor={blue},citecolor={blue},urlcolor={blue}}
\usepackage[applemac]{inputenc}
\usepackage[T1]{fontenc}
\usepackage{amsfonts}
\usepackage{float}
\usepackage{tikz}
\usepackage{enumitem}
\usetikzlibrary{shapes.geometric, arrows}
\usepackage{mathtools}

\usepackage{xr}
\newcommand{\fhat}{\widehat{f}}
\newcommand{\betahat}{\widehat{\beta}}
\usepackage{parskip}
\usepackage{multirow}

\usepackage{enumitem}

\usepackage{amsthm}
\newtheorem{theorem}{Theorem}
\newtheorem{lemma}{Lemma}
\newtheorem{Remark}{Remark}
\newtheorem{proposition}{Proposition}
\newtheorem{corollary}{Corollary}

\newtheorem{assumption}{Assumption}[]
\newcommand{\E}{\text{E}}
\newcommand{\tr}{\text{Tr}}

\newcommand{\Slinv}{(\Sigma + \lambda I)^{-1}}

\newcommand{\Hk}{\mathcal{H}_K}

\newcommand{\real}{\mathbb{R}}
\newcommand{\N}{\mathbb{N}}
\newcommand{\F}{\mathcal{F}}
\newcommand{\G}{\mathcal{G}}
\newcommand{\U}{\mathcal{U}}
\newcommand{\Pbb}{\mathbb{P}}

\newcommand{\cid}{\xrightarrow{d}}
\newcommand{\cip}{\xrightarrow{P}}

\def \mbb {\mathbb}

\def\expect{\mbb{E}}

\def\real{\mbb{R}}

\newcommand\ca[1]{{\cal{#1}}}
\newcommand\lo[1]{_{\nano{#1}}}

\def\tr{\mathrm{tr}}

\def\nano{\scriptscriptstyle}

\def\real{{\mathbb R}}

\def\L2T{L \lo 2 (T)}
\def\L2TX{L \lo 2 (T\lo X)}
\def\L2TX{L \lo 2 (T\lo Y)}

\def\eod{\end{document}}

\newcommand{\1}{\mathbf{1}}
\newcommand{\Prob}{\mathbb{P}}

\newcommand{\norm}[1]{\left\lVert #1 \right\rVert}
\newcommand{\fnorm}[1]{\left\lVert #1 \right\rVert_{\mathrm{F}}}

\newcommand{\R}{\mathbb{R}}
\DeclareMathOperator{\essinf}{\mathrm{ess inf}}

\title{Kernel Single-Index Bandits: Estimation, Inference, and Learning}
\author{
  Sakshi Arya\thanks{The first two authors contributed equally to the paper.} \ $^\dagger$
  \and
  Satarupa Bhattacharjee\footnotemark[1] \ $^\ddag$
\and
  Bharath Sriperumbudur$^{\S}$
}

\date{ {\small 
$^\dagger$Department of Mathematics, Applied Mathematics, and Statistics, Case Western Reserve University, \href{mailto:sxa1351@case.edu}{sxa1351@case.edu};\\
$^\ddag$Department of Statistics, University of Florida, \href{mailto:bhattacharjee.sa@ufl.edu}{bhattacharjee.sa@ufl.edu};\\
$^{\S}$Department of Statistics, The Pennsylvania State University, \href{mailto:bharathsv.ucsd@gmail.com}{bharathsv.ucsd@gmail.com}.}
}

\begin{document}
\maketitle

\begin{abstract}
    We study contextual bandits with finitely many actions in which the reward of each arm follows a single-index model with an arm-specific index parameter and an unknown nonparametric link function. We consider a regime in which arms correspond to stable decision options and covariates evolve adaptively under the bandit policy. This setting creates substantial statistical challenges: the sampling distribution depends endogenously on the allocation rule, observations are dependent across time, and inverse-propensity weighting induces significant variance inflation.
We propose a kernelized $\varepsilon$-greedy algorithm that combines Stein-based estimation of the index parameters with inverse-propensity-weighted kernel ridge regression for the nonparametric reward functions. This approach enables flexible semiparametric learning while retaining interpretable covariate effects.
Our analysis develops new tools for statistical inference with adaptively collected data. We establish asymptotic normality for the single-index estimator under adaptive sampling, yielding valid confidence regions for the index parameters, and we derive a directional functional central limit theorem for the RKHS estimator of the link function, which produces asymptotically valid pointwise confidence intervals. The analysis relies on high-probability concentration bounds for inverse-weighted Gram matrices together with martingale central limit theorems adapted to the bandit setting.
We further obtain finite-time regret guarantees, including $\tilde{O}(\sqrt{T})$ rates under common-link Lipschitz conditions, demonstrating that semiparametric structure can be exploited without sacrificing statistical efficiency. These results provide a unified framework for simultaneous learning and inference in single-index contextual bandits with adaptively generated data.
\end{abstract}

\section{Introduction}
Contextual bandits provide a fundamental framework for sequential decision-making under covariate-dependent uncertainty, in which an agent repeatedly observes a context, selects an action, and receives a stochastic reward. This paradigm underlies a wide range of applications, including adaptive experimentation, personalized medicine, mobile health interventions, online education, and recommendation systems. Substantial progress in the past decade, see, for example, \cite{lattimore2016causal}, \cite{tewari2017ads}, \cite{lattimore2020bandit}, \cite{chen2021statistical}, has led to algorithms with strong regret guarantees under increasingly flexible modeling assumptions, including parametric models such as linear and generalized linear bandits \cite{goldenshluger2013linear, Filippi2010_GLM_Bandits, chu2011contextual, yadkori2011_linear, agrawal_TS_Linear_13} and nonparametric approaches based on smoothness or kernel structure \cite{yang2002randomized,rigollet2010nonparametric,  gur2022smoothness, hu2020smooth,  chowdhury2017kernelized, zhu2021pure}.

Despite these advances, systematic approaches to statistical inference in structured contextual bandits remain limited. Existing methods primarily target regret minimization and provide limited tools for quantifying uncertainty about covariate effects or reward heterogeneity, particularly in semiparametric or nonparametric environments where adaptive sampling complicates statistical inference.

To address this gap, we introduce a semiparametric contextual bandit framework based on a per-arm single-index structure, a classical statistical model in which the conditional mean depends on the covariates through a one-dimensional projection. Single-index models provide a balance between interpretability and flexibility and have been extensively studied in static regression \citep{ichimura1993, horowitz2009} and more recently for complex data types \citep{bhattacharjee2023single, bouzebda2023k, tang2021partial, xiao2025model}, including functional, imaging, and manifold-valued observations \citep{park2021constrained, lu2024novel}. Despite their broad applicability, their development in contextual bandits is still emerging, and statistical inference under adaptive data collection remains largely open.

We propose a fully online single-index contextual bandit model with finitely many arms, where each arm is governed by its own index parameter and nonparametric link function. This formulation is particularly natural when arms correspond to stable decision options such as treatment levels, intervention policies, or dosage choices, so that arm-specific heterogeneity arises through distinct low-dimensional mechanisms.

Several recent works consider related models under different structural assumptions. \cite{arya2025semi} introduced a single-index framework in a batched  bandit regime where decisions and estimators are updated only at batch boundaries, whereas our setting is fully online with sequential estimation under adaptively collected data. A complementary line of work, \cite{kang2025single}, studies a global-index model with a shared index parameter and shared link function across arms in a setting where arm features are treated as i.i.d.\ across rounds. Our model instead allows arm-specific index parameters and arm-specific link functions with fixed arms, which is well suited to applications where decision options persist over time and their context-dependent effects differ across arms. A recent contribution by \cite{ma2025nonparametric} also proposes a single-index contextual bandit model, but focuses primarily on regret analysis and optimality, whereas our work develops a joint inferential and learning framework.

A growing literature studies statistical inference under adaptive data collection in bandit settings, where observations are neither independent nor identically distributed. One approach restores classical asymptotics by stabilizing adaptivity through batched designs \citep{zhang2020inference} and asymptotic theory for M-estimators under adaptive sampling \citep{zhang2021statistical, chen2021statistical}. A second line develops uncertainty quantification for policy evaluation using adaptively collected data, including confidence intervals and anytime-valid off-policy inference tools for contextual bandits \citep{hadad2021confidence, waudby2024anytime}. Beyond policy value estimation, recent works develop online inferential procedures for structured parametric contextual bandit models \citep{chen2022online, han2022online}, as well as related settings such as delayed feedback \citep{shi2023statistical} and model misspecification \citep{guo2025statistical}. Relatedly, \cite{dimakopoulou2021online} study bandit algorithms that explicitly incorporate adaptive inference objectives. However, most existing results focus on finite-dimensional targets and do not provide inference for a nonparametric reward function learned online. This motivates our RKHS-based inferential development for the arm-specific link functions $f_i$, together with inference for the index parameters $\beta_i$.

To learn and infer in this setting, we propose a kernel $\varepsilon$-greedy algorithm designed for single-index contextual bandits that combines two components: (i) a Stein-type estimator for each arm's index parameter, building on the approach of \cite{bala_PMLR}, and (ii) an inverse-propensity-weighted kernel ridge regression (IPW-KRR) estimator for the nonparametric reward functions, extending the estimator of \cite{arya:23} to adaptively sampled data. Both components must accommodate time dependence, endogenous covariate distributions, and the variance inflation introduced by inverse-propensity weighting. Addressing these challenges requires a combination of martingale arguments, high-dimensional operator concentration, and RKHS-based stochastic analysis.

A distinguishing feature of this work is the development of a unified inferential framework for both the parametric and nonparametric components of the model. While regret characterizes decision-making performance, statistical inference provides uncertainty quantification for estimated covariate effects and reward evaluations. Our analysis establishes inferential guarantees for both components:
\begin{itemize}
    \item[(I)] \textbf{Parametric inference} for the index parameter $\beta_i$, enabling interpretation of covariate influence and identification of features driving reward heterogeneity; and
    \item[(II)] \textbf{Nonparametric inference} for each arm-specific reward function $f_i$, including valid directional and pointwise confidence intervals derived from functional central limit theorems in an RKHS.
\end{itemize}

Our results account for the substantial challenges posed by adaptively collected data, where covariates, actions, and rewards are dependent and the sampling distribution evolves with the allocation rule. By combining inverse-propensity weighting with martingale central limit theorems, we establish valid inferential guarantees in this sequential setting. To the best of our knowledge, this is the first framework in contextual bandits that provides asymptotic confidence intervals for both the semiparametric index and the nonparametric reward function.
\subsection*{Contributions}
This work makes several contributions to the theory and methodology of contextual bandits with structured reward models: 
\begin{itemize}
    \item \textbf{Semiparametric modeling framework.} We introduce a formulation for joint estimation and inference of $(\beta_i,f_i)$ in a per-arm single-index contextual bandit with a finite action space, bridging sufficient dimension reduction with contextual bandits and extending GLM bandits to unknown link functions through low-dimensional representations of the context. This design enables interpretability and structure at the level of individual, stable decision options (see Sections~\ref{sec:problem} and~\ref{sec: algorithm}).

    \item \textbf{Parametric inference and estimation under adaptive sampling.} In Section~\ref{sec:asymptotic_inference_beta}, we develop a comprehensive inferential theory for the parametric component, establishing an asymptotic normal approximation for the estimator of $\beta_i$ under adaptively collected data. 
    Furthermore, finite sample estimation error bounds are also derived for the estimated index parameters in Section~\ref{sec:index_concentration_summary}, with detailed derivations provided in Section~\ref{subsecApp: proofs_beta_finite_time}
    of the supplement~\citep{suppl:arya:26}.

    \item \textbf{Directional functional CLT for the unknown link function.} For the nonparametric component, we derive a directional functional central limit theorem for the inverse-propensity-weighted kernel ridge estimator (IPW-KRR) of $f_i$ in the associated RKHS in Section~\ref{sec:rkhs-clt-setup} with additional results and proofs detailed in Section~\ref{secApp: proof_rkhs_clt}
    of the supplement. The limiting distribution holds in an infinite-dimensional space and remains valid under adaptively collected data, where the effective design distribution evolves endogenously with the allocation rule. Very few asymptotic results of this nature exist in the literature on nonparametric inference in Hilbert spaces, and, to our knowledge, none apply to online learning settings. The result enables valid uncertainty quantification for evaluations of $f_i$ along arbitrary finite collections of linear directions, including pointwise evaluations $f_i(x)$ at covariates of interest. It thus provides a principled foundation for constructing confidence intervals and performing model diagnostics.

    \item \textbf{Finite-time regret guarantees.} We establish finite-time bounds on the cumulative regret using new high-dimensional concentration inequalities for inverse-weighted kernel Gram matrices. The proposed method achieves sublinear regret in general, and under additional regularity conditions on $f_i$, attains the parametric rate \( \widetilde{O}(\sqrt{T}) \) (see Section~\ref{sec: regret_analysis} and~\ref{secApp: proofs_regret} 
    of the supplement~\citep{suppl:arya:26}).

    \item \textbf{Empirical validation and interpretability.} Through extensive simulations and real-data illustrations, we evaluate both learning and inferential performance under varying noise levels and context distribution dimensions. Our results in Sections~\ref{sec:simulations} and~\ref{sec:realdata_rice} highlight the advantages of single-index structure in improving sample efficiency and interpretability, while the constructed confidence intervals provide reliable uncertainty quantification for online decision-making.
\end{itemize}

Taken together, these contributions provide a unified framework for estimation, inference, and learning in single-index contextual bandits.

\paragraph*{Notation}
For $x \in \mathbb{R}^d$, $\|x\|$ denotes the Euclidean norm.
For a matrix $A$, $\|A\|$ denotes the operator norm, and
$\lambda_{\min}(A)$, $\lambda_{\max}(A)$ its extreme eigenvalues.
We write $A \succeq 0$ if $A$ is positive semidefinite.
Let $\Hk$ be a reproducing kernel Hilbert space (RKHS) with kernel $K$,
inner product $\langle\cdot,\cdot\rangle_K$,
and norm $\|\cdot\|_K$.
Convergence in probability and distribution are denoted by
$\cip$ and  $\xrightarrow{D}$,respectively.
Stochastic order is written $O_{\mathbb{P}}(\cdot)$.
The natural filtration is $\mathcal F_t$. For $f,g \in \mathcal H_k$, $f \otimes g$ denotes the rank-one operator
$(f \otimes g)(h) = \langle g,h\rangle_K f$. For a comprehensive notation table organized by section, see the Supplement.

\section{Problem Setup}\label{sec:problem}

In the contextual bandit problem with finitely many \emph{arms}, the decision-maker has
$L\in\mathbb{N}$ competing choices of arms/actions/decisions, denoted
$\ca A:=\{1,\dots,L\}$, and must select arms sequentially over time
$\ca T:=\{1,\dots,T\}$, where $T$ is the time horizon. At each time
$t\in\ca T$, the decision-maker observes \emph{covariate/contextual information}
$X_t\in\ca X\subset\mathbb{R}^d$ drawn i.i.d.\ from a distribution $P_X$, and
selects an arm $a_t\in\ca A$ based on the historical data
$\{X_s,a_s,Y_s\}_{s<t}$. After choosing $a_t$, a \emph{reward} $Y_t$ is observed.

We consider the following \emph{kernel single-index bandit (\textbf{KSIB})} model:
\begin{equation}
    Y_t = f_i(X_t^\top \beta_i) + \varepsilon_t, \qquad i \in \{1, \dots, L\},
    \label{eq: MABCmodel}
\end{equation}
where \( f_i: \mathbb{R} \to \mathbb{R} \) is an unknown nonlinear reward
function associated with arm \(a_t=i\), and \( \beta_i \in \mathbb{R}^d \) is
the corresponding index parameter.

Our primary focus is on the single-index model, where the effective covariate
dimension is one, namely $U_{t,i}=X_t^\top\beta_i$. The RKHS-based \emph{IPW-KRR}
estimator studied later is, however, more broadly applicable. In particular,
the covariate $U_{t,i}$ may take values in a general Banach space
$\mathcal{U}$, or more generally a topological space, provided a reproducing
kernel $K:\mathcal{U}\times\mathcal{U}\to\mathbb{R}$ is defined. By the
reproducing property, any function $g\in\mathcal{H}_K$ satisfies $g(x)=\langle g,K(\cdot,x)\rangle_{\mathcal{H}_K},  x\in\mathcal{U}.$
This observation allows possible extensions to multi-index or nonlinear models,
although we do not pursue such directions here.
We make the following model assumptions.

\begin{assumption}[Model assumptions]
\label{ass-model}
\begin{enumerate}[label = (M\arabic*), leftmargin=2.2em]
    \item \textbf{Covariate distribution:}\label{assump:cov_dist}
    The covariates \( X_t \in \mathbb{R}^d \) are drawn i.i.d.\ from a
    distribution with a continuously differentiable density function
    \( p : \mathbb{R}^d \to \mathbb{R}_{>0} \), that is,
    \( p \in C^1(\mathbb{R}^d) \) and \( p(x) > 0 \) for all
    \( x \in \operatorname{supp}(P_X) \). The associated score function is
    defined as
    \[
    S(x) := -\nabla_x \log p(x).
    \]
    Throughout the paper, we assume that the score function \(S\) is known.
    \item \textbf{Noise distribution:}\label{assump:noise}
    The noise sequence \( \{ \varepsilon_t \} \) is conditionally independent
    given the actions \( \{a_t\} \), satisfies
    \( \mathbb{E}[\varepsilon_t \mid a_t = i] = 0 \) and
    \( \operatorname{Var}(\varepsilon_t \mid a_t = i) = \sigma^2 < \infty \),
    and is independent of the covariates \( X_t \) given \( a_t = i \), for
    all \( i \in \{1, \dots, L\} \).

    \item \textbf{RKHS model class:}\label{assump:rkhs}
    Each reward function \( f_i \) lies in a reproducing kernel Hilbert space
    (RKHS) \( \mathcal{H}_K \), associated with a positive definite kernel
    \( K: \mathcal{U} \times \mathcal{U} \to \mathbb{R} \), where
    \( \mathcal{U} \subset \mathbb{R} \) is compact and contains the range of
    the index variable \( U_t = X_t^\top \beta_i \). The kernel \( K \) is
    assumed to be bounded and continuous; that is, there exists \( \kappa > 0 \)
    such that \( \sup_{x \in \mathcal{U}} K(x, x) \leq \kappa \).
\end{enumerate}
\end{assumption}

\begin{Remark}
(i) Assumption~\ref{assump:cov_dist} is standard in the single-index literature
(e.g., \citealp{ichimura1993,horowitz2009,bala:22}) and ensures that the score
function $S(x)=-\nabla_x\log p(x)$ is well-defined. In practice, when the covariate distribution is unknown, the score 
function may be replaced by a suitable estimator. We leave the analysis 
of such plug-in approaches to future work.
(ii) Assumption~\ref{assump:noise} imposes conditionally mean-zero,
homoscedastic noise within each arm, as is common in contextual bandits and
semiparametric regression (e.g.,
\citealp{PerchetRigollet2013,qianandyang2016,ichimura1993}). 
(iii) Assumption~\ref{assump:rkhs} places each reward function in a bounded
RKHS, which provides the regularity needed for kernel-based nonparametric
estimation and regret analysis (e.g.,
\citealp{srinivas2010gaussian,valko2013finite,arya:23}).
\end{Remark}

Let $a_t^\ast := \arg\max_{i \in \mathcal{A}} f_i(X_t^\top \beta_i)$ denote the
optimal arm at time $t$, and let $a_t$ be the arm chosen by the algorithm
$\pi$ at time $t$. The \emph{cumulative regret} incurred by $\pi$ over horizon $T$ is
defined as
\begin{equation}\label{eq:regret}
    R_T(\pi) := \sum_{t=1}^T \left(
    f_{a_t^\ast}(X_t^\top \beta_{a_t^\ast}) -
    f_{a_t}(X_t^\top \beta_{a_t})
    \right).
\end{equation}
The objective is to design a decision-making strategy $\pi$ that minimizes
$R_T(\pi)$, ideally achieving sublinear regret in $T$.

Our goal is to develop a contextual bandit framework that supports both
efficient learning and statistical inference for the single-index reward
model~\eqref{eq: MABCmodel}. We estimate the index parameter
\( \beta_i \in \mathbb{R}^d \) for each arm using a Stein-type score-based
method~\citep{bala_PMLR}, which is agnostic to the unknown link function
\( f_i \). We then estimate the arm-level reward function
\( f_i \in \mathcal{H}_K \) by inverse-propensity-weighted kernel ridge
regression (IPW-KRR) \citep{arya:23}, using the projected covariate
\( U_t = X_t^\top \widehat{\beta}_{i,t} \). Our analysis establishes a central
limit theorem for the index estimator and a directional functional CLT for the
RKHS estimator of \( f_i \), yielding valid confidence intervals for both model
components under adaptive sampling.

We proceed as follows: In Section~\ref{sec: algorithm}, we present the
proposed algorithm. In Section~\ref{sec:beta-estimation}, we develop estimation
and inference for the index parameter. In Section~\ref{sec:rkhs-clt-setup}, we
study estimation and inference for the link function. In
Section~\ref{sec: regret_analysis}, we derive finite-time regret bounds. All
proofs are deferred to the Supplement \cite{suppl:arya:26}.

\section{Algorithm}\label{sec: algorithm}

We propose \emph{K-SIEGE}, a kernelized $\varepsilon$-greedy strategy for contextual bandits with per-arm single-index reward models. The method combines Stein-based estimation of the arm-specific index parameters with RKHS-based nonparametric estimation of the reward functions, while balancing exploration and exploitation through a decaying exploration schedule. The full procedure is given in Algorithm~\ref{algo:kernel-eps-greedy}.

\begin{algorithm}[t]
	\caption{K-SIEGE (Kernelized Single-Index Epsilon-Greedy Exploration) Strategy} 
    \label{algo:kernel-eps-greedy}
	\begin{algorithmic}[1]

        \State \textbf{Warm start:}
Select arms $\hat a_1,\dots,\hat a_{t_0}$ so that each arm is pulled at least once,
and observe rewards $Y_1,\dots,Y_{t_0}$.
\State Using data $\{(X_s,\hat a_s,Y_s)\}_{s=1}^{t_0}$,
compute initial estimators 
$\{\hat\beta_i^{(t_0)}, \hat f_{i,t_0}\}_{i\in\mathcal A}$
via \eqref{eq:hatGamma} and the IPW-KRR estimator from \citep{arya:23}.
        \For {$t = t_0+1,\dots,T$}
\State Using $\{\hat\beta_i^{(t-1)},\hat f_{i,t-1}\}_{i\in\mathcal A}$,
compute the best performing arm:
$
A_t = \arg\max_{i\in\mathcal A}
\hat f_{i,t-1}(X_t^\top \hat\beta_i^{(t-1)}).
$
\State \emph{$\epsilon$-greedy step:}  For a non-increasing exploration probability sequence $\{\epsilon_t, t \geq 1\}$, the arm pulled is given by
        $$ \hat{a}_t = 
        \begin{cases}
            A_t\ \text{with probability} 1-\epsilon_t,\\
            \mathcal{A} \setminus A_t\  \text{with probability} \frac{\epsilon_t}{L-1}
        \end{cases}
        $$
        \State  Observe reward $Y_t$ corresponding to $\hat{a}_t$
        \State  For $i = \hat a_t$, update both $\hat\beta_{i}^{(t)}$ and $\hat f_{i,t}$ using ($X_t, Y_t)$.
For $i \neq \hat a_t$, set $\hat\beta_{i}^{(t)} = \hat\beta_{i}^{(t-1)}$ and $\hat f_{i,t} = \hat f_{i,t-1}$. 

		\EndFor
	\end{algorithmic} 
\end{algorithm}

Algorithm~\ref{algo:kernel-eps-greedy} proceeds as follows. After an initial warm start, the method estimates the arm-specific index parameters and reward functions from the available data. At each round, it selects the arm with the largest predicted reward based on the current estimators, while preserving exploration through the $\varepsilon_t$-greedy mechanism. Only the estimators for the sampled arm are updated, which keeps the procedure fully online and computationally tractable.

Tables~\ref{tab:notation_parametric} and~\ref{tab:notation_rkhs} in the supplement summarize the notation used throughout the problem setup. We next describe the estimation procedures in detail.
\section{Estimation and Inference for the Index Parameter}\label{sec:beta-estimation}

Classical single-index methodology studies estimation of $\beta$ in models of the form
$Y=f(X^\top\beta)+\varepsilon$ primarily under i.i.d.\ sampling, using profile least squares or likelihood-based methods that jointly estimate $\beta$ and the unknown link $f$; see, e.g., \cite{ichimura1993, horowitz2009}. At the population level, these approaches rely on first-order optimality conditions for the regression risk involving quantities of the form $\E[f'(X^\top\beta)X]$, and therefore require either explicit estimation of $f$ or strong structural assumptions on the covariate distribution. Consequently, the induced estimating equations are nonlinear in $\beta$ and are not readily adapted to sequential or selectively observed data.

A key insight of \cite{plan2016generalized} and \cite{bala_PMLR} is that Stein's identity can circumvent these difficulties when the covariate distribution admits a smooth density with score function. Stein's identity permits identification and inference for the index parameter without prior estimation of the nonparametric link function, thereby decoupling dimension reduction from subsequent reward-function learning under adaptive sampling. This allows estimation of $\beta_i$ prior to estimation of $f_i$, avoids joint ill-posed optimization, uses only $X$ through its score together with $Y$, and remains well-defined under adaptive sampling. More formally, the relevant form of Stein's identity is the following.

\begin{proposition}[First-order Non-Gaussian Stein's Identity~\citep{bala_PMLR}]
Let $X\in \real^d$ be a real-valued random vector with a differentiable density $p$, and 
let $g:\real^d\to \real$ be a continuous differentiable function such that $\expect[\nabla g(X)]$ exists. Then it holds that,
$$
\expect[g(X)S(X)] = \expect[\nabla g(X)].
$$
\end{proposition}

Applying Proposition~1 with $g(X_t)=f_i(X_t^\top\beta_i)$, assuming the required differentiability and integrability conditions, yields the population moment relation
\[
\mathbb{E}\!\left[Y_t\,S(X_t)\right]
=
\mathbb{E}\!\left[f(X_t^\top\beta_i)S(X_t)\right]
=
\mathbb{E}\!\left[f'(X_t^\top\beta_i)\right]\beta_i
=:\mu_i\,\beta_i,
\]
where $\mu_i=\mathbb{E}[f'(X_t^\top\beta_i)]$ is assumed nonzero. Thus Stein's identity identifies the direction of $\beta$ through the first moment $\mathbb{E}[Y_tS(X_t)]$.

Since Stein's identity characterizes the index direction through inner products with the score function, the natural estimand is the $L^2$-projection of $Y$ onto the linear span of $S(X_t)$. By the projection theorem, this projection is equivalently defined as the minimizer of the quadratic criterion $\mathbb E[(Y_t-S(X_t)^\top\beta_i)^2],$
whose first-order condition yields the normal equation $\Gamma\beta_i=m$, where
$
\Gamma=\mathbb E[S(X_t)S(X_t)^\top],$ and $m=\mathbb E[Y_tS(X_t)].
$
The solution 
$\Gamma^{-1}m$ satisfies the Stein's equation: $m = \mu_i \beta_i$, upto  
a linear transformation of the true index direction. The directional recovery is exact when $\Gamma$ is isotropic, or more generally, after an appropriate whitening transformation. A common sufficient condition is spherical symmetry of the covariate distribution; for example, when $X_t\sim N(0,I_d)$, one has $S(X_t)=X_t$ and $\Gamma=I_d$. This is consistent with the classical observation that, under standard Gaussian covariates, ordinary least squares recovers the index direction up to scale even under a nonlinear single-index model. More generally, if $X_t$ is elliptically symmetric and the score is affine in $X_t$, then $\Gamma$ has a structured form and the index direction may be recovered accordingly.

The present work is not merely an application of \cite{bala_PMLR}, but adapts the Stein-based approach to the adaptive contextual bandit setting, where outcomes are observed only for selected arms and the sampling law of $(X_t,Y_t)$ depends on the past. In this regime,
\[
\mathbb E[Y_t S(X_t)\mid a_t=i]\neq \mathbb E[Y_t S(X_t)],
\]
and the naive empirical moment $t^{-1}\sum_{s\le t}Y_t S(X_t)$ is biased for $m$. We therefore construct inverse-propensity-weighted analogues of $m$ and $\Gamma$, yielding an unbiased estimating equation whose solution remains colinear with the population index. The resulting estimator admits a martingale representation, which in turn supports concentration inequalities, finite-sample error bounds, and asymptotic normality via self-normalized martingale limit theory.

\paragraph*{Estimator via an empirical normal equation}
Let $W_s := S(X_s) \in \mathbb{R}^d, \ s\ge 1$ denote the score function evaluated at the covariate $X_s$, 
and let
\[
p_{s,i}(W_s) := \mathbb{P}(\hat{a}_s = i \mid \mathcal{F}_{s-1}, W_s)
\]
be the conditional probability of selecting arm $i$ at time $s$, given the past history
$\mathcal{F}_{s-1} := \sigma(X_1, \hat{a}_1, Y_1, \dots, X_{s-1}, \hat{a}_{s-1}, Y_{s-1})$
and current context $X_s$ (and hence $W_s$). Motivated by the population normal equation $\Gamma\beta=m$, we estimate $\beta_i$ for arm $i$ by the plug-in estimator
where the IPW covariance matrix and the population level covariance matrix are given by
	\begin{align}
		\label{eq:hatGamma}
\hat\beta_{i,t} := \hat\Gamma_{i,t}^{-1}\hat m_{i,t},
\qquad
\hat\Gamma_{i,t}:=\frac1t\sum_{s=1}^t \frac{\mathbf 1\{\hat a_s=i\}}{p_{s,i}(W_s)}\,W_sW_s^\top,\quad
\hat m_{i,t}:=\frac1t\sum_{s=1}^t \frac{\mathbf 1\{\hat a_s=i\}}{p_{s,i}(W_s)}\,W_sY_s,
	\end{align}
where $W_s=S(X_s)$. This estimator is equivalently the minimizer of the weighted empirical criterion
$
\beta_i \mapsto \frac1t\sum_{s=1}^t \frac{\mathbf 1\{\hat a_s=i\}}{p_{s,i}(W_s)}\,(Y_s-W_s^\top\beta_i)^2,$
whenever $\hat\Gamma_{i,t}$ is invertible. For interpretability under the single-index invariance, we report the normalized direction $\hat b_{i,t}:=\hat\beta_{i,t}/\|\hat\beta_{i,t}\|$ when needed (Section~4.2).

Because rewards are observed only for the selected arm at each round, and these selections depend on past data and covariates, the observed data are adaptively sampled. Consequently, the naive sample average of \( Y_t S(X_t) \) is not an unbiased estimator of the population moment \( \mathbb{E}[Y_t S(X_t)] \). To correct for this selection effect, we use inverse-probability weighting based on the arm-selection propensities. Then, for each $s$,
$
\smash{\mathbb E\!\left[\frac{\mathbf 1\{\hat a_s=i\}}{p_{s,i}(W_s)}\,W_s\,Y_s
\Bigm|\mathcal F_{s-1},W_s\right]}
=
W_s\,\mathbb E\!\left[Y_s(i)\mid \mathcal F_{s-1},W_s\right],
$
because $\mathbb E[\mathbf 1\{\hat a_s=i\}\mid\mathcal F_{s-1},W_s]=p_{s,i}(W_s)$.
Consequently,
\[
\mathbb E[\hat m_{i,t}]
=
\mathbb E\!\left[W_s\,\mathbb E[Y_s(i)\mid W_s]\right]
=
\mathbb E[W_sY_s(i)]
=:m_i.
\]
Thus, the IPW moment estimator $\hat m_{i,t}$ is unbiased for $m_i=\mathbb E[W_sY_s(i)]$. The estimator $\hat\beta_{i,t}=\hat\Gamma_{i,t}^{-1}\hat m_{i,t}$ is not exactly unbiased in finite samples, since expectation does not commute with matrix inversion. Rather, $\hat\beta_{i,t}$ is characterized as the solution to an unbiased estimating equation,
$
\mathbb E\big[\hat\Gamma_{i,t}\beta_i-\hat m_{i,t}\big]=0,
$
where $\Gamma=\mathbb E[W_sW_s^\top]$ and $m_i=\mathbb E[W_sY_s(i)]$ satisfy the population normal equation $\Gamma\beta_i=m_i$. This is the relevant formulation in $Z$-estimation and semiparametric theory and is sufficient for consistency and asymptotic normality under Assumption~2. Accordingly, we estimate $\Gamma$ and $m_i$ by their inverse-propensity-weighted empirical analogues and set $\hat\beta_{i,t}=\hat\Gamma_{i,t}^{-1}\hat m_{i,t}$.

Unbiasedness of the estimator itself is not required for valid inference. What matters is that the estimating equation has mean zero at the true parameter and that its derivative is well behaved. This principle underlies least squares, generalized method of moments, and modern adaptive inference.
We work under the following conditions.

\begin{assumption}[Regularity conditions for index estimation]\label{ass:beta-regularity}

\begin{enumerate}[label = (A\arabic*), leftmargin=2.2em]
    \item \label{ass:gram:para:second moment}  \textbf{Predictability (stationary second moment).}
    For each \( s \), the conditional second moment of the score vector \( W_s = S(X_s) \) is stable and predictable:
    \[
    \mathbb{E}[\,W_s W_s^\top \mid \mathcal{F}_{s-1}\,] = \Gamma,
    \qquad \Gamma \in \mathbb{R}^{d\times d} \text{ deterministic.}
    \]
    \item \label{ass:avg:explore:rt} \textbf{Average exploration (overlap).}
    Define
    \begin{align}
    \label{def: r_t}
    p_s^\star := \inf_{i\in[L]}\;\inf_{w\in\mathbb R^d} p_{s,i}(w),
\qquad
r_t := \frac{1}{t^2}\sum_{s=1}^t \frac{1}{p_s^\star}.
    \end{align}
    where \(p_{s,i}(w)\) is the propensity of selecting arm \(i\) given covariate \(w\) at time \(s\). Under $\varepsilon_s$-greedy exploration, $p_s^\star=\varepsilon_s/(L-1)$.
    We assume that \(r_t \to 0\) as \(t \to \infty\), or equivalently,
    \(\sum_{s \le t} 1/p_s^\star = o(t^2)\).

    \item \label{ass:subG} \textbf{Sub-Gaussian features.}
    There exists \(\sigma_W > 0\) such that for all unit vectors \(v \in \mathbb{R}^d\) and all \(\lambda \in \mathbb{R}\),
    \[
        \mathbb{E}\!\left[\exp\!\left(\lambda v^\top W_s\right)\right]
        \le \exp\!\left(\tfrac{\lambda^2 \sigma_W^2}{2}\right).
    \]
    Consequently, for every integer \(k \ge 1\), there exists \(M_{2k} < \infty\) such that
    \[
        M_{2k} := \sup_{s \ge 1}\, \mathbb{E}\!\left[\|W_s\|_2^{2k}\right] < \infty.
    \]

    \item \label{ass:identifiability} \textbf{Identifiability.}
    The population Gram matrix \(\Gamma\) is positive definite with
    \[
        \lambda_{\min}(\Gamma) \ge \mu > 0.
    \]
\end{enumerate}
\end{assumption}

\begin{Remark}
(i) Assumption~\ref{ass:gram:para:second moment} imposes a stable and predictable second-moment (Gram) structure for the score vectors \(W_s\). Although the sampling and filtration \(\{\F_{s-1}\}\) may be adaptive, the average quadratic signal carried by \(W_s\) remains time-homogeneous and converges to a deterministic information matrix \(\Gamma\). Predictability conditions of this form are standard in matrix martingale theory \citep{HallHeyde1980, Tropp2011, Oliveira2010}. They hold, for example, when \(W_s\) is independent of \(\F_{s-1}\) and identically distributed with \(\E[W_sW_s^\top]=\Gamma\). 
(ii) Assumption~\ref{ass:avg:explore:rt} is an overlap (or positivity) condition on the arm assignment mechanism. The quantity \(r_t = t^{-2}\sum_{s\le t} 1/p_s^\star\) measures the variance inflation induced by inverse-probability weighting. The requirement \(r_t \to 0\) prevents any arm from being systematically undersampled and ensures consistent estimation and nondegenerate asymptotics. Similar overlap conditions are standard in causal inference and off-policy evaluation \citep{RosenbaumRubin1983, Li2010, Dimakopoulou2019, hadad2021confidence}. Under the $\epsilon$-greedy exploration used here, \cite{arya:23} assume \(p_{s,i} \ge \epsilon_s/(L-1)\) for all \(i\in\{1,\dots,L\}\), so \(p_s^\star=\epsilon_s/(L-1)\) is a user-specified sequence independent of the data. 
(iii) Assumption~\ref{ass:subG} imposes sub-Gaussian tails for \(W_s\), guaranteeing finite high-order moments and exponential concentration for \(\|W_s\|_2\). Such conditions are standard in high-dimensional regression \citep{Vershynin2018}, kernel ridge regression, and RKHS learning \citep{caponnetto2007optimal, SteinwartChristmann2008}. The assumption holds under mild regularity of the covariate distribution: if the log-density \(\log p(x)\) has a bounded or Lipschitz gradient, as in Gaussian or more generally strongly log-concave distributions, then the score vector \(S(X)=-\nabla\log p(X)\) is sub-Gaussian. Consequently, the score-transformed covariates \(W_s\) have Gaussian-type tails and large deviations occur with exponentially small probability (see Lemma~\ref{lem:sg-tail} in Section~\ref{secApp: proofs_beta_estimation} of the supplement~\citep{suppl:arya:26}). 
(iv) Assumption~\ref{ass:identifiability} requires the population Gram matrix \(\Gamma\) to be strictly positive definite. This standard identifiability condition ensures a nondegenerate signal space, invertibility of the empirical Gram matrices, and stability of the projection used for directional inference. Similar eigenvalue conditions are common in random-design regression and RKHS learning \citep{Bach2013, SteinwartChristmann2008}.
\end{Remark}

\paragraph*{On the invertibility of $\hat{\Gamma}_{i,t}$ with high probability}
Practically, one could use a Tikhonov-regularized version
\[
\hat{\beta}_{i,t} =(\hat{\Gamma}_{i,t} + \lambda I)^{-1}   \sum_{s=1}^t Y_sW_s.
\]
However, Corollary~\ref{cor:PD} 
justifies the definition~\eqref{eq:hatGamma} without any tuning parameter or regularization. As such, the empirical inverse-propensity-weighted covariance matrix $\hat\Gamma_{i,t}$ concentrates around its population counterpart $\Gamma$, and thus its smallest eigenvalue converges to the true one despite adaptive sampling. It is then immediate from Assumption~\ref{ass:identifiability} that, since $\Gamma$ is strictly positive definite, $\hat\Gamma_{i,t}$ is also positive definite with high probability, ensuring that $\hat\Gamma_{i,t}^{-1}$ exists, is well-conditioned, and is stable for inference for large $t$. By Theorem~\ref{thm:eig-consistency} and Corollary~\ref{cor:PD} 
in Section~\ref{secApp: proofs_beta_estimation} 
of the supplement~\citep{suppl:arya:26},
$
\|\hat\Gamma_{i,t} - \Gamma\| \xrightarrow{\mathbb{P}} 0.
$
Since $\lambda_{\min}(\Gamma) \ge \mu > 0$, Weyl's inequality implies that
$
\lambda_{\min}(\hat\Gamma_{i,t}) > \mu/2
$
with high probability for $t$ sufficiently large, ensuring that $\hat\Gamma_{i,t}^{-1}$ exists and is stable. In particular, the estimator of $\beta_i$ is based on a nondegenerate signal, and the inversion in~\eqref{eq:hatGamma} is well-defined with high probability. This plays a central role in enabling asymptotic linearization in the subsequent section.

Next, we provide theoretical guarantees for the proposed estimator~\eqref{eq:hatGamma}, including a martingale central limit theorem (CLT) in Section~\ref{sec:asymptotic_inference_beta} that enables inference for the index parameter $\beta_i$ under adaptive sampling, as well as a non-asymptotic finite-sample error bound in Section~\ref{sec:index_concentration_summary}. The corresponding proofs can be found in Section~\ref{secApp: proofs_beta_estimation} 
of the Supplement.

\subsection{Asymptotic Inference for the Index Parameter}
\label{sec:asymptotic_inference_beta}

Recall the empirical moment estimator and its population counterpart
\[
\hat{m}_{i,t} := \frac{1}{t}\sum_{s=1}^{t} \frac{\1\{\hat{a}_s = i\}}{p_{s,i}(W_s)}\, W_s Y_s,
\qquad
m_i := \mathbb{E}[W_s Y_s].
\]
By construction of the inverse-propensity weights, $\E[\hat m_{i,t}] = m_i$.  
This unbiased estimating equation admits a martingale difference representation under adaptive sampling. 
Combined with stability of the conditional second moments and suitable moment conditions, this structure enables a central limit theorem for the estimator of the index parameter. The target parameter is
$
\beta_i := \Gamma^{-1} m_i .
$

Define the innovation vector
$
\Delta_s(\beta_i) := W_s \big(Y_s - W_s^\top \beta_i\big) \in \mathbb{R}^d,
$
and its covariance matrix
$
\Lambda_i := \mathbb{E}\!\left[\Delta_s(\beta_i)\Delta_s(\beta_i)^\top\right],
$
which is finite and positive semidefinite under Assumption~\ref{ass:resid}.

Under adaptive data collection, actions $\hat a_s$ depend on past observations.  
The ignorability property below of the inverse-propensity weights ensures correct variance scaling and underlies the martingale concentration arguments used in the derivations
\begin{equation}
\label{ass:ignorability}
\E\!\left[\left(\frac{\1\{\hat a_s = i\}}{p_{s,i}(W_s)}\right)^2 \middle| \F_{s-1},W_s\right]
=\frac{1}{p_{s,i}(W_s)} .
\end{equation}
We impose the following conditions.

\begin{assumption}[Regularity conditions for asymptotic inference]
\label{ass:beta-clt}

\begin{enumerate}[label = (B\arabic*), leftmargin=2.2em]

\item \label{ass:outcome:moments}
\textbf{Outcome moments.}
There exists $\delta>0$ such that
$
\sup_{s\ge1}\E[|Y_s|^{2+\delta}]<\infty,
\sup_{s\ge1}\E[\|W_s\|^{2+\delta}]<\infty .
$
Consequently
$
\sup_{s\ge1}\E[\|\Delta_s(\beta_i)\|^{2+\delta}]<\infty .
$

\item \label{ass:Lyapunov}
\textbf{No dominant inverse-propensity weight.}
Let \(p_s^\star := \essinf_w p_{s,i}(w)\). Then
\[
\frac{\sum_{s=1}^t (p_s^\star)^{-(1+\delta/2)}}{\left(\sum_{s=1}^t (p_s^\star)^{-1}\right)^{1+\delta/2}}
\to 0 .
\]

\item \label{ass:resid}
\textbf{Residual nondegeneracy.}
There exists $\sigma_\beta^2>0$ such that
$
\E[(Y_s-W_s^\top\beta_i)^2\mid W_s]\ge\sigma_\beta^2 .
$
Hence
$
\Lambda_i
=\E[W_sW_s^\top(Y_s-W_s^\top\beta_i)^2]
\succeq \sigma_\beta^2\Gamma .
$ is positive definite.
\end{enumerate}
\end{assumption}

\begin{Remark}
(i) Assumption~\ref{ass:outcome:moments} imposes uniform $(2+\delta)$-moment bounds on $Y_s$ and $W_s$, ensuring that the martingale increments $\Delta_s(\beta_i)$ have uniformly bounded higher moments. Such conditions are standard in martingale CLTs and semiparametric asymptotics \citep{HallHeyde1980, delaPenaLaiShao2009, NeweyMcFadden1994}.
(ii) Assumption~\ref{ass:Lyapunov} plays the role of a Lyapunov condition controlling heavy inverse-propensity weights. It prevents a small number of rounds with extremely small propensities from dominating the variance, a requirement closely related to self-normalized martingale CLTs \citep{delaPenaLaiShao2009, HallHeyde1980}.
(iii) Assumption~\ref{ass:resid} ensures nondegeneracy of the regression residual variance so that the limiting covariance matrix remains positive definite. Similar conditions appear in linear regression, GMM, and semiparametric efficiency theory \citep{White1984, NeweyMcFadden1994}.
\end{Remark}

Define the martingale difference array
$
\psi_{t,s,i}
:=
t^{\alpha-1}\Gamma^{-1}
\frac{\1\{\hat a_s=i\}}{p_{s,i}(W_s)}
\Delta_s(\beta_i)
$ and let \(S_{t,i}:=\sum_{s=1}^t\psi_{t,s,i}\).  
Since
$
\smash{\E\!\left(\frac{\1\{\hat a_s=i\}}{p_{s,i}(W_s)}\middle|\F_{s-1},W_s\right)}=1
\quad\text{and}\quad
\E[\Delta_s(\beta_i)\mid W_s]=0,
$
we have $\E[\psi_{t,s,i}\mid\F_{s-1}]=0$, so $\{S_{t,i},\F_t\}$ forms a vector-valued martingale.

Define the predictable quadratic variation:
\begin{equation}\label{eq:Vbeta:def}
V_{\beta,t}
=
\sum_{s=1}^t\E[\psi_{t,s,i}\psi_{t,s,i}^\top\mid\F_{s-1}]
=
t^{2\alpha-2}\sum_{s=1}^t
\Gamma^{-1}
\E\!\left[
\frac{\1\{\hat a_s=i\}}{p_{s,i}^2(W_s)}
\Delta_s(\beta_i)\Delta_s(\beta_i)^\top
\middle|\F_{s-1}
\right]
\Gamma^{-1}.
\end{equation}

We now state the main result.

\begin{theorem}[Asymptotic inference for the index parameter]\label{thm:main}
Under Assumptions~\ref{ass-model}, \ref{ass:beta-regularity}, and \ref{ass:beta-clt}, suppose
$t^{2\alpha}r_t\to0$ for some $\alpha\in(0,1/2)$.  
Then for each arm $i=1,\dots,L$,
\[
V_{\beta,t}^{-1/2}\,t^\alpha(\hat\beta_{i,t}-\beta_i)
\xrightarrow{D}
\mathcal N(0,I_d).
\]
Moreover
\[
t^\alpha(\hat\beta_{i,t}-\beta_i)\xrightarrow{\mathbb P}0 .
\]
\end{theorem}

The theorem establishes asymptotically valid inference for $\beta_i$ under adaptive sampling.  
The slower rate $t^{-\alpha}$ reflects nonuniform exploration but still yields consistent estimation and valid confidence regions.

\paragraph*{Key steps of the proof}

\begin{enumerate}

\item[(i)]
We decompose
\[
V_{\beta,t}^{-1/2}t^\alpha(\hat\beta_{i,t}-\beta_i)
=
V_{\beta,t}^{-1/2}S_{t,i}
+
V_{\beta,t}^{-1/2}R_t ,
\]
where $S_{t,i}$ is the leading martingale term and $R_t$ arises from replacing $\Gamma$ with $\hat\Gamma_{i,t}$.  
We show $V_{\beta,t}^{-1/2}S_{t,i}$ satisfies a multivariate martingale CLT while $V_{\beta,t}^{-1/2}R_t\cip0$.

\item[(ii)]
For any fixed vector $a$, the projection $\zeta_{t,s}=a^\top V_{\beta,t}^{-1/2}\psi_{t,s,i}$ forms a martingale difference array with predictable quadratic variation $\|a\|^2$. Lemmas \ref{lem:Vbeta}-\ref{lem:Vbeta-inv} verify convergence of the quadratic variation and Proposition~\ref{prop:lindeberg} establishes the Lindeberg condition, yielding
$
V_{\beta,t}^{-1/2}S_{t,i}\xrightarrow{D}\mathcal N(0,I_d).
$

\item[(iii)]
The remainder term depends on $\hat\Gamma_{i,t}-\Gamma$.  
Lemma~\ref{lem:mds} expresses this difference as a martingale sum and Proposition~\ref{prop:second-moment} controls its second moment. High-probability concentration for $\hat\Gamma_{i,t}$ (Theorem~\ref{thm:GammaHP} and Corollary~\ref{cor:scaled-consistency}) implies
$
t^\alpha\|\hat\Gamma_{i,t}-\Gamma\|\cip0 ,
$
so the perturbation term $R_t$ is negligible.

\end{enumerate}

\paragraph*{Special case: propensity-independent sampling}

If $p_{s,i}(W_s)\equiv p_s$, then
\begin{equation*}
V_{\beta,t}
=
t^{2\alpha}r_t\,\Gamma^{-1}\Lambda_i\Gamma^{-1}.
\end{equation*}
This occurs when exploration is independent of covariates, such as uniform or randomized arm-pulling, which includes the K-SIEGE strategy.

\begin{theorem}
\label{thm:Vhat-consistency}
Under the conditions of Theorem~\ref{thm:main}, suppose $t^{2\alpha}r_t\to0$ for $r_t$ as defined in \eqref{def: r_t} and $\|\widehat\Gamma_{i,t}-\Gamma\|\xrightarrow{\mathbb P}0$.  
Define
\begin{equation}\label{eq:Vhat-def}
\widehat V_{\beta,t,i}
=
\sum_{s=1}^t\widehat\psi_{t,s,i}\widehat\psi_{t,s,i}^\top,
\qquad
\widehat\psi_{t,s,i}
=
t^{\alpha-1}\widehat\Gamma_{i,t}^{-1}
\frac{\1\{\widehat a_s=i\}}{p_{s,i}(W_s)}
\widehat\Delta_{s,i}.
\end{equation}
Then
\begin{equation}\label{eq:Vhat-operator}
\|\widehat V_{\beta,t,i}-V_{\beta,t}\|\xrightarrow{\mathbb P}0
\quad\text{and}\quad
\|\widehat V_{\beta,t,i}^{-1/2}V_{\beta,t}^{1/2}-I_d\|\xrightarrow{\mathbb P}0 .
\end{equation}
\end{theorem}

\begin{corollary}[Feasible studentization]\label{cor:feasible-CLT}
Under the assumptions of Theorem~\ref{thm:main} and Theorem~\ref{thm:Vhat-consistency},
$
\widehat V_{\beta,t,i}^{-1/2}t^\alpha(\widehat\beta_{i,t}-\beta_i)
\xrightarrow{D}\mathcal N(0,I_d).
$
Consequently, for any $x\in\R^d$,
$
x^\top\widehat V_{\beta,t,i}^{-1/2}t^\alpha(\widehat\beta_{i,t}-\beta_i)
\xrightarrow{D}\mathcal N(0,1).
$
An asymptotic $(1-\tau)$ confidence region for $\beta_i$ is therefore given by 
$
\mathcal C_{t,1-\tau}
=
\{\theta\in\R^d:\,
t^\alpha\|\widehat V_{\beta,t,i}^{-1/2}(\widehat\beta_{i,t}-\theta)\|
\le\chi_{d,1-\tau}\}.
$
\end{corollary}

In practice, $\widehat V_{\beta,t,i}$ is computed using observable quantities by replacing $\Gamma$ with $\widehat\Gamma_{i,t}$ and using residuals $\widehat\Delta_{s,i}$.  
The resulting statistic
$
\widehat T_{t,i}
=
\widehat V_{\beta,t,i}^{-1/2}t^\alpha(\widehat\beta_{i,t}-\beta_i)
$
is asymptotically $\mathcal N(0,I_d)$, providing a practical basis for Wald-type inference and hypothesis testing for the index parameter.

\subsection{Delta Method for Directional Inference}

We derive a CLT for the normalized direction
$
\hat b_{i,t}:=\hat\beta_{i,t}/\|\hat\beta_{i,t}\|,
$
which is the natural inferential target under the identifiability constraint
\(\|\beta_i\|_2=1\) in model~\eqref{eq: MABCmodel}. 
 Consider the map \(g(x)=x/\|x\|\), which is smooth on \(\mathbb R^d\setminus\{0\}\). Its Jacobian at \(x\neq 0\) is
$ J(x)=\frac{1}{\|x\|}(I-bb^\top),
 b=x/\|x\|.$
Applying the delta method to Theorem~\ref{thm:main} yields
$t^{\alpha} V_{b,i}^{-1/2}(\hat b_{i,t}-b_i)
\;\xrightarrow{D}\;
\mathcal N(0,I_d),
V_{b,i}=J(\beta_i)\,V_{\beta,i}\,J(\beta_i)^\top.
$

Given a consistent estimator \(\widehat V_{\beta,t,i}\) of \(V_{\beta,i}\), a plug-in estimator of the directional covariance is
$
\widehat V_{b,t,i}
=
J(\hat\beta_{i,t})\,\widehat V_{\beta,t,i}\,J(\hat\beta_{i,t})^\top.
$
Since \(b_i\in\mathbb S^{d-1}\), the covariance \(V_{b,i}\) acts on the \((d-1)\)-dimensional tangent space at \(b_i\). An asymptotic \((1-\alpha)\) confidence ellipsoid for \(b_i\) is therefore
${\left\{
b:\;
(b-\hat b_{i,t})^\top \widehat V_{b,t,i}^{-1}(b-\hat b_{i,t})
\le \chi^2_{d-1,1-\alpha}
\right\}}.
$

\paragraph*{Implementation in Algorithm~\ref{alg:parametric_inference}}
In K-SIEGE, we use the delta-method covariance \(\widehat V_{b,t,i}\) to quantify uncertainty for each arm-specific direction at prescribed inference times. The construction is based on the feasible covariance estimator in~\eqref{eq:Vhat-def} together with the Jacobian \(J(\hat\beta_{i,t})\), and yields valid confidence regions for the normalized index direction.

\begin{algorithm}[h]
\caption{Parametric Directional Inference for Single-Index Bandits}
\label{alg:parametric_inference}
\begin{algorithmic}[1]
\State \textbf{Input:} Covariates $X_{1:T}\in\mathbb{R}^d$, arms $i\in\{1,\dots,L\}$, regularization $\lambda_t$, propensities $p_{s,i}$, inference times $\mathcal{T}_{\mathrm{infer}}\subseteq \{T_0,\dots,T\}$.
\For{$t = 1$ to $T$}
  \State Compute empirical mean $\hat\mu_t$ and covariance $\hat\Sigma_{t,i}$ of $X_{1:t}$; define score features $W_s = \hat\Sigma_{t,i}^{-1}(X_s - \hat\mu_t)$.
  \State For each arm $i$, define weights $w_{s,i} = \frac{\mathbf{1}\{a_s=i\}}{\max\{p_{s,i},p_\text{min}\}}$ and estimate
  \[
  \hat{\beta}_{i,t} = A_{i,t}^{-1} b_{i,t}, \quad \text{where} \quad A_{i,t} = \frac{1}{t}\sum_{s=1}^t w_{s,i} W_s W_s^\top + \lambda_t I, \quad b_{i,t} = \frac{1}{t} \sum_{s=1}^t w_{s,i} W_s Y_s.
  \]
  \State Normalize: $\hat{b}_{i,t} = \hat{\beta}_{i,t}/\|\hat{\beta}_{i,t}\|$.
  \If{$t \in \mathcal{T}_{\mathrm{infer}}$}
    \State Compute residuals $R_s = Y_s - W_s^\top \hat{\beta}_{i,t}$ and influence vectors $\widehat{\psi}_s = t^{\alpha - 1} A_{i,t}^{-1} w_{s,i} W_s R_s$ for $s \leq t$.
    \State Estimate covariance: $\hat V_{\beta,t,i} = \sum_{s=1}^t \widehat{\psi}_s \widehat{\psi}_s^\top$.
    \State For $\hat{b}_{i,t} = \hat{\beta}_{i,t}/\|\hat{\beta}_{i,t}\|$, Jacobian $J(\hat{b}_{i,t}) = (I - \hat{b}_{i,t}\hat{b}_{i,t}^\top)/\|\hat{\beta}_{i,t}\|$.
    \State Estimate directional covariance: $\hat V_{\mathrm{dir},i,t} = J(\hat{b}_{i,t}) \hat V_{i,t} J(\hat{b}_{i,t})^\top / t^{2\alpha}$.
    \State Form $(1-\delta)$ ellipsoid: $\{b: b^\top \hat V_{\mathrm{dir},i,t}^{-1} b \leq \chi^2_{d-1,1-\delta}\}$; extract marginal intervals from the diagonal.
  \EndIf
\EndFor
\end{algorithmic}
\end{algorithm}

\subsection{Finite-Sample Control of the Index Estimator}
\label{sec:index_concentration_summary}

To complement the asymptotic results of Section~\ref{sec:asymptotic_inference_beta}, we establish a non-asymptotic bound for the inverse-propensity-weighted estimator $\hat{\beta}_{i,t}$ under adaptive sampling. The analysis controls deviations of the inverse-propensity-weighted empirical Gram matrix and moment estimator and propagates them through the decomposition
$
\hat{\beta}_{i,t}-\beta_i
=
\Gamma^{-1}(\hat m_{i,t}-m_i)
+
(\hat{\Gamma}_{i,t}^{-1}-\Gamma^{-1})\hat m_{i,t}.
$
The bound depends on the exploration complexity
$
r_t = t^{-2}\sum_{s\le t}\frac{1}{p_s^\star},
$
which quantifies the cumulative variance inflation induced by inverse-propensity weighting. Under the regularity and overlap conditions introduced earlier, the estimator satisfies
$
\|\hat{\beta}_{i,t}-\beta_i\|
=
\smash{O_P\!\left(
\sqrt{r_t} + \sqrt{\frac{\log d}{t}}
\right)}.$
The second term corresponds to the standard concentration rate for empirical covariance and moment estimators in dimension \(d\), whereas $\sqrt{r_t}$ captures the additional variability arising from inverse-propensity weighting. In particular, if the exploration schedule ensures \(r_t=o(1)\), the estimator is consistent despite adaptive sampling.

This bound guarantees that the index estimator remains sufficiently accurate for the reward estimation and arm-selection components of the algorithm, and it plays a key role in the regret analysis by controlling the propagation of index estimation error into the allocation policy. The full finite-sample inequality and its proof are provided in Section~\ref{subsecApp: proofs_beta_finite_time} of the Supplement.

\section{Estimation and Inference for the link function}\label{sec:rkhs-clt-setup}

While estimation of $\beta_i$ is agnostic to the unknown link function, learning $f_i$ is necessary for implementing the allocation rule and analyzing regret. \cite{arya:23} establish consistency and finite-time guarantees for their inverse-propensity-weighted kernel ridge regression (IPW-KRR) estimator. Our goal is to develop an asymptotic theory enabling statistical inference for this estimator in the KSIB model. In particular, we derive a directional functional central limit theorem (CLT) for the IPW-KRR estimator, yielding pointwise confidence intervals for $f_i(\cdot)$ evaluated along the projected covariates in the estimated index direction.

Fix an arm $i\in\{1,\dots,L\}$ and let $(\F_s)_{s\ge0}$ denote the natural filtration of the algorithm.  
For a direction vector $\beta_i$, define the projected covariates $U_{s,i}=X_s^\top\beta_i\in\mathcal U\subset\mathbb R$.  
The outcome is $Y_s\in\mathbb R$, and the propensity function is
$
p_{s,i}(u)=\Prob(\hat a_s=i\mid\F_{s-1},U_{s,i}=u)\in(0,1].
$
Under Assumption~\ref{assump:rkhs}, each $f_i\in\Hk$, where $\Hk$ is an RKHS on $\mathcal U$ with bounded kernel $K$.  
Write $K_u=K(\cdot,u)\in\Hk$, so $\|K_u\|_K\le\kappa$.  
Table~\ref{tab:notation_rkhs} in the Supplement summarizes the notation used in this section.

Define the population and empirical inverse-propensity-weighted operators
\[
\Sigma_i f
=
\E\!\left[
\frac{\1\{\hat a_s=i\}}{p_{s,i}(U_{s,i})}
\langle f,K_{U_{s,i}}\rangle_K K_{U_{s,i}}
\right],
\qquad
h_i
=
\E\!\left[
\frac{\1\{\hat a_s=i\}}{p_{s,i}(U_{s,i})}Y_sK_{U_{s,i}}
\right],
\]
\[
\hat\Sigma_{i,t}
=
\frac1t\sum_{s=1}^t
\frac{\1\{\hat a_s=i\}}{p_{s,i}(U_{s,i})}
K_{U_{s,i}}\otimes K_{U_{s,i}},
\qquad
\hat h_{i,t}
=
\frac1t\sum_{s=1}^t
\frac{\1\{\hat a_s=i\}}{p_{s,i}(U_{s,i})}
Y_sK_{U_{s,i}} .
\]

The IPW-KRR estimator is
\[
\hat f_{i,t}
=
(\hat\Sigma_{i,t}+\lambda_t I)^{-1}\hat h_{i,t},
\]
where $\lambda_t>0$ is a regularization parameter.

\paragraph*{Directional inference.}
Since $\hat f_{i,t}\in\Hk$ is infinite-dimensional, inference is conducted via finite-dimensional projections.  
Let $\mathcal G\subset\Hk$ be a class of directions (typical representers $K_u$).  
For any finite set $G=(g_1,\dots,g_m)\subset\mathcal G^m$ we study
\[
\big(
\langle g_\ell,\;t^\gamma(\hat f_{i,t}-f_i)\rangle_K
\big)_{\ell=1}^m,
\qquad \gamma\in(0,1/2),
\]
which yields confidence intervals for $f_i$ evaluated at selected covariates.

\paragraph*{Bias--variance decomposition.}
Because $\hat f_{i,t}$ is regularized, it is biased for $f_i$.  
Write
\begin{equation}\label{eq: bias_variance_est_rkhs}
\hat f_{i,t}-f_i
=
(\hat f_{i,t}-f_i^{\lambda_t})
+
(f_i^{\lambda_t}-f_i),
\end{equation}
where $f_i^{\lambda_t}=(\Sigma_i+\lambda_t I)^{-1}h_i$.  
The first term captures stochastic variability due to adaptive sampling and inverse weighting and forms the basis of the CLT, while the second is the regularization bias.  
Under standard RKHS conditions (source smoothness of $f_i$ and eigenvalue decay of $\Sigma_i$) the bias is asymptotically negligible relative to stochastic fluctuations.

\paragraph*{Assumptions for asymptotic normality.}

Fix $\lambda>0$ and define $A_i(\lambda)=\Sigma_i+\lambda I$ and $\hat A_{t,i}(\lambda)=\hat\Sigma_{i,t}+\lambda I$.  
Then $f_i^\lambda=A_i(\lambda)^{-1}h_i$ and $\hat f_{i,t}=\hat A_{t,i}(\lambda)^{-1}\hat h_{i,t}$.

\begin{enumerate}[label = (C\arabic*), series = cltAssump]

\item \label{ass:cond:noise}
\textbf{Conditional noise nondegeneracy.}
There exists $\sigma^2>0$ such that
\[
\E[(Y_s-f_i^\lambda(U_{s,i}))^2K_{U_{s,i}}\otimes K_{U_{s,i}}\mid\F_{s-1}]
\succeq
\sigma^2\E[K_{U_{s,i}}\otimes K_{U_{s,i}}\mid\F_{s-1}] .
\]

\item \label{ass:avg:explore:rt2}
\textbf{Exploration decay.}
Let
\[
\tilde r_t
=
\frac1{t^2}\sum_{s=1}^t
\E\!\left[\frac1{p_{s,i}(U_{s,i})}\right].
\]
Assume $\tilde r_t=o(1)$ and
\begin{equation*}
t^{2\gamma}\tilde r_t\to0 .
\end{equation*}

\item \label{ass:direc:nondegeneracy}
\textbf{Directional nondegeneracy.}
For any finite $G=(g_1,\dots,g_m)\subset\mathcal G^m$, define
\[
\mathsf Q_\lambda(G)
=
\big[
\langle A_i(\lambda)^{-1}g_\ell,\,
\E[K_{U_{s,i}}\otimes K_{U_{s,i}}]
A_i(\lambda)^{-1}g_m
\rangle_K
\big]_{\ell,m}.
\]
Assume $\mathsf Q_\lambda(G)$ is positive definite.

\end{enumerate}

\begin{Remark}
(i) Section~\ref{sec:beta-estimation} introduced the uniform exploration coefficient  
$
r_t
=
t^{-2}\sum_{s\le t}\frac1{p_s^\star},$ $p_s^\star=\inf_{w,i}p_{s,i}(w).
$
In contrast,
$
\tilde r_t
=
t^{-2}\sum_{s\le t}\E[1/p_{s,i}(U_{s,i})]
$
is an average exploration coefficient along the realized trajectory.  
Since $\tilde r_t\le r_t$, Assumption~\ref{ass:avg:explore:rt} implies Assumption~\ref{ass:avg:explore:rt2}.  
The condition $t^{2\gamma}\tilde r_t\to0$ ensures the stochastic error of the RKHS estimator vanishes at the rate required for the CLT; similar exploration constraints arise in adaptive semiparametric learning \citep{ChenLiao2024, Greenewald2024}.
(ii) Assumption~\ref{ass:direc:nondegeneracy} ensures the asymptotic covariance is nondegenerate in any finite set of directions.  
Equivalently, the Gram operator $\E[K\otimes K]$ must be strictly positive definite on the span of $A_i(\lambda)^{-1}G$.  
Since $\lambda>0$, $\|A_i(\lambda)^{-1}\|\le\lambda^{-1}$, guaranteeing stability of both $A_i(\lambda)^{-1}$ and $\hat A_{t,i}(\lambda)^{-1}$.  
Such directional nondegeneracy conditions are standard in functional CLTs and RKHS inference \citep{GineNickl2016, SteinwartChristmann2008}.
\end{Remark}
\subsection{Studentized CLT for Finite Linear Projection}

\paragraph*{Linearization and martingale structure.}
We now derive a studentized central limit theorem for finite-dimensional projections of the RKHS estimator $\hat f_{i,t}$. The main idea is to linearize the estimator around the population ridge target $f_i^\lambda$ and to show that the leading stochastic term admits a martingale representation in the Hilbert space $\Hk$, while the remaining term is asymptotically negligible. This strategy allows the adaptive sampling dependence to be handled using martingale central limit theory.

Let $\Delta \Sigma_{t,i} := \hat \Sigma_{i,t} - \Sigma_i$ and $\Delta h_{t,i} := \hat h_{i,t} - h_i$. Using the identity $A_i(\lambda) f_i^\lambda = h_i$ and the resolvent identity, we write:
\begin{align}
	\hat f_{i,t}-f_i^\lambda
	&= A_i(\lambda)^{-1}\big(\Delta h_{t,i}-\Delta\Sigma_{t,i} f_i^\lambda\big)
	\ +\ \big(\hat A_{t,i} (\lambda)^{-1}-A_i(\lambda)^{-1}\big)\big(\Delta h_{t,i}-\Delta\Sigma_{t,i} f_i^\lambda\big)
	\notag\\
	&=: L_{t,i}\ +\ R_{t,i},
	\label{eq:lin-decomp}
\end{align}
with $A_i(\lambda):=\Sigma_i+\lambda I$, $\hat A_{t,i} (\lambda)^{-1}-A_i(\lambda)^{-1}=-A_i(\lambda)^{-1}\Delta\Sigma_{t,i}\,\hat A_{t,i} (\lambda)^{-1}$. Furthermore, the exact expression for the first-order stochastic term is given by $L_{t,i} \;:=\; A_i(\lambda)^{-1}\big(\Delta h_{t,i} - \Delta\Sigma_{t,i}\,f_i^\lambda\big),$ for $\gamma\in(0,1/2)$. Define
	$
	\Delta_{s,i}^\lambda \;:=\; Y_s\,K_{U_{s,i}} - (K_{U_{s,i}}\otimes K_{U_{s,i}})\,f_i^\lambda \in \Hk$ and
	$\xi_{t,s,i} \;:=\; t^{\gamma-1}\,A_i(\lambda)^{-1}\,\bigg(\frac{1\{\hat a_s=i\}}{p_{s,i}(U_{s,i})}\,\Delta_{s,i}^\lambda \;-\; \lambda f_i^\lambda\bigg).$
By Lemma~\ref{lem:Lt-equals-St}, we can show that
$t^\gamma\,L_{t,i}  =  \sum_{s=1}^t \xi_{t,s,i}.$
The above decomposition isolates the principal stochastic fluctuation term from the higher-order remainder. It forms the basis for establishing the functional CLT by showing that the leading term may be treated as a sum of martingale increments in the Hilbert space $\mathcal H_\kappa$.

 The (predictable) quadratic variation 
 operator is defined as 
\begin{align}
	V_t(\lambda)
	&:= \sum_{s=1}^t \E[\xi_{t,s,i}\otimes \xi_{t,s,i}\mid \F_{s-1}]
	= t^{2\gamma-2}\,A_i(\lambda)^{-1}H_i(\lambda)A_i(\lambda)^{-1}
	\notag\\
	&= t^{2\gamma-2}\,A_i(\lambda)^{-1}\left(\sum_{s=1}^t
	\E\!\Big[\frac{1}{p_{s,i}(U_{s,i})}\,\Delta_{s,i}^\lambda\otimes\Delta_{s,i}^\lambda\ \Bigm|\ \F_{s-1}\Big]\ -\ \lambda^2\,f_i^\lambda\otimes f_i^\lambda\right)A_i(\lambda)^{-1},
	\label{eq:Vt-correct}
\end{align}
where $H_i(\lambda):=\left(\sum_{s=1}^t
	\E\!\Big[\Big(\frac{1\{\hat a_s=i\}}{p_{s,i}(U_{s,i})}\,\Delta_{s,i}^\lambda - \lambda f_i^\lambda\Big)\otimes
	\Big(\frac{1\{\hat a_s=i\}}{p_{s,i}(U_{s,i})}\,\Delta_{s,i}^\lambda - \lambda f_i^\lambda\Big)\ \Bigm|\ \F_{s-1}\Big]\right)$. The last equality uses the IPW second-moment identity from Equation~\eqref{ass:ignorability} and the fact that
$\E[\frac{1}{p_{s,i}}\Delta_{s,i}^\lambda\mid \F_{s-1}]=\lambda f_i^\lambda$ , so the cross terms cancel.
The predictable covariance operator governing the fluctuations of the RKHS estimator remains well-scaled, provided the exploration rate does not decay too quickly. This ensures that the limiting distribution is non-degenerate along all inference directions of interest (see Proposition~\ref{prop:Vt-scale-correct2}). 
This characterization of the predictable covariance operator $V_t(\lambda)$ plays a central role in studentization and the derivation of Gaussian limits.

While the stochastic term $L_{t,i}$ is a sum of martingale differences scaled by the regularized inverse operator $A_i(\lambda)^{-1}$ (Lemma~\ref{lem:Lt-equals-St}), 
Proposition~\ref{prop:dev} 
verifies that the higher-order remainder term in the Hilbert-space linearization vanishes at the scale relevant for the functional CLT. Conceptually, the latter shows that the estimation problem is asymptotically linear, meaning that the dominant source of variability in $\widehat{f}_{i,t}$ arises from first-order martingale fluctuations. This property is crucial for proving that the limiting distribution is Gaussian rather than degenerate.

We combine the linearization decomposition with the remainder bound to derive a functional CLT for directional projections of \( \hat f_{i,t} \). Specifically, from Equation~\eqref{eq:lin-decomp}, Lemma~\ref{lem:Lt-equals-St}, and 
Proposition~\ref{prop:dev}, 
 we obtain $t^\gamma(\hat f_{i,t}-f_i^\lambda)
\ =\
\sum_{s=1}^t \xi_{t,s,i}\ +\ t^\gamma R_{t,i}.$
The former expression is the exact stochastic linear term, while the latter is the remainder term.
In the studentized statistic, the additional term \(V_t(\lambda)^{-1/2}\,t^\gamma R_{t,i}\) is $o_{\Prob}(1)$ under the ridge condition \(t^\gamma\lambda^{-2}(\tilde{r}_t+t^{-1})\to 0\), so the CLT is driven by the martingale sum alone.

After studentization, one may write
\[
V_t(\lambda)^{-1/2}t^\gamma(\hat f_{i,t}-f_i^\lambda)
\ =\
\sum_{s=1}^t V_t(\lambda)^{-1/2}\xi_{t,s,i}\ +\ V_t(\lambda)^{-1/2}t^\gamma R_{t,i},
\]
where the first term is a martingale array and the second term is $o_{\Prob}(1)$.
The studentizing operator $V_t(\lambda)^{-1/2}$ is well-defined and stable on the finite-dimensional subspace generated by the evaluation directions. In particular, the predictable covariance operator governing the fluctuation process must remain nonsingular when projected onto any fixed finite collection of directions $G :=(g_1,\dots,g_m)\in\G^m$. We show in Proposition~\ref{prop:Vt-matrix-correct2} of the Supplement that, under Assumptions~\ref{ass:cond:noise}--\ref{ass:direc:nondegeneracy}, the projected covariance matrix $G^\star V_t(\lambda) G$ remains uniformly positive definite for all sufficiently large $t$. Consequently, the studentizing operator $V_t(\lambda)^{-1/2}$ is well-defined and bounded on $\operatorname{span}(G)$, ensuring that the normalized fluctuation process retains nondegenerate variability along all inference directions.

Fix $m\in\N$ and $G=(g_1,\dots,g_m)\in \mathcal{G}^m$.
Define the jointly studentized $k$-vector 
$$\small{\ \big(\,\big\langle g_\ell,\ V_t(\lambda)^{-1/2}\,t^\gamma(\hat f_{i,t}-f_i^\lambda)\big\rangle_K\big)_{\ell=1}^m
\ =\ \sum_{s=1}^t Z_{t,s}(G),}$$
where
$
\small{Z_{t,s}(G)\ :=\ \big(\,\langle g_\ell,}$ $\small{V_t(\lambda)^{-1/2}\,\xi_{t,s,i}\rangle_K\big)_{\ell=1}^m}$ $\small{ \in \real^m}.
$
Then $\{Z_{t,s}(G),\F_s\}$ is a $\real^m$-valued martingale difference array.

\begin{theorem}[Studentized vector martingale CLT in $\Hk$]\label{thm:main-correct2}
Suppose Assumptions~\ref{ass:cond:noise}--\ref{ass:direc:nondegeneracy} hold. Let $\gamma\in(0,1/2)$ and let $\lambda_t>0$ depend on $t$. If
\begin{equation}\label{eq:lambda-condition}
t^\gamma\,\lambda_t^{-2}\,\big(\tilde{r}_t+t^{-1}\big)\ \longrightarrow\ 0,
\end{equation}
then for every fixed $m\in\N$ and $G=(g_1,\dots,g_m)\in\G^m$,
\[
\big(
\langle g_\ell,V_t(\lambda_t)^{-1/2}t^\gamma(\hat f_{i,t}-f_i^{\lambda_t})\rangle_K
\big)_{\ell=1}^m
\ \cid\ \mathcal N_m(0,I_m).
\]
\end{theorem}

The theorem establishes joint asymptotic normality of the studentized finite-dimensional projections of the RKHS estimator. The identity covariance reflects that the studentization removes the dependence structure across the chosen directions.

\begin{corollary}[Scalar standardized CLT]\label{cor:scalar-correct2}
Under the assumptions of Theorem~\ref{thm:main-correct2}, for any fixed $g\in\mathcal G$,
\[
\frac{\langle g,t^\gamma(\hat f_{i,t}-f_i^\lambda)\rangle_K}
{\sqrt{\langle g,V_t(\lambda)g\rangle_K}}
\ \cid\ \mathcal N(0,1),
\]
and the denominator is strictly positive with probability tending to one by Proposition~\ref{prop:Vt-matrix-correct2}.
\end{corollary}

The requirement $t^{2\gamma}\tilde r_t\to0$ implies that the covariance of the unstudentized process $t^\gamma(\hat f_{i,t}-f_i^\lambda)$ vanishes asymptotically, so an unnormalized CLT would be degenerate. Studentization therefore plays a crucial role in obtaining a nontrivial pivotal limit. In particular, if $\lambda_t\downarrow0$ is chosen so that
$
t^\gamma\lambda_t^{-2}\big(\tilde r_t+t^{-1}\big)\to0
\qquad
(\text{equivalently }\lambda_t\gg \max\{t^{\gamma/2}\sqrt{\tilde r_t},t^{(\gamma-1)/2}\}),
$
then the studentized finite-dimensional CLT holds with identity covariance. In the next subsection we address the bias term $f_i^{\lambda_t}-f_i$ in~\eqref{eq: bias_variance_est_rkhs} and show that it vanishes under suitable source and capacity conditions on the RKHS.

\subsection{Controlling Regularization Bias}
\label{sec:bias}

We control the regularization bias appearing in the studentized statistic by bounding the directional term
$
\langle g,\, t^\gamma (f_i^{\lambda_t}-f_i)\rangle_K
$
for evaluation directions $g$ in a suitable subset $\mathcal G\subset\Hk$. The goal is to ensure that the bias vanishes relative to the CLT scale as $\lambda_t\downarrow0$ and $t\to\infty$. The analysis relies on spectral calculus for the population covariance operator $\Sigma_i$.
For arm $i$, define the population operator
$
\Sigma_i f
=\mathbb E\!\left[\frac{1\{\hat a_s=i\}}{p_{s,i}(U_{s,i})}
\langle f,K_{U_{s,i}}\rangle_K K_{U_{s,i}}\right],$ where $K_u:=K(\cdot,u).$ Then, $\Sigma_i$ is self-adjoint, positive, and trace-class, admitting an eigen-expansion
$
\Sigma_i e_j=\mu_j e_j
$
with $\mu_1\ge\mu_2\ge\cdots\to0$.

Consider the population equation
$h_i=\Sigma_i f_i,
$
and its ridge solution
$
f_i^\lambda=(\Sigma_i+\lambda I)^{-1}h_i .
$
The regularization bias admits the standard Tikhonov identity
\begin{equation}\label{eq:tikhonov-bias-identity}
f_i^\lambda-f_i
=
-\lambda(\Sigma_i+\lambda I)^{-1}f_i
=:-r_\lambda(\Sigma_i)f_i,
\qquad
r_\lambda(\mu)=\frac{\lambda}{\mu+\lambda}.
\end{equation}
Throughout we fix $\gamma\in(0,1/2)$ and allow $\lambda=\lambda_t\downarrow0$. We quantify the regularity of $f_i$ and the evaluation direction $g$ relative to the spectrum of $\Sigma_i$.
\begin{enumerate}[label = (D\arabic*), series = biasAssump]
    \item \label{ass:source} \textbf{(H\"older source condition).}
    There exist $s > 0$ and $w \in \Hk$ with $\|w\|_k \le R_s < \infty$ such that
    $
    f_i = \Sigma_i^{\,s} w.
    $
    Equivalently, in the eigenbasis of $\Sigma_i$, this means
    $
    \langle f_i, e_j \rangle_K = \mu_j^s\, \langle w, e_j \rangle_K$ for all $j \ge 1.$
    \item \label{ass:dir} \textbf{(Directional regularity).}
    Fix $q \ge 0$ and $C_g < \infty$. The evaluation direction $g$ satisfies
    $
    g = \Sigma_i^{\,q} v,$ for some $v \in \Hk$ with $\|v\|_k \le C_g.
    $
    Equivalently,
    $
    \langle g, e_j \rangle_K = \mu_j^q\, \langle v, e_j \rangle_K$ for all $j \ge 1.$
\end{enumerate}
\begin{Remark}
Assumption~\ref{ass:source} is the classical H\"older source condition from inverse-problem regularization theory, ensuring that $f_i$ aligns with the spectral decay of $\Sigma_i$ and yielding optimal rates for kernel ridge regression \citep{caponnetto2007optimal,bauer2007regularization,SteinwartChristmann2008}.  
Assumption~\ref{ass:dir} imposes a compatible smoothness condition on the evaluation direction $g$, which is standard in statistical inverse problems and determines the attainable inference rates \citep{Cavalier2011,knapik2011bayesian,nickl2013confidence}. Together these conditions ensure that the relevant directional quadratic forms are well-defined and that the CLT variance remains finite.
\end{Remark}

\begin{proposition}[Directional bias vanishes]\label{prop:bias-vanish}
Let $\gamma\in(0,1/2)$ and suppose Assumptions~\ref{ass:source}--\ref{ass:dir} hold with indices $s>0$ and $q\ge0$.  
Define $\rho:=\min\{s+q,1\}$.  
If the ridge parameter satisfies
$t^\gamma \lambda_t^{\rho}\ \longrightarrow\ 0,
$
then
$
t^\gamma\big|\langle g,f_i^{\lambda_t}-f_i\rangle_K\big|
\longrightarrow 0.
$
\end{proposition}

The proof proceeds by first bounding the bias in the norm
$
\|h\|_q:=\|\Sigma_i^{\,q}h\|_K
$
and then transferring this control to directional evaluations. The detailed argument is given in Proposition~\ref{prop:bias} of the Supplement (Section~\ref{secApp: proof_rkhs_clt}).  
The result formalizes the standard principle from statistical inverse problems that smoother targets yield faster bias decay. Under the ridge window, $t^\gamma\lambda_t^\rho \to 0$, the bias is negligible at the CLT scale, ensuring that the limiting distribution derived in the next section is centered at the true function $f_i$.

\subsection{Directional FCLT for the Unknown Link Function}

\begin{theorem}[Pointwise studentized CLT centered at $f_i$]\label{thm:pointwise-studentized-CLT}
Fix $\gamma\in(0,1/2]$ and a nonzero direction $g\in\Hk$, and define
$
D_t(g):=\sqrt{\langle g,V_t(\lambda_t)g\rangle_K}.
$
Suppose Assumptions~\ref{ass:cond:noise}--\ref{ass:direc:nondegeneracy}, \ref{ass:source}, and~\ref{ass:dir} hold, and set $\rho:=\min\{s+q,1\}$. If the ridge parameter $\lambda_t$ satisfies
$t^\gamma \lambda_t^{-2}(\tilde r_t+t^{-1})\to0 \ 
\text{and} \ 
t^\gamma \lambda_t^\rho\to0
$, equivalently, $\lambda_t\gg \max\{t^{\gamma/2}\sqrt{\tilde r_t},\,t^{(\gamma-1)/2}\}$ and $\lambda_t\ll t^{-\gamma/\rho}$, then
\[
T_{t,i}(g):=
\frac{\langle g,\,t^\gamma(\hat f_{i,t}-f_i)\rangle_K}{D_t(g)}
\ \cid\ \mathcal N(0,1).
\]
\end{theorem}

This theorem yields asymptotically valid directional inference for $f_i$. In particular, after studentization, the fluctuation of the estimated reward function along a fixed RKHS direction behaves asymptotically as a standard Gaussian, despite adaptive sampling.

\begin{Remark}[Feasible choices of $\lambda_t$]
For the studentized CLT of Section~5.1, the stochastic remainder requires
\begin{equation}\label{eq:stoch-condition}
t^\gamma\lambda_t^{-2}(\tilde r_t+t^{-1})\xrightarrow[t\to\infty]{}0,
\end{equation}
where $\tilde r_t=t^{-2}\sum_{s=1}^t \E[1/p_{s,i}(U_{s,i})]$ and, by Assumption~\ref{ass:avg:explore:rt2}, $t^{2\gamma}\tilde r_t\to0$. Combining the ridge window: $t^\gamma \lambda_t^\rho$ with \eqref{eq:stoch-condition}, an admissible sequence $\lambda_t\downarrow0$ must satisfy
\[
\max\!\big\{t^{\gamma/2}\sqrt{\tilde r_t},\,t^{(\gamma-1)/2}\big\}\ll \lambda_t \ll t^{-\gamma/\rho}.
\]

If $\tilde r_t\asymp t^{-\xi}$ for some $\xi>0$, then $t^{\gamma/2}\sqrt{\tilde r_t}\asymp t^{(\gamma-\xi)/2}$, so the feasible window becomes
$
t^{(\gamma-\xi)/2}\ll \lambda_t \ll t^{-\gamma/\rho}.
$
Hence a power schedule $\lambda_t=t^{-\zeta}$ is admissible whenever
\[
\frac{\gamma}{\rho}<\zeta<\min\!\Big\{\frac{\xi-\gamma}{2},\,\frac{1-\gamma}{2}\Big\}.
\]
Since $t^{2\gamma}\tilde r_t\to0$ implies $\xi>2\gamma$, the right-hand side is strictly positive, so feasible choices of $\zeta$ exist provided $\gamma/\rho$ is sufficiently small.

A convenient canonical choice is $\lambda_t\asymp t^{-(1-\gamma)/3}$. This lies below the upper stochastic bound $(1-\gamma)/2$ and therefore satisfies the remainder condition. If, in addition, $\gamma/\rho<(1-\gamma)/3$, then it also satisfies the bias condition $\lambda_t\ll t^{-\gamma/\rho}$. Thus $\lambda_t\asymp t^{-(1-\gamma)/3}$ lies within the admissible window and ensures that both the bias and the stochastic remainder are negligible at the CLT scale.

Under near-uniform exploration, $\E[1/p_{s,i}(U_{s,i})]=O(1)$ and hence $\tilde r_t\asymp t^{-1}$, that is, $\xi=1$. In this case the admissible range reduces to
$
\frac{\gamma}{\rho}<\zeta<\frac{1-\gamma}{2},
$
so the above choice of $\lambda_t$ is feasible whenever $\gamma/\rho<(1-\gamma)/3$.
\end{Remark}

\begin{corollary}[Pointwise studentized CLT and confidence intervals]\label{cor:pointwise-x}
Fix $\gamma\in(0,1/2]$ and $x\in\U$, and let $g:=K_x:=K(\cdot,x)\in\Hk$, so that $q=0$. Assume the standing conditions of the stochastic RKHS CLT, directional nonsingularity along $K_x$, and a source condition $f_i=\Gamma_i^{\,s}w$ with $s>0$ and $w\in\Hk$. Set $\rho:=\min\{s,1\}$.

Let $\lambda_t>0$ satisfy
\begin{equation}\label{eq:lambda-window-point}
t^\gamma \lambda_t^{-2}(\tilde r_t+t^{-1})\xrightarrow[t\to\infty]{}0,
\qquad
t^\gamma \lambda_t^\rho\xrightarrow[t\to\infty]{}0,
\end{equation}
equivalently, $\lambda_t\gg \max\{t^{\gamma/2}\sqrt{\tilde r_t},\,t^{(\gamma-1)/2}\}$ and $\lambda_t\ll t^{-\gamma/\rho}$. 
Let $V_t(\lambda_t)$ be the predictable covariance operator \eqref{eq:Vt-correct} from the stochastic CLT. Then
$
T_{t,i}(x):=
\frac{t^\gamma(\hat f_{i,t}(x)-f_i(x))}{D_t(x)},$ where $D_t(x)^2:=\langle K_x,V_t(\lambda_t)K_x\rangle_K,$
satisfies
\[
T_{t,i}(x)\ \cid\ \mathcal N(0,1).
\]
\end{corollary}

We next establish consistency of the feasible pointwise variance estimator.

\begin{proposition}
\label{prop:variance-consistency}
Fix $x\in\mathcal U$ and let $K_x:=K(\cdot,x)\in\Hk$. Under the standing assumptions, and additionally under the following conditions:
\begin{enumerate}[label = (E\arabic*), series = ptCLTassump]
    \item \label{ass:pt-nonsingular}
    \textbf{Directional nonsingularity along $K_x$.}
    There exist constants $0<c_x'\le C_x'<\infty$ such that, for all sufficiently small $\lambda>0$,
    $
    c_x'
    \le
    \langle K_x,(\Gamma_i+\lambda I)^{-1}\Gamma_i(\Gamma_i+\lambda I)^{-1}K_x\rangle_K
    \le
    C_x'.
    $
    \item \label{ass:pt-ridge-window}
    \textbf{Ridge window.}
    The ridge parameter satisfies $\lambda_t\downarrow0$ with
    $
    \lambda_t^{-2}\left(\sqrt{\tilde r_t}+\frac{1}{\sqrt t}\right)\longrightarrow0.
    $
\end{enumerate}
If $\tilde r_t=t^{-2}\sum_{s=1}^t\E[1/p_{s,i}(U_{s,i})]\downarrow0$ and $t^{2\gamma}\tilde r_t\to0$ for $\gamma\in(0,1/2)$, then
$
\frac{\widehat D_t(x)^2}{D_t(x)^2}\xrightarrow{\Prob}1.
$
\end{proposition}

\paragraph*{Length of the pointwise confidence interval}
By Corollary~\ref{cor:pointwise-x} and Proposition~\ref{prop:variance-consistency} together with Slutsky's lemma,
$\frac{t^\gamma(\hat f_{i,t}(x)-f_i(x))}{\widehat D_t(x)} \xrightarrow{D} \mathcal N(0,1).$
Therefore, an asymptotic $(1-\alpha)$ confidence interval for $f_i(x)$ is
\begin{equation}\label{eq: final_CI_nonp}
\mathrm{CI}_{1-\alpha}(x)
=
\Big[
\hat f_{i,t}(x)-z_{1-\alpha/2}\,\widehat D_t(x)t^{-\gamma},
\;
\hat f_{i,t}(x)+z_{1-\alpha/2}\,\widehat D_t(x)t^{-\gamma}
\Big],
\end{equation}
where $z_{1-\alpha/2}$ denotes the $(1-\alpha/2)$ quantile of the standard normal law. 
Under the admissible ridge window:
$\max\!\big\{t^{\gamma/2}\sqrt{\tilde r_t},\,t^{(\gamma-1)/2}\big\}
\ll
\lambda_t
\ll
t^{-\gamma/\rho},\  \rho=\min\{s,1\},
$
the pointwise CLT remains valid, and therefore the interval in~\eqref{eq: final_CI_nonp} is asymptotically valid.

\subsection{Construction of Pointwise Confidence Intervals via RKHS CLT}
\label{sec: pointwise_nonp_CLT}

We construct a confidence interval for the conditional reward function
$f_i(X_{t+1}^{\top}\beta_i)$ evaluated along the estimated index direction
$\hat{\beta}_{i}^{(t)}$. The construction combines the RKHS CLT with a plug-in
estimator of the asymptotic covariance operator. Let
$U_{s,i}=X_s^\top\hat\beta_i^{(t)}$ denote the projected covariates. The estimated covariance operator is
\begin{equation}\label{eq:Vhat-np}
\widehat{V}_t^{(i)}(\lambda_t)
=
t^{2\gamma-2}
A_{t,i}^{-1}
\Bigg(
\sum_{s=T_0+1}^{t}
w_s\,
\big(Y_s-\hat f_{i,t}(U_{s,i})\big)^2\,
K(\cdot,U_{s,i})\otimes K(\cdot,U_{s,i})
\Bigg)
A_{t,i}^{-1},
\end{equation}
where
$A_{t,i}=\hat\Sigma_{i,t}+\lambda_t I.$
For any evaluation point $x\in\mathcal X$, the corresponding pointwise variance estimator is
$
\widehat D_t^2(x)
=
\big\langle K(\cdot,x),\widehat V_t^{(i)}(\lambda_t)K(\cdot,x)\big\rangle_{\Hk}, \ 
\widehat D_t(x)=\sqrt{\widehat D_t^2(x)}.
$
This quantity is the plug-in standard error entering the pointwise confidence interval in~\eqref{eq: final_CI_nonp}. Pseudocode is given in Algorithm~\ref{alg:np_directional_CI}.

\begin{algorithm}[t]
\caption{Nonparametric directional CI for arm $i$ at time $t+1$}
\label{alg:np_directional_CI}
\begin{algorithmic}[1]

\Require $X_{1:t+1}$, $\hat\beta_{i,t}$, kernel $K$, covariance estimator $\widehat V_t^{(i)}$, level $\alpha$, rate $\gamma$.

\State Compute projections $U_s = X_s^\top \hat\beta_{i,t}$ for $s=1,\dots,t+1$.

\State Form the kernel vector $K_{t+1} = (K(U_1,U_{t+1}),\dots,K(U_t,U_{t+1}))^\top$.

\State Compute $\widehat D_t^2(U_{t+1}) = K_{t+1}^\top \widehat V_t^{(i)} K_{t+1}$.

\State Compute $z_{1-\alpha/2} = \Phi^{-1}(1-\alpha/2)$.

\State Return $\mathrm{CI}_{t+1}^{(i)} =
\widehat f_i(U_{t+1}) \pm
z_{1-\alpha/2} \, t^{-\gamma} \sqrt{\widehat D_t^2(U_{t+1})}$.

\end{algorithmic}
\end{algorithm}

\begin{Remark}
An implementable form of \( \widehat V_t^{(i)}(\lambda_t) \) is obtained through the kernel trick. Define the projected covariates
$U_{s,i}=X_s^\top\widehat\beta_{i,t}$ for $s=1,\dots,t$, the Gram matrix
$K_t=[K(U_{r,i},U_{s,i})]_{r,s=1}^t\in\mathbb R^{t\times t}$, the reward vector
$\mathbf y_t=(Y_1,\dots,Y_t)^\top$, and the inverse-propensity weights
$\mathbf w^{(i)}=(w_{i,1},\dots,w_{i,t})^\top$, where
$w_{i,s}=\frac{\mathbf 1\{A_s=i\}}{p_{s,i}(X_s)}$. Let
$W_i=\operatorname{diag}(w_{i,1},\dots,w_{i,t})$.
The fitted responses of the kernel ridge estimator are
$
\widehat{\mathbf y}_{i,t}
=
K_t(K_tW_i+t\lambda_t I)^{-1}W_i\mathbf y_t.
$
Let $\mathbf r^{(i)}=\mathbf y_t-\widehat{\mathbf y}_{i,t}$ and define
$R_i=\operatorname{diag}(w_{i,s}(r_s^{(i)})^2)_{s=1}^t$.
Then
$
\widehat V_t^{(i)}(\lambda_t)
=
t^{2\gamma-2}
A_{t,i}^{-1}
\left(
\frac1t K_tR_iK_t
\right)
A_{t,i}^{-1},
$ and 
$A_{t,i}=\frac1tK_tW_iK_t+\lambda_t I$.
Thus, the covariance estimator depends only on kernel evaluations and can be computed entirely in dual form. In practice, small ridge or spectral truncation adjustments may be used for numerical stability.
\end{Remark}

\subsection{Comparison with Uniform RKHS Bounds}
\label{subsec:comparison-AS}

\cite{arya:23} derive uniform high-probability bounds for the RKHS estimation
error $\|\hat f_{i,t}-f_i\|_{\Hk}$ under adaptive sampling.
These results imply simultaneous confidence bands over compact domains.
In contrast, our approach focuses on local inference through a
studentized functional CLT, which yields pointwise confidence intervals
for $f_i(x)$ along the projected covariates.

This raises two natural questions:  
(i) whether the regularization schedules proposed in \cite{arya:23}
are compatible with the bias--variance conditions required by our
CLT-based inference, and  
(ii) how the resulting pointwise confidence intervals compare when both
methods are applicable.
The detailed analysis is provided in Section~\ref{subsec:App:AS_comparison} of the Supplement
\citep{suppl:arya:26}. The supplement establishes two main conclusions.

(i) \textbf{Compatibility of regularization schedules.}
Under the exploration conditions of the bandit policy,
there exists a nonempty range of ridge parameters $\lambda_{i,t}$
that simultaneously satisfies the regularization constraints required
by \cite{arya:23} and the two-sided bias--variance window required by
our RKHS CLT. Consequently, the two approaches can be implemented under
a common regularization schedule.

\noindent
(ii) \textbf{Asymptotic comparison of pointwise confidence intervals.}
When both methods are feasible, the pointwise confidence intervals
derived from our studentized CLT shrink at rate
$\Theta_{\Prob}(\sqrt{\tilde r_t})$, whereas the intervals obtained from
the uniform RKHS bounds of \cite{arya:23} shrink at the slower rate
$\Theta(\tilde r_t^{\,\theta_i})$ for some $\theta_i\in(0,1/2)$.
Since $\tilde r_t\to0$, this implies that the CLT-based intervals are
asymptotically shorter.

These results show that the uniform bounds of \cite{arya:23} remain
compatible with our asymptotic framework, while the CLT-based approach
provides sharper local inference when the goal is pointwise uncertainty
quantification. The precise compatibility conditions and rate
comparisons are derived in Theorem~\ref{thm:compat-AS-compact} of the Supplement.

\section{Finite-time Regret Analysis}\label{sec: regret_analysis}
We now analyze the regret performance of our K-SIEGE algorithm, which employs an $\varepsilon$-greedy strategy to balance exploration and exploitation. At each time step, the algorithm estimates both the index direction $\beta_i$ and the reward function $f_i$ for each arm $i$, using these estimates to guide informed decisions. These evolving estimators introduce intricate dependencies across rounds. To handle this, we decompose the cumulative regret into three components: (i) function estimation error due to nonparametric kernel regression; (ii) index estimation error arising from the single-index parameter estimation; and (iii) arm selection error caused by the $\varepsilon$-greedy randomization. This decomposition enables us to separately analyze each source of error and derive a finite-time regret bound that reflects both the statistical complexity of the RKHS function class and the difficulty of estimating the unknown index directions under adaptive sampling. Under suitable assumptions on identifiability, regularity, and exploration, we establish a finite-time upper bound on the cumulative regret that captures both the complexity of the RKHS function class and the cost of estimating the unknown index direction.
We make the following assumptions throughout this section whenever we are working under the assumption that the mean reward functions lie in an RKHS $\mathcal{H}_K$.

\begin{enumerate}[label = (R\arabic*), series = regretAssump]
    \item \label{ass:eval-decay} \textit{(Eigenvalue decay).}
    For all $i = 1,\ldots,L$, let $\Sigma_i := \mathbb{E}\left[K(\cdot, U_{s,i}) \otimes K(\cdot, U_{s,i})\right]$  denote the covariance operator associated with the single-index projection $U_{s,i} = X_s^\top \beta_i$. There exists $\alpha > 1$ and $\bar{C} \in (0,\infty)$ such that
    $
    \sup_{\beta_i} \eta_l(\Sigma_i)\ \le\ \bar{C} l^{-\alpha},
    $
    where $\eta_l(\Sigma_i)$ denotes the $l^\text{th}$ eigenvalue of $\Sigma_i$. 

    \item \label{ass:smoothness} \textit{(Source condition / function smoothness).}
    For each arm $i = 1,\ldots,L$, the true reward function $f_i$ lies in the range of a fractional power of the kernel operator, i.e.,
    $
    f_i \in \mathrm{Ran}(\Sigma_i^{\gamma_i}), \text{for some } \gamma_i \in \left(0, \tfrac{1}{2}\right].
    $
    That is, there exists $h \in \mathcal{H}_K$  such that $f_i = \Sigma_i^{\gamma_i} h$.

    \item \label{ass:kernel-lipschitz} \textit{(Kernel Lipschitz continuity).} 
    The kernel $K(\cdot, \cdot)$ is Lipschitz in its second argument: there exists $L_K > 0$ such that for all $u, v \in \mathbb{R}$,
    \[
    \|K(\cdot, u) - K(\cdot, v)\|_{\mathcal{H}_K} \le L_K \|u - v\|.
    \]
\end{enumerate}
Note that Assumption~\ref{ass:eval-decay} imposes a polynomial decay on the eigenvalues of the kernel covariance operator $\Sigma_i$, which in turn controls the complexity of the hypothesis space via the effective dimension $N_{\Sigma_i,1}(\lambda) := \tr(\Slinv \Sigma_i) \lesssim \lambda^{-1/\alpha}$. This assumption reflects the richness or smoothness of the function class induced by the kernel. Assumption~\ref{ass:smoothness} is a standard source condition that quantifies the regularity of the true mean reward function $f_i$ in terms of its membership in the range space of a fractional power of $\Sigma_i$. Together, these two assumptions govern the achievable bias-variance tradeoff in regularized estimation. Finally, Assumption~\ref{ass:kernel-lipschitz} ensures that the kernel map is sufficiently regular with respect to perturbations in its input, which is crucial for transferring estimation error from the index space to the RKHS norm.
These assumptions are satisfied by many commonly used kernels. For instance, Assumptions~\ref{ass:eval-decay} and~\ref{ass:smoothness} hold for Gaussian RBF and Mat\'ern kernels with $\alpha > 1$ and suitable $\gamma_i \le 1/2$, as the eigenvalues of the associated integral operators decay rapidly.  Assumption~\ref{ass:kernel-lipschitz} is also satisfied by all these kernels due to their Lipschitz continuity in the second argument.

\subsection{Regret Decomposition} \label{subsec:regret_decomp}

We analyze the cumulative regret $R_T(\pi)$ defined in \eqref{eq:regret} for the K-SIEGE policy $\pi$.  
The regret decomposes into three sources: function estimation, index estimation, and arm-selection mismatch.  
The derivation, based on RKHS inner-product expansions and the Lipschitz property of the kernel, is given in Section~\ref{secApp: proofs_regret} of the Supplement.

\begin{proposition}[Regret Decomposition]
\label{prop:regret-decomposition}
Let $R_T(\pi)$ denote the cumulative regret over $T$ rounds.  
Under the kernel regression model with index parameters $\beta_i\in\mathbb{R}^d$ and RKHS reward functions $f_i\in\Hk$, the regret satisfies
\begin{align}
R_T(\pi)
&\le
2\kappa
\sum_{t=1}^T
\sup_{i\in[L]}\|f_i-\fhat_{i,t}\|_{\Hk}
+2L_K \sup_{i\in[L]}\|f_i\|_{\Hk}
\sum_{t=1}^T \|X_t\|_2 \|\betahat_{i,t}-\beta_i\|_2
\nonumber\\
&\quad
+\kappa
\sum_{t=1}^T I\{a_t\neq A_t\}
\Big(
\|f_{A_t}-f_{a_t}\|_{\Hk}
+
C_\beta L_K \sup_{i\in[L]}\|f_i\|_{\Hk}\|X_t\|_2
\Big).
\label{eq:regret_breakdown}
\end{align}
\end{proposition}

Term~I and Term~III follow from the nonparametric analysis of \cite{arya:23}, while Term~II reflects propagation of index estimation error controlled in Section~\ref{sec:index_concentration_summary}.  
Combining these bounds with \eqref{eq:beta:finite:sample2} yields the following high-probability regret guarantee.
Define the shorthand quantities
$
E_T := \sum_{t=1}^T \frac{\epsilon_t}{L-1},$ and $
r_t := \frac{1}{t^2}\sum_{s=1}^t \frac{1}{\epsilon_s}.
$

\begin{theorem}[High-Probability Regret Bound]
\label{thm:main-regret-bound}
Suppose Assumptions~\ref{ass:eval-decay}, \ref{ass:smoothness}, and \ref{ass:kernel-lipschitz} hold, the ridge parameter satisfies \eqref{eq:INT-compact}, the contexts $\{X_t\}$ are i.i.d.\ sub-Gaussian with parameter $\sigma$, and $\{\beta_a\}$ are uniformly bounded.  
Then with probability at least $1-3\delta$,
\begin{align}\label{eq:final_regret_bound_three_terms}
R_T(\pi)
&\le
\tilde C_1 L\sigma\sqrt{d+\log(1/\delta)}
\sqrt{\tfrac1\delta}
\sum_{t=1}^T(\sqrt{r_t}\vee t^{-1/2})
\nonumber\\
&\quad
+\kappa\Theta
\sum_{t=1}^T
\Big[
I\{\tfrac{Lr_t}{\delta}<1\}
\big(\tfrac{Lr_t}{\delta}\big)^{\frac{(\min_i\gamma_i)\alpha}{2(\min_i\gamma_i)\alpha+\alpha+1}}
+
I\{\tfrac{Lr_t}{\delta}\ge1\}
\big(\tfrac{Lr_t}{\delta}\big)^{\frac{(\max_i\gamma_i)\alpha}{2(\max_i\gamma_i)\alpha+\alpha+1}}
\Big]
\nonumber\\
&\quad
+L_K C_\beta\sigma\sqrt{d+\log(1/\delta)}
\sup_{a,a'}\|f_a-f_{a'}\|_{\Hk}\,
T^{1/q}(E_T+\sqrt{E_T/\delta})^{1/p}.
\end{align}
Here $\Theta:=\sup_a\|f_a\|_{\Hk}$, $L_K$ is the kernel Lipschitz constant, and $(p,q)$ are H\"older conjugates.
\end{theorem}

The proof is given in Section~\ref{secApp: proofs_regret} of the Supplement.
Since
\[
\frac{\gamma_i\alpha}{2\gamma_i\alpha+\alpha+1}\le\frac12,
\qquad
0<\gamma_i\le\tfrac12,\ \alpha>1,
\]
the nonparametric regression error dominates the index estimation contribution.  
Thus the bound in \eqref{eq:final_regret_bound_three_terms} simplifies to the same order obtained in \cite{arya:23}.  
In particular, with probability at least $1-3\delta$,
\begin{align}
R_T(\pi)
&\lesssim
2\kappa\Theta
\sum_{t=1}^T
\Big[
I\{\tfrac{Lr_t}{\delta}<1\}
\big(\tfrac{Lr_t}{\delta}\big)^{\frac{(\min_i\gamma_i)\alpha}{2(\min_i\gamma_i)\alpha+\alpha+1}}
+
I\{\tfrac{Lr_t}{\delta}\ge1\}
\big(\tfrac{Lr_t}{\delta}\big)^{\frac{(\max_i\gamma_i)\alpha}{2(\max_i\gamma_i)\alpha+\alpha+1}}
\Big]
\nonumber\\
&\quad
+\tilde C_1 \sup_{a,a'}\|f_a-f_{a'}\|_{\Hk}\,
T^{1/q}(E_T+\sqrt{E_T/\delta})^{1/p}.
\label{eq:final_regret_bound}
\end{align}

The regret rate is therefore governed by the smoothness parameters $\{\gamma_i\}$ and the RKHS complexity parameter $\alpha$, matching the bound of \cite{arya:23}.  
Importantly, K-SIEGE performs nonparametric regression on the one-dimensional projections $X_t^\top\beta_i$, thereby avoiding the curse of dimensionality associated with kernel methods in $\mathbb{R}^d$.

\begin{Remark}
Let $\epsilon_t=t^{-\beta}$ for $0<\beta<1$. Then $\sum_{s=1}^t\epsilon_s^{-1}\lesssim t^{\beta+1}=o(t^2)$ and \eqref{eq:final_regret_bound} implies
\[
R_T \lesssim T^{(\beta-1)w+1}+T^{1-\beta/p},
\qquad
w=\frac{(\min_{i\in\mathcal A}\gamma_i)\alpha}{2(\min_{i\in\mathcal A}\gamma_i)\alpha+\alpha+1}.
\]
Balancing the two terms with $p=1$ gives the optimal exploration rate
$\beta^*=w/(w+1)$.  
This yields regret of order $T^{4/5}$ in the worst case and $R_T\lesssim T^{2/3}$ in the finite-dimensional RKHS setting.
\end{Remark}
In many practical applications, variability in rewards across arms may arise solely from differences in their parametric components, with all arms sharing a common nonlinear transformation of the index. Motivated by this structure, we consider a special case of our model where the nonparametric reward function is the same across all arms, i.e., $f := f_a$ for all $a \in \mathcal{A}$, and is assumed to be Lipschitz continuous.
\begin{theorem}[Regret under a common Lipschitz reward function]
\label{thm:special_case_lipschitz}
Suppose each arm $a \in \mathcal{A}$ shares a common reward function $f_a \equiv f$, where $f:\mathbb{R}\to\mathbb{R}$ is Lipschitz with $|f'(u)| \le C_{f'}$. Assume the covariates satisfy $\|X_t\|_2 \le C_x(\sigma\sqrt{d + \log(1/\delta)})$ with high probability, and that the arm parameters satisfy $\|\beta_a - \beta_{a'}\| \le C_\beta$ for all $a,a'\in\mathcal{A}$. Then, with probability at least $1-\delta$, the cumulative regret satisfies
\[
R_T(\pi)
\le 2 \tilde{C}_3 L (\sigma \sqrt{d + \log(1/\delta)})
\left[
\sqrt{\frac{1}{\delta}} \sum_{t=1}^T (\sqrt{r_t} \vee t^{-1/2})
+ T^{1/q}\!\left( E_T
+ \sqrt{\frac{1}{\delta} E_T } \right)^{1/p}
\right],
\]
for any conjugate H\"older pair $1/p+1/q=1$, where
$
r_t = \frac{1}{t^2}\sum_{s=1}^t \epsilon_s^{-1}, 
E_T = \sum_{t=1}^T \frac{\epsilon_t}{L-1}.
$
\end{theorem}
In particular, if $\epsilon_t = O(t^{-1/2})$, then $R_T(\pi)=O_{\mathbb{P}}(\sqrt{T})$, matching the minimax-optimal parametric rate. 

\section{Simulation study}
\label{sec:simulations}

We now present a simulation study to (i) verify the finite-sample behavior of the
directional central limit theorem (CLT) and associated confidence sets for the
index parameter $\beta_i$, (ii) compare our nonparametric pointwise confidence
intervals for $f_i$ with those derived from the uniform-in-time bands of
\cite{arya:23}, and (iii) illustrate the regret performance of
the kernel single-index $\varepsilon$-greedy (K-SIEGE) algorithm.

\subsection{Experimental setup}

We consider a two-armed single-index bandit with covariates
$X_t \in \mathbb{R}^d$ and rewards
$
Y_t = f_{a_t}(X_t^\top \beta_{a_t}) + \varepsilon_t,
$
where $\varepsilon_t \sim N(0,\sigma^2)$ and $X_t \sim N(0,I_d)$.
We examine dimensions $d \in \{2,5\}$ and noise levels
$\sigma \in \{0.05,0.10,0.20\}$.

For each $(d,\sigma)$ scenario and each arm $i \in \{1,2\}$ we sample a
``canonical'' index vector $\beta_i \in \mathbb{R}^d$ by drawing
$b \sim \mathsf{N}(0,I_d)$, forcing the first coordinate to be positive
($b_1 \leftarrow |b_1|$), and normalizing to unit norm
$\beta_i = b / \|b\|_2$.  The same pair $(\beta_1,\beta_2)$
is used for all Monte Carlo replications within a given $(d,\sigma)$
scenario.

The arm-specific mean reward functions are chosen to be 
non-linear, for $z \in \mathbb{R}$:
\begin{align*}
    g_1(z) &= \mu_1 + a \tanh(k z),
    &
    g_2(z) &= \mu_2 - a \tanh(k z),
\end{align*}
with $(\mu_1,\mu_2,a,k) = (0.6,\,0.4,\,0.4,\,1.0)$.  We work with the
single-index structure of \eqref{eq: MABCmodel}: for arm $i$ and context
$x \in \mathbb{R}^d$ the scalar index is
$
    z_i(x) = x^\top \beta_i,
$
and the corresponding mean reward is
$
    f_i(x) = g_i\bigl(z_i(x)\bigr)
           = g_i\bigl(x^\top \beta_i\bigr), i = 1,2.
$
plotted in Figure~\ref{fig:link_functions} of the Supplement.

The K-SIEGE algorithm is run for $T=1000$ rounds with an initial
forced-exploration phase. Kernel ridge regression is applied to the
projected indices using a Gaussian kernel. Additional implementation
details are given in Section \ref{sec:simulation_details} of the Supplement.

For each $(d,\sigma)$ scenario we run $N = 100$ independent replications of the
entire trajectory.  Within each replication we store the full paths of covariates,
rewards, arm selections, propensity scores, and the raw and normalized
estimators $\hat\beta_{i,t}$ for both arms.  We perform parametric and
nonparametric inference at a grid of inference times
$
    \mathcal{T}_{\mathrm{inf}}
    = \{200, 314, 428, 542, 657, 771, 885, 999\}.
$

\paragraph*{Directional parametric inference for $\beta_i$}
We first examine directional inference for the index parameter.
At each inference time $t\in\mathcal{T}_{\mathrm{inf}}$ we construct
95\% confidence sets for the normalized index direction
$\theta_i^\star = \beta_i/\|\beta_i\|$ using the
procedure of of Algorithm~\ref{alg:parametric_inference} in  Section~\ref{sec:asymptotic_inference_beta}.
Empirical coverage is estimated across Monte Carlo replications.


Figure~\ref{fig:results_parametric_coverage} reports empirical joint
coverage across inference times for all $(d,\sigma)$ scenarios.
For $d=2$ and low to moderate noise ($\sigma \in \{0.05,0.10\}$),
coverage is close to the nominal $95\%$ level and improves with $t$,
remaining between $0.84$ and $1.00$. With higher noise
($\sigma=0.20$), coverage decreases at early times but remains in the
range $0.61$--$0.83$.

For $d=5$, the problem becomes substantially harder, with noticeable
undercoverage at early times ($0.18$--$0.45$ for $t\approx200$--$428$),
which improves to about $0.56$--$0.80$ by $t=999$. These trends are
consistent with theory: higher dimension and noise slow the
convergence of the index and covariance estimators, requiring larger
sample sizes to approach nominal coverage.


\begin{figure}[h!]
    \centering
    \includegraphics[width=0.3\linewidth]{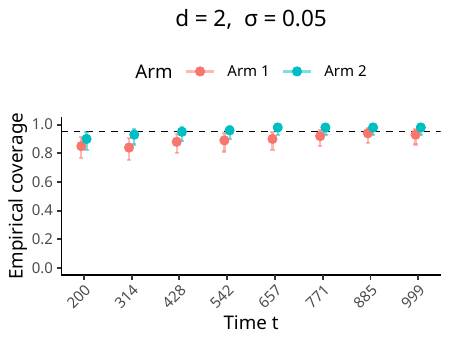}
    \includegraphics[width=0.3\linewidth]{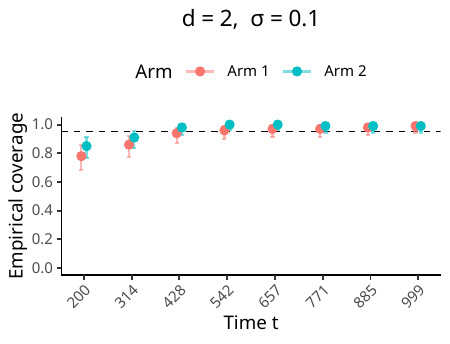}
    \includegraphics[width=0.3\linewidth]{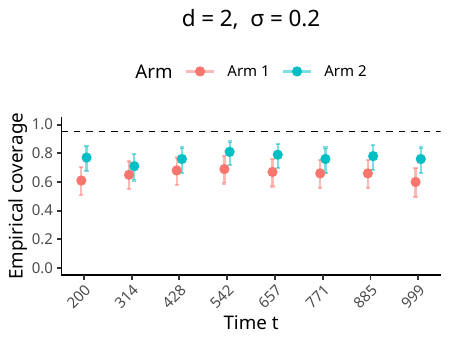}\\
    \includegraphics[width=0.3\linewidth]{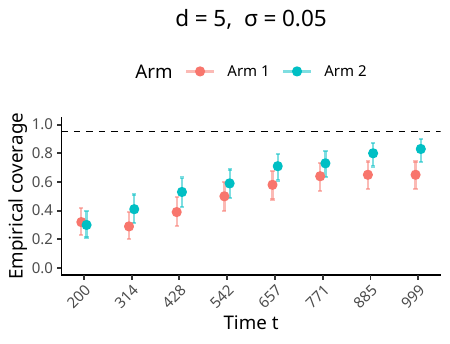}
    \includegraphics[width=0.3\linewidth]{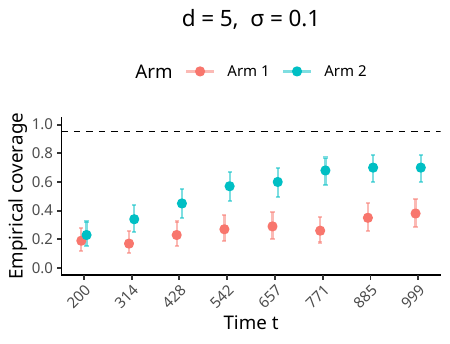}
    \includegraphics[width=0.3\linewidth]{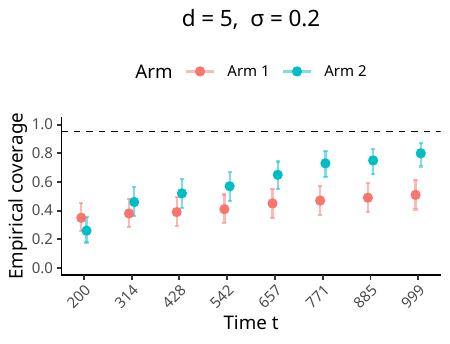}
    \caption{Empirical joint $95\%$ coverage for the single-index direction for both arms.}
    \label{fig:results_parametric_coverage}
\end{figure}

\paragraph*{Nonparametric pointwise inference for $f_i$}
We evaluate pointwise confidence intervals for the arm-specific reward
functions $f_i$. At each inference time $t\in\mathcal{T}_{\mathrm{inf}}$
and arm $i$, we construct two
intervals for $f_i(x_{t+1}^\top \widehat\beta_{i,t})$:
(i) our CLT-based interval from Algorithm~\ref{alg:np_directional_CI}
using the RKHS covariance estimator $\widehat V_t(\lambda_t)$ in
\eqref{eq:Vhat-np}, and
(ii) a pointwise interval obtained from the uniform-in-time bands of
\cite{arya:23} (A\&S) evaluated at the same index point.
Both methods use the same kernel ridge estimator with ridge schedule
$\lambda_t=t^{-\zeta}$ ($\zeta=0.05$) and a Gaussian RBF kernel (details in Supplement).

\begin{figure}[h!]
    \centering
    \includegraphics[width=0.45\linewidth]{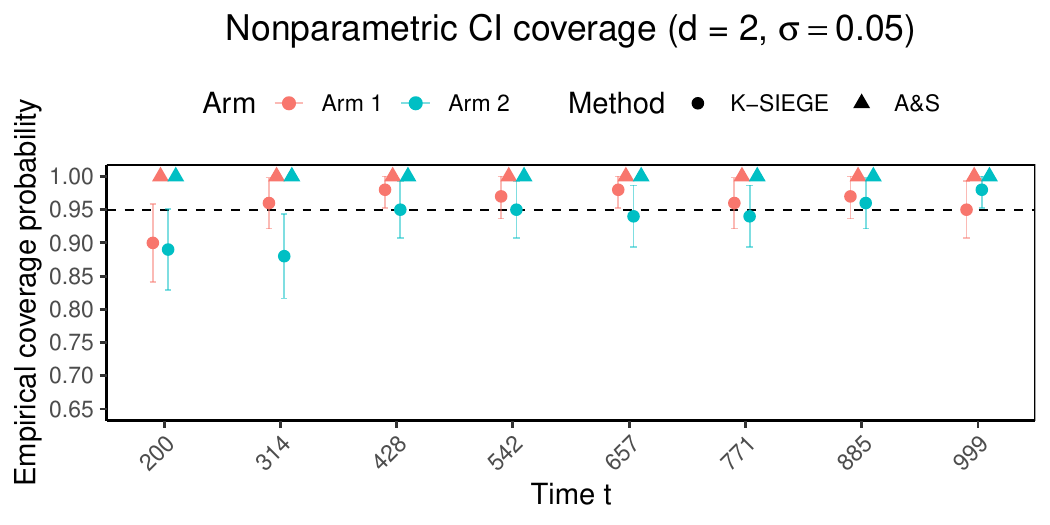}
    \includegraphics[width=0.45\linewidth]{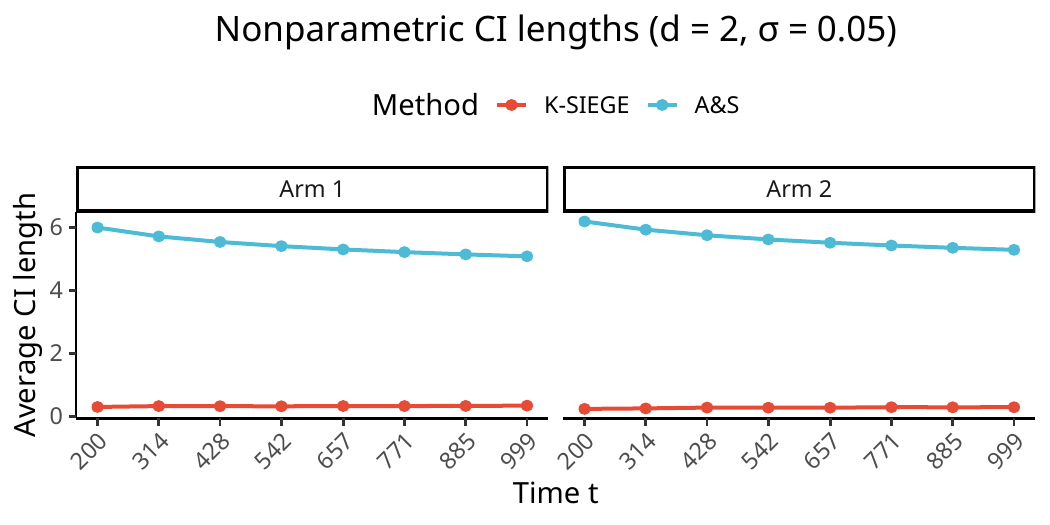}
    \caption{Nonparametric pointwise inference for $(d,\sigma)=(2,0.05)$.
    Left: empirical $95\%$ coverage of K-SIEGE and A\&S intervals at the
    selected inference times. Right: corresponding average interval lengths
    for each arm.}
    \label{fig:np_coverage}
\end{figure}

Figure~\ref{fig:np_coverage} (left) shows empirical $95\%$ pointwise
coverage for $(d,\sigma)=(2,0.05)$ across $N=100$ replications.
The K-SIEGE intervals attain coverage close to the nominal level for
both arms at all inference times. In contrast, the A\&S intervals are
uniformly conservative, with empirical coverage near $1.0$, reflecting
their calibration to a high-probability uniform band.
The right panel reports average interval lengths: the A\&S intervals are
approximately three to five times longer than the K-SIEGE intervals and
shrink more slowly with $t$. Similar behavior is observed across other $(d,\sigma)$ configurations
(see Section~\ref{sec:simulation_details}). Additional trajectory-level
illustrations and regret results are reported in the Appendix. In all
scenarios, the cumulative regret grows sublinearly, consistent with the
theoretical results of Section~\ref{sec: regret_analysis}.

\section{Real data experiment: Rice classification}
\label{sec:realdata_rice}

We illustrate the proposed inference procedures on the Rice Classification
dataset \citep{riceCammeo} from the UCI Machine Learning Repository, which
contains morphological features extracted from images of rice grains from
two varieties. Each observation consists of $d=7$ geometric covariates
(e.g., area, perimeter, axis lengths, eccentricity) and a binary class label.

We cast the problem as a two-armed contextual bandit with covariates
$X_t\in\mathbb{R}^7$ and rewards $Y_t\in\{0,1\}$, where arm
$i\in\{1,2\}$ corresponds to predicting variety $i$ and a reward of one
is obtained if the prediction is correct.  The conditional mean reward
is modeled using the single-index structure
$
\mathbb{E}[Y_t\mid X_t,A_t=i]=f_i(X_t^\top\beta_i),
$
where $\beta_i$ is an unknown arm-specific index vector and
$f_i$ is an unknown smooth function.

To emulate an online setting, we sample trajectories of length $T=1000$
from the dataset and reveal contexts sequentially, observing feedback only
for the selected arm. Results are aggregated over 40 random permutations.

\paragraph*{Directional inference and interpretability}
We first examine directional confidence sets for the normalized index
parameters $\theta_i^\star=\beta_i/\|\beta_i\|_2$.
Figure~\ref{fig:rice_dir_summary} summarizes the resulting inference.
Panel (a) shows that the directional confidence interval width contracts
rapidly over time, indicating that the index direction stabilizes quickly.
Panel (b) displays the absolute directional loadings at $t=900$, revealing
the dominant features in the learned projection.
\begin{figure}[t]
\centering
\begin{tabular}{@{}cc@{}}
\begin{minipage}[t]{0.45\linewidth}
\centering
\includegraphics[width=\linewidth]{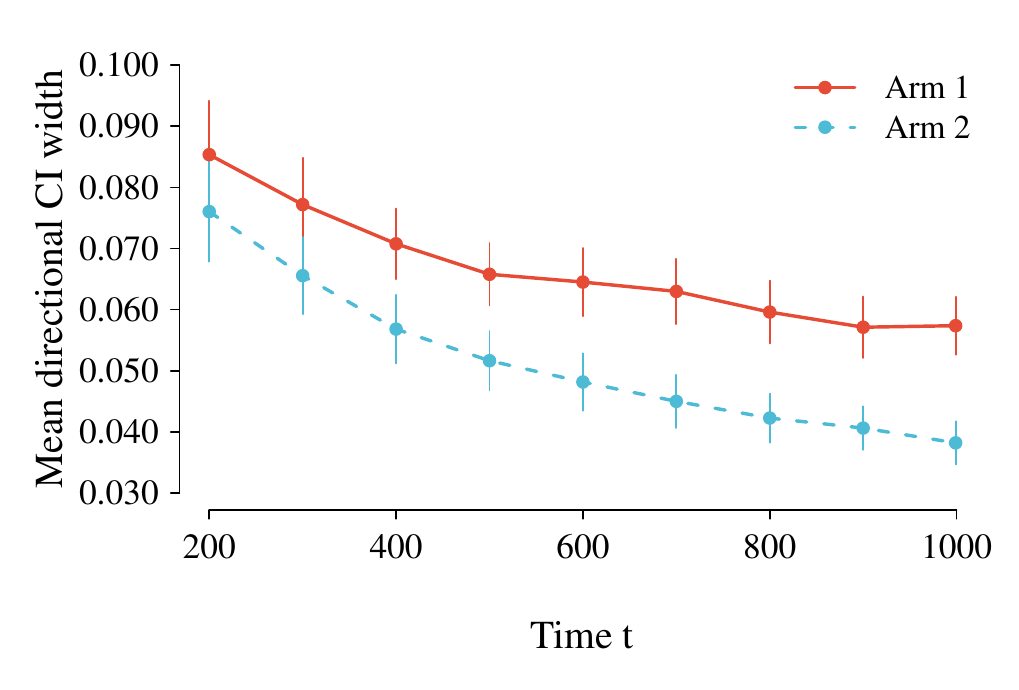}\\
\vspace{-0.2em}
{\small (a) Mean directional CI width vs.\ time}
\end{minipage}
&
\begin{minipage}[t]{0.50\linewidth}
\centering
\includegraphics[width=\linewidth]{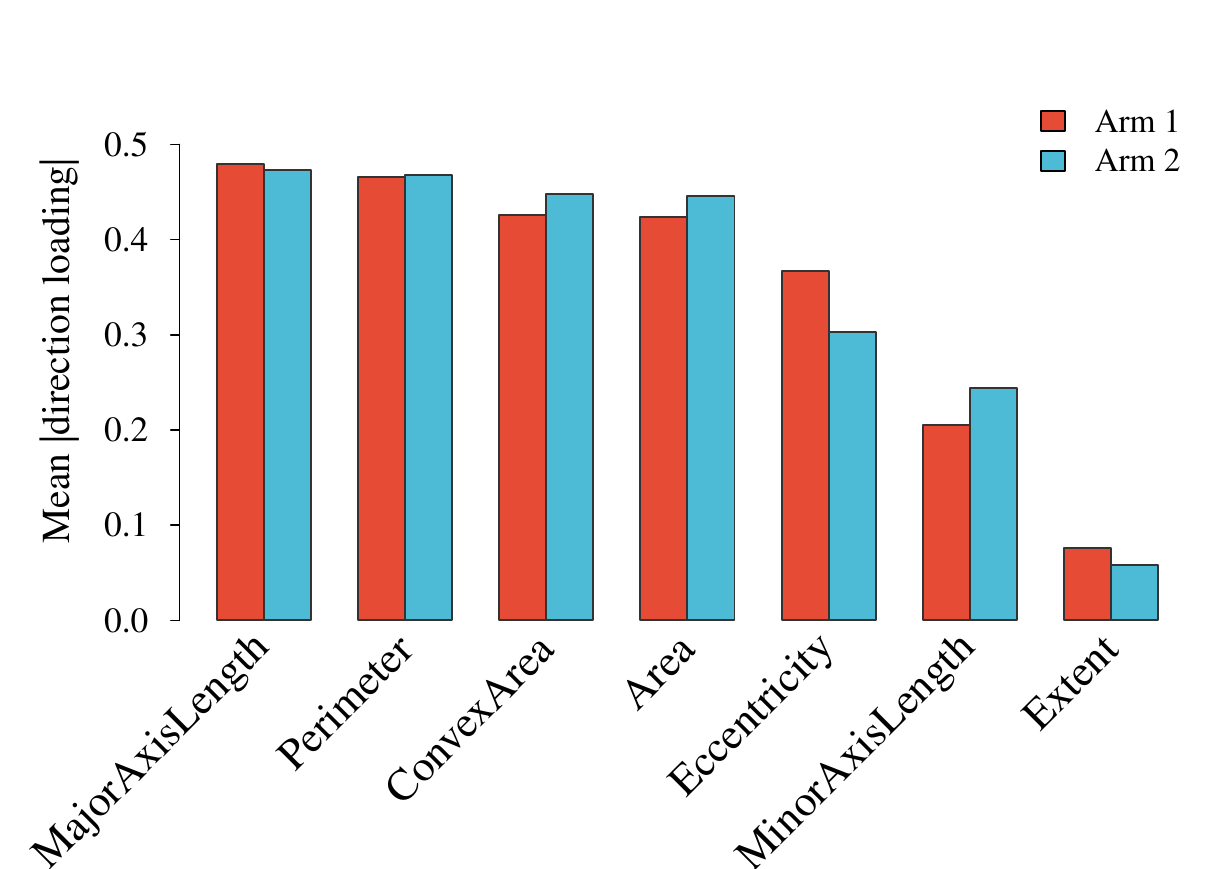}\\
\vspace{-0.2em}
{\small (b) Mean directional loadings by arm at $t=900$}
\end{minipage}
\end{tabular}
\caption{Directional inference summary on the Rice Classification dataset.}
\label{fig:rice_dir_summary}
\end{figure}
The results show that \emph{Perimeter}, \emph{MajorAxisLength},
\emph{ConvexArea}, and \emph{Area} consistently dominate the
single-index representation. These variables correspond to global
grain size and shape descriptors, which are known to be highly
discriminative in rice-grain classification
\citep{Cinarer2024RiceCA, ISLAM2025100788}.
Directional confidence intervals quantify uncertainty in the relative
importance of these features.
\begin{figure}[h!]
\centering
\begin{tabular}{@{}cc@{}}
\begin{minipage}[t]{0.55\linewidth}
\centering
\includegraphics[width=\linewidth]{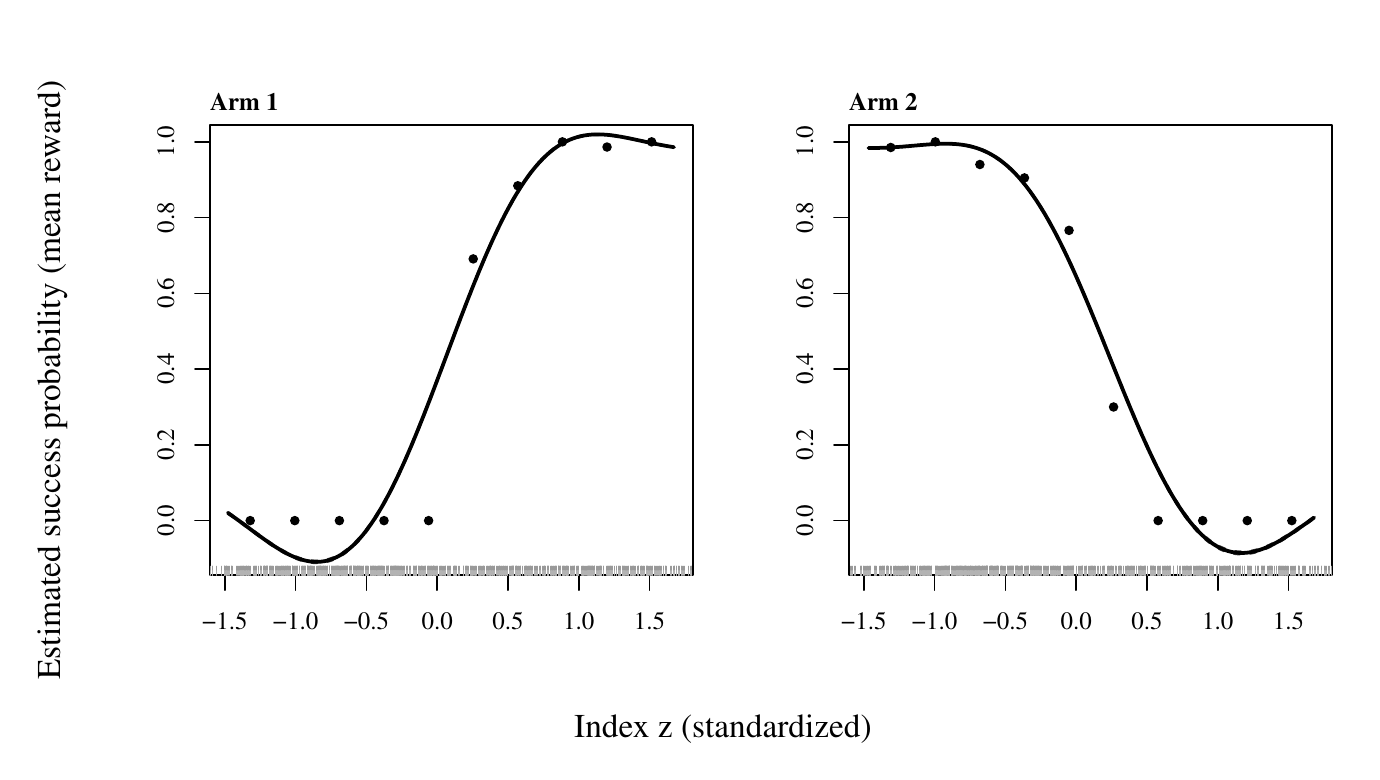}\\
\vspace{-0.35em}
{\small (a)}
\end{minipage}
\begin{minipage}[t]{0.45\linewidth}
\centering
\includegraphics[width=\linewidth]{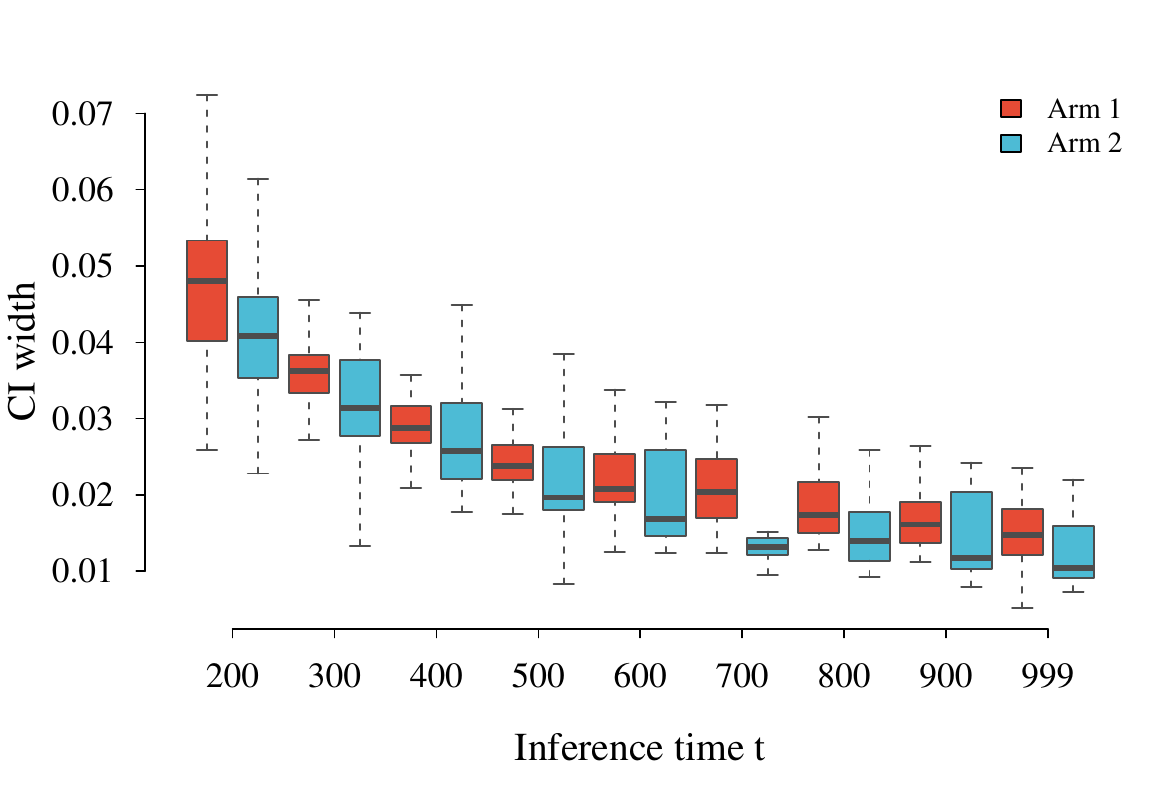}\\
\vspace{-0.35em}
{\small (b)}
\end{minipage}
\end{tabular}
\caption{Nonparametric inference summary on the Rice Classification dataset.
(a) Estimated success curves with pointwise $95\%$ confidence intervals at $t=900$
for a representative run. (b) Distribution of pointwise CI widths across replications,
by arm and time.}
\label{fig:rice_np_summary}
\end{figure}

\paragraph*{Nonparametric success curves}
We next illustrate uncertainty in the arm-specific reward functions.
Figure~\ref{fig:rice_np_summary} shows the estimated success curves
$\widehat f_{i,t}(z)$ as functions of the learned index
$z=X^\top\widehat\beta_{i,t}$ together with pointwise confidence intervals.
The two curves exhibit complementary behavior across the index domain,
reflecting the competing class-prediction actions, while the confidence
bands contract steadily as more data are observed.

Across all runs the cumulative regret grows sublinearly, indicating that
the algorithm rapidly approaches the performance of strong
full-information baselines. Additional implementation details,
diagnostics, and supplementary figures are provided in the Supplement.

\section{Conclusion}\label{sec: conclusion}

We study contextual bandits with finitely many arms under arm-specific single-index reward models with unknown link functions. Our primary contribution is a unified framework that enables both policy learning and statistically valid inference for the index parameters and reward functions using data generated by a contextual bandit algorithm. This setting poses substantial challenges: the sampling design is endogenous, observations are non-i.i.d., and unbiased estimation requires inverse-propensity weighting, which introduces dependence and variance inflation.

Our analysis shows that these challenges can nevertheless be addressed constructively. For the index parameters, we establish martingale-based asymptotic normality for a Stein-type estimator under adaptive sampling, yielding interpretable confidence regions for the index direction. For the reward functions, we derive a directional functional central limit theorem in an RKHS for inverse-propensity-weighted kernel ridge regression, leading to valid online confidence intervals for reward evaluations along the learned index. These inferential results are complemented by nonasymptotic estimation bounds and finite-time regret guarantees for the kernelized $\varepsilon$-greedy policy, demonstrating that statistical inference can be incorporated without sacrificing sublinear regret and, under standard smoothness conditions, achieving near-parametric learning rates.

From a bandit perspective, the contribution is orthogonal to improving regret rates or enlarging model classes. Instead, the work addresses a fundamental question: how to obtain statistically interpretable inference from adaptively collected bandit data in flexible reward models. The analysis provides a general blueprint, combining martingale limit theory with high-probability control of inverse propensity weighted operators for carrying out inference under endogenous sequential designs involving both low and infinite-dimensional parameters.

Several extensions remain of interest, including multi-index or partially linear reward structures, uniform or anytime-valid confidence bands for sequential monitoring, robustness to model misspecification, and exploration strategies that explicitly preserve inferential guarantees. More broadly, the results highlight that contextual bandits can serve not only as tools for sequential optimization but also as principled frameworks for statistical inference in adaptive experiments.

\bibliography{ref}

\newpage
\appendix
\section*{Supplement to ``Kernel Single-Index Bandits: Estimation, Inference, and Learning''}
Table~\ref{tab:notation_setup} summarizes the notation used in the problem setup (Section~\ref{sec:problem}). 
\begin{table}[h]
\centering
\caption{Notation for the problem setup.}
\label{tab:notation_setup}
\begin{tabular}{ll}
\toprule
\textbf{Symbol} & \textbf{Description} \\
\midrule
$T$, $\mathcal{T}=\{1,\dots,T\}$ & Time horizon and index set of rounds \\
$L$, $\mathcal{A}=\{1,\dots,L\}$ & Number of arms and action set \\
$d$ & Covariate dimension \\
$X_t \in \mathcal{X}\subset\mathbb{R}^d$ & Context observed at time $t$; $X_t \sim P_X$ i.i.d. \\
$a_t \in \mathcal{A}$ & Arm pulled at time $t$ (random under the policy) \\
$p_{t,i}$ & Propensity: $\mathbb{P}(a_t=i\mid \mathcal{F}_{t-1},X_t)$ \\
$Y_t$ & Reward observed at time $t$ \\
$f_i:\mathbb{R}\to\mathbb{R}$ & Unknown mean reward function for arm $i$ \\
$\beta_i \in \mathbb{R}^d$ & Unknown index parameter for arm $i$ \\
$Z_{t,i}=X_t^\top \beta_i$ & Arm-specific index variable at time $t$ \\
$\varepsilon_t$ & Noise term; $\mathbb{E}[\varepsilon_t\mid a_t=i]=0$, $\mathrm{Var}(\varepsilon_t\mid a_t=i)=\sigma^2$ \\
$\mathcal{F}_t$ & Filtration (history) up to time $t$ \\
$p(x)$, $\mathcal{S}(x)$ & Covariate density and score: $\mathcal{S}(x)=-\nabla_x\log p(x)$ \\
$W_t=\mathcal{S}(X_t)$ & Score-transformed covariate \\
$a_t^\star$ & Optimal arm: $a_t^\star=\arg\max_{i\in\mathcal{A}} f_i(Z_{t,i})$ \\
$\pi$ & Policy / bandit algorithm \\
$R_T(\pi)$ & Cumulative regret over $T$ rounds (Eq.~\eqref{eq:regret}) \\
\bottomrule
\end{tabular}
\end{table}
\section{Proofs of Section~\ref{sec:asymptotic_inference_beta}}
\label{secApp: proofs_beta_estimation}
First we provide a table (Table ~\ref{tab:notation_parametric}) for notations used in Section \ref{sec:asymptotic_inference_beta}.
\begin{table}[H]
\centering
\caption{Notation for parametric (index-direction) estimation and inference in Section~\ref{sec:beta-estimation}.}
\label{tab:notation_parametric}
\begin{tabular}{ll}
\toprule
\textbf{Symbol} & \textbf{Description} \\
\midrule

$p$ & Differentiable density of $X_t \sim P_X$ on $\mathbb{R}^d$ \\
$\mathcal{S}(x)$ & Score (w.r.t.\ covariates): $\mathcal{S}(x)=-\nabla_x\log p(x)$ \\
$W_s$ & Score feature: $W_s := \mathcal{S}(X_s)\in\mathbb{R}^d$ \\[0.2em]

$\beta_i$ & Arm-$i$ index parameter in the KSIB model \\
$b_i$ & Unit direction: $b_i := \beta_i/\|\beta_i\|_2$ \\
$f_i$ & Unknown link; $f_i':=\frac{d}{dz}f_i(z)$ \\
$\mu_i$ & Stein scaling: $\mu_i := \mathbb{E}\!\left[f_i'(X^\top\beta_i)\right]$ \\[0.2em]

$\mathcal{F}_{s-1}$ & History filtration up to $s-1$ \\
$\hat a_s$ & Arm pulled at time $s$ (random) \\
$p_{s,i}(W_s)$ & Propensity: $\mathbb{P}(\hat a_s=i\mid \mathcal{F}_{s-1},W_s)$ \\
$p_s^\star$ & Minimal propensity: $p_s^\star:=\inf_{w}\inf_{i\in[L]} p_{s,i}(w)$ \\
$r_t$ & Average inverse-propensity scale: $r_t:=t^{-2}\sum_{s=1}^t (p_s^\star)^{-1}$ \\[0.2em]

$\hat m_{i,t}$ & IPW moment: $\hat m_{i,t}:=\frac{1}{t}\sum_{s=1}^t \frac{\1\{\hat a_s=i\}}{p_{s,i}(W_s)}\,W_sY_s$ \\
$m_i$ & Population moment: $m_i:=\mathbb{E}[W_sY_s]$ \\
$\Gamma$ & Population Gram: $\Gamma:=\mathbb{E}[W_sW_s^\top]$ \\
$\hat\Gamma_{i,t}$ & IPW Gram: $\hat\Gamma_{i,t}:=\frac{1}{t}\sum_{s=1}^t \frac{\1\{\hat a_s=i\}}{p_{s,i}(W_s)}\,W_sW_s^\top$ \\
$\Delta_{i,t}$ & Gram error: $\Delta_{i,t}:=\hat\Gamma_{i,t}-\Gamma$ \\[0.2em]

$\hat\beta_i^{(t)}$ & IPW index estimator: $\hat\beta_i^{(t)}:=\hat\Gamma_{i,t}^{-1}\hat m_{i,t}$ \\
$\Delta_s(\beta_i)$ & Innovation: $\Delta_s(\beta_i):=W_s(Y_s-W_s^\top\beta_i)$ \\
$\Lambda_i$ & Innovation covariance: $\Lambda_i:=\mathbb{E}\!\left[\Delta_s(\beta_i)\Delta_s(\beta_i)^\top\right]$ \\[0.2em]

$\alpha$ & Scaling exponent in CLT: $t^\alpha(\hat\beta_i^{(t)}-\beta_i)$ \\
$\psi_{t,s,i}$ & CLT increment: $t^{\alpha-1}\Gamma^{-1}\frac{\1\{\hat a_s=i\}}{p_{s,i}(W_s)}\Delta_s(\beta_i)$ \\
$S_{t,i}$ & Martingale sum: $S_{t,i}:=\sum_{s=1}^t \psi_{t,s,i}$ \\
$V_{\beta,t}$ & Predictable QV matrix for $\hat\beta_i^{(t)}$ (Eq.~\eqref{eq:Vbeta:def}) \\
$\widehat V_{\beta,t,i}$ & Feasible (self-normalized) covariance estimator (Eq.~\eqref{eq:Vhat-def}) \\[0.2em]

$g(x)$ & Normalization map: $g(x)=x/\|x\|$ \\
$J(x)$ & Jacobian of $g$ at $x\neq 0$: $J(x)=(I-bb^\top)/\|x\|$ with $b=x/\|x\|$ \\
$V_{b,i}$ & Directional covariance: $V_{b,i}=J(\beta_i)V_{\beta,i}J(\beta_i)^\top$ \\
$\widehat V_{b,t,i}$ & Plug-in directional covariance estimate \\
$\chi^2_{d-1,1-\tau}$ & $(1-\tau)$ quantile of $\chi^2_{d-1}$ for directional ellipsoid \\
\bottomrule
\end{tabular}
\end{table}
\begin{lemma}[Sub-Gaussian Tail Bound]\label{lem:sg-tail}
Under Assumption~\ref{ass:subG}, there exist constants \(c, C > 0\), depending only on \(\sigma_W\) and \(d\), such that for all \(u \ge 0\),
\[
\mathbb{P}\big(\|W_s\|_2 > u\big) \le C e^{-c u^2}.
\]
\end{lemma}

	\begin{proof}[Proof of Lemma~\ref{lem:sg-tail}]
This is a standard consequence of sub-Gaussian concentration for Lipschitz functions (e.g.\ the norm) or by a net argument; see, for instance, Vershynin, \emph{High-Dimensional Probability}, Thm.\ 3.1. We include a short argument: for any \(u>0\), let \(\mathcal{N}\) be a \(1/2\)-net of the sphere with \(|\mathcal{N}|\le 5^d\). Then
		\(\norm{W_s}_2 \le 2 \max_{v\in\mathcal{N}} v^\top W_s\).
		By the union bound and (A2),
		\(\Prob(\norm{W_s}_2>u) \le \sum_{v\in\mathcal{N}} \Prob( v^\top W_s > u/2)
		\le 5^d \exp(- u^2/(8\sigma_W^2))\).
		Setting \(C_*=5^d\), \(c_*=1/(8\sigma_W^2)\) proves the claim.
	\end{proof}

\begin{Remark}
    For the ease of readability, the arm index $i = 1,\dots, L$ is omitted from some notations without loss of generality. For example, in the proof below, $H_s$ and $H_{s,i}$ are the same quantity, and the index $i$ is suppressed. This rule is followed generally, and the definitions are made clear in the respective contexts.
\end{Remark}

\begin{lemma}\label{lem:mds}
		Let
		\[
		H_{s,i} := \frac{\1\{\hat a_s=i\}}{p_{s,i}(W_s)}\, W_s W_s^\top,
		\qquad
		Z_{s,i} := H_{s,i} - \E\,[\,H_{s,i}\mid \F_{s-1}\,].
		\]
		Then \(\{Z_{s,i},\F_s\}\) is a self-adjoint matrix martingale difference sequence, and
		\[
		\sum_{s=1}^t Z_{s,i} \;=\; \sum_{s=1}^t \big(H_{s,i} - \Gamma\big),
		\qquad
		\hat\Gamma_{i,t}-\Gamma \;=\; \frac{1}{t}\sum_{s=1}^t Z_{s,i}.
		\]
	\end{lemma}

	\begin{proof}[Proof of Lemma~\ref{lem:mds}]
		Conditioning on \((\F_{s-1},W_s)\),
		\[
		\E\!\left[H_s \,\middle|\, \F_{s-1},W_s\right]
		= \E\!\left[\frac{\1\{\hat a_s=i\}}{p_{s,i}(W_s)} \,\middle|\, \F_{s-1},W_s\right] W_sW_s^\top
		= W_s W_s^\top.
		\]
		Therefore, by Assumption~\ref{ass:gram:para:second moment},
		\(
		\E[\,H_s\mid \F_{s-1}\,] = \E[\,W_s W_s^\top \mid \F_{s-1}\,] = \Gamma.
		\)
		Thus \(Z_s = H_s - \Gamma\) and \(\E[\,Z_s\mid \F_{s-1}\,]=0\), proving the martingale difference property. Summing and dividing by \(t\) gives the displayed relations.
	\end{proof}
	
	\begin{lemma}[Basic matrix inequalities]\label{lem:matrix-ineq}
		Let \(A,B\) be self-adjoint \(d\times d\) matrices. Then:
		\begin{enumerate}
			\item \((A-B)^2 \preceq 2A^2 + 2B^2\) (Loewner order).
			\item \(\norm{\,\E[A]\,} \le \E[\,\norm{A}\,]\) (operator norm is convex).
			\item \(\norm{W W^\top} = \norm{W}_2^2\) and \((W W^\top)^2 \preceq \norm{W}_2^2 \, (W W^\top) \preceq \norm{W}_2^4 I\).
		\end{enumerate}
	\end{lemma}
	
	\begin{proof}
		(1) Using \(A^2+B^2 - AB - BA = (A-B)^2\) and the PSD inequality \(A^2+B^2 - AB - BA = (A-B)^2 \preceq 2A^2 + 2B^2\) which follows from \(0\preceq (A-B)^2\preceq 2(A^2+B^2)\).
		(2) By Jensen's inequality since \(X\mapsto \norm{X}\) is convex.
		(3) \(\norm{WW^\top}=\sup_{\norm{v}=1} v^\top W W^\top v = \sup_{\norm{v}=1} (v^\top W)^2 = \norm{W}_2^2\). Also
		\((W W^\top)^2 = W (W^\top W) W^\top \preceq \norm{W}_2^2\, W W^\top \preceq \norm{W}_2^4 I\).
	\end{proof}

\begin{proposition}
\label{prop:second-moment}
		With \(Z_{s,i}\) as in Lemma~\ref{lem:mds}, for each \(s\),
		\[
		\E\big[\,\fnorm{Z_{s,i}}^2\,\big]
		=
		\E\left[\frac{\norm{W_s}_2^4}{\,p_{s,i}(W_s)\,}\right]
		-\fnorm{\Gamma}^2
		\;\le\;
		\frac{M_4}{\,p_s^\star\,}.
		\]
		Consequently, for any index set \(B\subset\{1,\dots,t\}\),
		\[
		\E\left\|\,\frac{1}{t}\sum_{s\in B} Z_{s,i} \right\|_F^2
		\;\le\;
		\frac{M_4}{t^2}\sum_{s\in B}\frac{1}{p_s^\star}\,.
		\]
\end{proposition}
	
	\begin{proof}[Proof of Proposition~\ref{prop:second-moment}]
		Since \(Z_s=H_s-\Gamma\) and \(\E[H_s\mid \F_{s-1}]=\Gamma\),
		\[
		\fnorm{Z_s}^2 = \fnorm{H_s}^2 - 2\langle H_s,\Gamma\rangle_F + \fnorm{\Gamma}^2.
		\]
		Conditioning on \((\F_{s-1},W_s)\) and noting \(\E[ \1\{\hat a_s = i\} / p_{s,i}(W_s) \mid \F_{s-1},W_s] = 1\) and
		\(\E[ (\1\{\hat a_s = i\}/p_{s,i}(W_s))^2 \mid \F_{s-1},W_s] = 1/p_{s,i}(W_s)\), we get
		\[
		\E\big[\fnorm{H_s}^2 \mid \F_{s-1},W_s\big]
		= \frac{1}{p_{s,i}(W_s)}\,\fnorm{W_s W_s^\top}^2
		= \frac{\norm{W_s}_2^4}{p_{s,i}(W_s)}.
		\]
		Similarly,
		\(
		\E[\langle H_s,\Gamma\rangle_F \mid \F_{s-1},W_s]
		= \langle W_s W_s^\top,\Gamma\rangle_F
		= W_s^\top \Gamma W_s.
		\)
		Taking expectations and using \(\E[W_s^\top \Gamma W_s] = \tr(\Gamma \E[W_sW_s^\top]) = \tr(\Gamma^2) = \fnorm{\Gamma}^2\),
		\[
		\E\fnorm{Z_s}^2 = \E\!\left[\frac{\norm{W_s}_2^4}{p_{s,i}(W_s)}\right] - \fnorm{\Gamma}^2 \le \E\!\left[\frac{\norm{W_s}_2^4}{p_s^\star}\right] \le \frac{M_4}{p_s^\star}.
		\]
		For the final bound, use orthogonality of martingale differences in \(L^2\):
		for \(r<s\),
		\(
		\E\langle Z_s,Z_r\rangle_F
		= \E\langle \E[Z_s\mid \F_{s-1}],Z_r\rangle_F = 0.
		\)
		Hence
		\[
		\E\left\|\,\frac{1}{t}\sum_{s\in B} Z_s\right\|_F^2
		= \frac{1}{t^2}\sum_{s\in B} \E\fnorm{Z_s}^2
		\le \frac{M_4}{t^2}\sum_{s\in B} \frac{1}{p_s^\star}.
		\]
	\end{proof}

\begin{lemma}[Variational characterization; Lipschitz in operator norm]\label{lem:weyl-min}
	For any symmetric \(A\in\R^{d\times d}\),
	\[
	\lambda_{\min}(A) \;=\; \min_{\norm{v}_2=1} v^\top A v.
	\]
	Consequently, for symmetric \(A,B\),
	\[
	\big|\,\lambda_{\min}(A) - \lambda_{\min}(B)\,\big|
	\;\le\; \norm{A-B}.
	\]
\end{lemma}

\begin{proof}
	The Rayleigh-Ritz variational principle yields the first identity. For the inequality, write
	\[
	\lambda_{\min}(A)-\lambda_{\min}(B)
	= \min_{\norm{v}=1} v^\top A v \;-\; \min_{\norm{v}=1} v^\top B v
	\;\le\; \max_{\norm{v}=1} v^\top (A-B) v
	\;\le\; \norm{A-B}.
	\]
	Symmetrically, exchanging \(A,B\) and applying the triangle inequality yields the absolute value bound.
\end{proof}

\begin{theorem}[Smallest-eigenvalue consistency]\label{thm:eig-consistency}
	Under Assumption~\ref{ass:gram:para:second moment},~\ref{ass:avg:explore:rt}, and~\ref{ass:subG}, for any \(\eta>0\),
	\[
	\Prob\!\left(\,\big|\,\lambda_{\min}(\hat\Gamma_{i,t}) - \lambda_{\min}(\Gamma)\,\big| > \eta\,\right)
	\;\xrightarrow[t\to\infty]{}\; 0.
	\]
\end{theorem}

\begin{proof}[Proof of Theorem~\ref{thm:eig-consistency}]
	By Lemma~\ref{lem:weyl-min},
	\[
	\big|\,\lambda_{\min}(\hat\Gamma_{i,t}) - \lambda_{\min}(\Gamma)\,\big|
	\;\le\; \norm{\hat\Gamma_{i,t}-\Gamma}.
	\]
	Thus, for any \(\eta>0\),
	\[
	\Prob\!\left(\,\big|\,\lambda_{\min}(\hat\Gamma_{i,t}) - \lambda_{\min}(\Gamma)\,\big| > \eta\,\right)
	\;\le\;
	\Prob\!\left(\,\norm{\hat\Gamma_{i,t}-\Gamma} > \eta\,\right)
	\;\xrightarrow[t\to\infty]{}\; 0,
	\]
	because \(\norm{\hat\Gamma_{i,t}-\Gamma}\xrightarrow{P} 0\) (to be proven in Theorem~\ref{thm:GammaHP}  and Corollary~\ref{cor:PD}).
\end{proof}

\begin{corollary}[Positive definiteness with high probability]\label{cor:PD}
	Under Assumption \ref{ass:identifiability}, we have that,
for any \(\varepsilon\in(0,\mu)\),
	\[
	\Prob\!\left(\,\lambda_{\min}(\hat\Gamma_{i,t}) \le \mu-\varepsilon\,\right)
	\;\le\;
	\Prob\!\left(\,\big|\,\lambda_{\min}(\hat\Gamma_{i,t}) - \lambda_{\min}(\Gamma)\,\big| \ge \varepsilon\,\right)
	\;\xrightarrow[t\to\infty]{}\; 0.
	\]
	In particular, \(\Prob(\hat\Gamma_{i,t}\ \text{is positive definite}) \to 1\), and \(\hat\Gamma_{i,t}^{-1}\) exists with high probability.
\end{corollary}

\begin{proof}[Proof of Corollary~\ref{cor:PD}]
	Immediate from Theorem~\ref{thm:eig-consistency} and \(\lambda_{\min}(\Gamma)=\mu>0\).
\end{proof}

\begin{lemma}[Matrix Freedman Inequality {\cite[Theorem~1.6]{Tropp2011}}]
\label{lem:freedman}
Let \(\{ \boldsymbol{Y}_k \}_{k\ge 0}\) be a matrix martingale taking values in the 
self-adjoint \(d\times d\) matrices, and write the difference sequence as
\(\boldsymbol{X}_k = \boldsymbol{Y}_k - \boldsymbol{Y}_{k-1}\) for \(k \ge 1\).
Assume that
\[
\mathbb{E}\!\left( \boldsymbol{X}_k \mid \mathcal{F}_{k-1} \right) = \boldsymbol{0},
\qquad 
\lambda_{\max}(\boldsymbol{X}_k) \le R \quad\text{almost surely}.
\]
Define the predictable quadratic variation process
\[
\boldsymbol{W}_k
\;=\;
\sum_{j=1}^k 
\mathbb{E}\!\left( \boldsymbol{X}_j^2 \mid \mathcal{F}_{j-1} \right).
\]
Then, for all \(t \ge 0\) and all \(\sigma^2 \ge 0\),
\[
\mathbb{P}\!\left\{
\lambda_{\max}(\boldsymbol{Y}_k) \ge t
\;\text{ and }\;
\lambda_{\max}(\boldsymbol{W}_k) \le \sigma^2
\right\}
\;\le\;
d \cdot 
\exp\!\left(
  -\,\frac{t^2/2}{\,\sigma^2 + Rt/3\,}
\right).
\]
\end{lemma}
    
\begin{theorem}[High-probability bound for the IPW Gram matrix]\label{thm:GammaHP}
Let Assumptions~\ref{ass:gram:para:second moment}-\ref{ass:identifiability} hold.
Fix \(\delta, \eta \in (0, 1/2)\), and define
\begin{align*}
M &:= \sqrt{M_8\,C_1}, \qquad
u_t^2 := \frac{2}{c_1}\,\max\!\left\{\log\frac{1}{\eta},\,\log\!\left(\frac{M^2\,(t\,r_t+1)}{\eta^2}\right)\right\}, \\
\Xi_t &:= 2\log M + \log(1 + t\,r_t) + 3\log\left(\frac{1}{\eta}\right).
\end{align*}
Then, with probability at least \(1 - 2\delta - \eta\),
\begin{align*}
\big\|\,\hat\Gamma_{i,t} - \Gamma\,\big\|
\;\le\;&
2\sqrt{M_4\,r_t\,\log\left(\tfrac{2d}{\delta}\right)}
\;+\;
2\sqrt{\frac{m_2^2}{t}\,\log\left(\tfrac{2d}{\delta}\right)}
\;+\;
\sqrt{\frac{M_4}{\delta}}\,\sqrt{r_t}
\;+\;
\frac{8}{3c_1}\,\frac{\log\left(\tfrac{2d}{\delta}\right)}{t}\,\Xi_t
\;+\;
\eta.
\end{align*}
\end{theorem}

\begin{proof}[Proof of Theorem~\ref{thm:GammaHP}]
    	We split the index set \(\{1,\dots,t\}\) depending on a feature threshold \(u>0\) and a weight threshold \(b\ge 1\):  
	\[
	G := \{\, s:\ \norm{W_s}_2 \le u,\ \ 1/p_s^\star \le b \,\},\quad
	H := \{\, s:\ \norm{W_s}_2 > u,\ \ 1/p_s^\star \le b \,\},\quad
	B := \{\, s:\ 1/p_s^\star > b \,\}.
	\]
	
	By Lemma~\ref{lem:mds},
	\begin{align*}
	\norm{\hat\Gamma_{i,t}-\Gamma}
	\;\le\; \frac{1}{t} \norm{\sum_{s\in G} Z_s}
		+ \frac{1}{t}\norm{\sum_{s\in H} Z_s}
	+ \frac{1}{t} \norm{\sum_{s\in B} Z_s}
	\end{align*}
	We bound the three terms in turn.
	\subsection*{Term over \(G\): bounded increments $\,+\,$ Freedman}
	
	On \(G\), by Lemma~\ref{lem:matrix-ineq},
	\[
	\norm{H_s} \le \frac{1}{p_s^\star}\,\norm{W_sW_s^\top}\le b\,u^2,
	\quad
	\norm{\E[H_s\mid \F_{s-1}]} \le \E[\norm{W_sW_s^\top}\,\1\{\norm{W_s}\le u\}] \le u^2.
	\]
	Hence \(\norm{Z_s}\le (b+1)u^2 \le 2 b u^2 =: R\).
	For the predictable quadratic variation,
	\begin{align*}
		\E\big[\,Z_s^2\mid \F_{s-1}\,\big]
		&\preceq 2 \E[\,H_s^2\mid \F_{s-1}\,] + 2 \big(\E[H_s\mid \F_{s-1}]\big)^2
		\quad\text{(Lemma~\ref{lem:matrix-ineq}(1))}\\
		&\preceq 2 \frac{\E[\norm{W_s}_2^4\,\1\{\norm{W_s}\le u\}]}{p_s^\star} I + 2 m_2^2 I
		\ \preceq\ 2 \frac{M_4}{p_s^\star} I + 2 m_2^2 I,
	\end{align*}
	where \(m_2 := \E[\norm{W_s}_2^2]\).
	Therefore,
	\[
	V_G \;:=\; \lambda_{\max}\!\Big(\sum_{s\in G} \E[Z_s^2\mid \F_{s-1}]\Big)
	\ \le\ 2 M_4 \sum_{s\in G} \frac{1}{p_s^\star} + 2 |G|\, m_2^2
	\ \le\ 2 M_4 \sum_{s=1}^t \frac{1}{p_s^\star} + 2 t m_2^2.
	\]
	By Lemma~\ref{lem:freedman}, for any \(\delta\in(0,1)\), with probability at least \(1-\delta\),
	\begin{align*}
		\frac{1}{t} \norm{\sum_{s\in G} Z_s}
		\ \le\
		\underbrace{\frac{4}{3}\,u^2 b\,\frac{\log(2d/\delta)}{t}}_{\mathrm{range}}
		\;+\;
		\underbrace{\sqrt{( \frac{4 M_4}{t^2}\sum_{s=1}^t \frac{1}{p_s^\star} + \frac{4 m_2^2}{t} )\,\log\frac{2d}{\delta}}}_{\mathrm{variance}}.
	\end{align*}
	
	\subsection*{Term over \(H\): large-feature tail via Markov}
	
	Since \(H_s\succeq 0\),
	\[
	\frac{1}{t}\norm{\sum_{s\in H} H_s}
	\ \le\ \frac{1}{t}\sum_{s\in H} \norm{H_s}
	\ \le\ \frac{1}{t}\sum_{s=1}^t \frac{1}{p_s^\star}\,\norm{W_s}_2^4\,\1\{\norm{W_s}_2>u\}.
	\]
	By (A2) and Lemma~\ref{lem:sg-tail}, using Cauchy-Schwarz
	\[
	\E\big[\norm{W_s}_2^4\,\1\{\norm{W_s}_2>u\}\big]
	\le \sqrt{M_8} \,\sqrt{\Prob(\norm{W_s}_2>u)}
	\le \sqrt{M_8 C_*} \, e^{-c_* u^2/2}.
	\]
	Thus
	\[
	\E\left[\frac{1}{t}\norm{\sum_{s\in H} H_s}\right]
	\le \sqrt{M_8 C_*} \, e^{-c_* u^2/2}\,\frac{1}{t}\sum_{s=1}^t \frac{1}{p_s^\star}
	= \sqrt{M_8 C_*}\,e^{-c_* u^2/2}\, t\, r_t.
	\]
	Similarly,
	\(
	\frac{1}{t}\norm{\sum_{s\in H} \E[H_s\mid \F_{s-1}]}
	\le \sqrt{M_8 C_*}\, e^{-c_* u^2/2}.
	\)
	Therefore, for any \(\eta\in(0,1)\), by Markov's inequality,
	\begin{align*}
		\Prob\!\left(
		\frac{1}{t}\norm{\sum_{s\in H} Z_s} > \eta
		\right)
		\ \le\
		\frac{\sqrt{M_8 C_*}}{\eta}\,(t\,r_t + 1)\,e^{-c_* u^2/2}.
	\end{align*}
	\subsection*{Term over \(B\): large-weight part via Chebyshev}
	
	By Proposition \ref{prop:second-moment},
	\[
	\E\left\|\,\frac{1}{t}\sum_{s\in B} Z_s \right\|_F^2
	\le \frac{M_4}{t^2}\sum_{s\in B} \frac{1}{p_s^\star}
	\le M_4\, r_t.
	\]
	Since \(\norm{\cdot}\le \norm{\cdot}_F\), Chebyshev yields, for any \(\delta\in(0,1)\),
	\begin{align*}
		\Prob\!\left(
		\frac{1}{t}\norm{\sum_{s\in B} Z_s} > \sqrt{\frac{M_4}{\delta}}\,\sqrt{r_t}
		\right)
		\ \le\ \delta.
	\end{align*}

\subsection*{Combining the three parts}

	Call the decomposition
	\[
	\|\hat\Gamma_{i,t}-\Gamma\| \;\le\; \|G_t\|+\|H_t\|+\|B_t\|
	\]
	and the three tail bounds:
	
	\begin{itemize}
		\item[(G)] For any $u>0$, with probability at least $1-\delta$,
		\[
		\|G_t\|\ \le\ \underbrace{\frac{4}{3}\,u^2 b\,\frac{\log(2d/\delta)}{t}}_{\text{(I)}}\
		+\ \underbrace{\sqrt{\Big( 4 M_4 r_t + \frac{4 m_2^2}{t}\Big)\,\log\frac{2d}{\delta}}}_{\text{(II)}}.
		\]
		\item[(H)] For any $u>0$, with probability at least $1-e^{-c_*u^2/2}$,
		\[
		\|H_t\|\ \le\ \sqrt{M_8 C_*\,\big(t\,r_t+1\big)\,e^{-c_*u^2/2}}\,.
		\]
		\item[(B)] With probability at least $1-\delta$,
		\[
		\|B_t\|\ \le\ \sqrt{\frac{M_4}{\delta}}\,\sqrt{r_t}\,.
		\]
	\end{itemize}
	
	Set $M:=\sqrt{M_8 C_*}$ for brevity.
	
	Fix confidence levels $\delta,\eta\in(0,1/2)$. We want to choose $u=u_t$ so that \emph{both}
	\begin{equation}\label{eq:Ht-magnitude-target}
		\|H_t\|\ \le\ \eta
	\end{equation}
	\emph{and} the failure probability of (H) is at most $\eta$, i.e.
	\begin{equation}\label{eq:Ht-failure-target}
		\Prob\big(\text{(H) fails}\big)\ \le\ \eta.
	\end{equation}
	Then, combined with (G) and (B) at level $\delta$ and a union bound, we get the master bound with overall probability at least $1-2\delta-\eta$.
	
	\paragraph*{Turning the two targets into explicit inequalities on $u$}
	From (H), on its ``good'' event we have
	\[
	\|H_t\|\ \le\ M\sqrt{(t r_t+1)\,e^{-c_*u^2/2}}.
	\]
	Thus the magnitude target \eqref{eq:Ht-magnitude-target} is ensured if
	\begin{equation*}
		M\sqrt{(t r_t+1)\,e^{-c_*u^2/2}}\ \le\ \eta
		\quad\Longleftrightarrow\quad
		e^{-c_*u^2/2}\ \le\ \frac{\eta^2}{M^2}\cdot\frac{1}{t r_t+1}.
	\end{equation*}
	Independently, the failure-probability target \eqref{eq:Ht-failure-target} is ensured if
	\begin{equation*}
		e^{-c_*u^2/2}\ \le\ \eta.
	\end{equation*}
	Therefore, it suffices to choose $u$ so that
	\begin{equation*}
		e^{-c_*u^2/2}\ \le\ \min\!\left\{\,\eta,\ \frac{\eta^2}{M^2}\cdot\frac{1}{t r_t+1}\right\},
	\end{equation*}
simplifying,
	\[
	\frac{c_*}{2}\,u^2\ \ge\ \max\!\left\{\,\log\frac{1}{\eta},\ \log\!\Big(\frac{M^2(t r_t+1)}{\eta^2}\Big)\right\}.
	\]
	Hence, a valid (indeed, minimal) choice is
	\begin{align}\label{eq:u-choice-max}
		u_t^2\ &:=\ \frac{2}{c_*}\ \max\!\left\{\,\log\frac{1}{\eta},\ \log\!\Big(\frac{M^2(t r_t+1)}{\eta^2}\Big)\right\}
		\\
        &\quad\ \Big(\text{equivalently }
		u_t=\sqrt{\frac{2}{c_*}\,\Big(\log(1/\eta)\ \vee\ \log(M^2(t r_t+1)/\eta^2)\Big)}\Big). \nonumber
	\end{align}
	With $u=u_t$ in \eqref{eq:u-choice-max} we have simultaneously:
	\[
	\|H_t\|\ \le\ \eta\quad\text{on the (H)-good event},\qquad
	\Prob\big(\text{(H)-good fails}\big)\ \le\ e^{-c_*u_t^2/2}\ \le\ \eta.
	\]
	
	Define the three ``good'' events
	\[
	\mathcal E_G:=\{\text{(G) holds at level }\delta\},\qquad
	\mathcal E_B:=\{\text{(B) holds at level }\delta\},\qquad
	\mathcal E_H:=\{\|H_t\|\le \eta\}\ \ (\text{with }u=u_t).
	\]
	Then
	\[
	\Prob(\mathcal E_G)\ge 1-\delta,\qquad
	\Prob(\mathcal E_B)\ge 1-\delta,\qquad
	\Prob(\mathcal E_H)\ge 1-\eta,
	\]
	so, by the union bound (no independence is required),
	\begin{equation*}
		\Prob(\mathcal E_G\cap \mathcal E_B\cap \mathcal E_H)\ \ge\ 1-(\delta+\delta+\eta)\ =\ 1-2\delta-\eta.
	\end{equation*}
		On $\mathcal E_G\cap \mathcal E_B\cap \mathcal E_H$,
	\[
	\|\hat\Gamma_{i,t}-\Gamma\|
	\ \le\
	\underbrace{\frac{4}{3}\,u_t^2\,b\,\frac{\log(2d/\delta)}{t}}_{(\mathrm{I})}
	\;+\;
	\underbrace{\sqrt{\Big( 4 M_4 r_t + \frac{4 m_2^2}{t}\Big)\,\log\frac{2d}{\delta}}}_{(\mathrm{II})}
	\;+\;
	\underbrace{\sqrt{\frac{M_4}{\delta}}\,\sqrt{r_t}}_{(\mathrm{III})}
	\;+\;
	\underbrace{\eta}_{(\mathrm{IV})},
	\]

Note that the parameter $b\geq 1$ is purely a truncation/splitting device; it controls how aggressively we truncate away small $p_s(W_s)$. Larger $b$ means we allow smaller propensities through $G_t$ and push more mass into $H_t$ and $B_t$. So setting $b=1$ is valid - it simply gives the loosest truncation, which makes the expression for (I) smallest (because the bound on $G_t$ scales with $b$), while the other terms (II), (III), (IV) already absorb the contribution from the $p_s(W_s)<1$ region. Thus, on $\mathcal E_G\cap \mathcal E_B\cap \mathcal E_H$, one can choose $b=1$ or choose $b$ as any constant greater than or equal to 1, and $u=u_t$ chosen as in~\eqref{eq:u-choice-max})
	\[
	\|\hat\Gamma_{i,t}-\Gamma\|
	\ \le\
	\underbrace{\frac{4}{3}\,u_t^2\,\frac{\log(2d/\delta)}{t}}_{(\mathrm{I})}
	\;+\;
	\underbrace{\sqrt{\Big( 4 M_4 r_t + \frac{4 m_2^2}{t}\Big)\,\log\frac{2d}{\delta}}}_{(\mathrm{II})}
	\;+\;
	\underbrace{\sqrt{\frac{M_4}{\delta}}\,\sqrt{r_t}}_{(\mathrm{III})}
	\;+\;
	\underbrace{\eta}_{(\mathrm{IV})},
	\]
	which holds with probability at least $1-2\delta-\eta$. 
\end{proof}

\begin{corollary}[Consistency under mild exploration]\label{cor:scaled-consistency}
Fix $\alpha\in(0,1/2)$. Suppose $t^{2\alpha}\,r_t\to 0$ as $t\to\infty$. With the following choice of parameters in Theorem~\ref{thm:GammaHP}: $\gamma\in(1/2-\alpha,1/2)$, $\eta_t:=t^{-\gamma}$, and $\delta\in(0,1/2)$, it follows that
\[
t^\alpha\,\big\|\,\hat\Gamma_{i,t}-\Gamma\,\big\|\ \xrightarrow{\ \Prob\ }\ 0.
\]
\end{corollary}

\begin{proof}[Proof of Corollary~\ref{cor:scaled-consistency}]
We now give a concrete, simplified form by fixing $\eta$ explicitly.

\paragraph*{Dominant terms and rates}
We start from the high-probability bound (with $b=1$ and $u=u_t$ chosen as in~\eqref{eq:u-choice-max})
\begin{equation}\label{eq:hp-master-dominant}
	\|\hat\Gamma_{i,t}-\Gamma\|
	\ \le\
	\underbrace{\frac{4}{3}\,u_t^2\,\frac{\log(2d/\delta)}{t}}_{(\mathrm{I})}
	\;+\;
	\underbrace{\sqrt{\Big( 4 M_4 r_t + \frac{4 m_2^2}{t}\Big)\,\log\frac{2d}{\delta}}}_{(\mathrm{II})}
	\;+\;
	\underbrace{\sqrt{\frac{M_4}{\delta}}\,\sqrt{r_t}}_{(\mathrm{III})}
	\;+\;
	\underbrace{\eta_t}_{(\mathrm{IV})},
\end{equation}
which holds with probability at least $1-2\delta-\eta_t$, where $\eta_t\in(0,1/2)$ is a free tuning parameter and
\[
u_t^2=\frac{2}{c_*}\ \max\!\left\{\,\log\frac{1}{\eta_t},\ \log\!\Big(\frac{M^2(t r_t+1)}{\eta_t^2}\Big)\right\}
\]
(cf.\ \eqref{eq:u-choice-max}), and $M:=\sqrt{M_8 C_*}$.
We next isolate the stochastic scales and compare orders.

\medskip
\noindent\textbf{Decomposition of (II) and deterministic bounds.}
Using $\sqrt{a+b}\le \sqrt a+\sqrt b$ for $a,b\ge0$,
\begin{equation}\label{eq:II-split}
	(\mathrm{II})
	\ \le\
	2\sqrt{M_4\,r_t\,\log\frac{2d}{\delta}}
	\;+\;
	2\sqrt{\frac{m_2^2}{t}\,\log\frac{2d}{\delta}}
	=: (\mathrm{IIa})+(\mathrm{IIb}).
\end{equation}
Thus the two stochastic scales are $\sqrt{r_t}$ and $t^{-1/2}$ (up to $\sqrt{\log(2d/\delta)}$ factors).

For (I), by the explicit form of $u_t$ and the trivial bound
$\max\{x,y\}\le x+y$,
\begin{align}\label{eq:I-bound}
	(\mathrm{I})
	&\le
	\frac{8}{3c_*}\,\frac{\log(2d/\delta)}{t}\,
	\left(\log\frac{1}{\eta_t}\ +\ \log\!\frac{M^2(tr_t+1)}{\eta_t^2}\right)\\
	&=
	\frac{8}{3c_*}\,\frac{\log(2d/\delta)}{t}\,
	\Big(2\log M+\log(tr_t+1)+3\log(1/\eta_t)\Big). \nonumber
\end{align}
In particular, for any choice with $\eta_t\downarrow 0$,
\begin{equation*}
	(\mathrm{I})
	=O\!\left(\frac{\log(1+t)+\log(1/\eta_t)}{t}\right)
	=o\!\left(\frac{1}{\sqrt t}\right),
	\qquad\text{since }\ \frac{\log(1+t)+\log(1/\eta_t)}{\sqrt t}\ \to\ 0.
\end{equation*}

Term (IV) is exactly $\eta_t$, which we are free to choose (subject to enlarging $u_t$ accordingly via \eqref{eq:u-choice-max}).

\medskip
\noindent\textbf{A uniform high-probability bound with tunable $\eta_t$.}
Combining \eqref{eq:hp-master-dominant}, \eqref{eq:II-split}, and \eqref{eq:I-bound} gives: for any $\delta\in(0,1/2)$ and any sequence
$\eta_t\in(0,1/2)$,
\begin{equation}\label{eq:HP-master-tunable}
	\Prob\!\left(
	\|\hat\Gamma_{i,t}-\Gamma\|
	\ \le\
	2\sqrt{M_4\,r_t\,\log\frac{2d}{\delta}}
	\ +\
	2\sqrt{\frac{m_2^2}{t}\,\log\frac{2d}{\delta}}
	\ +\
	\sqrt{\frac{M_4}{\delta}}\,\sqrt{r_t}
	\ +\
	\frac{8}{3c_*}\,\frac{\log(2d/\delta)}{t}\,\Xi_t
	\ +\
	\eta_t
	\right)
	\ \ge\ 1-2\delta-\eta_t,
\end{equation}
where $\Xi_t:=2\log M+\log(tr_t+1)+3\log(1/\eta_t)$.

\medskip
\paragraph*{Consistency of $t^\alpha\|\hat\Gamma_{i,t}-\Gamma\|$ for $\alpha<1/2$}
Starting from the high-probability inequality (with $b=1$ and $u=u_t$ as in~\eqref{eq:u-choice-max})
\begin{equation}\label{eq:HP-master-again}
	\|\hat\Gamma_{i,t}-\Gamma\|
	\ \le\
	\underbrace{\frac{4}{3}\,u_t^2\,\frac{\log(2d/\delta)}{t}}_{(\mathrm{I})}
	\ +\
	\underbrace{\sqrt{\Big( 4 M_4 r_t + \frac{4 m_2^2}{t}\Big)\,\log\frac{2d}{\delta}}}_{(\mathrm{II})}
	\ +\
	\underbrace{\sqrt{\frac{M_4}{\delta}}\,\sqrt{r_t}}_{(\mathrm{III})}
	\ +\
	\underbrace{\eta_t}_{(\mathrm{IV})},
\end{equation}
which holds with probability at least $1-2\delta-\eta_t$, we analyze $t^\alpha$ times the right-hand side.
We first choose a tuning $\eta_t\downarrow 0$ (this sets the oscillation term and the failure probability of the
$H$-bound) and then impose a mild rate on the exploration sequence $r_t$.

\medskip
\noindent\textbf{Choice of $\eta_t$.}
Fix $\alpha\in(0,1/2)$. Pick any $\gamma\in(1/2-\alpha,1/2)$ and set
\[
\eta_t:=t^{-\gamma}.
\]
Then $t^\alpha\eta_t=t^{\alpha-\gamma}\to 0$, and $\log(1/\eta_t)=\gamma\log t=o(\sqrt t)$ so that
the enlargement of $u_t$ (via~\eqref{eq:u-choice-max}) remains negligible for our purposes.

\medskip
Recall $r_t:=t^{-2}\sum_{s=1}^t\E[1/p_{s,i}]$. The bound~\eqref{eq:HP-master-again} shows the \emph{dominant}
stochastic term is of order $\sqrt{r_t}$ (up to logs), while the $t^{-1/2}$ term is smaller after scaling by
$t^\alpha$ when $\alpha<1/2$. Thus, to have $t^\alpha\|\hat\Gamma_{i,t}-\Gamma\|\to 0$ in probability, it suffices to require
\begin{equation}\label{eq:rate-cond}
	t^{\alpha}\sqrt{r_t}\ \longrightarrow\ 0
	\qquad\text{equivalently}\qquad
	t^{2\alpha}\,r_t\ \longrightarrow\ 0.
\end{equation}

\medskip
\noindent\textbf{Term-by-term verification.}
Multiply~\eqref{eq:HP-master-again} by $t^\alpha$ and control each term:

\begin{itemize}
	\item[(I)] Using $u_t^2=\frac{2}{c_*}\max\{\log(1/\eta_t),\log(M^2(tr_t+1)/\eta_t^2)\}\le C_1(\log t+\log(1+tr_t))$ with our choice
	of $\eta_t$,
	\[
	t^\alpha(\mathrm{I})\ \le\ C\,t^{\alpha-1}\,\big(\log t+\log(1+tr_t)\big)\,\log\frac{2d}{\delta}
	\ \xrightarrow[t\to\infty]{}\ 0,
	\]
	since $\alpha<1/2$ implies $t^{\alpha-1}\log t\to 0$ (and $\log(1+tr_t)\le \log t + O(1)$).
	
	\item[(II)] Split as $(\mathrm{IIa})+(\mathrm{IIb})$:
	\[
	t^\alpha(\mathrm{IIa})\ =\ 2\,t^\alpha\sqrt{M_4\,r_t\,\log\frac{2d}{\delta}}
	\ \le\ C\,\sqrt{\log\frac{2d}{\delta}}\;\underbrace{t^\alpha\sqrt{r_t}}_{\to 0\ \text{by}\ \eqref{eq:rate-cond}} \ \to\ 0,
	\]
	and
	\[
	t^\alpha(\mathrm{IIb})\ =\ 2\,\sqrt{m_2^2\,\log\frac{2d}{\delta}}\;t^{\alpha-1/2}\ \longrightarrow\ 0
	\qquad(\alpha<1/2).
	\]
	
	\item[(III)] $t^\alpha(\mathrm{III})=\sqrt{M_4/\delta}\,t^\alpha\sqrt{r_t}\to 0$ by~\eqref{eq:rate-cond}.
	
	\item[(IV)] $t^\alpha(\mathrm{IV})=t^\alpha\eta_t=t^{\alpha-\gamma}\to 0$ by the choice $\gamma>1/2-\alpha$.
\end{itemize}

\noindent Combining the four displays, we conclude
\[
t^\alpha\,\|\hat\Gamma_{i,t}-\Gamma\|\ \xrightarrow{\Prob}\ 0
\quad\text{provided}\quad t^{2\alpha}r_t\to 0
\ \ \text{and}\ \ \eta_t=t^{-\gamma},\ \gamma\in(1/2-\alpha,1/2).
\]

The requirement $t^{2\alpha}r_t\to 0$ is precisely the (mild) strengthening of the baseline
assumption $\sum_{s\le t}1/p_{s,i}=o(t^2)$ needed to make the $\sqrt{r_t}$ stochastic scale vanish after $t^\alpha$ scaling.
It is identical to the condition that appeared in the martingale CLT arguments (there for nondegeneracy/normalization),
so the conclusions are fully consistent:
\[
\|\hat\Gamma_{i,t}-\Gamma\|\ =\ O_{\Prob}\!\Big(\max\{\sqrt{r_t},\,t^{-1/2}\}\Big)
\quad\Longrightarrow\quad
t^\alpha\|\hat\Gamma_{i,t}-\Gamma\|\ =\ o_{\Prob}(1)
\ \text{iff}\ t^{2\alpha}r_t\to 0\ \text{and}\ \alpha<\tfrac12.
\]
If $d$ is fixed (or $\log d=o(t^{1-2\alpha})$), the logarithmic factors are innocuous and do not change the conclusion.

\end{proof}

\begin{lemma}[Exact predictable variance and Loewner bounds]\label{lem:Vbeta}
	Under Assumptions~\ref{ass:cond:noise}-\ref{ass:direc:nondegeneracy},
	\[
	V_{\beta,t}\ \text{ is given by \eqref{eq:Vbeta:def}, and}\quad
	t^{2\alpha-1}\,\Gamma^{-1}\Lambda_i\Gamma^{-1}\ \preceq\ V_{\beta,t}\ \preceq\ t^{2\alpha}r_t\,\Gamma^{-1}\Lambda_i\Gamma^{-1}.
	\]
\end{lemma}

\begin{proof}
	The identity \eqref{eq:Vbeta:def} follows by conditioning on \((\F_{s-1},W_s)\). For the bounds: since \(0<p_{s,i}(W_s)\le 1\), \(1/p_{s,i}(W_s)\ge 1\) a.s., hence
	\[
	V_{\beta,t}\ \succeq\ t^{2\alpha-2}\sum_{s=1}^t \Gamma^{-1}\,\E\!\big[\Delta_s\Delta_s^\top\mid\F_{s-1}\big]\,\Gamma^{-1}
	= t^{2\alpha-1}\,\Gamma^{-1}\Lambda_i\Gamma^{-1}.
	\]
	Also \(p_{s,i}(W_s)\ge p_s^\star\) a.s.\ implies \(1/p_{s,i}(W_s)\le 1/p_s^\star\), thus
	\[
	V_{\beta,t}\ \preceq\ t^{2\alpha-2}\sum_{s=1}^t \frac{1}{p_s^\star}\,\Gamma^{-1}\E[\Delta_s\Delta_s^\top]\,\Gamma^{-1}
	= t^{2\alpha}r_t\,\Gamma^{-1}\Lambda_i\Gamma^{-1}.
	\]
\end{proof}

\begin{lemma}[Invertibility and scale of $V_{\beta,t}$]\label{lem:Vbeta-inv}
	Under Assumptions~\ref{ass:cond:noise}-\ref{ass:direc:nondegeneracy}, for all large \(t\),
	\[
	V_{\beta,t}\ \succeq\ t^{2\alpha-1}\,\sigma_\beta^2\,\Gamma^{-1}\ \succ\ 0,
	\]
	hence \(V_{\beta,t}^{-1/2}\) exists. Moreover, there exist constants \(0<c\le C<\infty\) (depending only on \(\Gamma,\Lambda_i\)) such that
	\[
	c\,t^{2\alpha-1}\ \le\ \lambda_{\min}(V_{\beta,t})\ \le\ \lambda_{\max}(V_{\beta,t})\ \le\ C\,t^{2\alpha}r_t .
	\]
\end{lemma}

\begin{proof}
	By Assumption \ref{ass:resid}, \(\Lambda_i\succeq \sigma_\beta^2\Gamma\). Apply Lemma~\ref{lem:Vbeta} on both sides to get the lower and upper bounds.
\end{proof}

\begin{proposition}[Lindeberg condition for studentized projections]\label{prop:lindeberg}
	Fix \(a\in\R^d\) and set \(\zeta_{t,s}:=a^\top V_{\beta,t}^{-1/2}\psi_{t,s}\in\R\). Then \(\{\zeta_{t,s},\F_s\}\) is a scalar MDS with
	\[
	\sum_{s=1}^t \E[\zeta_{t,s}^2\mid \F_{s-1}]=\|a\|^2 \quad \text{a.s. for every }t,
	\]
	and, under Assumptions \ref{ass:avg:explore:rt2}-\ref{ass:direc:nondegeneracy}, for every \(\varepsilon>0\),
	\[
	\sum_{s=1}^t \E\big[\zeta_{t,s}^2\,\1\{|\zeta_{t,s}|>\varepsilon\}\mid \F_{s-1}\big]\ \xrightarrow[t\to\infty]{\Prob}\ 0.
	\]
\end{proposition}

\begin{proof}
	The MDS property follows from \(\E[\psi_{t,s}\mid\F_{s-1}]=0\). For the variance identity,
	\[
	\sum_{s=1}^t \E[\zeta_{t,s}^2\mid \F_{s-1}]
	= a^\top V_{\beta,t}^{-1/2}\Big(\sum_{s=1}^t \E[\psi_{t,s}\psi_{t,s}^\top\mid\F_{s-1}]\Big)V_{\beta,t}^{-1/2}a
	= \|a\|^2.
	\]
	For Lindeberg, with \(\eta:=\delta/2>0\) and \(\varepsilon>0\),
	\begin{align*}
&	\E[\zeta_{t,s}^2\1\{|\zeta_{t,s}|>\varepsilon\}\mid \F_{s-1}]\\ 
	\le & \varepsilon^{-\eta}\,\E[|\zeta_{t,s}|^{2+\eta}\mid \F_{s-1}] \\
	\le & C\,\|V_{\beta,t}^{-1/2}\|^{2+\eta}\,t^{(2+\eta)(\alpha-1)}\,
	\E\!\Big[\frac{1}{p_{s,i}(W_s)^{1+\eta/2}}\|\Delta_s(\beta_i)\|^{2+\eta}\,\Big|\,\F_{s-1}\Big],
	\end{align*}
	where we used \(\|\psi_{t,s}\|\le t^{\alpha-1}\|\Gamma^{-1}\|(1/p_{s,i}(W_s))\|\Delta_s(\beta_i)\|\).
	By \(p_{s,i}(W_s)\ge p_s^\star\) and Assumption~\ref{ass:cond:noise},
	\[
	\E[|\zeta_{t,s}|^{2+\eta}\mid \F_{s-1}]
	\le C'\,\|V_{\beta,t}^{-1/2}\|^{2+\eta}\,t^{(2+\eta)(\alpha-1)}\,(p_s^\star)^{-(1+\eta/2)}.
	\]
	Summing over \(s\le t\),
	\[
	\sum_{s=1}^t \E[\zeta_{t,s}^2\1\{|\zeta_{t,s}|>\varepsilon\}\mid \F_{s-1}]
	\le C'\,\|V_{\beta,t}^{-1/2}\|^{2+\eta}\,t^{(2+\eta)(\alpha-1)}
	\sum_{s=1}^t (p_s^\star)^{-(1+\eta/2)}.
	\]
	By Lemma~\ref{lem:Vbeta-inv}, \(\|V_{\beta,t}^{-1/2}\|\le C''\,t^{(1-2\alpha)/2}\). Hence
	\[
	\|V_{\beta,t}^{-1/2}\|^{2+\eta}\,t^{(2+\eta)(\alpha-1)}
	\le (C'')^{2+\eta}\,t^{(1-2\alpha)(1+\eta/2)}\,t^{(2+\eta)(\alpha-1)}
	= (C'')^{2+\eta}\,t^{-(1+\eta/2)}.
	\]
	Therefore,
	\[
	\sum_{s=1}^t \E[\zeta_{t,s}^2\1\{|\zeta_{t,s}|>\varepsilon\}\mid \F_{s-1}]
	\le \tilde C\, t^{-(1+\eta/2)}\sum_{s=1}^t (p_s^\star)^{-(1+\eta/2)}
	\ \xrightarrow{\Prob}\ 0,
	\]
	because \(\sum_{s=1}^t (p_s^\star)^{-(1+\eta/2)}=o\big(t^{1+\eta/2}\big)\) by Assumption~\ref{ass:avg:explore:rt2} (the denominator there is \(\ge t^{1+\eta/2}\)). This proves the Lindeberg condition.
\end{proof}

We now conclude with the vector martingale CLT.

\begin{theorem}[Studentized Martingale CLT]\label{thm:studentized-Vbeta}
Suppose Assumptions~\ref{ass:gram:para:second moment}-\ref{ass:identifiability} and \ref{ass:outcome:moments}-\ref{ass:resid} hold. Fix any \(\alpha \in (0,1/2)\) and any arm $i=1,\dots,L$, then,
\[
V_{\beta, t}^{-1/2} S_{t,i} \ \overset{D}{\longrightarrow}\ \mathcal{N}(0, I_d).
\]
\end{theorem}

\begin{proof}[Proof of Theorem~\ref{thm:studentized-Vbeta}]
	For any fixed \(a\in\R^d\), Proposition~\ref{prop:lindeberg} gives that the scalar array
	\(\{\zeta_{t,s}=a^\top V_{\beta,t}^{-1/2}\psi_{t,s}\}\) is an MDS with
	\(\sum_{s=1}^t\E[\zeta_{t,s}^2\mid \F_{s-1}]=\|a\|^2\) a.s.\ and satisfies the conditional Lindeberg condition.
	Hence, by the scalar martingale CLT (~\cite{HallHeyde1980}, Thm.\ 3.2),
	\(
	a^\top V_{\beta,t}^{-1/2}S_t=\sum_{s=1}^t \zeta_{t,s}\cid \mathcal N(0,\|a\|^2).
	\)
	By Cram\'er-Wold, \(V_{\beta,t}^{-1/2}S_t\cid \mathcal N(0,I_d)\).
\end{proof}

\begin{Remark}[On the regime $t^{2\alpha}r_t\to 0$]
	The scale bounds in Lemma~\ref{lem:Vbeta-inv} imply
	\(\|V_{\beta,t}\|=O(t^{2\alpha}r_t)\to 0\) and \(\lambda_{\min}(V_{\beta,t})\gtrsim t^{2\alpha-1}\).
	Thus, while the unstudentized \(\|S_t\|\) is \(o_{\Prob}(1)\) in norm under this regime, the \emph{studentized} statistic in Theorem~\ref{thm:studentized-Vbeta} has a stable nondegenerate Gaussian limit.
\end{Remark}

\begin{lemma}[Exact one-step identity]\label{lem:exact}
	Let \(A_t:=\Gamma^{-1}(\hat\Gamma_{i,t}-\Gamma)\).
	Then
	\[
	t^\alpha(\hat{\beta}_i^{(t)}-\beta_i)
	=(I+A_t)^{-1}\left\{\frac{1}{t}\sum_{s=1}^{t} t^{\alpha}\,\Gamma^{-1}\frac{\1\{\hat a_s=i\}}{p_{s,i}(W_s)}\,\Delta_s(\beta_i)\right\}.
	\]
\end{lemma}

\begin{proof}
	Use the standard resolvent rearrangement
	\(
	(I+A_t)(\hat{\beta}_i^{(t)}-\beta_i)
	=\Gamma^{-1}\big[(\hat m_{i,t}-m_i)-(\hat\Gamma_{i,t}-\Gamma)\beta_i\big]
	\),
	and note
	\(
	(\hat m_{i,t}-m_i)-(\hat\Gamma_{i,t}-\Gamma)\beta_i
	=\frac1t\sum_{s=1}^t \frac{\1\{\hat a_s=i\}}{p_{s,i}(W_s)}\,\Delta_s(\beta_i).
	\)
	Multiply by \(t^\alpha\).
\end{proof}

Recall the martingale differences
\[
\psi_{t,s}:=t^{\alpha-1}\,\Gamma^{-1}\,\frac{\1\{\hat a_s=i\}}{p_{s,i}(W_s)}\,\Delta_s(\beta_i),
\qquad
S_t:=\sum_{s=1}^{t}\psi_{t,s}.
\]
With respect to \((\F_s)\), \(\E[\psi_{t,s}\mid \F_{s-1}]=0\) by ignorability and \(\E[\Delta_s(\beta_i)\mid\F_{s-1}]=0\).
The predictable covariance in~\eqref{eq:Vbeta:def} is
\[
V_{\beta,t}:=\sum_{s=1}^t \E[\psi_{t,s}\psi_{t,s}^\top\mid \F_{s-1}]
= t^{2\alpha-2}\sum_{s=1}^t \Gamma^{-1}\,\E\!\Big[\frac{1}{p_{s,i}(W_s)}\,\Delta_s(\beta_i)\Delta_s(\beta_i)^\top\Big|\F_{s-1}\Big]\,\Gamma^{-1}.
\]

\begin{proposition}[Studentized linearization remainder is negligible]\label{prop:rem}

Let \(A_t:=\Gamma^{-1}(\hat\Gamma_{i,t}-\Gamma)\) and define the remainder
	\[
	R_t:=(I+A_t)^{-1}S_t - S_t.
	\]
	Under the standing assumptions used in the martingale CLT for \(V_{\beta,t}^{-1/2}S_t\) (namely~\ref{ass:cond:noise}-\ref{ass:direc:nondegeneracy}, and \(t^{2\alpha}r_t\to 0\)) together with the already established consistency
	\(
	\|\hat\Gamma_{i,t}-\Gamma\|\xrightarrow{\Prob}0
	\)
	(equivalently, \(\|A_t\|\xrightarrow{\Prob}0\)), we have
	\[
	\norm{\,V_{\beta,t}^{-1/2}\,R_t\,}\ \xrightarrow{\Prob}\ 0.
	\]
\end{proposition}

\begin{proof}
	Since \(I+A_t\) is invertible on the event \(\{\|A_t\|<1\}\), the resolvent identity yields
	\[
	(I+A_t)^{-1}-I=-(I+A_t)^{-1}A_t.
	\]
	Hence
	\[
	R_t=\big((I+A_t)^{-1}-I\big)S_t=-(I+A_t)^{-1}A_t S_t.
	\]
	Multiplying by \(V_{\beta,t}^{-1/2}\) and inserting two identities \(V_{\beta,t}^{1/2}V_{\beta,t}^{-1/2}=I\) gives the exact factorization
	\begin{equation*}
		V_{\beta,t}^{-1/2}R_t
		=-\underbrace{\,\Big(V_{\beta,t}^{-1/2}(I+A_t)^{-1}V_{\beta,t}^{1/2}\Big)\,}_{\mathsf{F}_{1,t}}
		\underbrace{\,\Big(V_{\beta,t}^{-1/2}A_t V_{\beta,t}^{1/2}\Big)\,}_{\mathsf{F}_{2,t}}
		\underbrace{\,\Big(V_{\beta,t}^{-1/2}S_t\Big)\,}_{\mathsf{F}_{3,t}}.
	\end{equation*}
	Therefore
	\begin{equation*}
		\norm{V_{\beta,t}^{-1/2}R_t}
		\ \le\ \norm{\mathsf{F}_{1,t}}\ \norm{\mathsf{F}_{2,t}}\ \norm{\mathsf{F}_{3,t}}.
	\end{equation*}
	
	\smallskip
	\textbf{Bound \(\mathsf{F}_{1,t}\):}
	By submultiplicativity,
	\[
	\norm{\mathsf{F}_{1,t}}
	=\norm{V_{\beta,t}^{-1/2}(I+A_t)^{-1}V_{\beta,t}^{1/2}}
	\ \le\ \norm{V_{\beta,t}^{-1/2}}\ \norm{(I+A_t)^{-1}}\ \norm{V_{\beta,t}^{1/2}}
	=\sqrt{\kappa(V_{\beta,t})}\ \norm{(I+A_t)^{-1}},
	\]
	where \(\kappa(V_{\beta,t}):=\|V_{\beta,t}\|\,\|V_{\beta,t}^{-1}\|\) is the spectral condition number.
	By Lemma~\ref{lem:Vbeta} and Assumption~\ref{ass:direc:nondegeneracy} (used earlier), there exist constants \(0<c\le C<\infty\) such that
	\[
	c\,\Gamma^{-1}\ \preceq\ \frac{V_{\beta,t}}{t^{2\alpha-2}\sum_{s=1}^t \E[1/p_{s,i}(W_s)]}\ \preceq\ C\,\Gamma^{-1},
	\]
	hence \(\kappa(V_{\beta,t})\le \kappa(\Gamma^{-1})\,C/c=:K<\infty\), deterministically. Therefore
	\[
	\norm{\mathsf{F}_{1,t}}\ \le\ \sqrt{K}\ \norm{(I+A_t)^{-1}}.
	\]
	Since \(\|A_t\|\xrightarrow{\Prob}0\), for any fixed \(\epsilon\in(0,1)\),
	\[
	\Prob\big(\|A_t\|\le \epsilon\big)\ \to\ 1,
	\qquad
	\text{and on this event } \ \ \norm{(I+A_t)^{-1}}\ \le\ \frac{1}{1-\|A_t\|}\ \le\ \frac{1}{1-\epsilon}.
	\]
	Thus \(\norm{\mathsf{F}_{1,t}}=O_{\Prob}(1)\).
	
	\smallskip
	\textbf{Bound \(\mathsf{F}_{2,t}\):}
	Recall \(A_t=\Gamma^{-1}(\hat\Gamma_{i,t}-\Gamma)\). Then
	\[
	\norm{\mathsf{F}_{2,t}}
	=\norm{V_{\beta,t}^{-1/2}\Gamma^{-1}V_{\beta,t}^{1/2}}\ \norm{\hat\Gamma_{i,t}-\Gamma}.
	\]
	As in Step 2,
	\[
	\norm{V_{\beta,t}^{-1/2}\Gamma^{-1}V_{\beta,t}^{1/2}}
	\ \le\ \norm{V_{\beta,t}^{-1/2}}\ \norm{\Gamma^{-1}}\ \norm{V_{\beta,t}^{1/2}}
	=\norm{\Gamma^{-1}}\,\sqrt{\kappa(V_{\beta,t})}
	\ \le\ \norm{\Gamma^{-1}}\,\sqrt{K},
	\]
	a deterministic constant. By the previously established consistency of \(\hat\Gamma_{i,t}\),
	\(\ \norm{\hat\Gamma_{i,t}-\Gamma}\xrightarrow{\Prob}0\).
	Hence
	\[
	\norm{\mathsf{F}_{2,t}}\ \xrightarrow{\Prob}\ 0.
	\]
	
	\smallskip
	\textbf{Tightness of \(\mathsf{F}_{3,t}\):}
	By the studentized martingale CLT proved earlier,
	\[
	V_{\beta,t}^{-1/2}S_t\ \cid\ \mathcal N(0,I_d),
	\]
	which implies tightness: \(\norm{\mathsf{F}_{3,t}}=\norm{V_{\beta,t}^{-1/2}S_t}=O_{\Prob}(1)\).
	
	\smallskip

	Fix any \(\varepsilon>0\) and \(\eta\in(0,1)\). Choose \(\epsilon\in(0,1)\) for Step 2 and a constant \(M<\infty\) such that
	\[
	\Prob\big(\|A_t\|\le \epsilon\big)\ge 1-\eta,\qquad
	\Prob\big(\|V_{\beta,t}^{-1/2}S_t\|\le M\big)\ge 1-\eta
	\]
	for all large \(t\). On the intersection of these events,
	\[
	\norm{V_{\beta,t}^{-1/2}R_t}
	\ \le\ \underbrace{\sqrt{K}\frac{1}{1-\epsilon}}_{=:C_1}\ \norm{\mathsf{F}_{2,t}}\ M.
	\]
	Since \(\norm{\mathsf{F}_{2,t}}\xrightarrow{\Prob}0\), there exists \(t_0\) such that for all \(t\ge t_0\),
	\(\Prob(\norm{\mathsf{F}_{2,t}}> \varepsilon/(C_1 M))\le \eta\).
	Therefore, for \(t\ge t_0\),
	\[
	\Prob\big(\norm{V_{\beta,t}^{-1/2}R_t}>\varepsilon\big)
	\ \le\ \Prob(\|A_t\|>\epsilon)+\Prob(\|V_{\beta,t}^{-1/2}S_t\|>M)+\Prob\!\Big(\norm{\mathsf{F}_{2,t}}>\frac{\varepsilon}{C_1 M}\Big)
	\ \le\ 3\eta.
	\]
	Letting \(\eta\downarrow 0\) yields \(\norm{V_{\beta,t}^{-1/2}R_t}\xrightarrow{\Prob}0\).
\end{proof}


\begin{proof}[Proof of Theorem~\ref{thm:main}]
	By Lemma~\ref{lem:exact},
	\(
	t^\alpha(\hat{\beta}_i^{(t)}-\beta_i)=(I+A_t)^{-1}S_t
	\),
	hence
	\[
	V_{\beta,t}^{-1/2}\,t^\alpha(\hat{\beta}_i^{(t)}-\beta_i)
	=V_{\beta,t}^{-1/2}S_t\;+\;V_{\beta,t}^{-1/2}R_t.
	\]
	By Theorem~\ref{thm:studentized-Vbeta}, \(V_{\beta,t}^{-1/2}S_t\cid \mathcal N(0,I_d)\).
	By Proposition~\ref{prop:rem}, \(V_{\beta,t}^{-1/2}R_t\to_{\Prob}0\).
	Slutsky's theorem yields the stated convergence.
	For the unstudentized statement, note that \(\|V_{\beta,t}\|\asymp t^{2\alpha}r_t\to 0\) while \(\norm{V_{\beta,t}^{-1/2}\,t^\alpha(\hat{\beta}_i^{(t)}-\beta_i)}=O_{\Prob}(1)\), hence
	\(\norm{t^\alpha(\hat{\beta}_i^{(t)}-\beta_i)}\le \|V_{\beta,t}^{1/2}\|\,O_{\Prob}(1)\to_{\Prob} 0\).
\end{proof}


\begin{proof}[Proof of Theorem~\ref{thm:Vhat-consistency}]
Throughout, ``$\lesssim$'' hides universal constants independent of $t$ and $s$; all norms are operator norms for matrices and Euclidean norms for vectors. Write
\[
\Xi_t\ :=\ \widehat V_{\beta,t}-V_{\beta,t}
\ =\ \underbrace{\sum_{s=1}^t\bigl(\widehat\psi_{t,s}\widehat\psi_{t,s}^\top-\psi_{t,s}\psi_{t,s}^\top\bigr)}_{:\,=\,B_t}
\ +\ \underbrace{\sum_{s=1}^t\Bigl(\psi_{t,s}\psi_{t,s}^\top-\mathbb E[\psi_{t,s}\psi_{t,s}^\top\mid\mathcal F_{s-1}]\Bigr)}_{:\,=\,C_t}.
\]
We show $\|B_t\|\xrightarrow{\mathbb P}0$ and $\|C_t\|\xrightarrow{\mathbb P}0$.

Introduce the shorthands
\[
D_t:=\widehat\Gamma_{i,t}^{-1}-\Gamma^{-1},\qquad
G_t:=\widehat\beta_{i,t}-\beta_i,\qquad
\varrho_{t,s}:=\frac{\mathbf 1\{\widehat a_s=i\}}{p_{s,i}(W_s)}.
\]
Observe
\[
\widehat\psi_{t,s}-\psi_{t,s}
= t^{\alpha-1}\Bigl[D_t\,\varrho_{t,s}\,\Delta_s
+ \widehat\Gamma_{i,t}^{-1}\,\varrho_{t,s}\,W_s\,(-W_s^\top G_t)\Bigr].
\]
Hence, using $\|xy^\top\|\le \|x\|\|y\|$ and $(X+Y)(X+Y)^\top-XX^\top=XY^\top+YX^\top+YY^\top$,
\begin{align*}
\|B_t\|
&\le 2\sum_{s=1}^t \|\psi_{t,s}\|\,\|\widehat\psi_{t,s}-\psi_{t,s}\|
 + \sum_{s=1}^t \|\widehat\psi_{t,s}-\psi_{t,s}\|^2\\
&\le C\,t^{2\alpha-2}\sum_{s=1}^t
\Bigl(\|\Gamma^{-1}\|\|\varrho_{t,s}\Delta_s\|
\Bigr)\Bigl(\|D_t\|\|\varrho_{t,s}\Delta_s\|
+ \|\widehat\Gamma_{i,t}^{-1}\|\|\varrho_{t,s}W_s W_s^\top G_t\|\Bigr)\\
&\quad + C\,t^{2\alpha-2}\sum_{s=1}^t
\Bigl(\|D_t\|^2\|\varrho_{t,s}\Delta_s\|^2
+ \|\widehat\Gamma_{i,t}^{-1}\|^2\|\varrho_{t,s}W_s W_s^\top G_t\|^2\Bigr).
\end{align*}
By the exploration and overlap conditions (part of Assumption~\ref{ass:direc:nondegeneracy}), $\varrho_{t,s}\le c^{-1}$ a.s.\ for some $c>0$.
By assumption, $\widehat\Gamma_{i,t}\to\Gamma$ in probability, hence $\|\widehat\Gamma_{i,t}^{-1}\|=O_{\mathbb P}(1)$ and
$\|D_t\|\xrightarrow{\mathbb P}0$. In addition, $t^\alpha G_t=O_{\mathbb P}(1)$ from Theorem~\ref{thm:main} in the main text
(unstudentized degeneracy plus the CLT). Using these facts and Cauchy-Schwarz,
\begin{align*}
\|B_t\|
&\lesssim t^{2\alpha-2}\sum_{s=1}^t\Bigl(
\|D_t\|\,\|\Delta_s\|^2
+ \|\Delta_s\|\,\|W_s W_s^\top G_t\|
+ \|D_t\|^2\|\Delta_s\|^2
+ \|W_s W_s^\top G_t\|^2\Bigr)\\
&\lesssim \|D_t\|\,t^{2\alpha-2}\sum_{s=1}^t\|\Delta_s\|^2
\ +\ t^{2\alpha-2}\sum_{s=1}^t \|W_s\|^3|\varphi_s|\,\|G_t\|\\
&\quad + \|D_t\|^2\,t^{2\alpha-2}\sum_{s=1}^t\|\Delta_s\|^2
\ +\ t^{2\alpha-2}\sum_{s=1}^t \|W_s\|^4\|G_t\|^2.
\end{align*}
Under the moment and predictability conditions in Assumption~\ref{ass:direc:nondegeneracy}, $\mathbb E[\|W_s\|^4+\varphi_s^4\mid\mathcal F_{s-1}]\le C$
a.s., uniformly in $s$, hence $t^{-1}\sum_{s=1}^t\|\Delta_s\|^2=O_{\mathbb P}(1)$ and
$t^{-1}\sum_{s=1}^t\|W_s\|^4=O_{\mathbb P}(1)$. Also $\|G_t\|=O_{\mathbb P}(t^{-\alpha})$.
Therefore,
\begin{align*}
\|B_t\|\ &\lesssim\ \|D_t\|\,t^{2\alpha-1}\frac{1}{t}\sum_{s=1}^t\|\Delta_s\|^2
\ +\ t^{2\alpha-2}\cdot t\cdot t^{-\alpha}\Bigl(\frac{1}{t}\sum_{s=1}^t\|W_s\|^3|\varphi_s|\Bigr)\\
&\qquad
\ +\ \|D_t\|^2\,t^{2\alpha-1}O_{\mathbb P}(1)
\ +\ t^{2\alpha-2}\cdot t\cdot t^{-2\alpha}O_{\mathbb P}(1).
\end{align*}
The second factor in parentheses is $O_{\mathbb P}(1)$ by Cauchy-Schwarz and the uniform fourth moments.
Hence,
\[
\|B_t\|\ =\ o_{\mathbb P}(1)\cdot t^{2\alpha-1}\ +\ O_{\mathbb P}(t^{\alpha-1})\ +\ o_{\mathbb P}(1)\cdot t^{2\alpha-1}\ +\ O_{\mathbb P}(t^{-1})
\ =\ O_{\mathbb P}(t^{\alpha-1})+o_{\mathbb P}(1).
\]
Since $\alpha\in(0,1]$ in the typical bandit scaling and our regime imposes $t^{2\alpha}r_t\to 0$, in particular $\alpha<1$, thus $t^{\alpha-1}\to 0$, hence $\|B_t\|\xrightarrow{\mathbb P}0$.

Set
\[
M_{t,s}\ :=\ \psi_{t,s}\psi_{t,s}^\top-\mathbb E\!\left[\psi_{t,s}\psi_{t,s}^\top\mid\mathcal F_{s-1}\right],
\qquad 1\le s\le t.
\]
Then $(\sum_{s=1}^u M_{t,s})_{0\le u\le t}$ is a square-integrable matrix-valued martingale (entrywise) with respect to $(\mathcal F_u)$. By Burkholder-Davis-Gundy and the tower property,
\[
\mathbb E\bigl[\|C_t\|\bigr]\ \lesssim\ 
\mathbb E\Bigl[\bigl\|\textstyle\sum_{s=1}^t\mathbb E\bigl[M_{t,s}^2\mid\mathcal F_{s-1}\bigr]\bigr\|^{1/2}\Bigr]
\ \le\ \Bigl(\sum_{s=1}^t \mathbb E\bigl[\|M_{t,s}\|^2\bigr]\Bigr)^{1/2}.
\]
Using $\|\psi\psi^\top\|\le \|\psi\|^2$ and $\|A-\mathbb E[A\mid\mathcal F]\|\le 2\|A\|$,
\[
\|M_{t,s}\|\ \le\ \|\psi_{t,s}\psi_{t,s}^\top\|+\bigl\|\mathbb E[\psi_{t,s}\psi_{t,s}^\top\mid\mathcal F_{s-1}]\bigr\|
\ \le\ 2\|\psi_{t,s}\|^2,
\]
so $\|M_{t,s}\|^2\lesssim \|\psi_{t,s}\|^4$. Since $\varrho_{t,s}\le c^{-1}$,
\[
\|\psi_{t,s}\|^4\ \lesssim\ t^{4\alpha-4}\,\|\Gamma^{-1}\|^4\,\|W_s\|^4\,\varphi_s^4,
\]
whence, by the uniform fourth-moment bound in {\rm(C4)},
\[
\sum_{s=1}^t \mathbb E[\|M_{t,s}\|^2]\ \lesssim\ t^{4\alpha-4}\,t\ =\ t^{4\alpha-3}.
\]
Thus $\mathbb E[\|C_t\|]\lesssim t^{2\alpha-3/2}\to 0$ since whenever $\alpha<1/2$ The verification mirrors the Lindeberg check for the score-array in the proof of Theorem~\ref{thm:main}; we omit duplication. The key point is that $t^{2\alpha-2}\sum_{s\le t} \|\Delta_s\|^2=O_{\mathbb P}(1)$ while the additional
$t^{2\alpha-2}$ scaling ensures the quadratic variation of $C_t$ is $o_{\mathbb P}(1)$.

Combining the above shows $\|\Xi_t\|\le \|B_t\|+\|C_t\|\xrightarrow{\mathbb P}0$, proving the first claim of \eqref{eq:Vhat-operator}.

By construction, $V_{\beta,t}\succeq 0$. Under Assumption~\ref{ass:direc:nondegeneracy} and the overlap condition, $V_{\beta,t}$ is a.s.\ positive definite for all large $t$, and $\lambda_{\min}(V_{\beta,t})$ is tight away from $0$ (this is the same nondegeneracy used in Theorem~\ref{thm:main} in the main text). By Weyl's inequality and the fact that the map $A\mapsto A^{\pm 1/2}$ is operator-Lipschitz on any cone $\{A\succeq \delta I\}$ with $\delta>0$,
\[
\bigl\|\widehat V_{\beta,t}^{-1/2}-V_{\beta,t}^{-1/2}\bigr\|\ \le\ C\,\|\widehat V_{\beta,t}-V_{\beta,t}\|\ \xrightarrow{\mathbb P}\ 0,
\]
whence
\[
\bigl\|\widehat V_{\beta,t}^{-1/2}V_{\beta,t}^{1/2}-I_d\bigr\|
\ =\ \bigl\|(\widehat V_{\beta,t}^{-1/2}-V_{\beta,t}^{-1/2})V_{\beta,t}^{1/2}\bigr\|
\ \le\ \bigl\|\widehat V_{\beta,t}^{-1/2}-V_{\beta,t}^{-1/2}\bigr\|\cdot \|V_{\beta,t}^{1/2}\|
\ \xrightarrow{\mathbb P}\ 0.
\]
This proves \eqref{eq:Vhat-operator}.
\end{proof}


\begin{proof}[Proof of Corollary~\ref{cor:feasible-CLT}]
By Theorem~\ref{thm:main} in the main text, $V_{\beta,t}^{-1/2}\,t^\alpha(\widehat\beta_{i,t}-\beta_i)\overset{D}{\longrightarrow}\mathcal N(0,I_d)$.
By Theorem~\ref{thm:Vhat-consistency} in the main text,
\[
\bigl\|\widehat V_{\beta,t}^{-1/2}-V_{\beta,t}^{-1/2}\bigr\|
\ \xrightarrow{\mathbb P}\ 0
\qquad\Longrightarrow\qquad
\bigl\|\widehat V_{\beta,t}^{-1/2}V_{\beta,t}^{1/2}-I_d\bigr\|\ \xrightarrow{\mathbb P}\ 0.
\]
Hence, writing $S_t:=V_{\beta,t}^{-1/2}\,t^\alpha(\widehat\beta_{i,t}-\beta_i)$,
\[
\widehat V_{\beta,t}^{-1/2}\,t^\alpha(\widehat\beta_{i,t}-\beta_i)
=\bigl(\widehat V_{\beta,t}^{-1/2}V_{\beta,t}^{1/2}\bigr)\,S_t
\ =\ S_t + o_{\mathbb P}(1)\cdot \|S_t\|.
\]
Since $S_t\overset{D}{\longrightarrow}\mathcal N(0,I_d)$ is tight, $\|S_t\|=O_{\mathbb P}(1)$; therefore, the difference between the feasible and infeasible studentized statistics is $o_{\mathbb P}(1)$. By Slutsky's theorem, the feasible statistic converges in distribution to $\mathcal N(0,I_d)$, proving the first claim. The scalar version with any fixed $x$ follows by the continuous mapping theorem. The confidence region statement is a direct corollary of the feasible CLT and the continuous mapping theorem for norms.
\end{proof}

    

\subsection{Proof and Auxiliary Results of Section~\ref{sec:index_concentration_summary}} \label{subsecApp: proofs_beta_finite_time}

A non-asymptotic bound for $\hat{\beta}_{i,t}-\beta_i$ is needed to connect the statistical
properties of the IPW estimators with the regret and policy allocation guarantees developed
later.  In the regret analysis, the algorithm must rely on $\hat{\beta}_{i,t}$ at every time step, so
uniform high-probability control of its estimation error is essential.  
The main challenge is that both $\hat\Gamma_{i,t}$ and $\hat m_{i,t}$ are constructed from
inverse-propensity-weighted martingale increments, whose fluctuations depend on the feature
distribution, the outcome moments, and the exploration schedule.  
The approach in this section establishes deviation bounds for these operators and then
propagates them through a perturbation analysis of $\hat\Gamma_{i,t}^{-1}$, yielding an explicit,
finite-sample tail inequality for $\hat{\beta}_{i,t}-\beta_i$ that will be used repeatedly in the
downstream regret and allocation results. All the proofs for this section can be found in Section S.1.1 
of the Supplement.

We make the following assumption along with the standing assumptions required earlier in Section 4.
\begin{enumerate}[label = (F\arabic*), leftmargin=2.2em]
\item \label{ass:moment:more} \textbf{Moment envelopes.} There exists \(\delta>0\) and finite constants \(M_2,M_4,M_8,\tilde M_2,\tilde M_4\) with
	$
	\sup_s\mathbb{E}\|W_s\|^2\le M_2,\ \sup_s\mathbb{E}\|W_s\|^4\le M_4,\ \sup_s\mathbb{E}\|W_s\|^8\le M_8,\ 
	\sup_s\mathbb{E}|Y_s|^2\le \tilde M_2,\ \sup_s\mathbb{E}|Y_s|^4\le \tilde M_4,
	$
	and \(\sup_s \mathbb{E}\|W_sY_s\|^{2+\delta}<\infty\).
	In particular, by Cauchy-Schwarz, \(\|m_i\|=\|\mathbb{E}[W_sY_s]\|<\infty\).
\end{enumerate}

This slightly strengthened moment envelope assumption is required only for the finite-sample
analysis: higher-order moments of the IPW-weighted terms ensure that the truncation and
Freedman-type martingale inequalities used to control $\Delta_{i,t}$ and $\hat m_{i,t}-m_i$
yield non-asymptotic bounds with explicit dependence on $\tilde{r}_t$ and $\delta$.  
These conditions are mild, automatically satisfied under sub-Gaussian or light-tailed
contexts, and do not strengthen the identifiability or asymptotic assumptions used elsewhere.

Denote the estimation errors as
\[
\Delta_{i,t}:=\hat\Gamma_{i,t}-\Gamma_i,\ \text{ such that } \hat{\beta}_{i,t}-\beta_i
\;=\;\Gamma_i^{-1}(\hat m_{i,t}-m_i)\;+\;(\hat\Gamma_{i,t}^{-1}-\Gamma_i^{-1})\,\hat m_{i,t}.
\]
We now state finite-sample deviation inequalities for \(\|\Delta_{i,t}\|\) and \(\|\hat m_{i,t}-m_i\|\).
They are proved via a truncation/splitting scheme and matrix/vector Freedman inequalities, detailed below.
Set
\[
X_s:=\frac{1\{\hat a_s=i\}}{p_{s,i}(W_s)}\,W_sW_s^\top - \mathbb{E}[W_sW_s^\top],\qquad
Z_s:=\frac{1\{\hat a_s=i\}}{p_{s,i}(W_s)}\,W_sY_s - \mathbb{E}[W_sY_s].
\]
Then \(\{X_s,\F_s\}\) is a self-adjoint matrix MDS (\(\mathbb{E}[X_s\mid\F_{s-1}]=0\)), and \(\{Z_s,\F_s\}\) is a vector MDS
(\(\mathbb{E}[Z_s\mid\F_{s-1}]=0\)). We have
\(
\Delta_t=\frac{1}{t}\sum_{s=1}^t X_s,\ \hat m_{i,t}-m_i=\frac{1}{t}\sum_{s=1}^t Z_s.
\)

\paragraph*{Deviation bounds for $\Delta_t$ and $\hat m_{i,t}-m_i$}

We first prove finite-sample deviation inequalities for \(\|\Delta_t\|\) and \(\|\hat m_{i,t}-m_i\|\).
They are proved via a truncation/splitting scheme and matrix/vector Freedman inequalities.
Proofs follow the blueprint detailed earlier: define a truncation level \(b\ge 1\) and split the summands by
\(\chi_s(b):=1\{p_{s,i}(W_s)\ge 1/b\}\), then handle the bounded-increments part with Freedman and the small-propensity region by a counting/Markov argument.

\begin{lemma}[Matrix deviation for $\Delta_{i,t}$]\label{lem:Gamma-dev}
	Fix \(\delta_\Gamma,\eta_\Gamma\in(0,1/2)\).
	There exist absolute constants \(c_*>0\) and finite \(M_4,M_8,C_*,m_2\) (depending only on the moment envelopes) such that,
	with
	\[
	u_{\Gamma,t}^2\ :=\ \frac{2}{c_*}\ \max\!\left\{\log\frac{1}{\eta_\Gamma},\ \log\!\Big(\frac{M_8 C_*\, (t \tilde{r}_t+1)}{\eta_\Gamma^2}\Big)\right\},
	\]
	the following holds: for \(b=1\), and 
	\begin{equation}\label{eq:Delta-bound_supp}
    B_\Gamma(t;\delta_\Gamma,\eta_\Gamma) := 
    \frac{4}{3}\,u_{\Gamma,t}^2\,\frac{\log(2d/\delta_\Gamma)}{t}
		+\sqrt{\Big( 4 M_4 \tilde{r}_t + \frac{4 m_2^2}{t}\Big)\,\log\frac{2d}{\delta_\Gamma}}
		+\sqrt{\frac{M_4}{\delta_\Gamma}}\,\sqrt{\tilde{r}_t}
		+\eta_\Gamma
	\end{equation}
\[
		\mathbb{P}\!\left(
		\|\Delta_{i,t}\|
		\ \le\ B_\Gamma(t;\delta_\Gamma,\eta_\Gamma)
		\right)\ \ge\ 1-2\delta_\Gamma-\eta_\Gamma.
\]
\end{lemma}

\begin{lemma}[Vector deviation for $\hat m_{i,t}-m_i$]\label{lem:m-dev}
	Fix \(\delta_m,\eta_m\in(0,1/2)\).
	There exist absolute constants \(\tilde c_*>0\) and finite \(\tilde M_4,\tilde M_8,\tilde C_*,\tilde C_2,\tilde c_2,\tilde C_3\) (depending only on the moment envelopes) such that,
	with
	\[
	u_{m,t}^2\ :=\ \frac{2}{\tilde c_*}\ \max\!\left\{\log\frac{1}{\eta_m},\ \log\!\Big(\frac{\tilde M_8 \tilde C_*\, (t \tilde{r}_t+1)}{\eta_m^2}\Big)\right\},
	\]
    and
	\begin{equation}\label{eq:m-bound-supp}
B_m(t;\delta_m,\eta_m) := \frac{4}{3}\,u_{m,t}^2\,\frac{\log(2d/\delta_m)}{t}
		+\sqrt{\Big( \tilde C_2\, \tilde{r}_t + \frac{\tilde c_2}{t}\Big)\,\log\frac{2d}{\delta_m}}
		+\sqrt{\frac{\tilde C_3}{\delta_m}}\,\sqrt{\tilde{r}_t}
		+\eta_m
	\end{equation}
	we have
\[
		\mathbb{P}\!\left(
		\|\hat m_{i,t}-m_i\|
		\ \le B_m(t;\delta_m,\eta_m)
		\right)\ \ge\ 1-2\delta_m-\eta_m.
\]
\end{lemma}

	The proofs of Lemmas~\ref{lem:Gamma-dev} and~\ref{lem:m-dev} are based on the truncation scheme outlined above and matrix/vector Freedman inequalities for martingales. 
	To avoid duplication, we omit the proofs here; similar line of arguments as Theorem~\ref{thm:GammaHP}, Equation~\eqref{eq:HP-master-tunable}, through which we establish the parametric CLT for $\hat{\beta}_i^{(t)}$ (in Theroem~\ref{thm:main} and Corollary~\ref{cor:feasible-CLT});  can easily be followed to construct a unified and detailed proof. The present section records the resulting deviation bounds for later use in both the CLT and finite-sample analyses.

\paragraph*{Invertibility of $\hat\Gamma_{i,t}$ and inverse perturbation control}

\begin{lemma}[Small perturbation and Neumann series]\label{lem:small-perturb}
	If $\|\Gamma_i^{-1}\|\,\|\Delta_t\|<1$, then:
	\begin{enumerate}
		\item $(I+\Gamma_i^{-1}\Delta_t)$ is invertible and admits the Neumann series
		\[
		(I+\Gamma_i^{-1}\Delta_t)^{-1}=\sum_{k=0}^\infty (-\Gamma_i^{-1}\Delta_t)^k,
		\qquad
		\big\|(I+\Gamma_i^{-1}\Delta_t)^{-1}\big\|\le \frac{1}{1-\|\Gamma_i^{-1}\|\,\|\Delta_t\|}.
		\]
		\item $\hat\Gamma_{i,t}=\Gamma_i(I+\Gamma_i^{-1}\Delta_t)$ is invertible and
		\[
		\hat\Gamma_{i,t}^{-1}=(I+\Gamma_i^{-1}\Delta_t)^{-1}\Gamma_i^{-1}.
		\]
		\item The inverse perturbation satisfies
		\[
		\|\hat\Gamma_{i,t}^{-1}-\Gamma_i^{-1}\|
		\ \le\ \frac{\|\Gamma_i^{-1}\|^2\,\|\Delta_t\|}{1-\|\Gamma_i^{-1}\|\,\|\Delta_t\|}.
		\]
	\end{enumerate}
\end{lemma}

\begin{proof}
	For 1., recall the Neumann series criterion: if $\|A\|<1$ for a bounded linear operator $A$, then
	\[
	(I+A)^{-1}=\sum_{k=0}^\infty (-A)^k,
	\qquad
	\|(I+A)^{-1}\|\le\frac{1}{1-\|A\|}.
	\]
	Apply this with $A=\Gamma_i^{-1}\Delta_t$, for which $\|A\|\le \|\Gamma_i^{-1}\|\,\|\Delta_t\|<1$ by assumption.

	For 2., observe that
	\[
	\hat\Gamma_{i,t}
	=\Gamma_i+\Delta_t
	=\Gamma_i(I+\Gamma_i^{-1}\Delta_t),
	\]
	so that $\hat\Gamma_{i,t}$ is a product of two invertible operators, and hence is invertible with inverse
	\[
	\hat\Gamma_{i,t}^{-1}
	=(I+\Gamma_i^{-1}\Delta_t)^{-1}\Gamma_i^{-1}.
	\]

	For 3., start from
	\[
	\hat\Gamma_{i,t}^{-1}-\Gamma_i^{-1}
	=(I+\Gamma_i^{-1}\Delta_t)^{-1}\Gamma_i^{-1}-\Gamma_i^{-1}
	=\bigl[(I+\Gamma_i^{-1}\Delta_t)^{-1}-I\bigr]\Gamma_i^{-1}.
	\]
	Using the identity
	\[
	(I+\Gamma_i^{-1}\Delta_t)^{-1}-I
	=-(I+\Gamma_i^{-1}\Delta_t)^{-1}\Gamma_i^{-1}\Delta_t,
	\]
	we obtain
	\[
	\hat\Gamma_{i,t}^{-1}-\Gamma_i^{-1}
	=-(I+\Gamma_i^{-1}\Delta_t)^{-1}\Gamma_i^{-1}\Delta_t\Gamma_i^{-1}.
	\]
	Taking operator norms and using submultiplicativity,
	\[
	\|\hat\Gamma_{i,t}^{-1}-\Gamma_i^{-1}\|
	\le \big\|(I+\Gamma_i^{-1}\Delta_t)^{-1}\big\|\,
	\|\Gamma_i^{-1}\|^2\,\|\Delta_t\|
	\le \frac{\|\Gamma_i^{-1}\|^2\,\|\Delta_t\|}{1-\|\Gamma_i^{-1}\|\,\|\Delta_t\|},
	\]
	where the last inequality follows from part~(1).
\end{proof}

\begin{lemma}[A sufficient condition]\label{lem:sufficient-mu}
	If $\|\Delta_t\|<\mu$, then $\|\Gamma_i^{-1}\|\,\|\Delta_t\|<1$ and all conclusions of Lemma~\ref{lem:small-perturb} hold.
	Moreover, if $\|\Delta_t\|\le \mu/2$, then
	\[
	\lambda_{\min}(\hat\Gamma_{i,t})\ \ge\ \mu-\|\Delta_t\|\ \ge\ \mu/2,\qquad
	\|\hat\Gamma_{i,t}^{-1}\|\ \le\ 2/\mu,
	\]
	and
	\[
	\|\hat\Gamma_{i,t}^{-1}-\Gamma_i^{-1}\|\ \le\ \frac{2}{\mu^2}\,\|\Delta_t\|.
	\]
\end{lemma}

\begin{proof}
	By Assumption~\ref{ass:identifiability} we have $\lambda_{\min}(\Gamma_i)\ge\mu>0$, hence $\|\Gamma_i^{-1}\|=1/\lambda_{\min}(\Gamma_i)\le 1/\mu$.
	If $\|\Delta_t\|<\mu$, then
	\[
	\|\Gamma_i^{-1}\|\,\|\Delta_t\|
	\le \mu^{-1}\|\Delta_t\|
	<1,
	\]
	so Lemma~\ref{lem:small-perturb} applies.

	For the eigenvalue bound, Weyl's inequality for Hermitian matrices yields
	\[
	\lambda_{\min}(\hat\Gamma_{i,t})
	=\lambda_{\min}(\Gamma_i+\Delta_t)
	\ge \lambda_{\min}(\Gamma_i)-\|\Delta_t\|
	\ge \mu-\|\Delta_t\|.
	\]
	If $\|\Delta_t\|\le\mu/2$, then
	\(
	\lambda_{\min}(\hat\Gamma_{i,t})\ge\mu/2,
	\)
	so
	\(
	\|\hat\Gamma_{i,t}^{-1}\|=1/\lambda_{\min}(\hat\Gamma_{i,t})\le 2/\mu.
	\)

	Finally, apply Lemma~\ref{lem:small-perturb} together with $\|\Gamma_i^{-1}\|\le 1/\mu$:
	\[
	\|\hat\Gamma_{i,t}^{-1}-\Gamma_i^{-1}\|
	\le \frac{\|\Gamma_i^{-1}\|^2\,\|\Delta_t\|}{1-\|\Gamma_i^{-1}\|\,\|\Delta_t\|}
	\le \frac{\mu^{-2}\,\|\Delta_t\|}{1-\mu^{-1}\|\Delta_t\|}.
	\]
	If $\|\Delta_t\|\le\mu/2$, then $\mu^{-1}\|\Delta_t\|\le 1/2$, so
	\(
	1-\mu^{-1}\|\Delta_t\|\ge 1/2
	\)
	and hence
	\(
	\|\hat\Gamma_{i,t}^{-1}-\Gamma_i^{-1}\|\le 2\mu^{-2}\|\Delta_t\|.
	\)
\end{proof}

Therefore, on a set of probability at least \(1-2\delta_\Gamma-\eta_\Gamma\), the small-perturbation condition holds and Lemma~\ref{lem:small-perturb} applies.

The concentration bound for $\hat{\beta}_{i,t}-\beta_i$ holds as follows:
\begin{theorem}[Finite-sample tail bound for $\hat{\beta}_{i,t}-\beta_i$]\label{thm:beta-tail-precise}
Suppose Assumptions~\ref{ass:avg:explore:rt},\ref{ass:identifiability},\ref{ass:Lyapunov},and~\ref{ass:moment:more}, and Equation~\eqref{ass:ignorability} hold.
	Fix tail parameters \(\delta_\Gamma,\eta_\Gamma,\delta_m,\eta_m\in(0,1/2)\) and define
	\[
	B_\Gamma(t)=B_\Gamma(t;\delta_\Gamma,\eta_\Gamma)\ \text{ from \eqref{eq:Delta-bound_supp}},\qquad
	B_m(t):=B_m(t;\delta_m,\eta_m)\ \text{ from \eqref{eq:m-bound-supp}}.
	\] 
	Then, on the event \(B_\Gamma(t)<\mu\), with probability at least \(1-(2\delta_\Gamma+\eta_\Gamma)-(2\delta_m+\eta_m)\),
	\begin{equation}\label{eq:beta-tail-general}
			\|\hat{\beta}_{i,t}-\beta_i\|\ \le\
			\frac{1}{\mu}\,B_m(t)\ +\ \frac{ B_\Gamma(t) }{ \mu^2 - \mu B_\Gamma(t) }\,\big(\,\|m_i\|+B_m(t)\,\big).
	\end{equation}
	In particular, if \(B_\Gamma(t)\le \mu/2\), then
	\begin{equation}\label{eq:beta-tail-simplified}
			\|\hat{\beta}_{i,t}-\beta_i\|\ \le\
			\frac{1}{\mu}\,B_m(t)\ +\ \frac{2}{\mu^2}\,B_\Gamma(t)\,\big(\,\\|m_i\|+B_m(t)\,\big),
	\end{equation}
	with the same probability.
\end{theorem}

\begin{proof}[Proof of Theorem~\ref{thm:beta-tail-precise}]
We start from the exact algebraic decomposition
\[
\hat{\beta}_i^{(t)}-\beta_i
\;=\;\Gamma_i^{-1}(\hat m_{i,t}-m_i)\;+\;(\hat\Gamma_{i,t}^{-1}-\Gamma_i^{-1})\,\hat m_{i,t}.
\]
Taking norms and using \(\|\Gamma_i^{-1}\|\le 1/\mu\) gives
\begin{equation}\label{eq:beta-basic}
	\|\hat{\beta}_i^{(t)}-\beta_i\|
	\ \le\ \frac{1}{\mu}\,\|\hat m_{i,t}-m_i\|
	\ +\ \|\hat\Gamma_{i,t}^{-1}-\Gamma_i^{-1}\|\;\big(\,\|m_i\|+\|\hat m_{i,t}-m_i\|\,\big).
\end{equation}

	Intersect the ``good'' events underlying \eqref{eq:Delta-bound_supp} and \eqref{eq:m-bound-supp}; by a union bound the intersection has probability at least \(1-(2\delta_\Gamma+\eta_\Gamma)-(2\delta_m+\eta_m)\).
	On this intersection we have \(\|\Delta_t\|\le B_\Gamma(t)\) and \(\|\hat m_{i,t}-m_i\|\le B_m(t)\).
	If \(B_\Gamma(t)<\mu\), then Lemma~\ref{lem:small-perturb} gives
	\[
	\|\hat\Gamma_{i,t}^{-1}-\Gamma_i^{-1}\|
	\ \le\ \frac{ B_\Gamma(t) }{ \mu^2 - \mu B_\Gamma(t) }.
	\]
	Insert these bounds into \eqref{eq:beta-basic} and use
	\(
	\|\hat m_{i,t}\|\le \|m_i\|+\|\hat m_{i,t}-m_i\|
	\)
	to obtain \eqref{eq:beta-tail-general}.
	If \(B_\Gamma(t)\le \mu/2\), then \(1-\mu^{-1}B_\Gamma(t)\ge 1/2\) and thus
	\(
	\|\hat\Gamma_{i,t}^{-1}-\Gamma_i^{-1}\|\le 2B_\Gamma(t)/\mu^2,
	\)
	which yields \eqref{eq:beta-tail-simplified} upon substitution into \eqref{eq:beta-basic}.
\end{proof}

A convenient one-shot choice, which will later be used in the regret analysis, is
\[
\delta_\Gamma=\delta_m=\frac{\delta}{4},
\qquad
\eta_\Gamma=\eta_m=\frac{\delta}{4},
\]
for a prescribed confidence parameter \(\delta\in(0,1/2)\).
Then Theorem~\ref{thm:beta-tail-precise} implies that, provided
\(B_\Gamma(t;\delta/4,\delta/4)\le \mu/2\),
\begin{align*}
\mathbb{P}\!\left(
\|\hat{\beta}_{i,t}-\beta_i\|\ \le\
\frac{1}{\mu}\,B_m\!\Big(t;\tfrac{\delta}{4},\tfrac{\delta}{4}\Big)
+\frac{2}{\mu^2}\,B_\Gamma\!\Big(t;\tfrac{\delta}{4},\tfrac{\delta}{4}\Big)\,
\Big(\|m_i\|+B_m\!\Big(t;\tfrac{\delta}{4},\tfrac{\delta}{4}\Big)\Big)
\right)\ \ge\ 1-\delta.
\end{align*}

For fixed dimension \(d\),
\begin{align*}
B_\Gamma(t;\delta/4,\delta/4),\ B_m(t;\delta/4,\delta/4)
&\;=\;
O\!\Bigg(
\frac{\log(2d/\delta)}{t}
\;+\;
\sqrt{\Big(r_t+\frac{1}{t}\Big)\,\log\frac{2d}{\delta}}
\;+\;
\sqrt{\frac{r_t}{\delta}}
\Bigg),
\end{align*}
where the last term explicitly exhibits the factor \(1/\sqrt{\delta}\) coming from
\(\sqrt{M_4/\delta_\Gamma}\sqrt{r_t}\) and \(\sqrt{\tilde C_3/\delta_m}\sqrt{r_t}\).
Consequently, for this choice of tail parameters,
\begin{align}
\label{eq:beta:finite:sample2}
\|\hat{\beta}_{i,t}-\beta_i\|
&=
O_{\mathbb{P}}\!\Bigg(
\frac{1}{\mu}
\Bigg\{
\frac{\log(2d/\delta)}{t}
+
\sqrt{\Big(r_t+\frac{1}{t}\Big)\,\log\frac{2d}{\delta}}
+
\sqrt{\frac{r_t}{\delta}}
\Bigg\}
\Bigg)
\\ \nonumber
&\qquad
+
O_{\mathbb{P}}\!\Bigg(
\frac{\|m_i\|}{\mu^2}
\Bigg\{
\frac{\log(2d/\delta)}{t}
+
\sqrt{\Big(r_t+\frac{1}{t}\Big)\,\log\frac{2d}{\delta}}
+
\sqrt{\frac{r_t}{\delta}}
\Bigg\}
\Bigg),
\end{align}
and the quadratic remainder term \(\frac{2}{\mu^2} B_\Gamma(t) B_m(t)\) is of smaller order.
In particular, the dependence on the confidence parameter \(\delta\) is fully explicit through
polylogarithmic factors in \(\log(1/\delta)\) and the factor \(\delta^{-1/2}\) multiplying
\(\sqrt{r_t}\).
Under the exploration condition \(r_t=o(1)\), these bounds converge to zero as \(t\to\infty\) for each fixed \(\delta\), and Assumptions~\ref{ass:moment:more} and~\ref{ass:identifiability} ensure that \(m_i\) is finite and \(\mu>0\), so no additional integrability or identifiability conditions are required.

\section{Proof and Auxiliary Results of Section \ref{sec:rkhs-clt-setup}}\label{secApp: proof_rkhs_clt}
First we put some essential notations used in  Section \ref{sec: pointwise_nonp_CLT} in Table ~\ref{tab:notation_rkhs}.
\begin{table}[h!]
\centering
\caption{Notation for RKHS estimation and inference of the link function in Section~\ref{sec:rkhs-clt-setup}.}
\label{tab:notation_rkhs}
\begin{tabular}{ll}
\toprule
\textbf{Symbol} & \textbf{Description} \\
\midrule

$U_{s,i}$ & Projected covariate: $U_{s,i}:=X_s^\top\beta_i\in\U$ \\

$p_{s,i}(u)$ & Propensity: $\Prob(\hat a_s=i\mid \F_{s-1},\,U_{s,i}=u)\in(0,1]$ \\
$\tilde r_t$ & Average exploration: $\tilde r_t:=t^{-2}\sum_{s=1}^t \E[\,1/p_{s,i}(U_{s,i})\,]$ \\[0.2em]

$K$ & Reproducing kernel on $\U\times\U$ (bounded) \\
$\Hk$ & RKHS induced by $K$ on $\U$ \\
$\langle\cdot,\cdot\rangle_K,\ \|\cdot\|_K$ & Inner product and norm in $\Hk$ \\
$K_u$ & Representer $K_u:=K(\cdot,u)\in\Hk$ \\
$\kappa$ & Kernel bound: $\sup_{u\in\U}\|K_u\|_K\le \kappa$ \\

$\Sigma_i$ & Population IPW covariance operator on $\Hk$ \\
$h_i$ & Population IPW moment in $\Hk$ \\
$\hat\Sigma_{i,t}$ & Empirical IPW covariance operator at time $t$ \\
$\hat h_{i,t}$ & Empirical IPW moment at time $t$ \\

$\lambda_t$ & Ridge regularization (time-dependent, $\lambda_t>0$) \\
$A_i(\lambda)$ & Population ridge operator: $A_i(\lambda):=\Sigma_i+\lambda I$ \\
$\hat A_{t,i}(\lambda)$ & Empirical ridge operator: $\hat A_{t,i}(\lambda):=\hat\Sigma_{i,t}+\lambda I$ \\
$f_i$ & True link function (arm $i$), assumed $f_i\in\Hk$ \\
$f_i^{\lambda}$ & Population ridge target: $f_i^{\lambda}:=A_i(\lambda)^{-1}h_i$ \\
$\hat f_{i,t}$ & IPW-KRR estimator: $\hat f_{i,t}:=(\hat\Sigma_{i,t}+\lambda_t I)^{-1}\hat h_{i,t}$ \\[0.2em]

$\gamma\in(0,1/2)$ & Scaling exponent for inference on $f_i$ \\
$\G\subset\Hk$ & Class of inference directions (e.g.\ $\{K_u:u\in\U\}$) \\
$G=(g_1,\dots,g_m)$ & Finite subset of directions with $g_\ell\in\G$ \\[0.2em]


$\Delta_{s,i}^\lambda$ & RKHS innovation: $\Delta_{s,i}^\lambda:=Y_sK_{U_{s,i}}-(K_{U_{s,i}}\otimes K_{U_{s,i}})f_i^\lambda$ \\
$\xi_{t,s,i}$ & RKHS martingale increment for $t^\gamma L_{t,i}$ \\
$V_t(\lambda)$ & Predictable quadratic-variation (bracket) operator (Eq.~\eqref{eq:Vt-correct}) \\[0.2em]

$\mathsf Q_\lambda(G)$ & Directional covariance matrix used in Assumption~\ref{ass:direc:nondegeneracy} \\
$\mathcal E_{t,i}(G)$ & Studentized vector of projections along $G$ \\
$Z_{t,s}(G)$ & $\R^m$-valued martingale difference increment in $\mathcal E_{t,i}(G)$ \\[0.2em]

$D_t(g)$ & Directional s.d.: $D_t(g):=\sqrt{\langle g,V_t(\lambda_t)g\rangle_K}$ \\
$T_{t,i}(g)$ & Studentized statistic: $T_{t,i}(g):=\langle g,t^\gamma(\hat f_{i,t}-f_i)\rangle_K/D_t(g)$ \\[0.2em]

$(s,q,\rho)$ & Smoothness/bias indices: $f_i=\Sigma_i^{\,s}w$, $g=\Sigma_i^{\,q}v$, and $\rho:=\min\{s+q,1\}$ \\
$r_\lambda(\mu)$ & Filter: $r_\lambda(\mu):=\lambda/(\mu+\lambda)$ \\
$(\mu_j,e_j)$ & Eigenpairs of $\Sigma_i$: $\Sigma_i e_j=\mu_j e_j$ \\[0.2em]

$D_t(x)^2$ & Pointwise variance: $D_t(x)^2:=\langle K_x,V_t(\lambda_t)K_x\rangle_K$ \\
$\widehat V_t^{(i)}(\lambda_t)$ & Feasible covariance estimator used in implementation (Eq.~\eqref{eq:Vhat-np}) \\
$\widehat D_t(x)^2$ & Feasible pointwise variance: $\widehat D_t(x)^2:=\langle K_x,\widehat V_t^{(i)}(\lambda_t)K_x\rangle_K$ \\
\bottomrule
\end{tabular}
\end{table}
Recall the notations: 
Let $\Delta \Sigma_{t,i} := \hat \Sigma_{i,t} - \Sigma_i$ and $\Delta h_{t,i} := \hat h_{i,t} - h_i$. Using the identity $A_i(\lambda) f_i^\lambda = h_i$ and the resolvent identity, we write:
\begin{align*}
	\hat f_{i,t}-f_i^\lambda
	&= A_i(\lambda)^{-1}\big(\Delta h_{t,i}-\Delta\Sigma_{t,i} f_i^\lambda\big)
	\ +\ \big(\hat A_{t,i} (\lambda)^{-1}-A_i(\lambda)^{-1}\big)\big(\Delta h_{t,i}-\Delta\Sigma_{t,i} f_i^\lambda\big)
	\notag\\
	&=: L_{t,i}\ +\ R_{t,i},
\end{align*}
with $\hat A_{t,i} (\lambda)^{-1}-A_i(\lambda)^{-1}=-A_i(\lambda)^{-1}\Delta\Sigma_{t,i}\,\hat A_{t,i} (\lambda)^{-1}$.

Define the raw $\Hk$-valued increment
\[
\phi_{s,i}\ :=\ \frac{1\{\hat a_s=i\}}{p_{s,i}(U_{s,i})}\,\big(Y_s-f_i^\lambda(U_{s,i})\big)\,K_{U_{s,i}}
\ =\ \frac{1\{\hat a_s=i\}}{p_{s,i}(U_{s,i})}\,\Delta_{s,i}^\lambda,
\]
and its centered version
\[
\widetilde\phi_{s,i}\ :=\ \phi_{s,i}\;-\;\E[\phi_{s,i}\mid \F_{s-1}]
\ =\ \frac{1\{\hat a_s=i\}}{p_{s,i}(U_{s,i})}\,\Delta_{s,i}^\lambda\;-\;\lambda\,f_i^\lambda.
\]
The sequence $(\widetilde\phi_{s,i})$ forms a martingale difference sequence in $\mathcal{H}_k$ adapted to $(\F_s)$, which drives the stochastic fluctuation in each directional projection $\langle g, \hat f_{i,t} - f_i^\lambda \rangle_K$. The residual term $R_{t,i}$ vanishes under mild operator convergence conditions and plays a secondary role in the CLT.

\begin{lemma}
\label{lem:rt:rt2}
For every $t\ge 1$,
\[
\tilde{r}_t \;\le\; r_t.
\]
Consequently, Assumption~\ref{ass:avg:explore:rt} implies the exploration requirement of
Assumption~\ref{ass:avg:explore:rt2}.
\end{lemma}

\begin{proof}[Proof of Lemma~\ref{lem:rt:rt2}]
Fix $s\in\{1,\dots,t\}$.  By definition,
\[
p_s^\star 
    = \inf_{w\in\mathbb{R}^d}\;\inf_{i\in[L]} p_{s,i}(w)
    \;\le\; p_{s,i}(U_{s,i})
    \qquad \text{almost surely}.
\]
Since the map $x\mapsto 1/x$ is decreasing on $(0,\infty)$, this implies
\[
\frac{1}{p_{s,i}(U_{s,i})}
    \;\le\; \frac{1}{p_s^\star}
    \qquad \text{almost surely}.
\]
Taking expectations and summing over $s=1,\dots,t$ yields
\[
\sum_{s=1}^t \mathbb{E}\!\left[\frac{1}{p_{s,i}(U_{s,i})}\right]
    \;\le\; \sum_{s=1}^t \frac{1}{p_s^\star}.
\]
Dividing by $t^2$ proves
\[
\tilde{r}_t 
    = \frac{1}{t^2}\sum_{s=1}^t 
        \mathbb{E}\!\left[\frac{1}{p_{s,i}(U_{s,i})}\right]
    \;\le\;
    \frac{1}{t^2}\sum_{s=1}^t \frac{1}{p_s^\star}
    = r_t.
\]
Thus, any upper bound or rate condition imposed on 
$r_t$ automatically implies the corresponding condition on 
$\tilde{r}_t$, establishing the claim.
\end{proof}

\begin{proof}[Proof of Lemma~\ref{lem:L2-MDS}]
    Expand the square and use $\E\langle D_s,D_r\rangle_H = 0$ for $r<s$, since
	$\E[D_s\mid \F_{s-1}]=0$ and $D_r$ is $\F_{s-1}$-measurable; hence cross-terms vanish.
\end{proof}

\begin{lemma}[Orthogonality of martingale differences]\label{lem:L2-MDS}
	Let $(H,\langle\cdot,\cdot\rangle_H)$ be a real Hilbert space and $\{D_s,\F_s\}$ an $H$-valued MDS with $\E\|D_s\|_H^2<\infty$ for all $s$. Then
	\[
	\E\Big\| \sum_{s=1}^t D_s \Big\|_H^2  =  \sum_{s=1}^t \E\|D_s\|_H^2.
	\]
\end{lemma}

\begin{proof}
Expand the square and use $\E\langle D_s,D_r\rangle_H = 0$ for $r<s$, since
	$\E[D_s\mid \F_{s-1}]=0$ and $D_r$ is $\F_{s-1}$-measurable; hence cross-terms vanish.    
\end{proof}
	
\begin{lemma}\label{lem:Lt-equals-St}
	With
	\[
	L_{t,i} \;:=\; A_i(\lambda)^{-1}\big(\Delta h_{t,i} - \Delta\Sigma_{t,i}\,f_i^\lambda\big),
	\qquad
	A_i(\lambda):=\Sigma_i+\lambda I,
	\]
	and
	\[
	\Delta h_{t,i}  =  \frac{1}{t}\sum_{s=1}^t \frac{1\{\hat a_s=i\}}{p_{s,i}(U_{s,i})}\,Y_s\,K_{U_{s,i}} - h_i,\qquad
	\Delta\Sigma_{t,i}  =  \frac{1}{t}\sum_{s=1}^t \frac{1\{\hat a_s=i\}}{p_{s,i}(U_{s,i})}\,K_{U_{s,i}}\otimes K_{U_{s,i}} - \Sigma_i,
	\]
	define for $\gamma\in(0,1/2)$
	\[
	\Delta_{s,i}^\lambda \;:=\; Y_s\,K_{U_{s,i}} - (K_{U_{s,i}}\otimes K_{U_{s,i}})\,f_i^\lambda \in \Hk,
	\]
	and
	\[
	\xi_{t,s,i} \;:=\; t^{\gamma-1}\,A_i(\lambda)^{-1}\,\bigg(\frac{1\{\hat a_s=i\}}{p_{s,i}(U_{s,i})}\,\Delta_{s,i}^\lambda \;-\; \lambda f_i^\lambda\bigg),
	\qquad
	E_{t,i} \;:=\; \sum_{s=1}^t \xi_{t,s,i}.
	\]
	Then
	\[
	t^\gamma\,L_{t,i}  =  \sum_{s=1}^t \xi_{t,s,i}  =  E_{t,i}.
	\]
\end{lemma}

\begin{proof}[Proof of Lemma~\ref{lem:Lt-equals-St}]
	Start by expanding $L_t$:
	\begin{align*}
		L_t
		&= A(\lambda)^{-1}\left\{\frac{1}{t}\sum_{s=1}^t \frac{1\{\hat a_s=i\}}{p_{s,i}(U_{s,i})}\,Y_s\,K_{U_{s,i}} - h_i
		\;-\;\Big(\frac{1}{t}\sum_{s=1}^t \frac{1\{\hat a_s=i\}}{p_{s,i}(U_{s,i})}\,K_{U_{s,i}}\otimes K_{U_{s,i}} - \Sigma_i\Big) f_i^\lambda\right\}\\
		&= A(\lambda)^{-1}\left\{\frac{1}{t}\sum_{s=1}^t \frac{1\{\hat a_s=i\}}{p_{s,i}(U_{s,i})}\,\Big(Y_s\,K_{U_{s,i}}-(K_{U_{s,i}}\otimes K_{U_{s,i}})\,f_i^\lambda\Big)
		\;-\; \big(h_i-\Sigma_i f_i^\lambda\big)\right\}.
	\end{align*}
	Since $A(\lambda)f_i^\lambda=h_i$, we have $h_i-\Sigma_i f_i^\lambda=\lambda f_i^\lambda$. Using the definition of $\Delta_s^\lambda$,
	\[
	L_t  =  A(\lambda)^{-1}\left\{\frac{1}{t}\sum_{s=1}^t \frac{1\{\hat a_s=i\}}{p_{s,i}(U_{s,i})}\,\Delta_s^\lambda
	\;-\; \lambda f_i^\lambda\right\}.
	\]
	Multiply both sides by $t^\gamma$:
	\[
	t^\gamma L_t
	 =  t^{\gamma-1}\,A(\lambda)^{-1}\,\sum_{s=1}^t
	\left(\frac{1\{\hat a_s=i\}}{p_{s,i}(U_{s,i})}\,\Delta_s^\lambda \;-\; \lambda f_i^\lambda\right).
	\]
	By the definition of $\xi_{t,s}$,
	\[
	t^\gamma L_t  =  \sum_{s=1}^t \xi_{t,s}  =  E_t.
	\]
	This completes the proof.
\end{proof}

\begin{proposition}[Positivity and scale of $V_t(\lambda)$]\label{prop:Vt-scale-correct2}
	Under Assumptions~\ref{ass:cond:noise},\ref{ass:avg:explore:rt2}, and~\ref{ass:direc:nondegeneracy},  $V_t(\lambda)\succeq 0$ is trace-class and
	\[
	\|V_t(\lambda)\|\ \le\ \|A_i(\lambda)^{-1}\|^2\,t^{2\gamma-2}\sum_{s=1}^t \E\!\Big[\frac{1}{p_{s,i}(U_{s,i})}\Big]
	\ =\ \|A_i(\lambda)^{-1}\|^2\,t^{2\gamma}\tilde{r}_t\ \xrightarrow{}\ 0,
	\]
	by \eqref{eq:exploration}. Moreover, for any finite $G=(g_1,\dots,g_m)\in\mathcal{G}^m$,
	\begin{equation}\label{eq:min-eig-VtG-correct2}
		\lambda_{\min}\!\Big( G^\star V_t(\lambda) G \Big)
		\ \ge\ c\, t^{2\gamma-2}\sum_{s=1}^t \E\!\left[\frac{1}{p_{s,i}(U_{s,i})}\right]
		\times \lambda_{\min}\!\big(\mathsf Q_\lambda(G)\big),
	\end{equation}
	for a constant $c>0$ depending only on $(\Hk,k)$ and the noise lower bound in Assumption~\ref{ass:direc:nondegeneracy}.
\end{proposition}

\begin{proof}[Proof of proposition~\ref{prop:Vt-scale-correct2}]
	Positivity: each summand in \eqref{eq:Vt-correct} is a conditional covariance operator conjugated by $A(\lambda)^{-1}$, hence PSD; the subtraction by $\lambda^2 f_i^\lambda\otimes f_i^\lambda$ is exactly the mean-outer-product term and is required for covariance. The trace-class and operator-norm bound follow from Assumptions~\ref{ass:outcome:moments}, the boundedness of the RKHS kernel, and the fact that $\|A(\lambda)^{-1}\|\le \lambda^{-1}$:
	\[
	\|V_t(\lambda)\| \ \le\ \|A(\lambda)^{-1}\|^2\,t^{2\gamma-2}\sum_{s=1}^t
	\E\!\Big[\frac{1}{p_{s,i}(U_{s,i})}\,\|\Delta_s^\lambda\|_K^2\ \Bigm|\ \F_{s-1}\Big]
	\ \lesssim\ \|A(\lambda)^{-1}\|^2\,t^{2\gamma-2}\sum_{s=1}^t \E\tfrac{1}{p_{s,i}},
	\]
	using $\|\Delta_s^\lambda\|_K\le \kappa\,|Y_s-f_i^\lambda(U_{s,i})|$ and $\sup_s \E|Y_s|^{2+\delta}<\infty$. Under our standing assumptions, $|Y_s-f_i^\lambda(U_{s,i})|$ admits a uniform (in $\lambda$) envelope of the form
\[
|Y_s-f_i^\lambda(U_{s,i})|
\;\le\;
|Y_s|\;+\;\kappa\,\|f_i^\lambda\|_{\Hk}
\;\le\;
|Y_s|\;+\;\kappa\,\|f_i\|_{\Hk},
\]
because the Tikhonov solution is a contraction in RKHS norm:
\[
\|f_i^\lambda\|_{\Hk}\ \le\ \|f_i\|_{\Hk}\qquad\text{for all }\lambda>0.
\]
Hence the $\lambda$-dependence does not cause a blow-up in the basic envelope; we only use the $(2+\delta)$-moment envelope for $Y_s$ and bounded kernel ($\sup_x K(x,x)\le\kappa^2<\infty$).

	For \eqref{eq:min-eig-VtG-correct2}, fix $a\in\real^k$, put $g_a:=\sum_\ell a_\ell g_\ell$, and use the lower noise bound in~\ref{ass:cond:noise}:
	\[
	\langle g_a, V_t(\lambda) g_a\rangle_K
	\ \ge\ t^{2\gamma-2}\,\langle A(\lambda)^{-1}g_a,\ \Big(\sum_{s=1}^t \E[\tfrac{1}{p_{s,i}}\,\E[\Delta_s^\lambda\otimes\Delta_s^\lambda\mid \F_{s-1},U_{s,i}]\mid \F_{s-1}]\Big) A(\lambda)^{-1} g_a\rangle_K
	\]
	\[
	\ \ge\ c\,t^{2\gamma-2}\,\Big(\sum_{s=1}^t \E\tfrac{1}{p_{s,i}}\Big)\ \langle A(\lambda)^{-1}g_a,\ \E[K_{U_{s,i}}\otimes K_{U_{s,i}}]\ A(\lambda)^{-1}g_a\rangle_K,
	\]
	which implies \eqref{eq:min-eig-VtG-correct2} after taking the infimum over $\|a\|_2=1$. With $g_a:=\sum_\ell a_\ell g_\ell$,
\[
\langle g_a, V_t(\lambda) g_a\rangle
\ \ge\
c\,t^{2\gamma-2}\!\!\sum_{s=1}^t \E\!\big[\tfrac{1}{p_{s,i}(U_{s,i})}\big]\;
\big\langle A(\lambda)^{-1}g_a,\ \E[K\otimes K]\ A(\lambda)^{-1}g_a\big\rangle .
\]
By Assumption~\ref{ass:direc:nondegeneracy}, the Gram form on the span of \(A(\lambda)^{-1}G\) is strictly positive, i.e.
\(\langle A(\lambda)^{-1}g_a,\ \E[K\otimes K]\ A(\lambda)^{-1}g_a\rangle
\ge \lambda_{\min}(Q_\lambda(G))\,\|a\|_2^2\).
Taking the infimum over \(\|a\|_2=1\) yields
\[
\lambda_{\min}\big(G^\star V_t(\lambda) G\big)\ \ge\ c'\,t^{2\gamma}\tilde{r}_t\,\lambda_{\min}\!\big(Q_\lambda(G)\big)\ >\ 0.
\]
\end{proof}

\begin{proposition}[Remainder bound]\label{prop:dev}
	Let $\lambda>0$ be fixed and define, for each arm $i=1,\dots,L$,
	\begin{align*}
	    A_i(\lambda):=\Sigma_i+\lambda I,\ \hat A_{t,i} (\lambda):=\hat\Sigma_{i,t}+\lambda I, \ \Delta\Sigma_{t,i}:=\hat\Sigma_{i,t}-\Sigma_i,\ \Delta h_{t,i}:=\hat h_{i,t}-h_i,
	\end{align*}
	and
	\[
	R_{t,i}\ :=\ \big(\hat A_{t,i} (\lambda)^{-1}-A_i(\lambda)^{-1}\big)\,\big(\Delta h_{t,i}-\Delta\Sigma_{t,i}\,f_i^\lambda\big).
	\]
	Under the IPS identities and moment envelope conditions in Equation~\eqref{ass:ignorability} and Assumptions~\ref{ass:outcome:moments}, 
	and that $\gamma\in(0,1/2]$ with $\tilde{r}_t:=t^{-2}\sum_{s=1}^t \E[1/p_{s,i}(U_{s,i})]=o(1)$ satisfies
	$t^{2\gamma}\tilde{r}_t\to 0$. Then
	\[
	t^\gamma\,\|R_{t,i}\|\ \xrightarrow{\ \Prob\ }\ 0\,.
	\]
\end{proposition}

\begin{proof}[Proof of Proposition~\ref{prop:dev}]
	Using the identity
	\[
	\hat A_t(\lambda)^{-1}-A(\lambda)^{-1}
	\ =\ -\,A(\lambda)^{-1}\big(\hat A_t(\lambda)-A(\lambda)\big)\,\hat A_t(\lambda)^{-1}
	\ =\ -\,A(\lambda)^{-1}\,\Delta\Sigma_t\,\hat A_t(\lambda)^{-1}.
	\]
	Therefore
	\begin{equation}\label{eq:R-decomp}
		R_t\ =\ -\,A(\lambda)^{-1}\,\Delta\Sigma_t\,\hat A_t(\lambda)^{-1}\,
		\Big(\Delta h_t-\Delta\Sigma_t\,f_i^\lambda\Big).
	\end{equation}

	Both $\Sigma_i$ and $\hat\Sigma_{i,t}$ are positive semidefinite (PSD) operators on $\Hk$,
	so for any $f\in\Hk$,
	\[
	\langle f,\hat A_t(\lambda)f\rangle_K
	=\langle f,\hat\Sigma_{i,t}f\rangle_K+\lambda\|f\|_K^2
	\ \ge\ \lambda\|f\|_K^2,
	\]
	and similarly for $A(\lambda)$. Hence
	\begin{equation}\label{eq:resolvent-bounds}
		\|\hat A_t(\lambda)^{-1}\|\ \le\ \lambda^{-1},\qquad
		\|A(\lambda)^{-1}\|\ \le\ \lambda^{-1}\qquad\text{for all }t.
	\end{equation}
	These bounds are deterministic (i.e.\ hold almost surely).

	Combining \eqref{eq:R-decomp} with \eqref{eq:resolvent-bounds} and using submultiplicativity
	of the operator norm,
	\begin{align}
		\|R_t\|
		&\le \|A(\lambda)^{-1}\|\,\|\Delta\Sigma_t\|\,\|\hat A_t(\lambda)^{-1}\|\,
		\Big(\|\Delta h_t\|+\|\Delta\Sigma_t\|\,\|f_i^\lambda\|\Big)\notag\\
		&\le \lambda^{-2}\,\|\Delta\Sigma_t\|\,
		\Big(\|\Delta h_t\|+\|\Delta\Sigma_t\|\,\|f_i^\lambda\|\Big).
		\label{eq:R-ineq}
	\end{align}

We assume the kernel is bounded:
\begin{equation}\label{eq:bounded-kernel}
	\sup_{u\in\U} \|K_u\|_K^2  =  \sup_{u\in\U} K(u,u) \;\le\; \kappa^2 \;<\;\infty.
\end{equation}

Fix an arm $i\in\{1,\dots,L\}$ and a filtration $(\F_s)_{s\ge 0}$.
For each time $s\in\N$ let $U_{s,i}\in\U$ be the (possibly adaptive) context, and $Y_s\in\real$ be the outcome.
Let the propensity be
\[
p_{s,i}(u)\ :=\ \Pr(\hat a_s=i \mid \F_{s-1},\, U_{s,i}=u)\ \in (0,1],
\]
and write $p_{s,i}:=p_{s,i}(U_{s,i})$ for brevity. Define the population (Hilbert--Schmidt) covariance operator and the population regression representer by
\[
\Sigma_i \ :=\ \E\big[\,K_{U_{s,i}}\otimes K_{U_{s,i}}\,\big]\ \in \mathcal S_2(\Hk), 
\qquad
h_i \ :=\ \E\big[\,Y_s\,K_{U_{s,i}}\,\big]\ \in \Hk,
\]
where, for $x,y\in\Hk$, the rank-one operator $x\otimes y:\Hk\to\Hk$ is
\[
(x\otimes y)f \ :=\ \langle f,y\rangle_K\,x, \qquad f\in\Hk.
\]
We use that $\|x\otimes y\|_{\mathrm{HS}}=\|x\|_K\|y\|_K$ and $\|x\otimes y\|_{\mathrm{op}}\le \|x\otimes y\|_{\mathrm{HS}}$. Furthermore, let the space of Hilbert-Schmidt operators on $\Hk$ be defined as
\[
  \mathcal{S}^2(\Hk)
  \ :=\ \Bigl\{\, T:\Hk\to\Hk \;\text{linear, bounded} \;\Big|\;
    \|T\|_{\mathrm{HS}}^2 := \sum_{j=1}^\infty \|T e_j\|_K^2 < \infty \,\Bigr\},
\]
where $\{e_j\}_{j\ge 1}$ is any orthonormal basis of $\Hk$.

We study the deviations
\[
\Delta\Sigma_t \ :=\ \hat\Sigma_{i,t}-\Sigma_i \ \in \mathcal S_2(\Hk),
\qquad
\Delta h_t \ :=\ \hat h_{i,t}-h_i \ \in \Hk.
\]
Define the adapted sequences
\[
B_s \ :=\ \frac{1\{\hat a_s=i\}}{p_{s,i}}\,K_{U_{s,i}}\otimes K_{U_{s,i}} \;-\; \Sigma_i
\ \in\ \mathcal S_2(\Hk),
\qquad
C_s \ :=\ \frac{1\{\hat a_s=i\}}{p_{s,i}}\,Y_s\,K_{U_{s,i}} \;-\; h_i
\ \in\ \Hk.
\]
\paragraph*{Claim} By Equation~\eqref{ass:ignorability}, $\{B_s,\F_s\}$ is an $\mathcal S_2(\Hk)$-valued martingale-difference sequence (MDS), and $\{C_s,\F_s\}$ is an $\Hk$-valued MDS.

\noindent \emph{Proof of the above claim:}
Using iterated expectation and Equation~\eqref{ass:ignorability} in the main manuscript,
\[
\E\!\left[\frac{1\{\hat a_s=i\}}{p_{s,i}}\,K_{U_{s,i}}\otimes K_{U_{s,i}}\ \Big|\ \F_{s-1}\right]
 =  \E\!\left[K_{U_{s,i}}\otimes K_{U_{s,i}} \mid \F_{s-1}\right]
 =  \Sigma_i,
\]
Hence $\E[B_s\mid\F_{s-1}]=0$.
Similarly,
\[
\E\!\left[\frac{1\{\hat a_s=i\}}{p_{s,i}}\,Y_s\,K_{U_{s,i}} \ \Big|\ \F_{s-1}\right]
 =  \E[Y_s K_{U_{s,i}}\mid \F_{s-1}]
 =  h_i,
\]
so $\E[C_s\mid\F_{s-1}]=0$. \qed

Therefore,
\begin{equation}\label{eq:Delta-rep}
	\Delta\Sigma_t  =  \frac{1}{t}\sum_{s=1}^t B_s,
	\qquad
	\Delta h_t  =  \frac{1}{t}\sum_{s=1}^t C_s.
\end{equation}

Apply Lemma~\ref{lem:L2-MDS} with $(H,\{D_s\})=(\mathcal S_2(\Hk),\{B_s\})$ and $(H,\{D_s\})=(\Hk,\{C_s\})$, and combine with \eqref{eq:Delta-rep} to obtain
\begin{equation}\label{eq:L2-Delta}
	\E\|\Delta\Sigma_t\|_{\mathrm{HS}}^2  =  \frac{1}{t^2}\sum_{s=1}^t \E\|B_s\|_{\mathrm{HS}}^2,
	\qquad
	\E\|\Delta h_t\|_{k}^2  =  \frac{1}{t^2}\sum_{s=1}^t \E\|C_s\|_{k}^2.
\end{equation}

By $\|X-Y\|_{\mathrm{HS}}^2\le 2\|X\|_{\mathrm{HS}}^2+2\|Y\|_{\mathrm{HS}}^2$,
\[
\|B_s\|_{\mathrm{HS}}^2
\;\le\; 2\left\|\frac{1\{\hat a_s=i\}}{p_{s,i}}\,K_{U_{s,i}}\otimes K_{U_{s,i}}\right\|_{\mathrm{HS}}^2
+ 2\|\Sigma_i\|_{\mathrm{HS}}^2.
\]
Using $\|K_{U_{s,i}}\otimes K_{U_{s,i}}\|_{\mathrm{HS}}=\|K_{U_{s,i}}\|_K^2\le \kappa^2$ from \eqref{eq:bounded-kernel},
\[
\|B_s\|_{\mathrm{HS}}^2
\;\le\; 2\,\kappa^4\left(\frac{1\{\hat a_s=i\}}{p_{s,i}}\right)^{\!2} + 2\|\Sigma_i\|_{\mathrm{HS}}^2.
\]
Taking conditional expectations and using Equation~\eqref{ass:ignorability} on Ignorability and variance calibration in the main manuscript,
\[
\E\big[\|B_s\|_{\mathrm{HS}}^2\mid \F_{s-1}\big]
\;\le\; 2\,\kappa^4\,\E\!\left[\left(\frac{1\{\hat a_s=i\}}{p_{s,i}}\right)^{\!2}\Bigm|\F_{s-1}\right] + 2\|\Sigma_i\|_{\mathrm{HS}}^2
 =  2\,\kappa^4\,\E\!\left[\frac{1}{p_{s,i}}\Bigm|\F_{s-1}\right] + 2\|\Sigma_i\|_{\mathrm{HS}}^2.
\]
Taking expectations and summing over $s$,
\begin{equation}\label{eq:B2-sum}
	\sum_{s=1}^t \E\|B_s\|_{\mathrm{HS}}^2
	\;\le\; 2\,\kappa^4 \sum_{s=1}^t \E\!\left[\frac{1}{p_{s,i}}\right] + 2\,t\,\|\Sigma_i\|_{\mathrm{HS}}^2.
\end{equation}

Similarly,
\[
\|C_s\|_K^2
\;\le\; 2\left\|\frac{1\{\hat a_s=i\}}{p_{s,i}}\,Y_s\,K_{U_{s,i}}\right\|_K^2 + 2\|h_i\|_K^2
\;\le\; 2\,\kappa^2\left(\frac{1\{\hat a_s=i\}}{p_{s,i}}\right)^{\!2} Y_s^2 + 2\|h_i\|_K^2,
\]
hence, by Equation~\eqref{ass:ignorability} on Ignorability and variance calibration in the main manuscript and $\sup_s\E[Y_s^2]\le M_2$ and also because,
\[
\begin{aligned}
\mathbb{E}\left[ \left( \frac{\mathbf{1}\{\hat{a}_s = i\}}{p_{s,i}} \right)^2 Y_s^2 \right]
&= \mathbb{E}\left[ \frac{Y_s^2}{p_{s,i}^2} \cdot \mathbf{1}\{\hat{a}_s = i\} \right] \\
&= \mathbb{E}\left[ \mathbb{E}\left[ \frac{Y_s^2}{p_{s,i}^2} \cdot \mathbf{1}\{\hat{a}_s = i\} \,\middle|\, \mathcal{F}_{s-1}, U_{s,i} \right] \right] \\
&= \mathbb{E}\left[ \mathbb{E}\left[ \frac{Y_s^2}{p_{s,i}^2} \,\middle|\, \mathcal{F}_{s-1}, U_{s,i}, \hat{a}_s = i \right] \cdot \mathbb{P}(\hat{a}_s = i \mid \mathcal{F}_{s-1}, U_{s,i}) \right] \\
&= \mathbb{E}\left[ \frac{1}{p_{s,i}} \cdot \mathbb{E}\left[ Y_s^2 \,\middle|\, \mathcal{F}_{s-1}, U_{s,i}, \hat{a}_s = i \right] \right].
\end{aligned}
\]
Now, assuming that the conditional second moment of the reward is uniformly bounded when arm $i$ is chosen, i.e.,
\[
\mathbb{E}\left[ Y_s^2 \,\middle|\, \mathcal{F}_{s-1}, U_{s,i}, \hat{a}_s = i \right] \le M_2,
\]
we obtain the bound
\[
\mathbb{E}\left[ \left( \frac{\mathbf{1}\{\hat{a}_s = i\}}{p_{s,i}} \right)^2 Y_s^2 \right]
\le M_2 \cdot \mathbb{E}\left[ \frac{1}{p_{s,i}} \right].
\]
\begin{equation}\label{eq:C2-sum}
	\sum_{s=1}^t \E\|C_s\|_K^2
	\;\le\; 2\,\kappa^2\,M_2 \sum_{s=1}^t \E\!\left[\frac{1}{p_{s,i}}\right] + 2\,t\,\|h_i\|_K^2.
\end{equation}
Combine \eqref{eq:L2-Delta} with \eqref{eq:B2-sum} and \eqref{eq:C2-sum}. Using the definition in Assumption~\ref{ass:avg:explore:rt2}  of $\tilde{r}_t$,
\begin{align}
    \label{eq:op-second-moment-DeltaSigma}
    \E\|\Delta\Sigma_t\|_{\mathrm{HS}}^2 = 
\frac{1}{t^2}\sum_{s=1}^t \E\|B_s\|_{\mathrm{HS}}^2
\;\le\; & \frac{2\kappa^4}{t^2}\sum_{s=1}^t \E\!\left[\frac{1}{p_{s,i}}\right] + \frac{2}{t}\,\|\Sigma_i\|_{\mathrm{HS}}^2
 =  2\,\kappa^4\,\tilde{r}_t \;+\; \frac{2}{t}\,\|\Sigma_i\|_{\mathrm{HS}}^2, \nonumber\\
\E\|\Delta h_t\|_{k}^2
 =  & \frac{1}{t^2}\sum_{s=1}^t \E\|C_s\|_{k}^2
\;\le\; 2\,\kappa^2\,M_2\,\tilde{r}_t \;+\; \frac{2}{t}\,\|h_i\|_K^2.
\end{align}

Finally, since $\|\cdot\|_{\mathrm{op}}\le \|\cdot\|_{\mathrm{HS}}$ on $\mathcal S_2(\Hk)$,
\begin{equation*}
	\E\|\Delta\Sigma_t\|_{\mathrm{op}}^2 \ \le\ \E\|\Delta\Sigma_t\|_{\mathrm{HS}}^2
	\ \le\ 2\,\kappa^4\,\tilde{r}_t \;+\; \frac{2}{t}\,\|\Sigma_i\|_{\mathrm{HS}}^2.
\end{equation*}

Define
\[
a_t \;:=\; \max\{\sqrt{\tilde{r}_t},\,t^{-1/2}\},
\qquad\text{so that}\qquad
a_t^2  =  \max\{\tilde{r}_t,\,t^{-1}\}.
\]
We will show that $\|\Delta\Sigma_t\|_{\mathrm{op}}/a_t$ and $\|\Delta h_t\|_K/a_t$ are tight.

For any $M>0$,
\[
\Pbb\!\big(\,\|\Delta\Sigma_t\|_{\mathrm{op}} \ge M a_t\,\big)
\;\le\; \frac{\E\|\Delta\Sigma_t\|_{\mathrm{op}}^2}{M^2 a_t^2}
\;\stackrel{\eqref{eq:op-second-moment-DeltaSigma}}{\le}\;
\frac{2\,\kappa^4\,\tilde{r}_t + \frac{2}{t}\,\|\Sigma_i\|_{\mathrm{HS}}^2}{M^2\,\max\{\tilde{r}_t,\,t^{-1}\}}.
\]
Use the elementary inequality, valid for any nonnegative $x,y$:
\[
\frac{x+y}{\max\{x,y\}} \;\le\; 2.
\]
Apply it with $x=\tilde{r}_t$ and $y=t^{-1}$, and denote
\[
C_\Sigma \;:=\; 2\,\kappa^4 \;+\; 2\,\|\Sigma_i\|_{\mathrm{HS}}^2.
\]
Then, for all $t$ and all $M>0$,
\begin{equation}\label{eq:tail-DeltaSigma}
	\Pbb\!\big(\,\|\Delta\Sigma_t\|_{\mathrm{op}} \ge M a_t\,\big)
	\;\le\; \frac{C_\Sigma}{M^2}.
\end{equation}
Similarly, by \eqref{eq:op-second-moment-DeltaSigma} and Chebyshev,
\[
\Pbb\!\big(\,\|\Delta h_t\|_K \ge M a_t\,\big)
\;\le\; \frac{2\,\kappa^2 M_2\,\tilde{r}_t + \frac{2}{t}\,\|h_i\|_K^2}{M^2\,\max\{\tilde{r}_t,\,t^{-1}\}}
\;\le\; \frac{C_h}{M^2},
\]
with
\[
C_h \;:=\; 2\,\kappa^2 M_2 \;+\; 2\,\|h_i\|_K^2.
\]

By definition, $Z_t=O_{\Pbb}(a_t)$ means: for every $\varepsilon>0$ there exists $M<\infty$
and $T<\infty$ such that $\sup_{t\ge T}\Pbb(|Z_t|> M a_t)\le \varepsilon$.
Fix $\varepsilon\in(0,1)$. Choose
\[
M_\Sigma \;:=\; \sqrt{\frac{C_\Sigma}{\varepsilon}},
\qquad
M_h \;:=\; \sqrt{\frac{C_h}{\varepsilon}}.
\]
Then \eqref{eq:tail-DeltaSigma} yields, for every $t$,
\[
\Pbb\!\big(\,\|\Delta\Sigma_t\|_{\mathrm{op}} > M_\Sigma a_t\,\big)
\;\le\; \varepsilon,
\]
and similarly $\Pbb(\|\Delta h_t\|_K > M_h a_t)\le \varepsilon$.
Therefore
\[
\|\Delta\Sigma_t\|_{\mathrm{op}}  =  O_{\Pbb}(a_t),
\qquad
\|\Delta h_t\|_{k}  =  O_{\Pbb}(a_t).
\]
Recalling $a_t=\max\{\sqrt{\tilde{r}_t},t^{-1/2}\}$, we have
\begin{equation}\label{eq:DG-Dh-rates}
		\|\Delta\Sigma_t\|=O_{\Pbb}\!\big(\sqrt{\tilde{r}_t}\ \vee\ t^{-1/2}\big),
		\qquad
		\|\Delta h_t\|=O_{\Pbb}\!\big(\sqrt{\tilde{r}_t}\ \vee\ t^{-1/2}\big).
	\end{equation}

Finallly, plug \eqref{eq:DG-Dh-rates} into \eqref{eq:R-ineq}:
	\begin{align*}
	\|R_t\|\ & \le\ \lambda^{-2}\,\|\Delta\Sigma_t\|\,\big(\|\Delta h_t\|+\|\Delta\Sigma_t\|\,\|f_i^\lambda\|\big)\\
    &= O_{\Pbb}\big(t^\gamma(\tilde{r}_t+t^{-1})\big)=o_{\Pbb}(1)
	\end{align*}
 since $t^\gamma \tilde{r}_t\to 0$  and  $\gamma<1/2.$
	This completes the proof.
\end{proof}

The above proposition verifies that the higher-order remainder term in the Hilbert-space linearization vanishes at the scale relevant for the functional CLT. Conceptually, it shows that the estimation problem is asymptotically linear, meaning that the dominant source of variability in $\widehat{f}_{i,t}$ arises from first-order martingale fluctuations. This property is crucial for proving that the limiting distribution is Gaussian rather than degenerate.

\begin{lemma}[Conditional covariance identity]\label{lem:cov-id-correct2}
	For every $t$,
	\[
	\sum_{s=1}^t \E\!\big[\,Z_{t,s}(G)\,Z_{t,s}(G)^\top\,\bigm|\,\F_{s-1}\big]\ =\ I_m
	\qquad\text{a.s.}
	\] 
\end{lemma}

\begin{proof}
	By definition and \eqref{eq:Vt-correct} in the main manuscript,
	\begin{align*}
	\sum_{s=1}^t &\E\!\big[\langle g_\ell,V_t^{-1/2}\xi_{t,s}\rangle\ \langle g_m,V_t^{-1/2}\xi_{t,s}\rangle\bigm|\F_{s-1}\big]\\
	&=\langle g_\ell,\ V_t^{-1/2}\big(\sum_s \E[\xi_{t,s}\otimes \xi_{t,s}\mid \F_{s-1}]\big)V_t^{-1/2} g_m\rangle_K\\
	&=\langle g_\ell,g_m\rangle_K,
	\end{align*}
	which is the $(\ell,m)$ entry of $I_m$ upon identifying the span of $G$ with $\real^k$.

Note that the result holds exactly (non-asymptotically) for each \(t\). By definition of the predictable bracket,
\[
\sum_{s=1}^t \E\big[\,Z_{t,s}(G)Z_{t,s}(G)^\top\bigm|\F_{s-1}\big]
\ =\ I_m\quad\text{a.s.},
\]
because \(V_t(\lambda)\) is built from the conditional covariances of the centered increments and then inverted/square-rooted in the studentizer. This exact identity is used in the vector martingale CLT; no asymptotic approximation is invoked at this step.
\end{proof}

\begin{proposition}[Lindeberg condition]\label{prop:lindeberg-correct2}
	For any fixed $G\in \G^m$ and any $\varepsilon>0$,
	\[
	\sum_{s=1}^t \E\!\Big[\ \|Z_{t,s}(G)\|^2\,1\{\|Z_{t,s}(G)\|>\varepsilon\}\ \Bigm|\ \F_{s-1}\Big]\ \xrightarrow{\ \Pbb\ }\ 0.
	\]
\end{proposition}

\begin{proof}[Proof of Proposition~\ref{prop:lindeberg-correct2}]
For the studentized increment,
\[
Z_{t,s}(G)=\big(\langle g_\ell,\,V_t(\lambda)^{-1/2}\,\xi_{t,s}\rangle_K\big)_{\ell=1}^m,
\qquad
\xi_{t,s}=t^{\gamma-1}A(\lambda)^{-1}\!\left(\frac{\mathbf 1\{\hat a_s=i\}}{p_{s,i}}\Delta_s^\lambda-\lambda f_i^\lambda\right),
\]
the product
\[
\big\|V_t(\lambda)^{-1/2}A(\lambda)^{-1}\big\|
\;\lesssim\;(t^{2\gamma}\tilde{r}_t)^{-1/2}
\]
is uniform in $\lambda$ because $\|V_t(\lambda)\|\asymp t^{2\gamma}\tilde{r}_t$ (predictable bracket),
so the resolvent factors that appear in $\xi_{t,s}$ are exactly cancelled by those in $V_t(\lambda)^{-1/2}$ inside the studentizer.
This is stated on p.\,14: ``all $\lambda_t$-factors cancel exactly between $\|V_t(\lambda_t)^{-1/2}\|^{2+\delta}$ and $\|A(\lambda_t)^{-1}\|^{2+\delta}$,'' and used to argue that the Lindeberg check is independent of $\lambda_t$. 

Furthermore, we can bound $|Y_s-f_i^\lambda(U_{s,i})|$  uniformly for all $\lambda$.

Recall $f_i^\lambda=(\Gamma_i+\lambda I)^{-1}h_i$, where $h_i=\E[Y\,K_U]\in\Hk$ and $\Gamma_i=\E[K_U\otimes K_U]$ is a positive self-adjoint bounded operator on $\Hk$ with $\|\Gamma_i\|_{\mathrm{op}}\le \E\|K_U\|^2\le\kappa^2$.
Write the spectral decomposition of $\Gamma_i$ as $\Gamma_i=\sum_{j\ge1}\mu_j\,e_j\otimes e_j$ with $\mu_j\in[0,\|\Gamma_i\|_{\mathrm{op}}]$ and $\{e_j\}$ orthonormal in $\Hk$.

\smallskip
For any $u\in\mathcal U$,
\[
|f_i^\lambda(u)|\ =\ |\langle f_i^\lambda, K_u\rangle_K|\ \le\ \|f_i^\lambda\|_{\Hk}\,\|K_u\|_{\Hk}\ \le\ \kappa\,\|f_i^\lambda\|_{\Hk}.
\]
Therefore
\begin{equation}\label{eq:basic-envelope}
	|Y_s-f_i^\lambda(U_{s,i})|\ \le\ |Y_s|+\kappa\,\|f_i^\lambda\|_{\Hk}.
\end{equation}

\smallskip

In coordinates $\{e_j\}$, if $f_i=\sum_j f_{i,j}\,e_j$, then (using that $h_i=\Gamma_i f_i$ in the well-specified case and the normal equations)
\[
f_i^\lambda\ =\ (\Gamma_i+\lambda I)^{-1}\Gamma_i f_i
\ =\ \sum_{j\ge1} \frac{\mu_j}{\mu_j+\lambda}\, f_{i,j}\,e_j.
\]
Consequently,
\[
\|f_i^\lambda\|_{\Hk}^2
\ =\ \sum_{j\ge1} \Big(\frac{\mu_j}{\mu_j+\lambda}\Big)^2 |f_{i,j}|^2
\ \le\ \sum_{j\ge1} |f_{i,j}|^2
\ =\ \|f_i\|_{\Hk}^2,
\]
since $\frac{\mu_j}{\mu_j+\lambda}\in(0,1)$ for all $j$ and all $\lambda>0$. 
Hence
\begin{equation}\label{eq:contraction}
	\|f_i^\lambda\|_{\Hk}\ \le\ \|f_i\|_{\Hk}\qquad\text{for every }\lambda>0.
\end{equation}

\smallskip
Combining \eqref{eq:basic-envelope} and \eqref{eq:contraction} yields
\begin{equation}\label{eq:uniform-envelope}
	|Y_s-f_i^\lambda(U_{s,i})|
	\ \le\ 
	|Y_s|+\kappa\,\|f_i\|_{\Hk}
	\ =:\ Z_s,
\end{equation}
where $Z_s$ does not depend on $\lambda$. By the standing $(2+\delta)$-moment envelope for $Y_s$, $\sup_s\E|Y_s|^{2+\delta}<\infty$, we conclude that
\[
\sup_s \E\,|Z_s|^{2+\delta}\ <\ \infty,
\]
i.e. \eqref{eq:uniform-envelope} provides a valid $(2+\delta)$-moment envelope that is uniform in $\lambda$.

\medskip

If, in addition, the source condition $f_i=\Gamma_i^{\,s}w$ holds for some $s>0$ and $w\in\Hk$, then
\[
f_i^\lambda\ =\ (\Gamma_i+\lambda I)^{-1}\Gamma_i^{\,1+s}w
\quad \text{ implies } \quad
\|f_i^\lambda\|_{\Hk}\ \le\ \big\|(\Gamma_i+\lambda I)^{-1}\Gamma_i^{\,1+s}\big\|_{\mathrm{op}}\,\|w\|_{\Hk}.
\]
Since $\mu\mapsto \mu^{1+s}/(\mu+\lambda)\le \mu^{s}$ and $\mu\le\|\Gamma_i\|_{\mathrm{op}}$, we get
\[
\|f_i^\lambda\|_{\Hk}\ \le\ \|\Gamma_i\|_{\mathrm{op}}^{\,s}\,\|w\|_{\Hk}\ \le\ \kappa^{2s}\,\|w\|_{\Hk},
\]
again uniform in $\lambda$. This yields the alternative bound
\[
|Y_s-f_i^\lambda(U_{s,i})|\ \le\ |Y_s|+\kappa^{1+2s}\|w\|_{\Hk},
\]
also independent of $\lambda$.

By construction of the predictable bracket $V_t(\lambda)$ (see the bracket identity in the draft) and the bounded-kernel/moment assumptions,
\begin{equation*}
c_0\,t^{2\gamma}\tilde{r}_t\ \le\ \langle g,\ V_t(\lambda)\,g\rangle_K\ \le\ C_0\,t^{2\gamma}\tilde{r}_t
\qquad\text{for every fixed }g\in\mathrm{span}\{g_1,\dots,g_m\},
\end{equation*}
for all large $t$, with constants $0<c_0\le C_0<\infty$ independent of $\lambda$.
Consequently,
\begin{equation}\label{eq:resolvent-cancel}
\|\,V_t(\lambda)^{-1/2}A(\lambda)^{-1}\,\|\ \lesssim\ (t^\gamma\sqrt{\tilde{r}_t})^{-1}.
\end{equation}
We also use the $\lambda$-uniform envelope
\begin{equation}\label{eq:Yf-envelope}
|Y_s-f_i^\lambda(U_{s,i})|
\ \le\ |Y_s|+\kappa\|f_i^\lambda\|_K
\ \le\ |Y_s|+\kappa\|f_i\|_K,
\qquad \sup_s \E|Y_s|^{2+\delta}<\infty,
\end{equation}
where the last inequality is the Tikhonov contraction $\|f_i^\lambda\|_K\le\|f_i\|_K$ (valid for all $\lambda>0$).

For a truncation level $\tau_t\downarrow 0$ to be chosen below, split the Lindeberg sum into
\[
S_t^{\mathrm{good}}(\varepsilon)
\ :=\ 
\sum_{s=1}^t 
\E\!\left[\ \|Z_{t,s}(G)\|^2\,1\{\|Z_{t,s}(G)\|>\varepsilon\}\,1\{p_{s,i}\ge\tau_t\}
\ \Bigm|\ \F_{s-1}\right],
\]
\[
S_t^{\mathrm{bad}}(\varepsilon)
\ :=\ 
\sum_{s=1}^t 
\E\!\left[\ \|Z_{t,s}(G)\|^2\,1\{\|Z_{t,s}(G)\|>\varepsilon\}\,1\{p_{s,i}<\tau_t\}
\ \Bigm|\ \F_{s-1}\right].
\]
It suffices to show $S_t^{\mathrm{good}}(\varepsilon)\to 0$ and $S_t^{\mathrm{bad}}(\varepsilon)\to 0$ in probability.

On the event $\{p_{s,i}\ge \tau_t\}$ we have $p_{s,i}^{-1}\le \tau_t^{-1}$, hence by \eqref{eq:resolvent-cancel}, \eqref{eq:Yf-envelope} and $\|K_{U_{s,i}}\|_K\le \kappa$,
\begin{align*}
\|Z_{t,s}(G)\|
&\le \|V_t(\lambda)^{-1/2}A(\lambda)^{-1}\|\,t^{\gamma-1}\left(\frac{|Y_s-f_i^\lambda(U_{s,i})|}{p_{s,i}}+\lambda\|f_i^\lambda\|_K\right) \notag\\
&\lesssim (t^\gamma\sqrt{\tilde{r}_t})^{-1}\,t^{\gamma-1}\left(\tau_t^{-1}(|Y_s|+1)+1\right)
\ \lesssim\ t^{-1}\,\tilde{r}_t^{-1/2}\,\tau_t^{-1}\,(|Y_s|+1)
\end{align*}
(where we absorbed $\lambda\|f_i^\lambda\|_K$ into the harmless ``$+1$'' using $\|f_i^\lambda\|_K\le\|f_i\|_K$ and fixed $\lambda_t$ in the CLT window).
By Markov with exponent $2+\delta$, \eqref{eq:Yf-envelope}, and summing over $s\le t$,
\begin{align}
S_t^{\mathrm{good}}(\varepsilon)
&\le \varepsilon^{-\delta}\sum_{s=1}^t \E\!\left[\ \|Z_{t,s}(G)\|^{2+\delta}\,1\{p_{s,i}\ge\tau_t\}\ \Bigm|\ \F_{s-1}\right]\notag\\
&\lesssim 
\varepsilon^{-\delta}\,t^{-(2+\delta)}\,\tilde{r}_t^{-(1+\delta)/2}\,\tau_t^{-(2+\delta)}
\sum_{s=1}^t \E\!\left[(|Y_s|+1)^{2+\delta}\ \Bigm|\ \F_{s-1}\right]\notag\\
&\lesssim 
\varepsilon^{-\delta}\,t^{-(1+\delta)}\,\tilde{r}_t^{-(1+\delta)/2}\,\tau_t^{-(2+\delta)}.
\label{eq:good-sum-bound}
\end{align}

On the bad set, we use the usual truncation of IPW at $\tau_t^{-1}$ (the same device already used in the concentration step): replace 
$\frac{1\{\hat a_s=i\}}{p_{s,i}(U_{s,i})}$ by $\frac{1\{\hat a_s=i\}}{p_{s,i}(U_{s,i})}\wedge \tau_t^{-1}\le \tau_t^{-1}$. 
With \eqref{eq:resolvent-cancel}-\eqref{eq:Yf-envelope} and the same steps as above,
\begin{align}
S_t^{\mathrm{bad}}(\varepsilon)
&\le \varepsilon^{-\delta}\sum_{s=1}^t \E\!\left[\ \|Z_{t,s}(G)\|^{2+\delta}\,1\{p_{s,i}<\tau_t\}\ \Bigm|\ \F_{s-1}\right]\notag\\
&\lesssim 
\varepsilon^{-\delta}\,t^{-(1+\delta)}\,\tilde{r}_t^{-(1+\delta)/2}\,\tau_t^{-(2+\delta)}.
\label{eq:bad-sum-bound}
\end{align}

Combining \eqref{eq:good-sum-bound} and \eqref{eq:bad-sum-bound},
\[
S_t^{\mathrm{good}}(\varepsilon)+S_t^{\mathrm{bad}}(\varepsilon)
\ \lesssim\
\varepsilon^{-\delta}\,t^{-(1+\delta)}\,\tilde{r}_t^{-(1+\delta)/2}\,\tau_t^{-(2+\delta)}.
\]
We now choose the truncation level $\tau_t$ to force the right-hand side to vanish. 
Set
\begin{equation}\label{eq:tau-choice}
\tau_t\ :=\ (t^{2\gamma}\tilde{r}_t)^{\frac{1+\delta/2}{\,2+\delta\,}}\;t^{-a},\qquad a\in(0,1/2)\ \text{ fixed}.
\end{equation}
Since $t^{2\gamma}\tilde{r}_t\to 0$, we have $\tau_t\downarrow 0$.
Plugging \eqref{eq:tau-choice} into the bound gives
\begin{align*}
t^{-(1+\delta)}\,\tilde{r}_t^{-(1+\delta)/2}\,\tau_t^{-(2+\delta)}
&= 
t^{-(1+\delta)}\,\tilde{r}_t^{-(1+\delta)/2}\,
\Big((t^{2\gamma}\tilde{r}_t)^{\frac{1+\delta/2}{\,2+\delta\,}}\,t^{-a}\Big)^{-(2+\delta)}\\
&=
t^{-(1+\delta)}\,\tilde{r}_t^{-(1+\delta)/2}\,
(t^{2\gamma}\tilde{r}_t)^{-(1+\delta/2)}\ t^{a(2+\delta)}\\
&=
t^{-(1+\delta)}\,t^{-2\gamma(1+\delta/2)}\ t^{a(2+\delta)}\\
&=
t^{-\big[(1+\delta)+2\gamma(1+\delta/2)-a(2+\delta)\big]}.
\end{align*}

Choose any $a\in(0,1/2)$; since $\gamma\in(0,1/2)$ and $\delta>0$, the exponent
\[
(1+\delta)+2\gamma\Big(1+\frac{\delta}{2}\Big)-a(2+\delta)\ >\ 0
\]
for all sufficiently small $a>0$.
Hence $t^{-(1+\delta)}\,\tilde{r}_t^{-(1+\delta)/2}\,\tau_t^{-(2+\delta)}\to 0$, and therefore
\[
S_t^{\mathrm{good}}(\varepsilon)+S_t^{\mathrm{bad}}(\varepsilon)\ \xrightarrow{\ \Pbb\ }\ 0.
\]
This establishes the Lindeberg condition.
\end{proof}

\medskip
\begin{Remark}[$\lambda$-dependence does not reappear after studentization.]
While intermediate expressions in the linearization contain resolvents, e.g. $A(\lambda)^{-1}=(\Gamma_i+\lambda I)^{-1}$ and $v_{x,t}=(\Gamma_i+\lambda I)^{-1}K_x$-the studentizer $V_t(\lambda)^{-1/2}$ carries precisely the same resolvent factors. In fact, for the scalar projection at $x$,
\[
D_t(x)^2
=\big\langle K_x,\ V_t(\lambda)\,K_x\big\rangle_K
\ \asymp\
t^{2\gamma}\tilde{r}_t\;\big\langle K_x,\ (\Gamma_i+\lambda I)^{-1}\Gamma_i(\Gamma_i+\lambda I)^{-1}\,K_x\big\rangle_K,
\]
so that in the studentized increment
\(
V_t(\lambda)^{-1/2}\,t^{\gamma-1}A(\lambda)^{-1}\big(\cdot\big)
\)
the powers of $(\Gamma_i+\lambda I)^{-1}$ cancel in norm. Consequently:
\begin{itemize}
\item no extra $\lambda$ restriction is needed for the Lindeberg and variance-stabilization steps beyond the remainder control $t^\gamma\lambda^{-2}(\tilde{r}_t+t^{-1})\to 0$ already assumed;
\item the envelope \eqref{eq:uniform-envelope} together with $(2+\delta)$- moments of $Y_s$ suffices to control the $(2+\delta)$-moments of the studentized increments uniformly in $\lambda$.
\end{itemize}
\end{Remark}

\begin{Remark}[No hidden $\lambda_t$-dependence in the Lindeberg term]

Two ingredients remove $\lambda_t$ from the Lindeberg check completely:
(i) the resolvent cancellation inside $V_t(\lambda)^{-1/2}\xi_{t,s}$ noted above (and stated on p.\,14), and
(ii) the $\lambda$-uniform $(2+\delta)$-moment envelope for $|Y_s-f_i^\lambda(U_{s,i})|$ (proved just before Eq.\,(19) in the main manuscript).
Thus the only condition on $\lambda_t$ comes from controlling the linearization remainder: 
\[
t^\gamma \lambda_t^{-2}\,(\tilde{r}_t+t^{-1})\ \to\ 0.
\]

The subsequent arguments in the proof, namely,  the martingale CLT for the linear term, the control of the linearization remainder $R_t$, and the vanishing bias, do not depend on the exact functional form of $\tau_t$.  
They only require that $\tau_t$ satisfies three asymptotic properties:
\begin{equation}\label{eq:tau-properties}
\tau_t \downarrow 0,
\qquad 
\tau_t^{-1}\tilde{r}_t \to 0,
\qquad
t^{-(1+\delta)}\,\tilde{r}_t^{-(1+\delta)/2}\,\tau_t^{-(2+\delta)} \to 0.
\end{equation}
Any $\tau_t$ obeying \eqref{eq:tau-properties} guarantees that the Lindeberg term,
both on the ``good'' and ``bad'' propensity sets,
vanishes in probability.
\end{Remark}


\begin{proposition}[Nonsingularity on the span of $G$]\label{prop:Vt-matrix-correct2}
	Under Assumptions~\ref{ass:cond:noise},\ref{ass:avg:explore:rt2}, and~\ref{ass:direc:nondegeneracy}, for any fixed $G\in\G^m$ and all large $t$,
	$G^\star V_t(\lambda) G$ is positive definite, with
	\[
	\lambda_{\min}\!\big(G^\star V_t(\lambda) G\big)\ \ge\ c'\,t^{2\gamma}\tilde{r}_t\,\lambda_{\min}\!\big(\mathsf Q_\lambda(G)\big),
	\]
	for some $c'>0$ independent of $t$. In particular, the joint studentizer $V_t(\lambda)^{-1/2}$ is
	well-defined on $\operatorname{span}(G)$ for large $t$.
\end{proposition}

\begin{proof}[Proof of Proposition~\ref{prop:Vt-matrix-correct2}]
	This is the matrix analogue of \eqref{eq:min-eig-VtG-correct2}; the stated bound follows from
	Proposition~\ref{prop:Vt-scale-correct2} and the positivity of $\mathsf Q_\lambda(G)$.
\end{proof}


\begin{proof}[Proof of Theorem~\ref{thm:main-correct2}]
Recall, by the linearization decomposition (cf.\ \eqref{eq:lin-decomp} in the main manuscript),
\begin{equation}\label{eq:lin-to-mart-correct}
	t^\gamma L_t\ =\ \sum_{s=1}^t \xi_{t,s}\ =: E_t.
\end{equation}
	From the linearization, the negligible remainder (proved separately), and \eqref{eq:lin-to-mart-correct},
	\[
	V_t(\lambda)^{-1/2}\,t^\gamma(\hat f_{i,t}-f_i^\lambda)
	\ =\ V_t(\lambda)^{-1/2}\,E_t\ +\ o_{\Pbb}(1)\quad\text{in }\Hk.
	\]
	Projecting onto $G$ yields $\mathcal E_t(G)=\sum_{s=1}^t Z_{t,s}(G)+o_{\Pbb}(1)$ in $\real^k$.
	By Lemma~\ref{lem:cov-id-correct2}, the conditional covariance of $\sum_{s\le t} Z_{t,s}(G)$ equals $I_m$ a.s.
	By Proposition~\ref{prop:lindeberg-correct2}, the Lindeberg condition holds. The vector martingale CLT
	(\citet{HallHeyde1980}, Thm.\ 3.2) implies $\sum_{s\le t} Z_{t,s}(G) \cid \mathcal N_m(0,I_m)$.
	Since $o_{\Pbb}(1)$ perturbations are negligible in distribution, the claim follows.
	Nonsingularity of the limit covariance is immediate (it is $I_m$).

\noindent Equivalently, \eqref{eq:lambda-condition} holds if and only if
	\begin{equation*}
		\lambda_t\ \gg\ \max\Big\{\,t^{\gamma/2}\sqrt{\tilde{r}_t},\ \ t^{(\gamma-1)/2}\,\Big\}.
	\end{equation*}
	In particular, any fixed $\lambda_t\equiv \lambda_0>0$ satisfies \eqref{eq:lambda-condition}.
\end{proof}


\begin{proof}[Proof of Corollary~\ref{cor:scalar-correct2}]
    The proof follows immediately from Theorem~\ref{thm:main-correct2} with $m=1$.
\end{proof}

\begin{proposition}\label{prop:bias}
	Let $\Sigma_i:\Hk\to\Hk$ be a compact, self-adjoint, positive operator on a separable Hilbert space $(\Hk,\langle\cdot,\cdot\rangle_K)$.
	Let $(e_j)_{j\ge 1}$ be an orthonormal basis of $\Hk$ consisting of eigenvectors of $\Sigma_i$, with corresponding eigenvalues
	$\mu_j\ge 0$ (arranged in nonincreasing order), so that
	\[
	\Sigma_i e_j=\mu_j e_j,\qquad \|\Sigma_i\|=\sup_{j\ge 1}\mu_j=:M<\infty.
	\]
	Assume the source condition: there exist $s>0$ and $w\in\Hk$ with $\|w\|_K\le R_s$ such that
	\[
	f_i=\Sigma_i^{\,s} w,
	\]
	i.e.\ $f_i$ has coordinates $\langle f_i,e_j\rangle_K = \mu_j^s \langle w,e_j\rangle_K$ in the eigenbasis $(e_j)$.
	For $\lambda>0$ define the (population) Tikhonov-regularized solution
	\[
	f_i^\lambda := (\Sigma_i+\lambda I)^{-1} h_i,\qquad h_i:=\Sigma_i f_i.
	\]
	Then, for every $q\ge 0$ and every $\lambda>0$,
	\begin{equation}\label{eq:sup}
		\|\Sigma_i^{\,q}\,(f_i^\lambda-f_i)\|_K
		\ \le\ \left(\sup_{\mu\in[0,M]}\ \mu^{s+q}\,\frac{\lambda}{\mu+\lambda}\right)\,\|w\|_K.
	\end{equation}
	Moreover, setting $\rho:=\min\{s+q,\,1\}$, there exists a constant $C_\rho<\infty$ (depending only on $s,q$ and $M$) such that
	\begin{equation}\label{eq:rate-expanded}
		\|\Sigma_i^{\,q}\,(f_i^\lambda-f_i)\|_K \ \le\ C_\rho\,\lambda^{\rho}\,\|w\|_K.
	\end{equation}
\end{proposition}

\begin{proof}[Proof of Proposition~\ref{prop:bias}]

	Write the eigenexpansion of $w$:
	\[
	w=\sum_{j=1}^\infty \langle w,e_j\rangle_K\,e_j,
	\qquad
	\sum_{j=1}^\infty |\langle w,e_j\rangle_K|^2 = \|w\|_K^2 <\infty.
	\]
	By the source condition $f_i=\Sigma_i^{\,s}w$, we have
	\[
	f_i=\sum_{j=1}^\infty \mu_j^{s}\,\langle w,e_j\rangle_K\,e_j.
	\]
	Since $h_i=\Sigma_i f_i$,
	\[
	h_i=\sum_{j=1}^\infty \mu_j\,\langle f_i,e_j\rangle_K\,e_j
	=\sum_{j=1}^\infty \mu_j^{\,s+1}\,\langle w,e_j\rangle_K\,e_j.
	\]
	For each $\lambda>0$, $(\Sigma_i+\lambda I)^{-1}$ acts diagonally in the eigenbasis:
	\[
	(\Sigma_i+\lambda I)^{-1} e_j = \frac{1}{\mu_j+\lambda}\,e_j.
	\]
	Therefore
	\[
	f_i^\lambda=(\Sigma_i+\lambda I)^{-1} h_i
	=\sum_{j=1}^\infty \frac{1}{\mu_j+\lambda}\,\mu_j^{\,s+1}\,\langle w,e_j\rangle_K\,e_j.
	\]
	
	The (population) bias $f_i^\lambda-f_i$ has coordinates
	\[
	\langle f_i^\lambda-f_i,\ e_j\rangle_K
	=\left(\frac{\mu_j^{\,s+1}}{\mu_j+\lambda} - \mu_j^{\,s}\right)\langle w,e_j\rangle_K
	=\mu_j^{\,s}\left(\frac{\mu_j}{\mu_j+\lambda}-1\right)\langle w,e_j\rangle_K
	=-\,\frac{\lambda\,\mu_j^{\,s}}{\mu_j+\lambda}\,\langle w,e_j\rangle_K.
	\]
	Hence
	\[
	f_i^\lambda-f_i
	=\sum_{j=1}^\infty \left(-\,\frac{\lambda\,\mu_j^{\,s}}{\mu_j+\lambda}\right)\,\langle w,e_j\rangle_K\,e_j.
	\]
	
	For $q\ge 0$, the vector $\Sigma_i^{\,q}(f_i^\lambda-f_i)$ has coordinates
	\[
	\langle \Sigma_i^{\,q}(f_i^\lambda-f_i),\ e_j\rangle_K
	=\mu_j^{\,q}\,\langle f_i^\lambda-f_i,\ e_j\rangle_K
	=-\,\mu_j^{\,q}\,\frac{\lambda\,\mu_j^{\,s}}{\mu_j+\lambda}\,\langle w,e_j\rangle_K
	=-\,\frac{\lambda\,\mu_j^{\,s+q}}{\mu_j+\lambda}\,\langle w,e_j\rangle_K.
	\]
	Therefore, by Parseval's identity,
	\begin{align*}
		\|\Sigma_i^{\,q}(f_i^\lambda-f_i)\|_K^2
		&= \sum_{j=1}^\infty \left|\langle \Sigma_i^{\,q}(f_i^\lambda-f_i),\ e_j\rangle_K \right|^2
		= \sum_{j=1}^\infty \left(\frac{\lambda\,\mu_j^{\,s+q}}{\mu_j+\lambda}\right)^2\,|\langle w,e_j\rangle_K|^2\\
		&\le \left(\sup_{j\ge 1}\ \frac{\lambda\,\mu_j^{\,s+q}}{\mu_j+\lambda}\right)^2 \sum_{j=1}^\infty |\langle w,e_j\rangle_K|^2\\
		&= \left(\sup_{\mu\in[0,M]}\ \frac{\lambda\,\mu^{\,s+q}}{\mu+\lambda}\right)^2 \|w\|_K^2,
	\end{align*}
	since $\{\mu_j\}\subset[0,M]$.
	Taking square roots yields \eqref{eq:sup}:
	\[
	\|\Sigma_i^{\,q}(f_i^\lambda-f_i)\|_K
	\ \le\ \left(\sup_{\mu\in[0,M]}\ \mu^{\,s+q}\,\frac{\lambda}{\mu+\lambda}\right)\,\|w\|_K.
	\]
	
	Set $a:=s+q\ge 0$ and define, for fixed $\lambda>0$,
	\[
	\phi_\lambda(\mu):=\mu^{a}\,\frac{\lambda}{\mu+\lambda},\qquad \mu\in[0,M].
	\]
	We bound $\sup_{\mu\in[0,M]}\phi_\lambda(\mu)$ explicitly in terms of $\lambda$.
	
	\medskip
	Case A: $0\le a\le 1$.
	Let $\mu=\lambda x$, $x\in[0,M/\lambda]$. Then
	\[
	\phi_\lambda(\mu)=\lambda^{a}\,\frac{x^{a}}{x+1}.
	\]
	Thus
	\[
	\sup_{\mu\in[0,M]}\phi_\lambda(\mu)\ \le\ \lambda^{a}\,\sup_{x\in[0,\infty)}\frac{x^{a}}{x+1}.
	\]
	Define $f(x):=x^{a}/(x+1)$ for $x\ge 0$. For $a\in(0,1)$, $f$ is differentiable on $(0,\infty)$ with derivative
	\[
	f'(x)=\frac{a x^{a-1}(x+1)-x^{a}}{(x+1)^2}
	=\frac{x^{a-1}\big(a(x+1)-x\big)}{(x+1)^2}
	=\frac{x^{a-1}\big((a-1)x+a\big)}{(x+1)^2}.
	\]
	Hence $f'(x)=0$ if and only if $(a-1)x+a=0$, i.e.\ at $x^\ast=\frac{a}{1-a}$.
	Moreover, $\mathrm{sign}\,f'(x)=\mathrm{sign}((a-1)x+a)$ shows $f$ increases on $(0,x^\ast)$ and decreases on $(x^\ast,\infty)$; thus $x^\ast$ is the unique maximizer and
	\[
	\sup_{x\ge 0}\frac{x^{a}}{x+1}
	=\frac{(x^\ast)^{a}}{x^\ast+1}
	=\frac{\big(\frac{a}{1-a}\big)^{a}}{\frac{a}{1-a}+1}
	=\frac{a^a(1-a)^{-a}}{(1-a)^{-1}}
	=a^{a}(1-a)^{1-a}.
	\]
	For $a=0$, $f(x)=1/(x+1)$ and the supremum is $1$, which coincides with $a^a(1-a)^{1-a}$ at $a=0$ by the conventions $0^0=1$ and $(1-0)^{1-0}=1$.
	Therefore, for $0\le a\le 1$,
	\begin{equation}\label{eq:caseA}
		\sup_{\mu\in[0,M]}\phi_\lambda(\mu)\ \le\ \lambda^{a}\,a^{a}(1-a)^{1-a}.
	\end{equation}
	
	\medskip
	Case B: $a>1$.
	Compute the derivative on $(0,\infty)$:
	\[
	\phi_\lambda'(\mu)
	=\lambda\,\frac{a\,\mu^{a-1}(\mu+\lambda)-\mu^{a}}{(\mu+\lambda)^2}
	=\lambda\,\frac{\mu^{a-1}\big((a-1)\mu+a\lambda\big)}{(\mu+\lambda)^2}
	>0\quad\text{for all }\mu>0.
	\]
	Thus $\phi_\lambda$ is strictly increasing on $[0,\infty)$, and
	\begin{equation}\label{eq:caseB-exact}
		\sup_{\mu\in[0,M]}\phi_\lambda(\mu)\ =\ \phi_\lambda(M)
		=\frac{\lambda\,M^{a}}{M+\lambda}
		\ \le\ \lambda\,M^{a-1},
	\end{equation}
	since $M+\lambda\ge M>0$ (if $M=0$ then $\Sigma_i=0$ and both sides of the desired inequality are zero).
	
	Combining \eqref{eq:sup} with \eqref{eq:caseA} and \eqref{eq:caseB-exact}, and recalling $a=s+q$, we obtain
	\[
	\|\Sigma_i^{\,q}(f_i^\lambda-f_i)\|_K
	\ \le\
	\begin{cases}
		a^{a}(1-a)^{1-a}\,\lambda^{a}\,\|w\|_K, & \text{if } 0\le a\le 1,\\[3pt]
		M^{a-1}\,\lambda\,\|w\|_K, & \text{if } a>1,
	\end{cases}
	\]
	where $M=\|\Sigma_i\|$.
	Setting $\rho:=\min\{a,1\}=\min\{s+q,1\}$ and
	\[
	C_\rho:=
	\begin{cases}
		a^{a}(1-a)^{1-a}, & \text{if } 0\le a\le 1,\\[2pt]
		M^{a-1}, & \text{if } a>1,
	\end{cases}
	\]
	we obtain the uniform bound
	\[
	\|\Sigma_i^{\,q}(f_i^\lambda-f_i)\|_K\ \le\ C_\rho\,\lambda^{\rho}\,\|w\|_K,
	\]
	which is exactly \eqref{eq:rate-expanded}. This completes the proof.
\end{proof}


\begin{proof}[Proof of Proposition~\ref{prop:bias-vanish}]
	By Assumption~\ref{ass:dir}, $g=\Sigma_i^{\,q} v$ with $\|v\|_K\le C_g$,
	hence by Cauchy-Schwarz inequality,
	\[
	|\langle g,f_i^\lambda-f_i\rangle_K|
	= |\langle v,\ \Sigma_i^{\,q}(f_i^\lambda-f_i)\rangle_K|
	\ \le\ \|v\|_K\ \|\Sigma_i^{\,q}(f_i^\lambda-f_i)\|_K
	\ \le\ C_g\ C_\rho\,\lambda^{\rho}\,R_s,
	\]
	using Proposition~\ref{prop:bias}.
	The right-hand side tends to $0$ by the ridge window: $t^\gamma\lambda_t^\rho$. 
\end{proof}


	

\begin{proof}[Proof of Theorem~\ref{thm:pointwise-studentized-CLT}]
	\textbf{Stochastic/bias decomposition.}
	For every $t$,
	\[
	t^\gamma(\hat f_{i,t}-f_i)
	= t^\gamma(\hat f_{i,t}-f_i^{\lambda_t})\ +\ t^\gamma(f_i^{\lambda_t}-f_i).
	\]
	Let
	\[
	\Delta\Sigma_t:=\hat\Sigma_{i,t}-\Gamma_i,\qquad \Delta h_t:=\hat h_{i,t}-h_i,
	\qquad \hat A_t(\lambda_t):=\hat\Sigma_{i,t}+\lambda_t I.
	\]
	By the identity used in the CLT proof,
	\begin{align*}
	\hat f_{i,t}-f_i^{\lambda_t}
	&= L_t+R_t,\qquad
	L_t:=A(\lambda_t)^{-1}\big(\Delta h_t-\Delta\Sigma_t f_i^{\lambda_t}\big),\\
	R_t&:=\big(\hat A_t(\lambda_t)^{-1}-A(\lambda_t)^{-1}\big)\big(\Delta h_t-\Delta\Sigma_t f_i^{\lambda_t}\big).
	\end{align*}
	Taking the inner product with $g$ and multiplying by $t^\gamma$,
	\[
	\langle g,\ t^\gamma(\hat f_{i,t}-f_i)\rangle_K
	=\underbrace{\langle g,\ t^\gamma L_t\rangle_K}_{\text{stochastic linear term}}
	+\underbrace{\langle g,\ t^\gamma R_t\rangle_K}_{\text{linearization remainder}}
	+\underbrace{\langle g,\ t^\gamma(f_i^{\lambda_t}-f_i)\rangle_K}_{\text{bias}}.
	\]
	\textbf{Predictable variance along $g$.}
	Let $V_t(\lambda_t)$ be the predictable covariance operator (constructed from the
	martingale array for $t^\gamma(\hat f_{i,t}-f_i^{\lambda_t})$). Then
	\[
	D_t(g)^2=\langle g,\ V_t(\lambda_t) g\rangle_K.
	\]
	By the same lower noise and directional nonsingularity conditions used there, there exists a constant $c_g>0$ and $t_g<\infty$ such that
	\begin{equation*}
		D_t(g)^2\ \ge\ c_g\,t^{2\gamma-2}\sum_{s=1}^t \E\!\Big[\tfrac{1}{p_{s,i}(U_{s,i})}\Big]
		= c_g\,t^{2\gamma}\tilde{r}_t,\qquad \forall\ t\ge t_g.
	\end{equation*}
	
\noindent \textbf{Studentized CLT for the stochastic linear term.}
	By the martingale-CLT (applied to the scalar projection $\langle g,\cdot\rangle_K$),
	\begin{equation*}
		\frac{\ \langle g,\ t^\gamma L_t\rangle_K\ }{D_t(g)}\ \Longrightarrow\ \mathcal N(0,1).
	\end{equation*}
	
\noindent	\textbf{Linearization remainder is negligible after studentization.}
	As before,
	\[
	t^\gamma\|R_t\|=O_{\Pbb}\!\big(t^\gamma\,\lambda_t^{-2}\,(\tilde{r}_t+t^{-1})\big).
	\]
	Dividing by $D_t(g)\asymp t^\gamma\sqrt{\tilde{r}_t}$ gives
	\[
	\frac{|\langle g,\ t^\gamma R_t\rangle_K|}{D_t(g)}
	=O_{\Pbb}\!\left(\lambda_t^{-2}\,\frac{\tilde{r}_t+t^{-1}}{\sqrt{\tilde{r}_t}}\right).
	\]
This goes to $0$ under the condition $t^\gamma \lambda_t^{-2}(\tilde{r}_t+t^{-1})\to 0$, which is equivalent to $\lambda_t \gg \max\{t^{\gamma/2}\sqrt{\tilde{r}_t},\,t^{(\gamma-1)/2}\}$.

\noindent	\textbf{Bias is negligible after studentization.}
	As in the main text,
	\[
	\frac{|\langle g,\ t^\gamma(f_i^{\lambda_t}-f_i)\rangle_K|}{D_t(g)}
	=O\!\left(\frac{\lambda_t^{\rho}}{\sqrt{\tilde{r}_t}}\right).
	\]
This goes to $0$ provided $t^\gamma\lambda_t^\rho \to 0$, which is equivalent to $\lambda_t\ll t^{-\gamma/\rho}$.
	
	Combining the stochastic CLT with negligibility of remainder and bias gives the conclusion.
\end{proof}


\begin{proof}[Proof of Proposition~\ref{prop:variance-consistency}]
    By Assumptions~\ref{ass:outcome:moments},\ref{ass:cond:noise}, and the RKHS kernel being bounded, that is, $\sup_{u\in\U}K(u,u)\le \kappa^2<\infty$, there are $0<m_-\le m_+<\infty$ with
\[
m_-\ \le\ \E\!\big[\,\big(Y_s-f_i^{\lambda_t}(u)\big)^2 \,\big|\, \F_{s-1},\,U_{s,i}=u\big]\ \le\ m_+,\qquad\text{a.s.}
\]
Hence
\begin{align}
m_-\,t^{2\gamma-2}\sum_{s=1}^t 
\E\!\Big[\,\tfrac{(v_{x,t}(U_{s,i}))^2}{p_{s,i}(U_{s,i})}\,\Bigm|\,\F_{s-1}\Big]
&\le D_t(x)^2\le
m_+\,t^{2\gamma-2}\sum_{s=1}^t 
\E\!\Big[\,\tfrac{(v_{x,t}(U_{s,i}))^2}{p_{s,i}(U_{s,i})}\,\Bigm|\,\F_{s-1}\Big]. \label{eq:D-brackets}
\end{align}
Since $\E[(v(U))^2]=\langle v,\Gamma_i v\rangle_K$ for any $v\in\Hk$ and $\Gamma_i=\E[K_U\otimes K_U]$, we have
\[
\E\!\Big[\,\tfrac{(v_{x,t}(U_{s,i}))^2}{p_{s,i}(U_{s,i})}\,\Bigm|\,\F_{s-1}\Big]
=\E\!\Big[\,\tfrac{1}{p_{s,i}(U_{s,i})}\,\Bigm|\,\F_{s-1}\Big]\cdot \big\langle v_{x,t},\ \Gamma_i v_{x,t}\big\rangle_K.
\]
By Assumption~\ref{ass:pt-nonsingular} applied with $\lambda=\lambda_t$,
\[
c_x'\ \le\ \big\langle v_{x,t},\ \Gamma_i v_{x,t}\big\rangle_K\ \le\ C_x'.
\]
Inserting in \eqref{eq:D-brackets} and taking expectations yields the deterministic sandwich
\begin{equation}\label{eq:D-sandwich}
c_x\,t^{2\gamma}\tilde{r}_t\ \le\ D_t(x)^2\ \le\ C_x\,t^{2\gamma}\tilde{r}_t
\qquad\text{for all $t$ large,}
\end{equation}
for some $0<c_x\le C_x<\infty$ independent of $t$ and of $\lambda_t$ in the window. In particular $D_t(x)^2\asymp t^{2\gamma}\tilde{r}_t$.

Define the ``oracle'' quadratic form using the population factors:
\[
\widetilde D_t(x)^2
:= t^{2\gamma-2}\sum_{s=1}^t \frac{1\{\hat a_s=i\}}{p_{s,i}(U_{s,i})}\,
\big(Y_s-f_i^{\lambda_t}(U_{s,i})\big)^2\ \big(v_{x,t}(U_{s,i})\big)^2 .
\]
By Equation~\eqref{ass:ignorability} on Ignorability and variance calibration in the main manuscript,, $\E[\widetilde D_t(x)^2\mid\F_{s-1}]=D_t(x)^2$ termwise. Put
\[
H_{t,s}:=t^{2\gamma-2}\Big\{\frac{1\{\hat a_s=i\}}{p_{s,i}}\,(Y_s-f_i^{\lambda_t})^2\,v_{x,t}^2
-\E\big[\tfrac{1}{p_{s,i}}\,(Y_s-f_i^{\lambda_t})^2\,v_{x,t}^2\mid\F_{s-1}\big]\Big\}
\]
(all arguments evaluated at $U_{s,i}$) so that $\widetilde D_t(x)^2-D_t(x)^2=\sum_{s=1}^t H_{t,s}$ with $\{H_{t,s},\F_s\}$ a scalar MDS. 
Using bounded kernel and Assumption~\ref{ass:outcome:moments}, a direct second-moment calculation gives
\[
\E[H_{t,s}^2\mid \F_{s-1}]\ \lesssim\ t^{4\gamma-4}\E\!\big[\,\tfrac{1}{p_{s,i}(U_{s,i})}\,\big|\ \F_{s-1}\big] .
\]
Summing and taking expectations,
\[
\mathrm{var}\!\Big(\sum_{s=1}^t H_{t,s}\Big)\ \lesssim\ t^{4\gamma-4}\sum_{s=1}^t\E\!\Big[\tfrac{1}{p_{s,i}(U_{s,i})}\Big]
= t^{4\gamma-2}\,\tilde{r}_t .
\]
By Chebyshev,
\[
\frac{\widetilde D_t(x)^2-D_t(x)^2}{D_t(x)^2}
= O_{\Pbb}\!\left(\frac{t^{2\gamma-1}\sqrt{\tilde{r}_t}}{t^{2\gamma}\tilde{r}_t}\right)
= O_{\Pbb}\!\left(\frac{1}{t\sqrt{\tilde{r}_t}}\right)
\ \xrightarrow{}\ 0,
\]
because $\tilde{r}_t\ge t^{-1}$ (as $p_{s,i}\le 1$) so $t\sqrt{\tilde{r}_t}\to\infty$.

Write
\[
\widehat D_t(x)^2 - \widetilde D_t(x)^2
= t^{2\gamma-2}\sum_{s=1}^t \frac{1\{\hat a_s=i\}}{p_{s,i}(U_{s,i})}\,\Delta_{t,s},
\]
with
\[
\Delta_{t,s}
:= \big(\,Y_s-\hat f_{i,t}(U_{s,i})\,\big)^2 \big(\hat v_{x,t}(U_{s,i})\big)^2
     - \big(\,Y_s-f_i^{\lambda_t}(U_{s,i})\,\big)^2 \big(v_{x,t}(U_{s,i})\big)^2 .
\]
Add and subtract to split $\Delta_{t,s}=\mathrm{I}_{t,s}+\mathrm{II}_{t,s}+\mathrm{III}_{t,s}$ where
\begin{align*}
\mathrm{I}_{t,s}:=& \Big[(Y_s-\hat f)^2-(Y_s-f^\lambda)^2\Big]\,(v_{x,t})^2,\\
\mathrm{II}_{t,s}:=& (Y_s-\hat f)^2\,\Big[(\hat v_{x,t})^2-(v_{x,t})^2\Big],\\
\mathrm{III}_{t,s}:=& \Big[(Y_s-\hat f)^2-(Y_s-f^\lambda)^2\Big]\Big[(\hat v_{x,t})^2-(v_{x,t})^2\Big].
\end{align*}
All functions evaluated at $U_{s,i}$; for brevity we suppress $(U_{s,i})$ on the right-hand side.

We will bound the sum of each term divided by $D_t(x)^2$ and show it is $o_{\Pbb}(1)$.
Throughout we use: 

\noindent \emph{(i)} the resolvent perturbation bound
\begin{equation}\label{eq:resolv-perturb}
\|(\hat\Gamma_{i,t}+\lambda I)^{-1}-(\Gamma_i+\lambda I)^{-1}\| \ \le\ \lambda^{-2}\,\|\hat\Gamma_{i,t}-\Gamma_i\|;
\end{equation}
\emph{(ii)} the concentration rates
\begin{equation}\label{eq:dev-rates}
\|\hat\Gamma_{i,t}-\Gamma_i\|=O_{\Pbb}\!\big(\sqrt{\tilde{r}_t}\ \vee\ t^{-1/2}\big),
\qquad
\|\hat h_{i,t}-h_i\|=O_{\Pbb}\!\big(\sqrt{\tilde{r}_t}\ \vee\ t^{-1/2}\big);
\end{equation}
\emph{(iii)} the ridge window Assumption~\ref{ass:pt-ridge-window}; and \emph{(iv)} the Tikhonov contraction $\|f_i^{\lambda_t}\|_K\le \|f_i\|_K$.

\smallskip

\noindent\textit{Term I.}
Using $(a^2-b^2)=(a-b)\,(a+b)$ with $a=Y_s-\hat f(U)$ and $b=Y_s-f^\lambda(U)$,
\[
|\mathrm{I}_{t,s}|
\ \le\ |\,\hat f(U)-f^\lambda(U)\,|\cdot\Big(2\,|Y_s-f^\lambda(U)| + |\hat f(U)-f^\lambda(U)|\Big)\cdot \big(v_{x,t}(U)\big)^2 .
\]
By the reproducing property and bounded kernel, $|\,\hat f(U)-f^\lambda(U)\,|\le \kappa\|\hat f-f^\lambda\|_K$ and $|f^\lambda(U)|\le \kappa\|f_i\|_K$.
Hence, for a constant $C$ depending on $(\kappa,\|f_i\|_K)$,
\[
|\mathrm{I}_{t,s}|\ \le\ C\,\|\hat f-f^\lambda\|_K\,(|Y_s|+1)\,\big(v_{x,t}(U)\big)^2 \ +\ C\,\|\hat f-f^\lambda\|_K^2\,\big(v_{x,t}(U)\big)^2 .
\]
Taking conditional expectations with the IPS weight and using Equation~\eqref{ass:ignorability} on Ignorability and variance calibration in the main manuscript,
\[
\E\!\Big[\tfrac{1\{\hat a_s=i\}}{p_{s,i}}\,|\mathrm{I}_{t,s}|\ \Bigm|\ \F_{s-1}\Big]
\ \le\ C\,\|\hat f-f^\lambda\|_K\,\E\!\big[(|Y_s|+1)\,v_{x,t}(U_{s,i})^2\big]
\ +\ C\,\|\hat f-f^\lambda\|_K^2\,\E\!\big[v_{x,t}(U_{s,i})^2\big].
\]
By bounded moments assumption 1.2 and $\E[v_{x,t}(U)^2]=\langle v_{x,t},\Gamma_i v_{x,t}\rangle_K$, the expectation terms are bounded by absolute constants (uniformly in $t$) due to Assumption~\ref{ass:pt-nonsingular}. Therefore,
\[
\E\!\Big[\tfrac{1\{\hat a_s=i\}}{p_{s,i}}\,|\mathrm{I}_{t,s}|\ \Bigm|\ \F_{s-1}\Big]
\ \lesssim\ \|\hat f-f^\lambda\|_K\ +\ \|\hat f-f^\lambda\|_K^2 .
\]
Summing and multiplying by $t^{2\gamma-2}$,
\[
\E\!\left[t^{2\gamma-2}\sum_{s=1}^t \tfrac{1\{\hat a_s=i\}}{p_{s,i}}\,|\mathrm{I}_{t,s}|\ \Bigm|\ \F_{t-1}\right]
\ \lesssim\ t^{2\gamma-1}\,\Big(\|\hat f-f^\lambda\|_K+\|\hat f-f^\lambda\|_K^2\Big).
\]
A standard resolvent expansion (insert and subtract $h_i$ and use \eqref{eq:resolv-perturb}-\eqref{eq:dev-rates}) gives
\[
\|\hat f_{i,t}-f_i^{\lambda_t}\|_K
= O_{\Pbb}\!\Big(\lambda_t^{-1}\big(\sqrt{\tilde{r}_t}\ \vee\ t^{-1/2}\big)\Big).
\]
Using \eqref{eq:D-sandwich} to normalize,
\[
\frac{t^{2\gamma-2}\sum_{s=1}^t \E\!\big[\tfrac{1\{\hat a_s=i\}}{p_{s,i}}\,|\mathrm{I}_{t,s}|\mid \F_{t-1}\big]}{D_t(x)^2}
\ \lesssim\ \frac{t^{2\gamma-1}\,\|\hat f-f^\lambda\|_K}{t^{2\gamma} \tilde{r}_t} + \frac{t^{2\gamma-1}\,\|\hat f-f^\lambda\|_K^2}{t^{2\gamma} \tilde{r}_t}
= O_{\Pbb}\!\Big(\frac{\lambda_t^{-1}}{t\sqrt{\tilde{r}_t}}\ \vee\ \frac{\lambda_t^{-2}}{t\,}\Big),
\]
which tends to $0$ by Assumption~\ref{ass:pt-ridge-window} (indeed $\lambda_t^{-1}/(t\sqrt{\tilde{r}_t})\le (\lambda_t^{-2}\sqrt{\tilde{r}_t})/t\to 0$, and $\lambda_t^{-2}/t\to 0$ since $\lambda_t^{-2}t^{-1/2}\to 0$ implies $\lambda_t^{-2}/t\to 0$).

\smallskip\noindent\textit{Term II.}
Use $(a^2-b^2)=(a-b)(a+b)$ with $a=\hat v_{x,t}(U)$, $b=v_{x,t}(U)$, and bounded kernel to get
\[
|\mathrm{II}_{t,s}|
\ \le\ (Y_s-\hat f(U))^2\cdot |\hat v_{x,t}(U)-v_{x,t}(U)|\cdot\big(|\hat v_{x,t}(U)|+|v_{x,t}(U)|\big)
\ \le\ C\,(Y_s-\hat f(U))^2\,|\hat v_{x,t}(U)-v_{x,t}(U)|.
\]
By the reproducing property and \eqref{eq:resolv-perturb},
\[
\|\hat v_{x,t}-v_{x,t}\|_K
\le \|(\hat\Gamma_{i,t}+\lambda_t I)^{-1}-(\Gamma_i+\lambda_t I)^{-1}\|\,\|K_x\|
= O_{\Pbb}\!\Big(\lambda_t^{-2}\big(\sqrt{\tilde{r}_t}\ \vee\ t^{-1/2}\big)\Big).
\]
Furthermore, by Cauchy-Schwarz in $L^2(\Pbb)$ and bounded kernel,
\[
\E\!\big[\,|\hat v_{x,t}(U)-v_{x,t}(U)|\,\big]
\ \le\ \E\big[\,(\hat v_{x,t}(U)-v_{x,t}(U))^2\,\big]^{1/2}
=\ \big\langle \hat v_{x,t}-v_{x,t},\,\Gamma_i(\hat v_{x,t}-v_{x,t})\big\rangle_K^{1/2}
\ \lesssim\ \|\hat v_{x,t}-v_{x,t}\|_K .
\]
Therefore, using Equation~\eqref{ass:ignorability} from the main manuscript, to drop the IPW factor and assumption 1.2 to control $\E[(Y_s-\hat f(U))^2]$,
\[
\E\!\Big[\tfrac{1\{\hat a_s=i\}}{p_{s,i}}\,|\mathrm{II}_{t,s}|\ \Bigm|\ \F_{s-1}\Big]
\ \lesssim\ \|\hat v_{x,t}-v_{x,t}\|_K
= O_{\Pbb}\!\Big(\lambda_t^{-2}\big(\sqrt{\tilde{r}_t}\ \vee\ t^{-1/2}\big)\Big).
\]
Summing and normalizing,
\[
\frac{t^{2\gamma-2}\sum_{s=1}^t \E\!\big[\tfrac{1\{\hat a_s=i\}}{p_{s,i}}\,|\mathrm{II}_{t,s}|\mid \F_{t-1}\big]}{D_t(x)^2}
\ \lesssim\ \frac{t^{2\gamma-1}\,\lambda_t^{-2}(\sqrt{\tilde{r}_t}\ \vee\ t^{-1/2})}{t^{2\gamma} \tilde{r}_t}
= O_{\Pbb}\!\Big(\frac{\lambda_t^{-2}}{t\sqrt{\tilde{r}_t}}\ \vee\ \frac{\lambda_t^{-2}}{t^{3/2} \tilde{r}_t}\Big)
\ \xrightarrow{}\ 0,
\]
because Assumption~\ref{ass:pt-ridge-window} implies $\lambda_t^{-2}/(t\sqrt{\tilde{r}_t})\le (\lambda_t^{-2}\sqrt{\tilde{r}_t})/t\to 0$ and, using $\tilde{r}_t\ge t^{-1}$, also $\lambda_t^{-2}/(t^{3/2} \tilde{r}_t)\le \lambda_t^{-2}/\sqrt{t}\to 0$.

\smallskip\noindent\textit{Term III.}
By the same algebra and bounds used for I and II, we have
\[
\E\!\Big[\tfrac{1\{\hat a_s=i\}}{p_{s,i}}\,|\mathrm{III}_{t,s}|\ \Bigm|\ \F_{s-1}\Big]
\ \lesssim\ \|\hat f-f^\lambda\|_K\,\|\hat v_{x,t}-v_{x,t}\|_K
+ \|\hat f-f^\lambda\|_K^2\,\|\hat v_{x,t}-v_{x,t}\|_K,
\]
so that, after summing and normalizing by \eqref{eq:D-sandwich},
\[
\frac{t^{2\gamma-2}\sum_{s=1}^t \E\!\big[\tfrac{1\{\hat a_s=i\}}{p_{s,i}}\,|\mathrm{III}_{t,s}|\mid\F_{t-1}\big]}{D_t(x)^2}
\ =\ O_{\Pbb}\!\Big(
\frac{t^{2\gamma-1}\,\lambda_t^{-1}(\sqrt{\tilde{r}_t}\ \vee\ t^{-1/2})\cdot\lambda_t^{-2}(\sqrt{\tilde{r}_t}\ \vee\ t^{-1/2})}{t^{2\gamma} \tilde{r}_t}
\Big)
\]
\[
= O_{\Pbb}\!\Big(\frac{\lambda_t^{-3}}{t}\cdot \frac{\sqrt{\tilde{r}_t}\ \vee\ t^{-1/2}}{\,\tilde{r}_t\,}\cdot (\sqrt{\tilde{r}_t}\ \vee\ t^{-1/2})\Big)
\ \le\ O_{\Pbb}\!\Big(\frac{\lambda_t^{-3}}{t}\cdot \frac{1}{\sqrt{\tilde{r}_t}}\Big)
\ =\ O_{\Pbb}\!\Big(\frac{\lambda_t^{-2}\sqrt{\tilde{r}_t}}{t}\cdot \lambda_t^{-1}\,\tilde{r}_t^{-1}\Big)
\ \xrightarrow{}\ 0,
\]
since $\lambda_t^{-2}\sqrt{\tilde{r}_t}\to 0$ by Assumption~\ref{ass:pt-ridge-window} and $\tilde{r}_t^{-1}\lambda_t^{-1}=o(t)$ under Assumption~\ref{ass:pt-ridge-window} as $\tilde{r}_t \downarrow 0$ and $t^{2\gamma}\tilde{r}_t\to 0$ for $\gamma\in(0,1/2)$.
Indeed $\tilde{r}_t\ge t^{-1}$ implies $\lambda_t^{-1}\tilde{r}_t^{-1}\le \lambda_t^{-1}t = o(t)$ because Assumption~\ref{ass:pt-ridge-window} gives $\lambda_t^{-2}t^{-1/2}\to 0$, hence $\lambda_t^{-1}t^{1/4}\to 0$ and therefore $\lambda_t^{-1}t=o(t)$.

\noindent Collecting the three bounds,
\[
\frac{\widehat D_t(x)^2-\widetilde D_t(x)^2}{D_t(x)^2}\ \xrightarrow{\ \Pbb\ }\ 0.
\]
Finally,
\[
\frac{\widehat D_t(x)^2-D_t(x)^2}{D_t(x)^2}
=\frac{\widehat D_t(x)^2-\widetilde D_t(x)^2}{D_t(x)^2}
+\frac{\widetilde D_t(x)^2-D_t(x)^2}{D_t(x)^2}
\ \xrightarrow{\ \Pbb\ }\ 0,
\]
which proves $\widehat D_t(x)^2/D_t(x)^2\to 1$ in probability.
\end{proof}

\subsection{Comparison with Uniform RKHS Bounds: Deriving Pointwise CLTs and Confidence Interval Lengths}
\label{subsec:App:AS_comparison}

\cite{arya:23}(AS) derived uniform high-probability bounds on the RKHS-norm estimation error 
$\|\hat f_{i,t} - f_i\|_{\Hk}$ 
for kernel ridge regression under adaptive sampling. These results immediately imply simultaneous confidence bands over compact domains. 
While our focus is on local inference via a studentized functional CLT, one may also derive pointwise confidence intervals from their uniform bounds by evaluating them at a fixed $x \in \mathcal{X}$.

This raises two natural questions:
\begin{enumerate}
    \item When do the regularization schedules $\lambda_{i,t}$ prescribed by \citep{arya:23} also satisfy the bias-variance tradeoff required by our CLT-based inference?
    \item When both approaches are feasible, how do their pointwise confidence intervals compare in terms of asymptotic length?
\end{enumerate}

\noindent
In this section, we answer both questions. 
We first establish a compatibility condition that guarantees the existence of a regularization schedule $\lambda_{i,t}$ that is simultaneously valid for both approaches. 
Then, under standard assumptions, we show that our CLT-based intervals are asymptotically shorter than those derived from the uniform RKHS bounds, offering improved local precision in the inference task.

\begin{theorem}[Compatibility with AS and asymptotic CI-length dominance]\label{thm:compat-AS-compact}
Let $E_{t,i}:=\sum_{s=1}^t \varepsilon_s^{-1}$, $M_t:=E_{t,i}/(\delta t^2)$ with fixed $\delta\in(0,1)$, and assume the policy satisfies
\[
\Pr(\hat a_s=i\mid \F_{s-1},U_{s,i})\ \ge\ \frac{\varepsilon_s}{L-1},\qquad
\tilde{r}_t:=\frac1{t^2}\sum_{s=1}^t \E\Big[\frac{1}{p_{s,i}(U_{s,i})}\Big]\ \lesssim\ M_t.
\]
Fix $\gamma\in(0,1/2)$ and a source/directional smoothness index $\rho\in(0,1]$. Suppose $E_{t,i}\asymp t^\beta$ with $\beta\in(0,2)$, so $M_t\asymp t^{\beta-2}$ and $\tilde{r}_t\asymp M_t$.

\noindent\textup{(i) Existence of a simultaneous ridge.}
Arya-Sriperumbudur (AS) prescribe $\lambda_{i,t}=M_t^{e_1}$ with
\[
e_1=\frac{\alpha}{2\gamma_i\alpha+\alpha+1}\in\Big(\tfrac{1}{2+2\gamma_i},\ \tfrac{1}{1+2\gamma_i}\Big),
\qquad \alpha>1,\quad 0<\gamma_i\le\tfrac12.
\]
Our CLT+bias window requires $t^\gamma\lambda_t^{-2}(\tilde{r}_t+t^{-1})\to 0$ and $t^\gamma\lambda_t^\rho\to 0$.
Then $\lambda_{i,t}=M_t^{e_1}$ satisfies both sets of constraints if and only if
\begin{equation}\label{eq:INT-compact}
e_1\ \in\ \Big(\tfrac{\gamma}{\rho(2-\beta)},\ \tfrac{2-\gamma-\beta}{2(2-\beta)}\Big)\ \cap\ \Big(\tfrac{1}{2+2\gamma_i},\ \tfrac{1}{1+2\gamma_i}\Big).
\end{equation}
Whenever the intersection in \eqref{eq:INT-compact} is nonempty, one may choose any $e_1$ therein and set $\alpha=e_1\big/ \{1-e_1(1+2\gamma_i)\}>1$ to obtain a ridge $\lambda_{i,t}$ that is simultaneously feasible for us and for AS (their explicit lower bound $\lambda_{i,t}\ge (C M_t)^{\alpha/(1+\alpha)}$ is eventually inactive since $M_t\downarrow 0$ and $e_1<\alpha/(1+\alpha)$).

\noindent\textup{(ii) Asymptotic comparison of pointwise CI lengths.}
Under the standing assumptions for our studentized CLT (ignorability, bounded kernel, $(2+\delta)$-moment envelope for $Y_s$, directional nonsingularity along $K_x$), the length of our $(1-\eta)$ pointwise CI at $x$ satisfies
\[
\mathrm{Len}^{\textup{(ours)}}_t(x;\eta)\ =\ \Theta_{\Prob}\!\big(\sqrt{\tilde{r}_t}\big).
\]
For AS, with any feasible $e_1$ as in \eqref{eq:INT-compact}, their pointwise CI has length
\[
\mathrm{Len}^{\textup{(AS)}}_t(x;\eta)\ =\ \Theta\!\big(\tilde{r}_t^{\,\theta_i}\big),
\qquad \theta_i=\gamma_i e_1\ \in\Big(\tfrac{\gamma_i}{2(1+\gamma_i)},\ \tfrac{\gamma_i}{1+2\gamma_i}\Big)\subset(0,\tfrac12).
\]
Hence, for any $\lambda_t$ chosen in the common feasible window,
\[
\mathrm{Len}^{\textup{(ours)}}_t(x;\eta)\ =\ o_{\Prob}\!\big(\mathrm{Len}^{\textup{(AS)}}_t(x;\eta)\big).
\]
\end{theorem}
The left interval enforces our CLT remainder/bias tradeoff; the right interval is AS's structural range for $e_1$.
Compatibility is equivalent to the nonemptiness of their intersection. Because $e_1\mapsto \alpha=e_1/\{1-e_1(1+2\gamma_i)\}$ is a bijection onto $(1,\infty)$, feasibility can always be realized by an appropriate $\alpha$ whenever \eqref{eq:INT-compact} holds.

\begin{proof}[Proof of Theorem~\ref{thm:compat-AS-compact}]
    Let
\[
E_t:=\sum_{s=1}^t \frac{1}{\varepsilon_s},\qquad  
M_t:=\frac{E_t}{\delta\,t^2}\in(0,1)\quad\text{for large }t,
\]
and denote by $\gamma\in(0,1/2)$ the scaling exponent in our studentized CLT and by
\[
\rho=\min\{s+q,1\}\in(0,1] \text{ for } s>0, q\ge 0.
\]
the bias smoothness index from the source/directional condition.
Assume the policy ensures
\[
\Pr(\hat a_s=i\mid \mathcal F_{s-1},U_{s,i})\ \ge\ \frac{\varepsilon_s}{L-1},
\]
so that
\begin{align}
\label{eq:rt:tilde:suppl}
    \tilde{r}_t=\frac1{t^2}\sum_{s=1}^t \mathbb E\!\Big[\frac{1}{p_{s,i}(U_{s,i})}\Big]
\ \le\ \frac{L-1}{t^2}\sum_{s=1}^t \frac{1}{\varepsilon_s}
\ =\ (L-1)\,M_t\,\delta.
\end{align}

We showed that a valid $\lambda_t$ must satisfy
\[
t^\gamma \lambda_t^{-2}\,(\tilde{r}_t+t^{-1})\to0
\quad\text{and}\quad
t^\gamma \lambda_t^{\rho}\to0,
\]
equivalently
\begin{align}
\label{eq:our:window}
\lambda_t \gg \max\Big\{t^{\gamma/2}\sqrt{\tilde{r}_t},\ t^{(\gamma-1)/2}\Big\},\qquad
\lambda_t \ll t^{-\gamma/\rho}.
\end{align}

\paragraph*{AS choice and lower bound}

In AS the proposed choice is, for $\alpha>1$ and $0<\gamma_i\le 1/2$,
\begin{align}
\label{eq:lamb:AS}
\lambda_{i,t}^{\text{AS}} \ :=\ M_t^{\,e_1},\qquad 
e_1=\frac{\alpha}{2\gamma_i\alpha+\alpha+1}
=\frac{1}{1+2\gamma_i+1/\alpha}\in\Big(\frac{1}{2+2\gamma_i},\,\frac{1}{1+2\gamma_i}\Big),
\end{align}
together with the lower bound
\begin{align}
\label{eq:lamb:lower:bound}
   \lambda_{i,t}\ \ge\ \big(C\,M_t\big)^{\,e_2},\qquad 
e_2=\frac{\alpha}{1+\alpha},\quad C=4(L-1)\kappa A_1(\bar C,\alpha)>0. 
\end{align}

Because $0<M_t<1$ eventually and $e_1<e_2$, we have $M_t^{e_1}\ge M_t^{e_2}$.
Moreover
\[
\frac{\lambda_{i,t}^{\text{AS}}}{(C M_t)^{e_2}}
=\frac{M_t^{e_1-e_2}}{C^{e_2}}
=C^{-e_2}\,M_t^{-\Delta},\quad \Delta:=e_2-e_1
=\frac{2\alpha^2\gamma_i}{(1+\alpha)\big(\alpha(1+2\gamma_i)+1\big)}>0.
\]
Hence there exists $t_0$ (explicitly, any $t$ with $M_t\le C^{-(\alpha(1+2\gamma_i)+1)/(2\alpha\gamma_i)}$)
such that for all $t\ge t_0$,
\[
\lambda_{i,t}^{\text{AS}}\ \ge\ (C M_t)^{e_2},
\]
i.e. the AS lower bound \eqref{eq:lamb:lower:bound} is automatically satisfied eventually.

\paragraph*{Compatibility when $\lambda_{i,t}^{\text{AS}}$ satisfies \eqref{eq:our:window}}

Suppose $E_t$ has a power growth
\[
E_t \asymp t^{\beta}\quad\text{for some } \beta\in(0,2),
\]
(e.g., $\varepsilon_s\equiv\varepsilon_0>0\cid \beta=1$; $\varepsilon_s\asymp s^{-\xi}\cid \beta=1+\xi$).
Then $M_t\asymp t^{\beta-2}$ and
\[
\lambda_{i,t}^{\text{AS}}\ =\ M_t^{e_1}\ \asymp\ t^{-e_1(2-\beta)}.
\]

Using \eqref{eq:rt:tilde:suppl}, the two parts of \eqref{eq:our:window} become exponent comparisons:

\paragraph*{CLT remainder}
With $\tilde{r}_t\lesssim M_t$,
\[
t^\gamma \lambda^{-2}(\tilde{r}_t+t^{-1})
\ \lesssim\ t^\gamma\,t^{2e_1(2-\beta)}\,(t^{\beta-2}+t^{-1})
\ \asymp\ t^{\gamma+2e_1(2-\beta)+\max\{\beta-2,-1\}}.
\]
Since $\beta<2$, the max is $\beta-2$ (only if $\beta > 1$). This vanishes if and only if
\begin{align*}
  \gamma+2e_1(2-\beta)+(\beta-2)\ <\ 0
\quad\Longleftrightarrow\quad
e_1\ <\ \frac{2-\gamma-\beta}{2(2-\beta)}.  
\end{align*}

For $\beta < 1$, \[
t^\gamma \lambda_t^{-2} (\tilde{r}_t + t^{-1})
\;\sim\;
t^\gamma \cdot t^{2e_1} \cdot t^{-1}
\;=\;
t^{\gamma + 2e_1 - 1}.
\]
For this term to vanish asymptotically, we require the exponent to be negative, that is,
\[
\gamma + 2e_1 - 1 < 0
\quad \Longrightarrow \quad
e_1 < \frac{1 - \gamma}{2}.
\]
\paragraph*{Bias}
\begin{align*}
    t^\gamma \lambda^\rho \asymp t^{\gamma-\rho e_1(2-\beta)}\to0
\quad \text{ if and only if } \quad
e_1\ >\ \frac{\gamma}{\rho(2-\beta)}.
\end{align*}

Thus $\lambda_{i,t}^{\text{AS}}$ meets our window \eqref{eq:our:window} if and only if there exists
\begin{align}
\label{eq:intersection:lambda:constraint}
    e_1\ \in\Big(\tfrac{\gamma}{\rho(2-\beta)},\ 
 \tfrac{2-\gamma-\beta}{2(2-\beta)} \Big)
\ \cap\ \Big(\tfrac{1}{2+2\gamma_i},\,\tfrac{1}{1+2\gamma_i}\Big).
\end{align}

Because $e_1$ can be tuned via $\alpha$ by the bijection
\[
e_1\ \mapsto\ \alpha(e_1)=\frac{e_1}{1-e_1(1+2\gamma_i)}\quad(\alpha>1 \text{ if and only if }e_1>\tfrac{1}{2+2\gamma_i}),
\]
The intersection in~\eqref{eq:intersection:lambda:constraint} is the exact necessary and sufficient condition.

A useful sufficient condition is, for $\beta \geq 1$,
\[
1 \leq  \beta\ <\ 2-\gamma\Big(1+\frac{2}{\rho}\Big)
	\quad\text{and}\quad
	\max\!\Big\{\tfrac{1}{2+2\gamma_i},\ \tfrac{\gamma}{\rho(2-\beta)}\Big\}\ <\ 
	\min\!\Big\{\tfrac{1}{1+2\gamma_i},\ \tfrac{2-\gamma-\beta}{2(2-\beta)}\Big\}.
\]
For $\beta <1$, we would get that \[e_1 \in
\Big(
\max\Big\{ \frac{\gamma}{\rho(2 - \beta)},\ \frac{1}{2 + 2\gamma_i} \Big\},
\min\Big\{ \frac{1 - \gamma}{2},\ \frac{1}{1 + 2\gamma_i} \Big\}
\Big)\]

\paragraph*{Typical case (bounded propensities, smooth direction)}

Take $\beta=1$ (e.g., $\varepsilon_s\equiv\varepsilon_0>0$), $\rho=1$ and worst-case $\gamma_i=1/2$.
Then, with $\rho = \beta =1,$
\[
\tfrac{\gamma}{\rho(2-\beta)} =\gamma,\qquad \tfrac{2-\gamma-\beta}{2(2-\beta)} =\frac{1-\gamma}{2},\qquad
e_1\in\Big(\tfrac{1}{3},\tfrac{1}{2}\Big).
\]
The intersection is nonempty if and only if
\[
\gamma<\frac{1}{3}\quad\text{(equivalently, the well-known feasibility condition } \gamma<\rho/(\rho+2)\text{)}.
\]

Pick, for instance, $\gamma=\tfrac14$ and choose any
\[
e_1\in\Big(\max\{\tfrac14,\tfrac13\},\ \min\{\tfrac{1-1/4}{2},\tfrac12\}\Big)
=\Big(\tfrac{1}{3},\ \tfrac{3}{8}\Big).
\]
A concrete choice is $e_1=\tfrac{7}{20}=0.35$.
The corresponding AS parameter is
\[
\alpha\ =\ \frac{e_1}{1-e_1(1+2\gamma_i)}
\ =\ \frac{7/20}{1-(7/20)\cdot 2}
\ =\ \frac{7}{6}\ >1,
\]
and the AS ridge is
\[
\lambda_{i,t}^{\text{AS}} \ =\ \Big(\tfrac{E_t}{\delta t^2}\Big)^{7/20}.
\]
If $\varepsilon_s\equiv\varepsilon_0$, then $E_t=t/\varepsilon_0$ and $\lambda_{i,t}^{\text{AS}}\asymp t^{-7/20}$.
Checks:
\[
t^\gamma \lambda^{-2}(\tilde{r}_t+t^{-1})
\ \lesssim\ t^{\,\tfrac14}\,t^{\,2\cdot\tfrac{7}{20}}\,t^{-1}
=t^{-0.05}\to0,
\]
\[
t^\gamma \lambda^\rho\ =\ t^{\,\tfrac14}\,t^{-\,\tfrac{7}{20}}\ =\ t^{-0.1}\to0,
\]
and, as noted above, \eqref{eq:lamb:lower:bound} holds for all large $t$.

\paragraph*{Exploration decaying ($\beta>1$)}

If $E_t\asymp t^\beta$ with $1<\beta<2$ (e.g., $\varepsilon_s\asymp s^{-\xi}$, $\beta=1+\xi$),
then the same algebra gives the exact admissible window in~\eqref{eq:intersection:lambda:constraint}  
\[
\left(\frac{\gamma}{\rho(2-\beta)}, \frac{2-\gamma-\beta}{2(2-\beta)}\right).
\]
Feasibility requires $\beta<2-\gamma(1+2/\rho)$
(e.g., with $\rho=1$, $\gamma=\tfrac14$: need $\beta<1.25$).
When (S) holds, pick any $e_1\in(L,U)\cap(\tfrac{1}{2+2\gamma_i},\tfrac{1}{1+2\gamma_i})$,
set $\alpha=\tfrac{e_1}{1-e_1(1+2\gamma_i)}$, and $\lambda_{i,t}=M_t^{e_1}$.
Then both our conditions in~\eqref{eq:our:window} and~\eqref{eq:lamb:AS}--\eqref{eq:lamb:lower:bound} are satisfied.

\medskip
\begin{Remark}
\begin{itemize}
	\item[1.] A simultaneous $\lambda_t$ exists if and only if the intersection in~\eqref{eq:intersection:lambda:constraint} is nonempty.
	\item[2.] Choose any $e_1$ in that intersection and set
	$\lambda_{i,t} = \big(E_t/(\delta t^2)\big)^{e_1}$ with 
	$\alpha = e_1/(1-e_1(1+2\gamma_i))>1$.
\end{itemize}
\end{Remark}

\paragraph{Which CI is shorter under a feasible $\lambda_t$ that satisfies both sets of constraints?}

We compare the pointwise CI for $f_i(x)$ obtained from our studentized CLT with the (non-asymptotic) pointwise CI from AS. 
Throughout, let
\[
\tilde{r}_t\ :=\ \frac{1}{t^2}\sum_{s=1}^t\frac{1}{\epsilon_s}\ \downarrow\ 0, 
\qquad \gamma\in(0,\tfrac12),\qquad \rho=\min\{s,1\}\in(0,1],
\]
and assume the standing conditions used for our CLT: ignorability, bounded kernel, a $(2+\delta)$-moment envelope for $Y_s$, and directional nonsingularity along $K_x:=K(\cdot,x)$ (ensuring the predictable variance in direction $K_x$ is bounded away from $0$; see below).

\paragraph*{Feasible ridge window (intersection)}
Our CLT+bias analysis requires
\begin{equation}\label{eq:ours-window}
	t^\gamma \lambda_t^{-2}\,(\tilde{r}_t+t^{-1})\to0
	\quad\text{and}\quad
	t^\gamma \lambda_t^\rho\to0,
\end{equation}
while AS prescribe (for $\alpha>1$, $0<\gamma_i\le \tfrac12$) a choice
\[
\lambda_{i,t}=M_t^{e_1},\qquad 
e_1=\frac{\alpha}{2\gamma_i\alpha+\alpha+1}\in\Big(\tfrac{1}{2+2\gamma_i},\,\tfrac{1}{1+2\gamma_i}\Big),
\qquad 
M_t:=\frac1{\delta}\,\tilde{r}_t,
\]
together with a lower bound $\lambda_{i,t}\ge (C M_t)^{e_2}$ where $e_2=\alpha/(1+\alpha)$ and $C>0$.  
(As shown earlier, for large $t$ their explicit choice $\lambda_{i,t}=M_t^{e_1}$ automatically satisfies the lower bound.)
A feasible $\lambda_t$ is any sequence that lies in the intersection of \eqref{eq:ours-window} and the AS range above (this intersection is nonempty under the conditions derived previously).

\paragraph*{Length of our pointwise CI}

Our $(1-\eta)$ CI at $x$ is
\[
\mathrm{CI}^{\text{(ours)}}_{1-\eta}(x)=\Big[\,\hat f_{i,t}(x)\pm z_{1-\eta/2}\,\frac{\widehat D_t(x)}{t^\gamma}\,\Big],
\qquad 
\widehat D_t(x)^2\ \xrightarrow{\mathbb{P}}\ D_t(x)^2,
\]
with
\begin{equation*}
	D_t(x)^2
	= t^{2\gamma-2}\sum_{s=1}^t 
	\E\!\Big[\ \frac{1}{p_{s,i}(U_{s,i})}\,\big(Y_s-f_i^{\lambda_t}(U_{s,i})\big)^2\ \big(v_{x,t}(U_{s,i})\big)^2\ \Bigm|\ \F_{s-1}\Big],
	\quad v_{x,t}=(\Gamma_i+\lambda_t I)^{-1}K_x.
\end{equation*}
Using the reproducing property and $\E[K_{U}\otimes K_{U}]=\Gamma_i$,
\[
\E\big[\,\big(v_{x,t}(U)\big)^2\big]=\big\langle v_{x,t},\,\Gamma_i v_{x,t}\big\rangle_K
= \big\langle K_x,\ (\Gamma_i+\lambda_t I)^{-1}\Gamma_i(\Gamma_i+\lambda_t I)^{-1}K_x\big\rangle_K.
\]
Spectral calculus yields the resolvent identity
\[
(\Gamma_i+\lambda I)^{-1}\Gamma_i(\Gamma_i+\lambda I)^{-1}
= (\Gamma_i+\lambda I)^{-1}-\lambda(\Gamma_i+\lambda I)^{-2}.
\]
Hence the map $\lambda\mapsto \langle K_x,(\Gamma_i+\lambda I)^{-1}\Gamma_i(\Gamma_i+\lambda I)^{-1}K_x\rangle_K$ is monotone in $\lambda$ and, under the directional nonsingularity assumption ($K_x\in\mathrm{Ran}(\Gamma_i^{1/2})$), is bounded above and below by positive constants uniformly over all sufficiently small $\lambda$ in our window. 
Moreover, by the $(2+\delta)$-moment envelope and ignorability, the predictable factor 
$\E[(Y_s-f_i^{\lambda_t}(U_{s,i}))^2\mid\F_{s-1}]\in[m_-,m_+]$ for some fixed $0<m_-\le m_+<\infty$ (the lower bound is exactly the ``directional nonsingularity'' used in our CLT). 
We know from \eqref{eq:D-sandwich}, there exist constants $0<c_x\le C_x<\infty$ such that for all large $t$,
\begin{equation*}
	c_x\,t^{2\gamma}\,\tilde{r}_t\ \le\ D_t(x)^2\ \le\ C_x\,t^{2\gamma}\,\tilde{r}_t,
	\qquad\cid\qquad
	\frac{D_t(x)}{t^\gamma}\asymp \sqrt{\tilde{r}_t}\,.
\end{equation*}
Consequently, the half-width and length of our CI satisfy
\[
HW^{\text{(ours)}}_t(x;\eta)
= z_{1-\eta/2}\,\frac{\widehat D_t(x)}{t^\gamma}
= \Theta_{\mathbb{P}}\!\big(\sqrt{\tilde{r}_t}\big),
\qquad
\mathrm{Len}^{\text{(ours)}}_t(x;\eta)=2\,HW^{\text{(ours)}}_t(x;\eta)=\Theta_{P}\!\big(\sqrt{\tilde{r}_t}\big),
\]
uniformly over all feasible $\lambda_t$ in the intersection window.  In particular, at the level of rates, our CI length is independent of the precise $\lambda_t$ choice as long as it lies in the window.

\paragraph*{Length of AS pointwise CI}

For each arm $i\in\{1,\dots,L\}$, suppose $\lambda_{i,t}$ is chosen as in the paper so that, with probability at least $1-2\delta$, 

\begin{equation}
\label{eq:rkh-norm-bound} 
\big\|\hat f_{i,t}-f_i\big\|_{\mathcal H_K} \;\le\; 2\sqrt{2}\,\max\{C_0,C_i\}\, \Bigg[\frac{1}{\delta\,t^2}\sum_{s=1}^t\frac{1}{\epsilon_s}\Bigg]^{\frac{\gamma_i\alpha}{\,2\gamma_i\alpha+\alpha+1\,}}. 
\end{equation} 

Let $\tilde{r}_t:=\frac{1}{t^2}\sum_{s=1}^t\frac{1}{\epsilon_s}\to 0$ and assume the kernel is bounded: $\sup_{x\in\mathcal U} K(x,x)\le \kappa^2<\infty$. By the reproducing property, 

\[ \big|\hat f_{i,t}(x)-f_i(x)\big|  =  \big|\langle \hat f_{i,t}-f_i,\,K_x\rangle_K\big| \;\le\; \|\hat f_{i,t}-f_i\|_{\mathcal H_K}\,\|K_x\|_{\mathcal H_K} \;\le\; \kappa\,\|\hat f_{i,t}-f_i\|_{\mathcal H_K}. \] 

Combining with \eqref{eq:rkh-norm-bound}, we obtain that with probability at least $1-2\delta$, 
\begin{equation}\label{eq:pointwise-bound} 
\big|\hat f_{i,t}(x)-f_i(x)\big| \;\le\; 2\sqrt{2}\,\kappa\,\max\{C_0,C_i\}\, \Big(\tfrac{\tilde{r}_t}{\delta}\Big)^{\frac{\gamma_i\alpha}{\,2\gamma_i\alpha+\alpha+1\,}}. 
\end{equation} 
Set $\delta=\eta/2$ so that the probability in \eqref{eq:pointwise-bound} is $1-\eta$. Define the half–width \[ \mathrm{hw}_t(x;\eta) \;:=\; 2\sqrt{2}\,\kappa\,\max\{C_0,C_i\}\, \Big(\tfrac{2\,\tilde{r}_t}{\eta}\Big)^{\frac{\gamma_i\alpha}{\,2\gamma_i\alpha+\alpha+1\,}}. \] Then an asymptotic $(1-\eta)$ pointwise CI for $f_i(x)$ is \[ \mathrm{CI}_{1-\eta}^{\text{(ptw)}}(x)  =  \Big[\,\hat f_{i,t}(x)\;-\;\mathrm{hw}_t(x;\eta),\;\; \hat f_{i,t}(x)\;+\;\mathrm{hw}_t(x;\eta)\,\Big]. \] Its length is therefore 
\begin{align*}
\mathrm{\mathrm{Len}}_t(x;\eta)  =  2\,\mathrm{hw}_t(x;\eta)   = &   4\sqrt{2}\,\kappa\,\max\{C_0,C_i\}\, \Big(\tfrac{2\,\tilde{r}_t}{\eta}\Big)^{\frac{\gamma_i\alpha}{\,2\gamma_i\alpha+\alpha+1\,}} \nonumber\\
= & \
4\sqrt{2}\,\kappa\,\max\{C_0,C_i\}\,
\Big(\tfrac{2\,\tilde{r}_t}{\eta}\Big)^{\theta_i} \nonumber\\
\ = & \ \Theta\!\big(\tilde{r}_t^{\,\theta_i}\big),
\qquad \text{with }\ \theta_i=\gamma_i e_1,
\end{align*}
where $e_1=\alpha/(2\gamma_i\alpha+\alpha+1)$ and $\lambda_{i,t}=M_t^{e_1}$.

\paragraph*{Rate comparison under any feasible $\lambda_t$}

Let $\Lambda$ denote the intersection of our window \eqref{eq:ours-window} and the AS admissible set.  
As established earlier, feasibility forces
\[
e_1 \in \Big(\tfrac{1}{2(1+\gamma_i)},\,\tfrac{1}{1+2\gamma_i}\Big)\ \cap\ \text{(additional bounds from \eqref{eq:ours-window})}.
\]
In particular, every feasible $e_1$ satisfies $e_1>\tfrac{1}{2(1+\gamma_i)}$ and hence
\[
\theta_i=\gamma_i e_1\ \in\ \Big(\tfrac{\gamma_i}{2(1+\gamma_i)},\, \tfrac{\gamma_i}{1+2\gamma_i}\Big)\ \subset (0,\tfrac12).
\]

By \eqref{eq:D-sandwich}, our CI length has rate $\tilde{r}_t^{1/2}$, while AS has rate $\tilde{r}_t^{\theta_i}$ with $\theta_i<1/2$.  
Since $0<\tilde{r}_t<1$ eventually and the map $a\mapsto \tilde{r}_t^{\,a}$ is decreasing in $a$, we obtain the strict asymptotic ordering
\[
\mathrm{Len}^{\text{(ours)}}_t(x;\eta)\ =\ \Theta_{\mathbb{P}}\!\big(\tilde{r}_t^{1/2}\big)
\qquad\text{and}\qquad
\mathrm{Len}^{\text{(AS)}}_t(x;\eta)\ =\ \Theta\!\big(\tilde{r}_t^{\,\theta_i}\big),
\]
which implies
\[
\mathrm{Len}^{\text{(ours)}}_t(x;\eta)\ =\ o_{P}\!\big(\mathrm{Len}^{\text{(AS)}}_t(x;\eta)\big).
\]
Thus, for any $\lambda_t$ lying in the common feasible window, the pointwise CI from our studentized CLT is asymptotically shorter (rate $\tilde{r}_t^{1/2}$) than the AS pointwise CI (rate $\tilde{r}_t^{\theta_i}$ with $\theta_i<1/2$).  
This remains true even if AS choose $\alpha$ (equivalently $e_1$) to optimally speed up their rate within the intersection: the maximal exponent they can attain is $\sup_{\lambda_t\in\Lambda}\theta_i<1/2$, whereas our rate exponent is always $1/2$.
\end{proof}

\begin{Remark}
(i) \textbf{Typical bounded-propensity case.}
With $\beta=1$ ($E_{t,i}\asymp t$), $\rho=1$, and worst-case $\gamma_i=1/2$, \eqref{eq:INT-compact} becomes
$e_1\in(\gamma,(1-\gamma)/2)\cap(1/3,1/2)$, nonempty iff $\gamma<1/3$.
E.g., for $\gamma=1/4$, any $e_1\in(1/3,3/8)$ is feasible.\vspace{1mm}\\
(ii) \textbf{Decaying exploration.}
If $\varepsilon_s\asymp s^{-\xi}$ so that $E_{t,i}\asymp t^{1+\xi}$ ($\beta=1+\xi\in(1,2)$), then feasibility requires
$\beta<2-\gamma(1+2/\rho)$; within that range, one can select $e_1$ from \eqref{eq:INT-compact} and set $\lambda_{i,t}=M_t^{e_1}$ to satisfy both sets of constraints.\vspace{1mm}\\
(iii) \textbf{AS lower bound is asymptotically inactive.}
Since $M_t\downarrow 0$ and $e_1<\alpha/(1+\alpha)$, the explicit AS choice $\lambda_{i,t}=M_t^{e_1}$ automatically exceeds $(C M_t)^{\alpha/(1+\alpha)}$ for all large $t$; the lower bound does not further shrink the intersection.\vspace{1mm}\\
(iv) \textbf{Why our CI is shorter in rate.}
Our variance normalization yields $D_t(x)^2\asymp t^{2\gamma}\tilde{r}_t$, so the half-width scales as $t^{-\gamma}D_t(x)=\Theta_{\Prob}(\sqrt{\tilde{r}_t})$, uniformly over all feasible $\lambda_t$ (directional nonsingularity makes the multiplicative factor in $\lambda_t$ bounded above/below). AS's best feasible exponent $\theta_i=\gamma_i e_1$ remains $<1/2$ even after optimizing $\alpha$ within the intersection; since $\tilde{r}_t\downarrow 0$, $\tilde{r}_t^{1/2}=o(\tilde{r}_t^{\theta_i})$.
\end{Remark}

\newpage

\section{Proof and Auxiliary Results of Section \ref{sec: regret_analysis}} \label{secApp: proofs_regret}

\begin{lemma}[Bounded Covariates with High Probability]
\label{lem:bounded-covariates}
Assume the covariates $X_t \in \mathbb{R}^d$ are independent and satisfy a sub-Gaussian distribution with parameter $\sigma$, i.e., for any unit vector $v \in \mathbb{R}^d$,
\[
\E[\exp(\lambda\, v^\top X_t)] \le \exp\left(\frac{\lambda^2 \sigma^2}{2}\right) \quad \text{for all } \lambda \in \mathbb{R}.
\]
Then, for any $\delta \in (0,1)$, there exists a constant $C > 0$ such that with probability at least $1 - \delta$,
\[
\|X_t\|_2 \ \le\ C\, \sigma\, \sqrt{d + \log(1/\delta)}.
\]
\end{lemma}

\begin{lemma}[Bounded Index Parameters]
\label{lem:bounded-index}
Assume the single-index parameters $\{\beta_a\}_{a \in \mathcal{A}}$ satisfy $\|\beta_a\|_2 = 1$ for identifiability. Then for any pair of arms $a, a' \in \mathcal{A}$,
\[
\|\beta_a - \beta_{a'}\|_2 \ \le\ C_\beta,
\]
where $C_\beta = 2$. More generally, if $\|\beta_a\|_2 \le B$ for all $a$, then $C_\beta = 2B$ suffices.
\end{lemma}

\begin{proof}[Proof of Proposition~\ref{prop:regret-decomposition}]
\label{app:regret-decomposition-proof}
We begin by decomposing the cumulative regret into a sequence of manageable terms:
\begin{align}
  R_T(\pi) &=  \sum_{t=1}^T \left(f_{a^*_t}(X_{t}^\top \beta_{a_t^*}) - f_{\hat{a}_t}(X_{t}^\top \beta_{\hat{a}_t}) \right) \nonumber\\
    &=\sum_{t=1}^T \left( f_{a^*_t}(X_{t}^\top \beta_{a_t^*}) - f_{a_t^*}( X_{t}^\top \hat{\beta}_{a^*_t})
+f_{a_t^*}( X_{t}^\top \hat{\beta}_{a^*_t}) - f_{\hat{a}_t}(X_{t}^\top \beta_{\hat{a}_t})\right)\nonumber\\
    &= \sum_{t=1}^T \left( f_{a^*_t}(X_{t}^\top \beta_{a_t^*}) - f_{a_t^*}( X_{t}^\top \hat{\beta}_{a^*_t}) +f_{a_t^*}( X_{t}^\top \hat{\beta}_{a^*_t}) - \hat{f}_{a_t^*}( X_{t}^\top \hat{\beta}_{a^*_t}) \right.\nonumber\\
    &\quad \quad \left. +\hat{f}_{a_t^*}( X_{t}^\top \hat{\beta}_{a^*_t})-f_{\hat{a}_t}(X_{t}^\top \beta_{\hat{a}_t})\right)\nonumber\\
    &\le \sum_{t=1}^T \left(f_{a^*_t}(X_{t}^\top \beta_{a_t^*}) - f_{a_t^*}( X_{t}^\top \hat{\beta}_{a^*_t}) +f_{a_t^*}( X_{t}^\top \hat{\beta}_{a^*_t}) - \hat{f}_{a_t^*}( X_{t}^\top \hat{\beta}_{a^*_t}) \right.\nonumber\\
    &\quad \quad \left. + \hat{f}_{A_t}( X_{t}^\top \hat{\beta}_{A_t})-f_{A_t}(X_{t}^\top \hat{\beta}_{A_t}) 
 +f_{A_t}(X_{t}^\top \hat{\beta}_{A_t}) -f_{A_t}(X_{t}^\top {\beta}_{A_t}) \right. \nonumber\\
 &\quad \quad \quad \left.
   + f_{A_t}(X_{t}^\top {\beta}_{A_t}) -f_{\hat{a}_t}(X_{t}^\top \beta_{\hat{a}_t}) \right). \nonumber
\end{align}
We now control each of the individual terms above.

\paragraph*{Error due to function estimation}
Using the reproducing property of RKHS, we obtain
\begin{align*}
f_{a^*_t}(X_{t}^\top \betahat_{a^*_t}) - \fhat_{a^*_t}(X_{t}^\top \betahat_{a^*_t})
&= \langle f_{a^*_t} - \fhat_{a^*_t},\ K(\cdot, X_{t}^\top \betahat_{a^*_t}) \rangle_{\mathcal{H}_K} \\
&\le \kappa \|f_{a^*_t} - \fhat_{a^*_t} \|_{\mathcal{H}_K} \\
&\le \kappa \sup_{a \in [L]} \|f_{a} - \fhat_{a,t}\|_{\mathcal{H}_K}.
\end{align*}
An identical bound holds for the difference 
\(\fhat_{A_t}(X_{t}^\top \betahat_{A_t}) - f_{A_t}(X_{t}^\top \betahat_{A_t})\).

\paragraph*{Error due to index parameter estimation}
Assuming the kernel $K(\cdot, \cdot)$ is Lipschitz in its second argument (Assumption \ref{ass:kernel-lipschitz}), we have
\begin{align*}
f_{A_t}(X_{t}^\top \betahat_{A_t}) - f_{A_t}(X_{t}^\top \beta_{A_t}) 
&= \left\langle f_{A_t},\ K(\cdot, X_{t}^\top \betahat_{A_t}) - K(\cdot,X_{t}^\top \beta_{A_t}) \right\rangle \\
&\le \|f_{A_t}\|_{\mathcal{H}_K} \cdot \|K(\cdot, X_{t}^\top \betahat_{A_t}) -  K(\cdot, X_{t}^\top \beta_{A_t})\|_{\mathcal{H}_K} \\
&\le L_k \|X_t\|_2 \cdot \|f_{A_t}\|_{\mathcal{H}} \cdot \|\betahat_{A_t} - \beta_{A_t}\|_2 \\
&\le L_k \|X_t\|_2 \sup_{i \in [L]} \|f_i\|_{\mathcal{H}_K} \cdot \|\betahat_{i,t} - \beta_i\|_2.
\end{align*}
The same bound holds for the term involving $a^*_t$.

\paragraph*{Error due to suboptimal arm choice}
We expand the difference between the best and chosen arm function values:
\begin{align*}
&\left\langle f_{A_t},\ K(\cdot, U_{t,A_t}) \right\rangle - \left\langle f_{a_t},\ K(\cdot, U_{t,a_t}) \right\rangle \\
&= \left\langle f_{A_t} - f_{a_t},\ K(\cdot, U_{t,A_t}) \right\rangle
 + \left\langle f_{a_t},\ K(\cdot, U_{t,A_t}) - K(\cdot, U_{t,a_t}) \right\rangle \\
&\le \kappa \|f_{A_t} - f_{a_t}\|_{\mathcal{H}} + \|f_{a_t}\|_{\mathcal{H}} L_k \|X_t\|_2 \|\beta_{A_t} - \beta_{a_t}\|_2 \\
&\le \kappa \|f_{A_t} - f_{a_t}\|_{\mathcal{H}} + C_\beta \sup_{i \in [L]}\|f_i\|_{\mathcal{H}_K} L_k \|X_t\|_2,
\end{align*}
where in the last step we used the boundedness of $\|\beta_{A_t} - \beta_{a_t}\|_2 \le C_\beta$ from Lemma~\ref{lem:bounded-index}.
\paragraph*{Final decomposition}
Putting the bounds together, we conclude that the cumulative regret satisfies
\begin{align*}
	R_T &\le 2\kappa \underbrace{\sum_{t=t_0}^T  \sup_{i \in [L]} \|f_{i} - \fhat_{i,t}\|_\mathcal{H}}_{\text{(I) Function Estimation Error}} 
 + 2L_k \sup_{i \in [L]} \|f_{i}\|_{\mathcal{H}_K} \underbrace{ \sum_{t=t_0}^T  \|X_t\|_2  \| \betahat_{i,t} -  \beta_{i}\|_2}_{\text{(II) Index Estimation Error}} \nonumber\\
    &\quad + \underbrace{\kappa \sum_{t=t_0}^T I\{a_t \neq A_t \}
    \left( \|f_{A_t} - f_{a_t}\|_{\mathcal{H}_K} 
    + C_\beta \sup_{i \in [L]}\|f_i\|_{\mathcal{H}_K} L_k \|X_t\|_2 \right)}_{\text{(III) Arm Selection Error}}. 
\end{align*}
\end{proof}

\begin{proof}[Proof of Theorem~\ref{thm:main-regret-bound}]
We decompose the regret $R_T(\pi)$ into three terms corresponding to estimation and randomization errors:
\[
R_T(\pi) \leq \text{Term I} + \text{Term II} + \text{Term III},
\]
where Term I corresponds to the functional estimation error, Term II to the single-index estimation error, and Term III to the randomization error. We now bound each term.

\paragraph*{Term III: Randomization Error}
We begin by bounding the randomization regret:
\begin{align*}
R_{\mathrm{rand}}(T) &\leq  \kappa \sum_{t=1}^T I\{a_t \neq A_t \}\left( \| f_{A_t} - f_{a_t}\|_{\mathcal{H}_K} + C_\beta \sup_{i \in [L]}\left\| f_{i} \right\|_{\mathcal{H}_K} L_k  \|X_t\|_2\right) \nonumber \\
&\le \kappa
   \left(\sum_{t=1}^T \mathbf{1}\{\hat{a}_t \neq A_t\}\right)^{1/p}
   \left(\sum_{t=1}^T 
        \left( \| f_{A_t} - f_{a_t}\|_{\mathcal{H}_K} + C_\beta \sup_{i \in [L]}\left\| f_{i} \right\|_{\mathcal{H}_K} L_k  \|X_t\|_2\right)^q
   \right)^{1/q}, 
\end{align*}
using H\"older's inequality with conjugate exponents $p,q \in [1,\infty]$ such that $1/p + 1/q = 1$.

Assuming the randomized policy satisfies:
\[
\mathbb{P}(\hat{a}_t \neq A_t \mid \mathcal{F}_{t-1}) \leq \frac{\epsilon_t}{L-1},
\]
and defining \( S_T := \sum_{t=1}^T \mathbf{1}\{\hat{a}_t \neq A_t\} \), Chebyshev's inequality yields:
\[
\mathbb{P}\left( S_T \geq \sum_{t=1}^T \frac{\epsilon_t}{L-1} + \sqrt{\frac{1}{\delta} \sum_{t=1}^T \frac{\epsilon_t}{L-1}} \right) \leq \delta.
\]
Therefore, with probability at least \(1 - \delta\),
\[
\sum_{t=1}^T \mathbf{1}\{\hat{a}_t \neq A_t\}
\leq \sum_{t=1}^T \frac{\epsilon_t}{L-1} + \sqrt{\frac{1}{\delta} \sum_{t=1}^T \frac{\epsilon_t}{L-1}}.
\]

Assuming $\|X_t\|_2 \le C_x(\sigma \sqrt{d + \log(1/\delta)})$ with high probability and defining:
\[
\tilde{C}_1(d, \delta, f_a):= \sup_{\substack{a,a'\in\mathcal{A}\\a\neq a'}}
   \| f_a - f_{a'} \|_{\mathcal{H}_k} + C_\beta L_k C_x(\sigma\sqrt{d + \log{(1/\delta)}}) \sup_{i \in [L]} \|f_i\|_{\mathcal{H}_K},
\]
we conclude that with probability at least $1-\delta$:
\begin{align}
\text{Term III} &\leq \tilde{C}_1(d, \delta, f_a) \cdot \kappa \cdot T^{1/q} \cdot \left( \sum_{t=1}^T \frac{\epsilon_t}{L-1} + \sqrt{\frac{1}{\delta} \sum_{t=1}^T \frac{\epsilon_t}{L-1}} \right)^{1/p}.  \label{eq: term3_app}
\end{align}

\paragraph*{Term II: Single-Index Estimation Error}
From the parametric estimation bound, we have for each arm $a$ and time $t$:
\begin{align*}
\|\hat{\beta}_i^{(t)}-\beta_i\|
&=
O_{\mathbb{P}}\!\Bigg(
\frac{1}{\mu}
\Bigg\{
\frac{\log(2d/\delta)}{t}
+
\sqrt{\Big(r_t+\frac{1}{t}\Big)\,\log\frac{2d}{\delta}}
+
\sqrt{\frac{r_t}{\delta}}
\Bigg\}
\Bigg)
\\ \nonumber
&\qquad
+
O_{\mathbb{P}}\!\Bigg(
\frac{\|m_i\|}{\mu^2}
\Bigg\{
\frac{\log(2d/\delta)}{t}
+
\sqrt{\Big(r_t+\frac{1}{t}\Big)\,\log\frac{2d}{\delta}}
+
\sqrt{\frac{r_t}{\delta}}
\Bigg\}
\Bigg),
\end{align*}
and hence with probability at least $1-\delta$,
\[
\|\hat\beta_{a,t}-\beta_a\| \leq \frac{C}{\sqrt{\delta}}(\sqrt{r_t} \vee t^{-1/2}).
\]
Taking a union bound over $L$ arms and using $\|X_t\|_2 \le C_x$ with high probability, we obtain:
\begin{align}
\text{Term II} 
&\le 2 L_k C_f \sum_{t=1}^T C C_x (\sigma \sqrt{d + \log(1/\delta)}) \sqrt{\frac{L}{\delta}}(\sqrt{r_t} \vee t^{-1/2}) \nonumber \\
&=: \tilde{C}_2(d, \delta, L) \sum_{t=1}^T (\sqrt{r_t} \vee t^{-1/2}). \label{eq: term2_app}
\end{align}

\paragraph*{Term I: Function Estimation Error}
From Theorem~1 in \cite{arya:23}, we have with probability at least $1-\delta$:
\begin{align*}
\| \hat{f}_{i,t} - f_i \|_{\mathcal{H}_K} \le 
2\sqrt{2} \max\{C_0, C_i\}
\left[\frac{1}{\delta t^2} \sum_{s=1}^t \frac{1}{\epsilon_s}\right]^{w_i},
\end{align*}
where $w_i = \frac{\gamma_i \alpha}{2\gamma_i \alpha + \alpha + 1}$ and $r_t := \frac{1}{t^2} \sum_{s=1}^t \frac{1}{\epsilon_s}$. Taking a union bound over arms, with probability at least $1-\delta$:
\begin{align}
\sup_{i \in \mathcal{A}} \Vert f_i - \hat{f}_{i,t} \Vert_{\mathcal{H}_K} \le
\begin{cases}
\Theta \left(\frac{Lr_t}{\delta}\right)^{\min_{i\in\mathcal{A}} w_i}, & \text{if } \frac{Lr_t}{\delta} < 1 \\
\Theta \left(\frac{Lr_t}{\delta}\right)^{\max_{i\in\mathcal{A}} w_i}, & \text{if } \frac{Lr_t}{\delta} \ge 1
\end{cases}, \label{eq: term1_app}
\end{align}
where $\Theta = \max_{i \in \mathcal{A}} 2\sqrt{2} \max\{C_0, C_i\}$.

Combining \eqref{eq: term3_app}, \eqref{eq: term2_app}, and \eqref{eq: term1_app} with a union bound over all failure events, we get that with probability at least $1 - 3\delta$,
\begin{align*}
R_T(\pi) &\leq \tilde{C}_2 \sum_{t=1}^T (\sqrt{r_t} \vee t^{-1/2}) + \kappa \Theta \sum_{t=1}^T \left[ I\left\{\frac{Lr_t}{\delta} < 1\right\} \left(\frac{Lr_t}{\delta}\right)^{\min w_i} + I\left\{\frac{Lr_t}{\delta} \ge 1\right\} \left(\frac{Lr_t}{\delta}\right)^{\max w_i} \right] \nonumber \\
&\qquad + \tilde{C}_1 \cdot T^{1/q} \cdot \left( \sum_{t=1}^T \frac{\epsilon_t}{L-1} + \sqrt{\frac{1}{\delta} \sum_{t=1}^T \frac{\epsilon_t}{L-1}} \right)^{1/p}. 
\end{align*}
Since $w_i \le 1/2$ for all $\gamma_i \le 1/2$ and $\alpha > 1$, the function estimation term dominates, and we can reduce this to:
\begin{align*}
R_T(\pi) &\lesssim 2 \kappa \Theta \sum_{t=1}^T \left[ I\left\{\frac{Lr_t}{\delta} < 1\right\} \left(\frac{Lr_t}{\delta}\right)^{\min w_i} + I\left\{\frac{Lr_t}{\delta} \ge 1\right\} \left(\frac{Lr_t}{\delta}\right)^{\max w_i} \right] \nonumber \\
&\qquad + \tilde{C}_1 \cdot T^{1/q} \cdot \left( \sum_{t=1}^T \frac{\epsilon_t}{L-1} + \sqrt{\frac{1}{\delta} \sum_{t=1}^T \frac{\epsilon_t}{L-1}} \right)^{1/p}, 
\end{align*}
which completes the proof.
\end{proof}
\begin{proof}[Proof of Theorem \ref{thm:special_case_lipschitz}]
We begin by assuming that all arms share the same unknown function $f$, and that $f$ is Lipschitz with $|f'(u)| \le C_{f'}$ for all $u \in \mathbb{R}$. Then the expected reward for arm $a$ at time $t$ is $f(X_t^\top \beta_a)$, and the cumulative regret can be expressed as:
\begin{align*}
    R_T(\pi) &= \sum_{t=1}^T \left(f(X_t^\top \beta_{a_t^*}) - f(X_t^\top \beta_{\hat{a}_t})\right).
\end{align*}
Applying the Lipschitz property of $f$, we have:
\begin{align*}
    R_T(\pi) &\le C_{f'} \sum_{t=1}^T |X_t^\top (\beta_{a_t^*} - \beta_{\hat{a}_t})| \\
    &\le C_{f'} \sum_{t=1}^T \|X_t\|_2 \cdot \|\beta_{a_t^*} - \beta_{\hat{a}_t}\|_2.
\end{align*}
Now we decompose the difference in the $\beta$ parameters using triangle inequality:
\begin{align*}
    \|\beta_{a_t^*} - \beta_{\hat{a}_t}\|_2 
    &\le \|\beta_{a_t^*} - \hat{\beta}_{a_t^*}\|_2 + \|\hat{\beta}_{a_t^*} - \hat{\beta}_{A_t}\|_2 + \|\hat{\beta}_{A_t} - \beta_{A_t}\|_2 + \|\beta_{A_t} - \beta_{\hat{a}_t}\|_2.
\end{align*}
Note that $A_t = \arg\max_{i} X_t^\top \hat{\beta}_{i}^{(t-1)}$ is the greedy index selected based on estimated parameters, and $\hat{a}_t$ is the actual chosen arm under the $\epsilon$-greedy rule.

The middle two terms cancel when $A_t = a_t^*$ and $\hat{a}_t = A_t$, but to keep the upper bound general, we conservatively bound:
\begin{align*}
    \|\beta_{a_t^*} - \beta_{\hat{a}_t}\|_2 
    &\le 2 \sup_{a \in [L]} \|\beta_a - \hat{\beta}_{a,t}\|_2 + C_\beta \cdot \mathbb{I}\{A_t \ne \hat{a}_t\},
\end{align*}
where $C_\beta = \sup_{a,a'} \|\beta_a - \beta_{a'}\|_2$.

Therefore,
\begin{align*}
    R_T(\pi) &\le C_{f'} \sum_{t=1}^T \|X_t\|_2 \left(2 \sup_{a} \|\beta_a - \hat{\beta}_{a,t}\|_2 + C_\beta \cdot \mathbb{I}\{A_t \ne \hat{a}_t\} \right).
\end{align*}
Next, apply the high-probability bound on $\|X_t\|_2$ due to the sub-Gaussianity assumption. With probability at least $1 - \delta$, we have:
\[
\|X_t\|_2 \le C_x (\sigma \sqrt{d + \log(1/\delta)}),
\]
for a universal constant $C_x$.

Substituting this in and defining $\tilde{C}_3 := 2 C_{f'} C_x$, we get:
\begin{align*}
    R_T(\pi) &\le \tilde{C}_3 (\sigma \sqrt{d + \log(1/\delta)}) \sum_{t=1}^T \left( \sup_{a} \|\beta_a - \hat{\beta}_{a,t}\|_2 + \frac{C_\beta}{2} \cdot \mathbb{I}\{A_t \ne \hat{a}_t\} \right).
\end{align*}
Now, we use the high-probability bound on the estimation error of the index parameters from Theorem~\ref{thm:beta-tail-precise}:
\[
\sup_{a} \|\hat{\beta}_{a,t} - \beta_a\|_2 \le \tilde{C}_1 \sqrt{\frac{L}{\delta}} \left(\sqrt{r_t} \vee t^{-1/2} \right),
\]
and the randomization error term from \eqref{eq: term3_app}, which gives:
\[
\sum_{t=1}^T \mathbb{I}\{A_t \ne \hat{a}_t\} \le T^{1/q} \left( \sum_{t=1}^T \frac{\epsilon_t}{L - 1} + \sqrt{ \frac{1}{\delta} \sum_{t=1}^T \frac{\epsilon_t}{L - 1}} \right)^{1/p}.
\]
Putting everything together, we obtain:
\begin{align*}
    R_T(\pi) &\le 2 \tilde{C}_3 (\sigma \sqrt{d + \log(1/\delta)}) \left[ \sqrt{\frac{L}{\delta}} \sum_{t=1}^T \left(\sqrt{r_t} \vee t^{-1/2} \right) \right. \\
    &\quad + \left. T^{1/q} \left( \sum_{t=1}^T \frac{\epsilon_t}{L - 1} + \sqrt{ \frac{1}{\delta} \sum_{t=1}^T \frac{\epsilon_t}{L - 1}} \right)^{1/p} \right].
\end{align*}
This completes the proof.
\end{proof}

\section{Simulation details}
\label{sec:simulation_details}
This section provides additional implementation details and supplementary
simulation results referenced in Section~\ref{sec:simulations} of the main paper.

\subsection{Algorithmic setup}

We run the K-SIEGE algorithm for $T=1000$ rounds with an initial
forced-exploration phase of $T_0=50$ rounds, during which each arm is
selected in a round-robin fashion. The exploration probability is
\[
\varepsilon_t
=\max\{0.005,\min(0.35,\,0.15\,t^{-0.4})\}, \qquad t\ge1,
\]
which encourages more exploration at early stages and gradually
anneals toward a near-greedy policy.

Within each arm, kernel ridge regression is performed using a
one-dimensional Gaussian RBF kernel applied to the projected index
$U_{t,i}=X_t^\top\widehat\beta_{i,t}$.
The bandwidth $\sigma_t$ is chosen as the median pairwise distance of
$\{U_{s,i}\}_{s\le t}$, and the regularization schedule is
\[
\lambda_K(t)=t^{-\zeta}, \qquad \zeta=0.05.
\]
The ridge parameter for the parametric score-regression step is fixed
at $\lambda_\beta=2\times10^{-3}$.
For the quadratic variation covariance estimator
$\widehat V_{\beta,t}$ in \eqref{eq:Vhat-def} we use scaling exponent
$\alpha=1/2$, and for the nonparametric covariance operator in
\eqref{eq:Vhat-np} we use $\gamma=1/2$.

The link functions used in the simulation are shown in
Figure~\ref{fig:link_functions}.

\begin{figure}[t]
\centering
\includegraphics[width=0.5\linewidth]{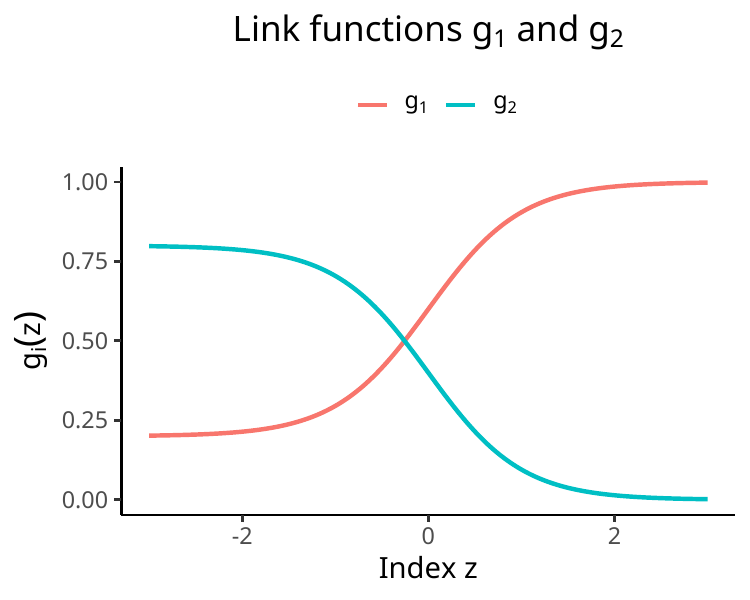}
\caption{Link functions $g_1$ and $g_2$ as functions of the single-index $z$.}
\label{fig:link_functions}
\end{figure}

\paragraph*{Nonparametric pointwise inference}

We evaluate pointwise confidence intervals for the arm-specific mean
reward functions $f_i$. At each inference time
$t\in\mathcal{T}_{\mathrm{inf}}$ and arm $i$, we construct two
intervals for $f_i(X_{t+1}^\top\widehat\beta_{i,t})$:
(i) the CLT-based interval from Algorithm~\ref{alg:np_directional_CI},
based on the RKHS covariance estimator $\widehat V_t(\lambda_t)$ in
\eqref{eq:Vhat-np}, and
(ii) a pointwise interval derived from the uniform-in-time bands of
\citet{arya:23} evaluated at the same index point.
Both methods use the same ridge schedule
$\lambda_t=t^{-\zeta}$ with $\zeta=0.05$ and the same Gaussian RBF kernel.
Figure~\ref{fig:np_sim10} illustrates the per-trajectory behavior of
the intervals for a representative run.

\begin{figure}[h]
\centering
\includegraphics[width=0.45\linewidth]{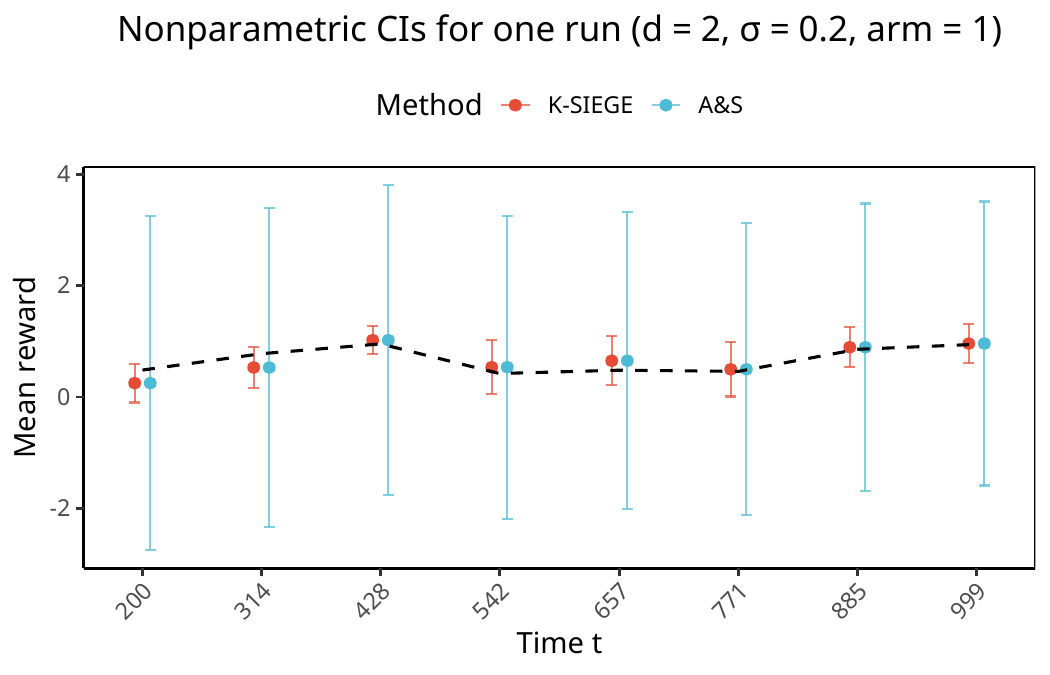}
\includegraphics[width=0.45\linewidth]{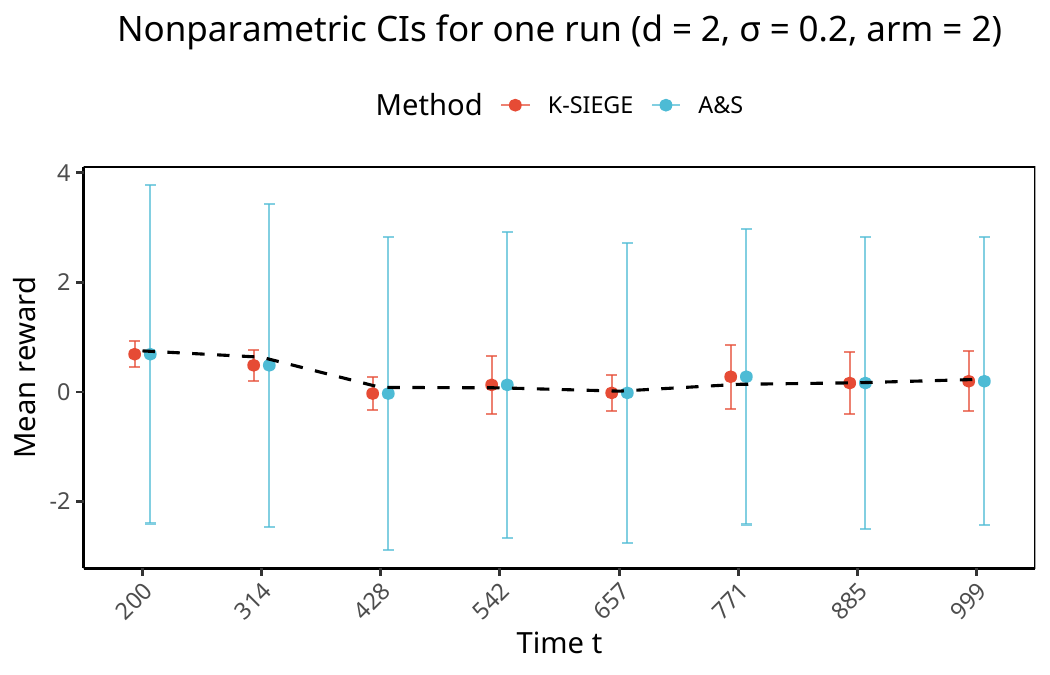}
\caption{Nonparametric pointwise $95\%$ confidence intervals for a single
Monte Carlo replication in the $(d,\sigma)=(2,0.20)$ scenario.
The dashed curve shows the true mean reward, points show fitted values,
and vertical bars give the K-SIEGE and A\&S intervals.}
\label{fig:np_sim10}
\end{figure}

\paragraph*{Regret behavior}
We also report additional regret diagnostics.
For each replication and scenario we compute cumulative regret as in
\eqref{eq:regret} and its time-averaged version $\bar R_t=R_t/t$.
Figure~\ref{fig:regret} plots the mean $\bar R_t$ across Monte Carlo
replications together with pointwise $95\%$ confidence bands. Across all scenarios the average per-step regret decreases steadily
with $t$, with slower decay in higher-noise and higher-dimensional
settings. This behavior is consistent with the theoretical regret
bounds in Section~\ref{sec: regret_analysis}.

\begin{figure}[h]
\centering
\includegraphics[width=0.65\linewidth]{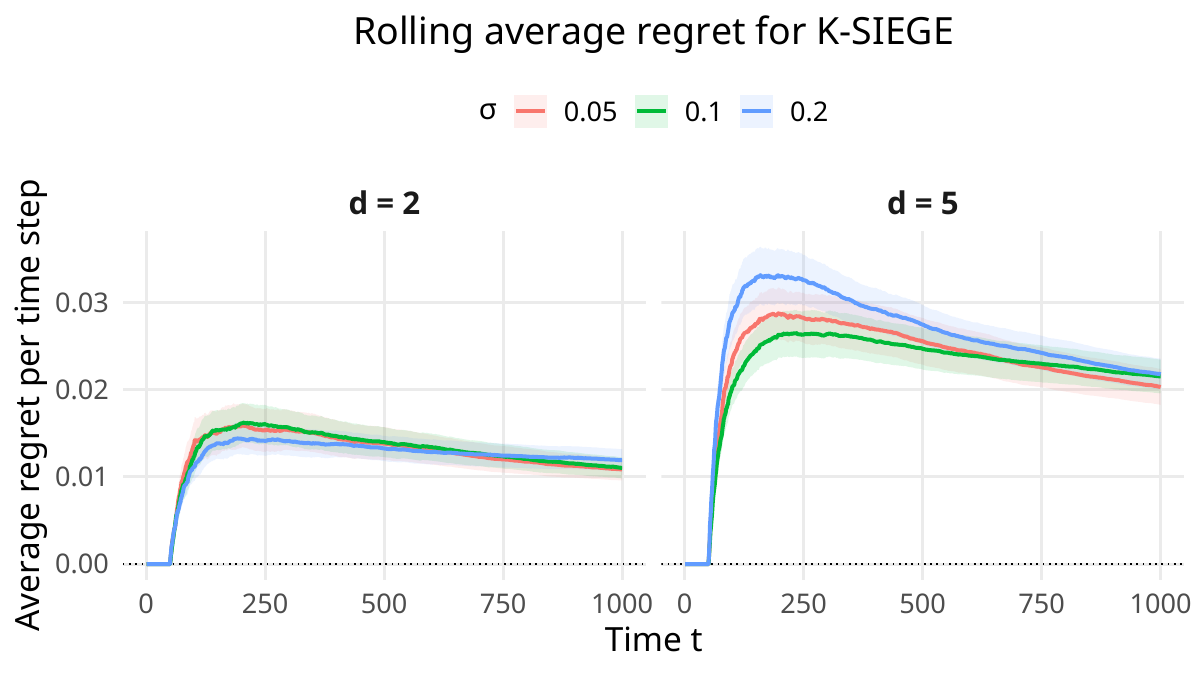}
\caption{Time-averaged regret $\bar R_t=R_t/t$ across simulation
scenarios. Curves show the mean over replications with pointwise
$95\%$ confidence bands.}
\label{fig:regret}
\end{figure}

\section{Additional Real Data Results}\label{sec:realdata_appendix}
Section~\ref{sec:realdata_rice} in the main paper presents the primary
real-data illustration on the Rice Classification dataset.  Here we
report additional supporting results, including coordinate-wise
directional confidence intervals and diagnostics for the Rice experiment,
as well as a complementary EEG eye-state study examining the behavior of
K-SIEGE in a noisier time-indexed setting.

\subsection{Additional details for the Rice classification experiment}
\label{sec:rice_appendix}

\paragraph*{Dataset and preprocessing}
The Rice Classification dataset \citep{riceCammeo} contains morphological
features extracted from images of rice grains belonging to two varieties.
Each observation comprises $d=7$ continuous covariates describing geometric
properties such as area, perimeter, major and minor axis lengths,
eccentricity, convex area, and extent. All covariates are standardized to
zero mean and unit variance prior to analysis.

From the full dataset we repeatedly sample random subsets of size $T=1000$
without replacement, each treated as an independent online trajectory.
Within each subset, contexts are revealed sequentially and bandit feedback
is observed only for the selected arm. Results are aggregated over
40 random permutations.

\paragraph*{Algorithmic setup}

We run the K-SIEGE algorithm for $T=1000$ rounds with a forced exploration
phase of $T_0=20$ rounds and exploration schedule
\[
\varepsilon_t=\max\{0.005,\min(0.35,0.15\,t^{-0.4})\}.
\]

Kernel ridge regression is performed using a one-dimensional Gaussian
RBF kernel applied to the projected index variable
$U_{t,i}=X_t^\top\widehat\beta_{i,t}$, with adaptive bandwidth chosen as
the median pairwise distance of $\{U_{s,i}\}_{s\le t}$.
The kernel regularization schedule is $\lambda_K(t)=t^{-\zeta}$ with
$\zeta=0.05$, while the ridge parameter for the parametric
score-regression step is fixed at $\lambda_\beta=2\times10^{-3}$.
Inference is performed at time points
$
\mathcal{T}_{\mathrm{inf}}
=\{200,300,400,500,600,700,800,900\}.
$

\paragraph*{Directional inference diagnostics}

Table~\ref{tab:rice_dir_ci} reports representative coordinate-wise
directional confidence intervals for Arm 1 at $t=900$, with variables
ordered by the magnitude of the estimated directional loading
$|\widehat\theta_{i,t}|$.

\begin{table}[t]
\centering
\begin{tabular}{llll}
\hline
Variable & Center & CI$_{\text{lo}}$ & CI$_{\text{hi}}$ \\
\hline
Perimeter       & 0.47405 & 0.47016 & 0.47794 \\
MajorAxisLength & 0.46718 & 0.45704 & 0.47731 \\
ConvexArea      & 0.44804 & 0.43794 & 0.45813 \\
Area            & 0.44534 & 0.43628 & 0.45440 \\
Eccentricity    & 0.28564 & 0.23618 & 0.33510 \\
MinorAxisLength & 0.26861 & 0.23365 & 0.30356 \\
Extent          & -0.06491 & -0.09285 & -0.03697 \\
\hline
\end{tabular}
\caption{Directional parametric inference (95\% confidence intervals) at $t=900$
for a representative run (Arm 1). Variables are ordered by
$\lvert \widehat{\theta}_{i,t} \rvert$.}
\label{tab:rice_dir_ci}
\end{table}
Because the two arms correspond to complementary class predictions,
the estimated directions are often approximately sign-reversed across arms.
We therefore interpret feature importance primarily through absolute
loadings and rank stability rather than the sign of individual
coordinates.
Figure~\ref{fig:width_box_t900_byarm} shows the distribution of
coordinate-wise directional CI widths at $t=900$ across variables and
replications.

\begin{figure}[t]
\centering
\includegraphics[width=0.75\linewidth]{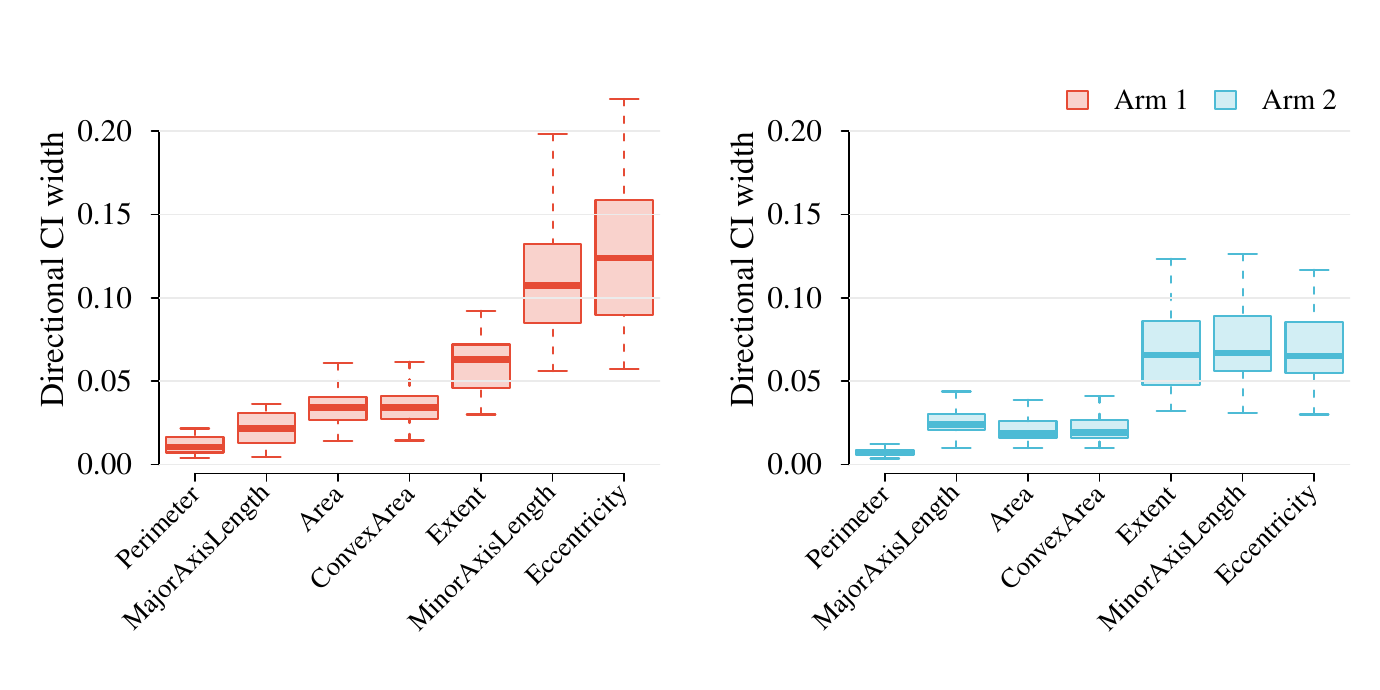}
\caption{Distribution of coordinate-wise directional CI widths across
40 replications at $t=900$, shown separately by arm.}
\label{fig:width_box_t900_byarm}
\end{figure}

\paragraph*{Sequential performance}

We evaluate sequential performance using time-averaged regret
$\bar R_t=R_t/t$ and excess classification error relative to a
cross-fitted logistic regression oracle trained on the full feature set.
Within each replication, the oracle is constructed by cross-fitting a
logistic classifier and forming out-of-sample predicted labels for all
observations.

We report the excess error
\[
\mathrm{ExcessErr}_t
=
\frac{1}{t}\sum_{s=1}^t \mathbf{1}\{A_s \neq Y_s\}
-
\frac{1}{t}\sum_{s=1}^t \mathbf{1}\{\widehat{A}^{\mathrm{orc}}_s \neq Y_s\}.
\]

Figure~\ref{fig:rice_seq_perf} reports both metrics averaged over
replications. The regret decreases steadily over time, while excess error
relative to the logistic oracle shrinks and stabilizes, indicating that
K-SIEGE approaches the performance of a strong full-information baseline
despite bandit feedback and explicit exploration.

\begin{figure}[t]
\centering
\begin{tabular}{@{}cc@{}}
\begin{minipage}[t]{0.49\linewidth}
\centering
\includegraphics[width=\linewidth]{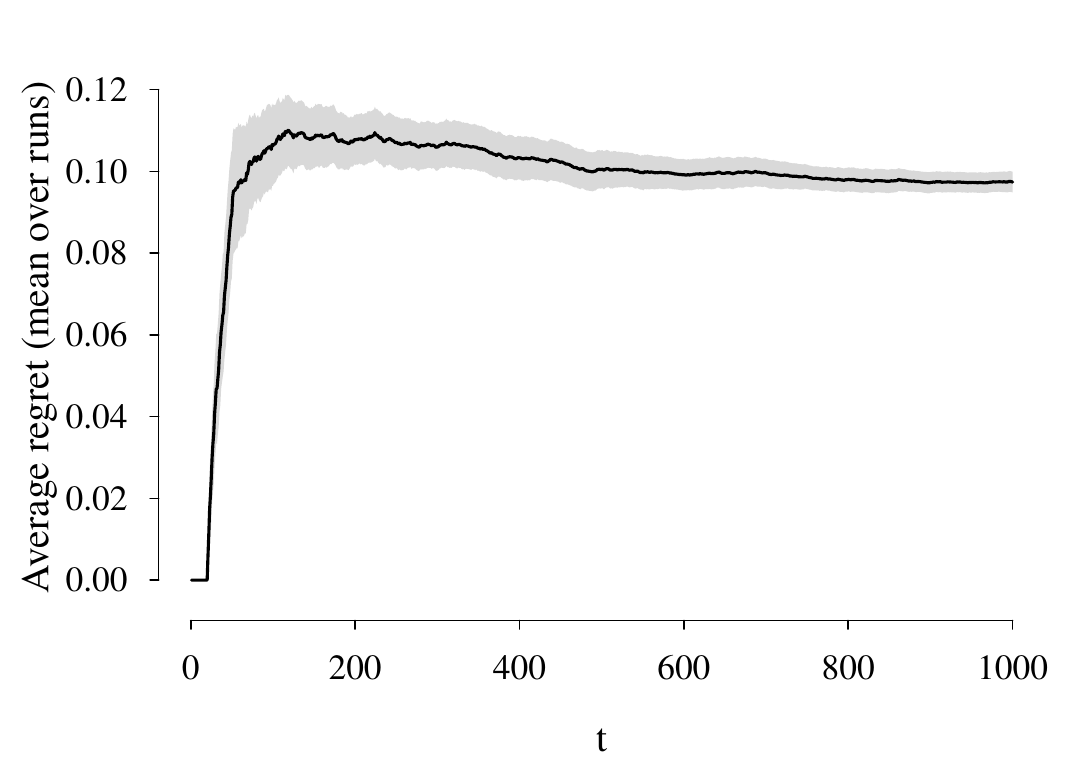}\\
\vspace{-0.35em}
{\small (a)}
\end{minipage}
&
\begin{minipage}[t]{0.49\linewidth}
\centering
\includegraphics[width=\linewidth]{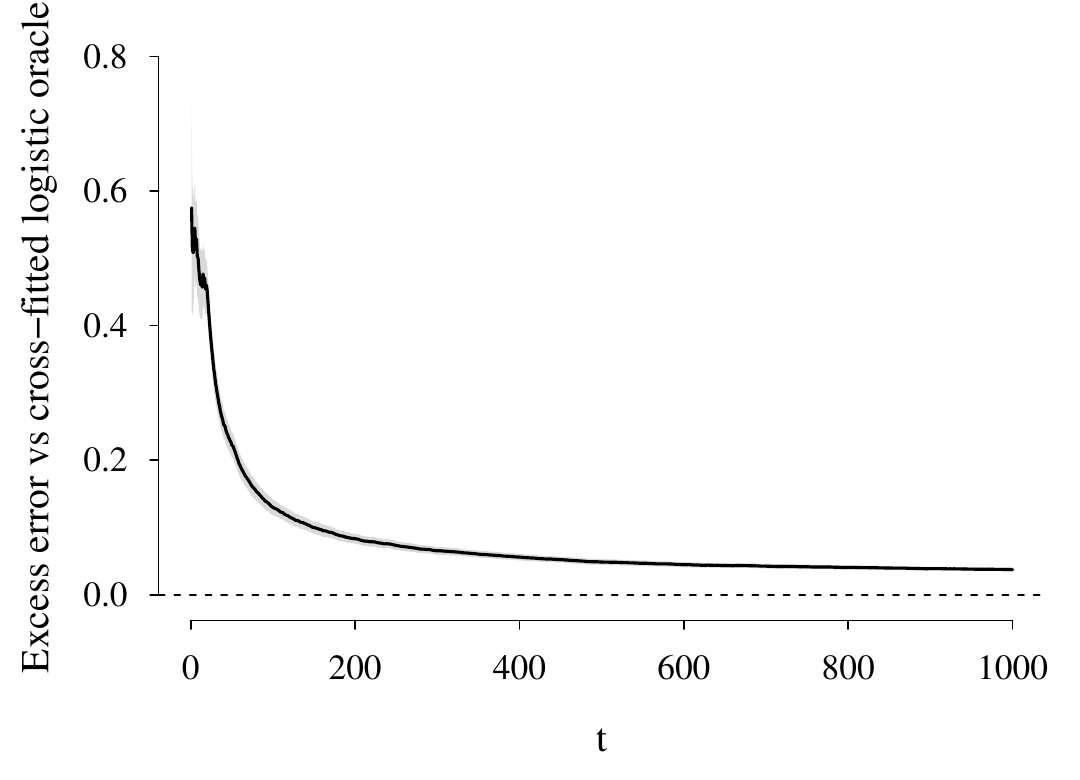}\\
\vspace{-0.35em}
{\small (b)}
\end{minipage}
\end{tabular}
\caption{Sequential performance of K-SIEGE on the Rice Classification dataset.
(a) Time-averaged regret $\bar R_t$ (mean $\pm$ 95\% CI).
(b) Excess error relative to the cross-fitted logistic oracle
(mean $\pm$ 95\% CI).}
\label{fig:rice_seq_perf}
\end{figure}

\subsection{Rice classification: representative directional CIs.}\label{sec: rice_appendix}
Tables~\ref{tab:dir_ci_run01_arm1_t300_600_900}--\ref{tab:dir_ci_run01_arm2_t300_600_900} report
coordinate-wise directional confidence intervals (CIs) for a representative run at
$t\in\{300,600,900\}$. The dominant geometric features
(\texttt{Perimeter}, \texttt{MajorAxisLength}, \texttt{ConvexArea}, \texttt{Area}) remain highly ranked
across time, and their uncertainty decreases as $t$ grows, indicating stable identification of the
primary discriminative direction. Since the two arms correspond to complementary class-prediction
actions, the learned directions are approximately sign-reversed; we therefore interpret results
primarily via absolute loadings and rankings.

\begin{table}[h!]
\centering
\begin{tabular}{llll}
  \hline
 Variable & $t=300$ & $t=600$ & $t=900$ \\ 
  \hline
 Perimeter & 0.4728 (0.4631, 0.4826) & 0.4768 (0.4714, 0.4823) & 0.4741 (0.4702, 0.4779) \\ 
 MajorAxisLength & 0.4733 (0.4627, 0.4839) & 0.4668 (0.4545, 0.4790) & 0.4672 (0.4570, 0.4773) \\ 
ConvexArea & 0.4469 (0.4364, 0.4575) & 0.4511 (0.4394, 0.4628) & 0.4480 (0.4379, 0.4581) \\ 
 Area & 0.4456 (0.4342, 0.4571) & 0.4471 (0.4368, 0.4575) & 0.4453 (0.4363, 0.4544) \\ 
 Eccentricity & 0.2941 (0.2500, 0.3382) & 0.2716 (0.2135, 0.3296) & 0.2856 (0.2362, 0.3351) \\ 
 MinorAxisLength & 0.2568 (0.2211, 0.2926) & 0.2739 (0.2336, 0.3141) & 0.2686 (0.2337, 0.3036) \\ 
   Extent & -0.0403 (-0.0969, 0.0163) & -0.0515 (-0.0888, -0.0143) & -0.0649 (-0.0929, -0.0370) \\ 
   \hline
\end{tabular}
\caption{Directional parametric inference (representative run 01, arm 1): entries are center and $95\%$ directional CI in parentheses.} 
\label{tab:dir_ci_run01_arm1_t300_600_900}
\end{table}

\begin{table}[h!]
\centering
\begin{tabular}{llll}
  \hline
  Variable & $t=300$ & $t=600$ & $t=900$ \\ 
  \hline
Perimeter & -0.4696 (-0.4746, -0.4645) & -0.4688 (-0.4729, -0.4646) & -0.4663 (-0.4695, -0.4630) \\ 
   MajorAxisLength & -0.4784 (-0.4894, -0.4675) & -0.4727 (-0.4844, -0.4611) & -0.4646 (-0.4737, -0.4554) \\ 
 ConvexArea & -0.4463 (-0.4564, -0.4361) & -0.4512 (-0.4602, -0.4422) & -0.4573 (-0.4635, -0.4510) \\ 
   Area & -0.4443 (-0.4531, -0.4355) & -0.4503 (-0.4587, -0.4419) & -0.4565 (-0.4625, -0.4505) \\ 
   Eccentricity & -0.3033 (-0.3315, -0.2751) & -0.2941 (-0.3220, -0.2662) & -0.2723 (-0.2937, -0.2508) \\ 
 MinorAxisLength & -0.2324 (-0.2627, -0.2020) & -0.2488 (-0.2778, -0.2198) & -0.2706 (-0.2924, -0.2487) \\ 
   Extent & 0.0898 (0.0557, 0.1240) & 0.0454 (0.0183, 0.0725) & 0.0442 (0.0254, 0.0631) \\ 
   \hline
\end{tabular}
\caption{Directional parametric inference (representative run 01, arm 2): entries are center and $95\%$ directional CI in parentheses.} 
\label{tab:dir_ci_run01_arm2_t300_600_900}
\end{table}

\subsection{Real data experiment: EEG eye-state classification}
\label{sec:realdata_eeg}

We provide a complementary real-data illustration on the EEG Eye State dataset \citep{eeg_eye_state_264},
a time-indexed binary classification task with $d=13$ EEG electrode features from a single subject.
We cast the task as a two-armed contextual bandit where arm $i\in\{1,2\}$ corresponds to predicting
one of the two eye states, and the reward equals one if the prediction is correct.

We use the EEG Eye State dataset \citep{eeg_eye_state_264} from the UCI Machine Learning Repository, consisting of
time-indexed electroencephalogram measurements from a single subject. Each observation contains
$d=13$ continuous covariates corresponding to EEG electrode signals and a binary label indicating
whether the subject's eyes are open or closed.

We cast the task as a two-armed contextual bandit with contexts $X_t\in\mathbb{R}^{13}$ and binary
rewards
\[
Y_t\in\{0,1\},\qquad t=1,\dots,T,
\]
where arm $i\in\{1,2\}$ corresponds to predicting eye state $i$, and the reward equals one if the
prediction is correct. The conditional mean reward is modeled via an arm-specific single-index form
\[
\mathbb{E}[Y_t\mid X_t, A_t=i] = f_i(X_t^\top \beta_i^\star),
\]
where $\beta_i^\star\in\mathbb{R}^{13}$ is an unknown arm-specific index vector and $f_i$ is an
unknown link function. This structure allows the algorithm to learn a low-dimensional projection of
EEG features that is most informative for each arm's decision boundary.

All covariates are standardized column-wise to zero mean and unit variance. We extract a subsequence
of length $T=1000$ and treat it as a single online trajectory. To assess stability with respect to
data ordering, we repeat the experiment over multiple random permutations of the data.

We run K-SIEGE for $T=1000$ rounds with a forced exploration period of $T_0=20$. The exploration
probability is
\[
\varepsilon_t
= \max\!\bigl\{0.005,\,
\min\{0.35,\,0.15\,t^{-0.4}\}\bigr\},
\]
which induces substantial exploration early and gradually anneals toward near-greedy behavior.
Within each arm, kernel ridge regression is performed using a one-dimensional Gaussian RBF kernel on
the projected variable $Z_{t,i}=X_t^\top \hat\beta_{i,t}$, with bandwidth set adaptively to the
median pairwise distance of $\{Z_{s,i}\}_{s\le t}$. The nonparametric ridge schedule is
$\lambda_K(t)=t^{-\zeta}$ with $\zeta=0.05$, while the ridge parameter for the parametric
score-regression step is fixed at $\lambda_\beta=2\times 10^{-3}$. Inference is evaluated at
\[
\mathcal{T}_{\mathrm{inf}}=\{200,300,400,500,600,700,800,900\}.
\]

\subsubsection*{Directional parametric inference and interpretability}
We first examine directional inference for the index directions. At each $t\in\mathcal{T}_{\mathrm{inf}}$
and arm $i$, we construct a $95\%$ joint confidence set for the unit direction
$\theta_i^\star=\beta_i^\star/\|\beta_i^\star\|_2$ using the quadratic-variation-based estimator
from Section~\ref{sec:asymptotic_inference_beta}. These directional sets quantify uncertainty in the
orientation of the estimated index while respecting the unit-norm constraint.

Figure~\ref{fig:eeg_dir_summary}(a) reports the mean directional confidence interval width over time,
aggregated across random permutations. The widths contract with $t$ for both arms, providing
empirical support for the directional CLT behavior in this noisy setting.
\begin{table}[ht]
\centering
\begin{tabular}{llll}
  \hline
  Variable & $t=300$ & $t=600$ & $t=900$ \\ 
  \hline
 V8 (O2) & 0.1707 (-0.0261, 0.3674) & 0.3688 (0.1869, 0.5506) & 0.3819 (0.2747, 0.4892) \\ 
V9 (P8) & 0.4968 (0.2792, 0.7145) & 0.4337 (0.2490, 0.6184) & 0.3681 (0.2609, 0.4753) \\ 
 V10 (T8) & 0.4665 (0.2836, 0.6495) & 0.4132 (0.2129, 0.6135) & 0.3493 (0.2312, 0.4673) \\ 
 V7 (O1) & -0.0647 (-0.4643, 0.3348) & 0.2504 (-0.1544, 0.6553) & 0.3461 (0.1259, 0.5663) \\ 
 V13 (F8) & 0.2368 (-0.0320, 0.5055) & 0.3605 (0.0850, 0.6360) & 0.3376 (0.1813, 0.4940) \\ 
 V12 (F4) & 0.2983 (0.1458, 0.4509) & 0.2996 (0.1234, 0.4758) & 0.3150 (0.2184, 0.4116) \\ 
 V11 (FC6) & -0.0072 (-0.1966, 0.1823) & 0.1429 (-0.0310, 0.3168) & 0.2973 (0.2054, 0.3891) \\ 
 V2 (F7) & -0.4133 (-0.6057, -0.2209) & -0.3234 (-0.5354, -0.1113) & -0.2795 (-0.4082, -0.1508) \\ 
 V14 (AF4) & 0.2139 (0.0323, 0.3955) & 0.1576 (-0.0302, 0.3454) & 0.2094 (0.0982, 0.3207) \\ 
 V3 (F3) & 0.1974 (0.0207, 0.3741) & 0.0838 (-0.0991, 0.2668) & 0.1533 (0.0432, 0.2634) \\ 
 V6 (P7) & -0.0213 (-0.1664, 0.1239) & 0.0590 (-0.1094, 0.2273) & 0.1441 (0.0351, 0.2531) \\ 
 V5 (T7) & -0.2660 (-0.4842, -0.0479) & -0.1429 (-0.3674, 0.0817) & 0.0879 (-0.0416, 0.2173) \\ 
 V4 (FC5) & -0.1741 (-0.3847, 0.0364) & -0.2049 (-0.4312, 0.0215) & -0.0368 (-0.1795, 0.1060) \\ 
   \hline
\end{tabular}
\caption{Directional parametric inference (representative run 01, arm 2): entries are center and $95\%$ directional CI in parentheses.} 
\label{tab:eeg_dir_ci_run01_arm2_t300_600_900}
\end{table}
\begin{figure}[H]
\centering

\begin{tabular}{@{}cc@{}}
\begin{minipage}[t]{0.45\linewidth}
\centering
\includegraphics[width=\linewidth]{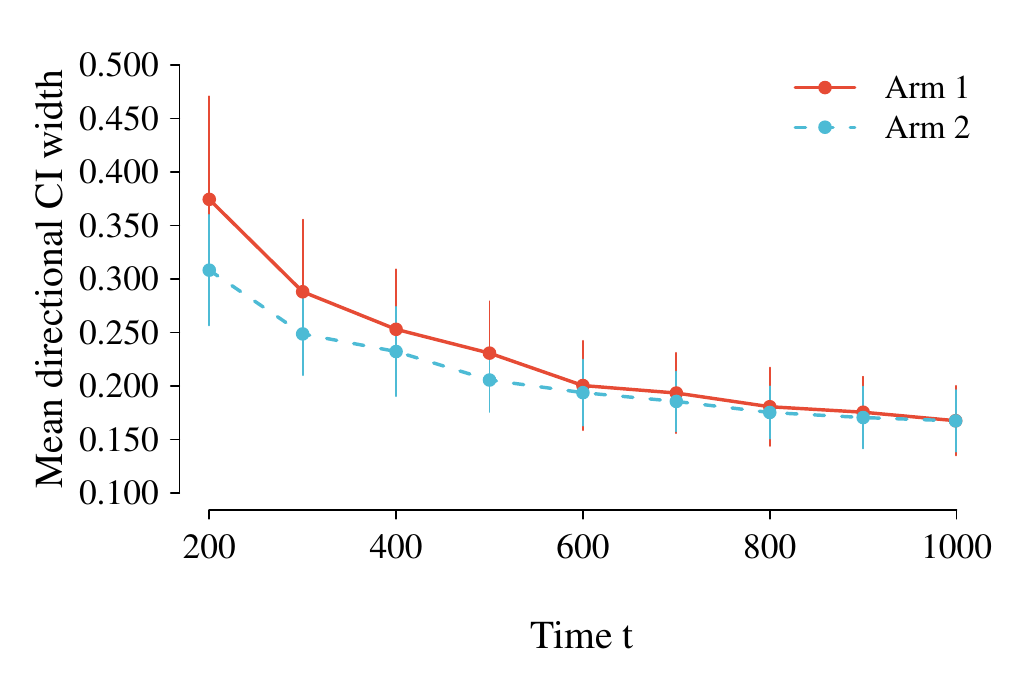}\\
\vspace{-0.2em}
{\small (a) Mean directional CI width vs.\ time}
\end{minipage}
&
\begin{minipage}[t]{0.50\linewidth}
\centering
\includegraphics[width=\linewidth]{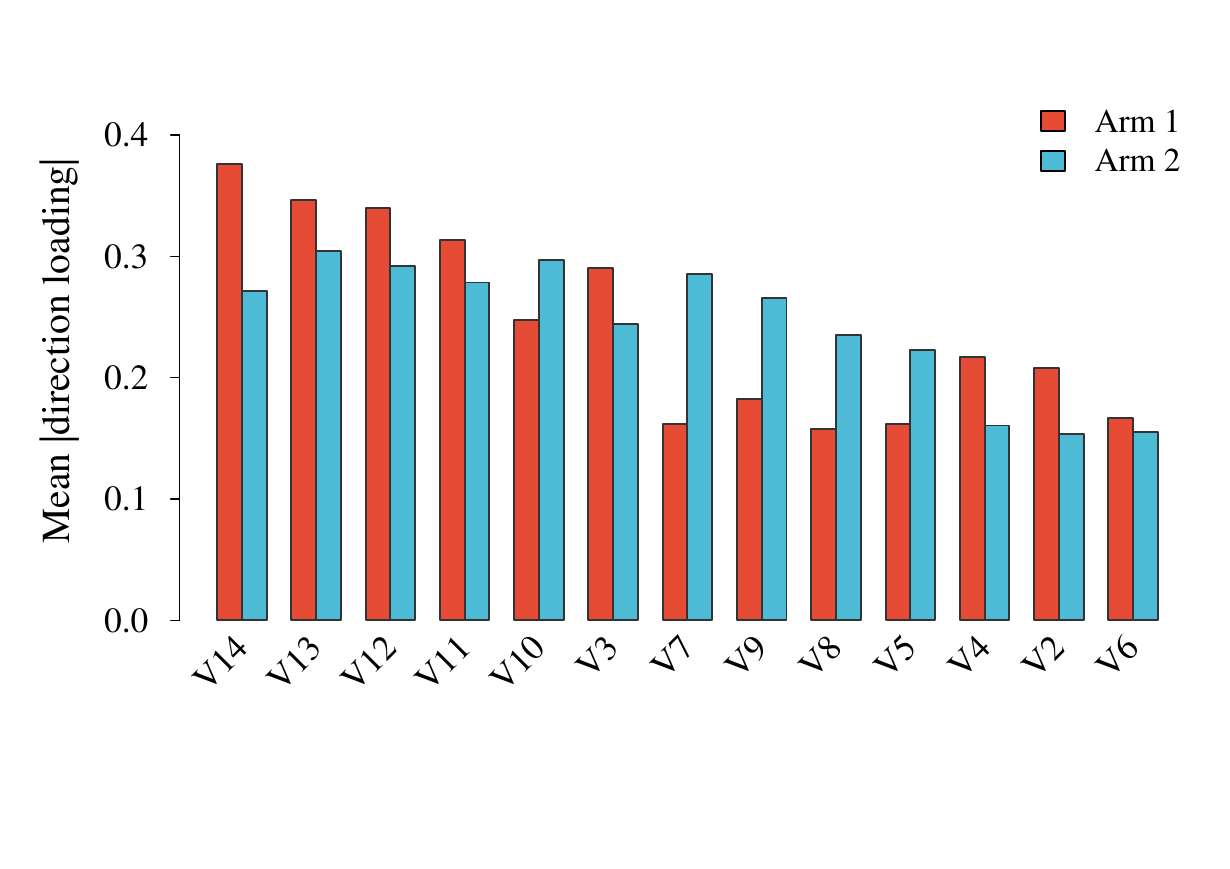}\\
\vspace{-0.2em}
{\small (b) Mean directional loadings by arm at $t=900$}
\end{minipage}
\end{tabular}
\caption{Directional inference summary on the EEG dataset.}
\label{fig:eeg_dir_summary}
\end{figure}
Interpretability is illustrated by comparing the learned directional loadings across arms. In this
dataset, the two arms correspond to predicting \emph{eyes open} versus \emph{eyes closed}. Since the
single-index model is arm-specific, the learned directions $\hat\theta_{1,t}$ and $\hat\theta_{2,t}$
need not coincide: each arm may rely on a different linear combination of electrodes to separate
correct from incorrect predictions for that label. Figure~\ref{fig:eeg_dir_summary}(b) shows the mean
absolute loadings by arm at $t=900$, averaged across permutations, and indicates that the relative
importance of electrodes can differ between the two labels. Such differences are consistent with an
asymmetric classification problem in which the EEG patterns most diagnostic of ``open'' need not
match those most diagnostic of ``closed''.  In particular, electrodes in occipital regions (e.g., O1/O2/P7/P8, labeled V6/V7/V8/V9) plausibly encode
changes related to visual cortex activity, while frontal and fronto-temporal sites (e.g., AF/F/FC
locations, labeled V11-V14) can reflect eye-movement and blink-related artifacts; thus, differences in the dominant
electrodes across arms can be read as \emph{asymmetric evidence patterns} that distinguish open vs.\
closed states.

To provide a representative snapshot with uncertainty quantification, Table~\ref{tab:eeg_dir_ci_run01_arm2_t300_600_900}
reports the estimated direction $\hat\theta_{i,t}$ together with coordinate-wise $95\%$ directional
confidence intervals at $t=300,600,900$ for one run (arm~2). Since covariates are standardized, the
coordinate magnitudes are directly comparable across electrodes. Variables with both large magnitude
and tighter intervals indicate stable contributions to the learned index direction, while variables
with intervals that include zero are less reliably informative at the corresponding time.

Finally, we illustrate the nonparametric inference for the arm-specific reward
functions $f_i$.  Focusing on a representative run and a late inference time
($t = 900$), we visualize the estimated success probability as a function of the
learned index $z = X^\top \hat\beta_{i,t}$.  For each arm, we plot the fitted
kernel ridge regression curve together with pointwise $95\%$ confidence bands
constructed using the RKHS covariance estimator of
Section~\ref{sec: pointwise_nonp_CLT}.  For additional context, we overlay
binned empirical success rates computed from observations where the corresponding
arm was selected.
Figure~\ref{fig:eeg_np_summary}(a) shows the estimated mean reward (success probability) as a
function of the standardized learned index $z$ for each arm. The curves are smooth and exhibit
pronounced nonlinear variation, indicating that a one-dimensional projection of the EEG features
captures meaningful signal for distinguishing the two eye states. Notably, the two arms produce
different functional relationships along their respective indices, consistent with arm-specific
single-index structure.

\begin{figure}[H]
\centering
\begin{tabular}{@{}cc@{}}
\begin{minipage}[t]{0.55\linewidth}
\centering
\includegraphics[width=\linewidth]{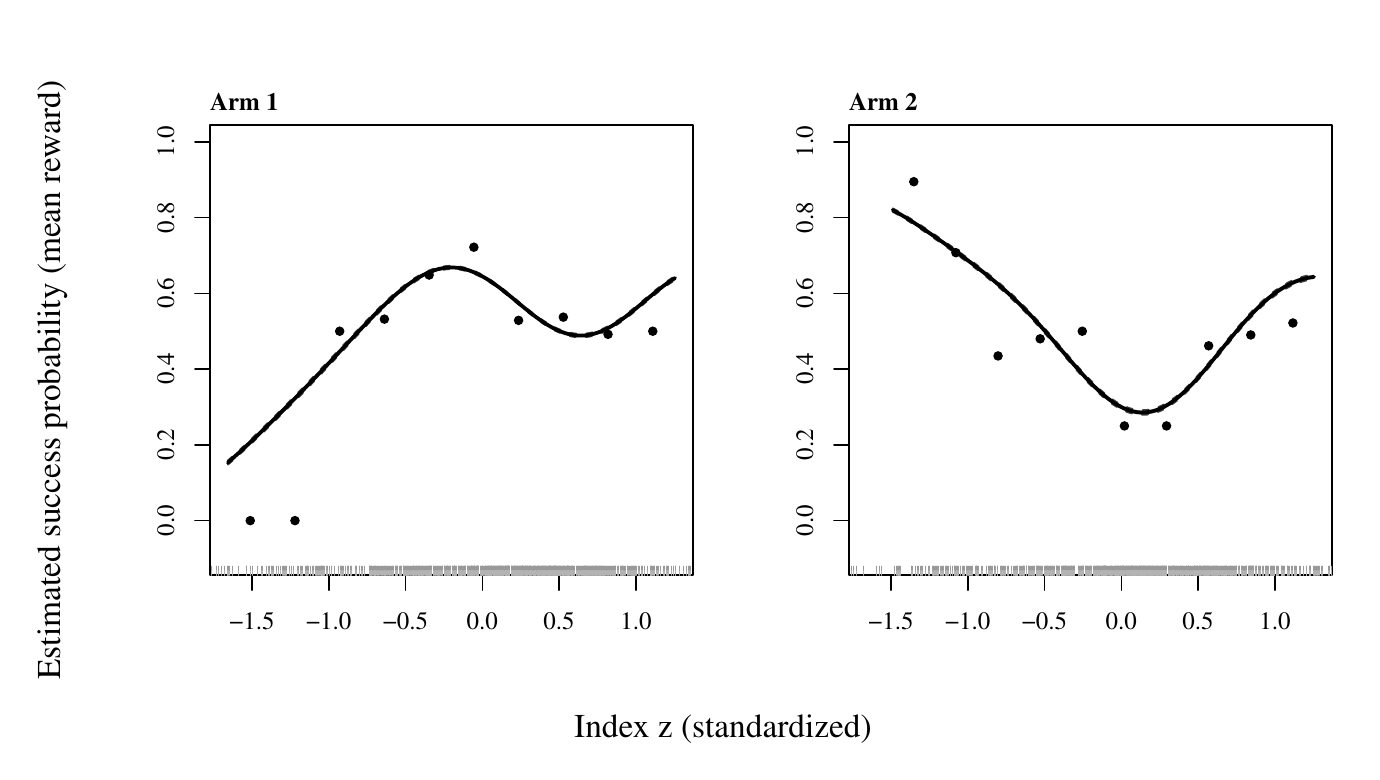}\\
\vspace{-0.35em}
{\small (a)}
\end{minipage}
\begin{minipage}[t]{0.45\linewidth}
\centering
\includegraphics[width=\linewidth]{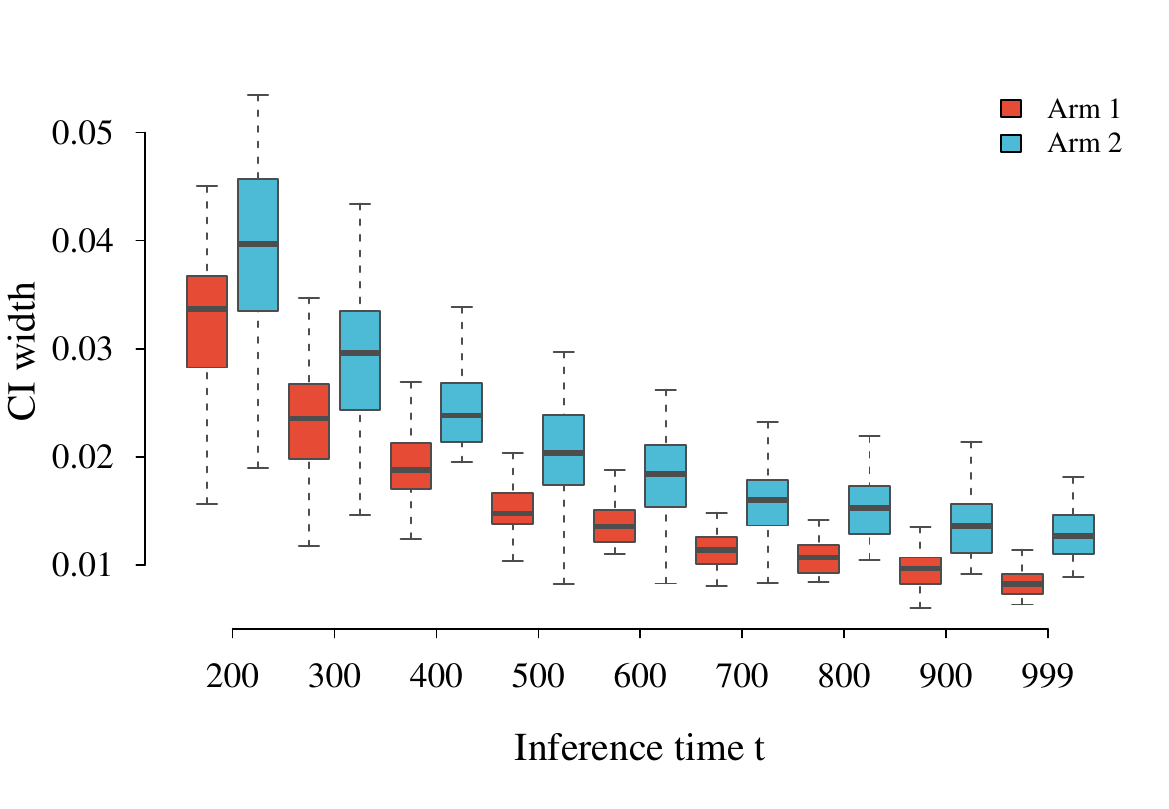}\\
\vspace{-0.35em}
{\small (b)}
\end{minipage}
\end{tabular}
\caption{Nonparametric inference summary on the EEG dataset. (a) NP success curves with pointwise $95\%$ CIs at $t=900$ (representative run). (b) Distribution of NP pointwise CI widths across replications, by arm and time.}
\label{fig:eeg_np_summary}
\end{figure}

To quantify how nonparametric uncertainty evolves over time, we compute the
average width of the pointwise $95\%$ confidence intervals for $f_i$ at each
inference time $t \in \mathcal{T}_{\mathrm{inf}}$.  For each arm, widths are
averaged over replications and evaluation points.
Figure~\ref{fig:eeg_np_summary}(b) plots the resulting mean interval width as a
function of time.  For both arms, the nonparametric confidence intervals
contract steadily with $t$, reflecting the increasing effective sample size
along the learned index.  While the rate of contraction is slower than in the
low-dimensional simulations, the monotone decay is consistent with the
$t^{-\gamma}$ scaling predicted by the theory.
Overall, the EEG experiment demonstrates that the proposed kernel single-index
bandit framework yields meaningful uncertainty quantification and interpretable
structure in a realistic, noisy, and moderately high-dimensional setting.

\end{document}